%% file: main.tex
\RequirePackage{fix-cm}
\documentclass{svjour3}                     
\smartqed  
\usepackage{graphicx}
\usepackage{mathptmx}      
\usepackage[utf8]{inputenc} 
\usepackage[T1]{fontenc}    
\usepackage{hyperref}       
\usepackage{url}            
\usepackage{booktabs}       
\usepackage{amsfonts}       
\usepackage{nicefrac}       
\usepackage{microtype}      
\usepackage{xcolor}         

\usepackage{enumitem}
\usepackage{microtype} 
\usepackage{marvosym} 
\usepackage{algorithm}
\usepackage{algorithmic}

\usepackage{appendix}

\usepackage{multirow}
\usepackage{bm}
\usepackage{mathtools}

\definecolor{links}{HTML}{0078b0} 
\definecolor{files}{HTML}{fc6160}
\hypersetup{
    colorlinks=true,
    linkcolor=files,
    filecolor=files,      
    urlcolor=links,
    citecolor=links,
}

\makeatletter
\newenvironment{breakablealgorithm}
  {
   \begin{center}
     \refstepcounter{algorithm}
     \hrule height.8pt depth0pt \kern2pt
     \renewcommand{\caption}[2][\relax]{
       {\raggedright\textbf{\ALG@name~\thealgorithm} ##2\par}%
       \ifx\relax##1\relax 
         \addcontentsline{loa}{algorithm}{\protect\numberline{\thealgorithm}##2}%
       \else 
         \addcontentsline{loa}{algorithm}{\protect\numberline{\thealgorithm}##1}%
       \fi
       \kern2pt\hrule\kern2pt
     }
  }{
     \kern2pt\hrule\relax
   \end{center}
  }
\makeatother

\begin{document}

\title{On the Importance of Sampling in Training GCNs: Tighter Analysis and Variance Reduction
}


\author{Weilin Cong \and Morteza Ramezani \and Mehrdad Mahdavi
}


\institute{ Weilin Cong \at
            Pennsylvania State University \\
            \email{wxc272@psu.edu}
            \and
            Morteza Ramezani \at
            Pennsylvania State University \\
            \email{morteza@cse.psu.edu}
            \and
            Mehrdad Mahdavi \at
            Pennsylvania State University \\
            \email{mzm616@psu.edu}
}

\date{Received: date / Accepted: date}

\maketitle

\begin{abstract}
Graph Convolutional Networks (GCNs) have achieved impressive empirical advancement across a wide variety of semi-supervised node classification tasks. Despite their great success, training GCNs on large graphs suffers from computational and memory issues. A potential path to circumvent these obstacles is sampling-based methods, where at each layer a subset of nodes is sampled. Although recent studies have empirically demonstrated the effectiveness of sampling-based methods, these works lack theoretical convergence guarantees under realistic settings and cannot fully leverage the information of evolving parameters during optimization. In this paper, we describe and analyze a general \textbf{\textit{doubly variance reduction}} schema that can accelerate any sampling method under the memory budget. The motivating impetus for the proposed schema is a careful analysis for the variance of sampling methods where it is shown that the induced variance can be decomposed into node embedding approximation variance (\emph{zeroth-order variance}) during forward propagation and layerwise-gradient variance (\emph{first-order variance}) during backward propagation. We theoretically analyze the convergence of the proposed schema and show that it enjoys an $\mathcal{O}(1/T)$ convergence rate.  We complement our theoretical results by integrating the proposed schema in different sampling methods and applying them to different large real-world graphs. 
\end{abstract}

\input{section/1-introduction}
\input{section/2-related}

\input{section/3-SGCN}
\input{section/4-SGCN+}

\input{section/5-SGCN++}
\input{section/6-experiment}

\section{Conclusion}\label{section:conclusion}
In this work, we develop a theoretical framework for analyzing the convergence of sampling based
mini-batch GCNs training. We show that the node embedding approximation variance and layerwise
gradient variance are two key factors that slow down the convergence of these methods. Furthermore, we propose doubly variance reduction schema and theoretically analyzed its convergence.
Experimental results on benchmark datasets demonstrate the effectiveness of proposed schema to
significantly reduce the variance of different sampling strategies to achieve better generalization.




\clearpage
\appendix



\input{supplementary/algorithm_summarize}
\input{appendix/proof_sgcn_thm_1}
\input{appendix/proof_sgcn_thm_2}
\input{appendix/proof_sgcn_thm_3}
\input{appendix/connection_to_existing_results}

\clearpage

\bibliographystyle{spmpsci}
\bibliography{reference}

\end{document}

%% file: section/1-introduction.tex
\section{Introduction}
Recently, there has been a thrust towards representation learning approaches for machine learning on 
graph structured data for
various applications in different domains including social networks~\cite{kipf2016semi,hamilton2017inductive,wang2019mcne,deng2019learning,qiu2018deepinf}, traffic prediction~\cite{cui2019traffic,rahimi2018semi,li2019predicting,kumar2019predicting}, knowledge graphs~\cite{wang2019knowledge,wang2019kgat,park2019estimating}, drug reaction~\cite{do2019graph,duvenaud2015convolutional} and recommendation system~\cite{berg2017graph,ying2018graph}. The overarching goal is to learn representations, i.e., embeding nodes as points in a low-dimensional vector space, by encoding the structural information captured by the underlying graph along with  node and (possibly) edge features  that can be used as feature inputs for downstream machine
learning tasks.

The field of graph representation learning has been greatly developed over
the past decade~\cite{zhang2020deep,ma2021deep},  and there has been a surge of approaches that seek to learn representations in graph data. In particular, graph convolutional networks (GCNs) have achieved great success in many graph-related applications, such as semi-supervised node classification~\cite{semigcn}, supervised graph classification~\cite{graph_isomorphism_network}, protein interface prediction~\cite{Fout2017ProteinIP}, and knowledge graph~\cite{schlichtkrull2018modeling,wang2017knowledge}. However, most works on GCNs focus on relatively small graphs, and scaling GCNs for large-scale graphs is not straightforward.
Due to the dependency of the nodes in the graph, we need to consider a large receptive-field to calculate the representation of each node, while the receptive field grows exponentially with respect to the number of convolutional layers or filters-- a phenomenon known as ``neighborhood explosio''~\cite{ma2021deep}. This  issue in turn hinders the {\em scalability} of stochastic optimization methods, in particular Stochastic Gradient Descent (SGD), and prevents
them from being adapted to large-scale graphs  to train GCNs.

To alleviate the exponential computation and memory requirements of training GCNs with multiple graph convolutional layers,  and  correspondingly improve their scalability,  sampling-based methods, such as node-wise sampling~\cite{graphsage,pinsage,control_variacne_gcn}, layer-wise sampling~\cite{fastgcn,ladies},  subgraph sampling~\cite{clustergcn,graphsaint},  bandit sampling~\cite{liu2020bandit,zhang2021biased}, minimal-variance sampling~\cite{cong2020minimal}, and lazy sampling with recycling~\cite{ramezani2020gcn} are proposed to be utilized   in  mini-batch training of GCNs  to accelerate the optimization.  The main idea of sampling methods is to reduce the number of nodes involved in the computing the representation of nodes and hence lower the required time and memory requirements. 

Although empirical results show that sampling-based methods can scale GCN training to large graphs, these methods suffer from a few key issues. On the one hand, sampling methods can significantly degrade the convergence rate of SGD due to bias and high variance introduced by sampling nodes at intermediate layers. This issue calls for novel algorithmic methods to reduce the negative effect of sampling methods. Moreover,  the theoretical understanding of sampling-based methods is still lacking. On the other hand, the aforementioned sampling strategies are only based on the static structure of the graph.
Although  recent attempts~\cite{adaptivegcn,cong2020minimal,liu2020bandit} propose to utilize adaptive importance sampling strategies to constantly re-evaluate the relative importance of nodes during training (e.g., current gradient or representation of nodes), finding the optimal adaptive sampling distribution is computationally inadmissible, as it requires to calculate the full gradient or node representations in each iteration. 
This necessitates developing alternative solutions that can efficiently be computed and that come with theoretical guarantees.

\begin{figure}[ht]
    \centering
    \includegraphics[width=0.99\textwidth]{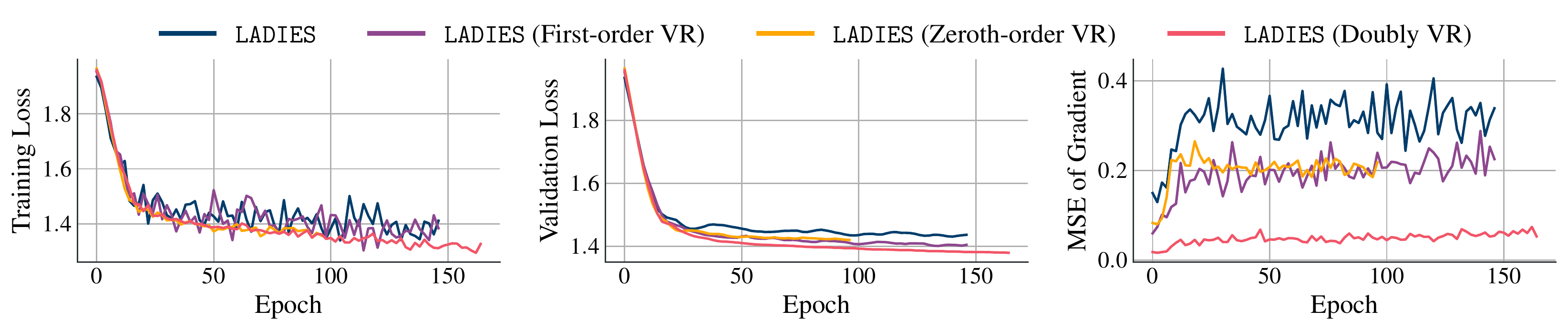}
    \vspace{-2mm}
    \caption{The effect of doubly variance reduction on training loss, validation loss, and mean-square error (MSE) of gradient on \texttt{Flickr} dataset using layer-wise sampling schema (\texttt{LADIES}~\cite{ladies}).}
    \label{fig:decomposition_loss}
\end{figure}

In this paper, we develop a novel variance reduction schema that can be applied to any sampling strategy
to significantly reduce the induced variance. The key idea is to use the historical node embeddings and the historical layerwise gradient of each graph convolution layer as control variants.
The main motivation behind the proposed schema stems from our theoretical analysis of the sampling methods' variance in training GCNs. 
Specifically, due to the composite structure of objective in GCN training~\cite{cong2020minimal}, any sampling strategy introduces two types of variances in estimating the stochastic gradients: 
node embedding approximation variance (\emph{zeroth-order variance}) which results from  embeddings approximation during forward propagation, and layerwise-gradient variance (\emph{first-order variance}) which results from gradient estimation during backward propagation.
In Figure~\ref{fig:decomposition_loss}, we exhibit the performance of proposed schema when utilized in the sampling strategy introduced in~\cite{ladies}. 
The plots show that applying our proposal can lead to a significant reduction in mean-square error of stochastic gradients; hence faster convergence rate and better test accuracy while enjoying the scalability feature of sampling strategy. 
We can also see that both zeroth-order and first-order methods are equally important and demonstrate significant improvement when applied jointly (i.e, \emph{doubly variance reduction}).

\paragraph{Contributions.}
To this end, we summarize the contributions of this paper as follows:
\begin{itemize}
\item We provide the theoretical analysis for sampling-based GCN training (\texttt{SGCN}) with a non-asymptotic convergence rate. 
We show that due to the node embedding approximation variance, \texttt{SGCN}s suffer from \emph{residual error} that hinders their convergence. \vspace{2mm}
\item We mathematically show that the aforementioned \emph{residual error} can be resolved by employing zeroth-order variance reduction to node embedding approximation (dubbed as \texttt{SGCN+}), which explains why \texttt{VRGCN}~\cite{control_variacne_gcn} enjoys a better convergence than \texttt{GraphSAGE}~\cite{graphsage}, even with less sampled neighbors. \vspace{2mm}
\item We extend the algorithm from node embedding approximation to stochastic gradient approximation, and propose a generic and efficient doubly variance reduction schema (\texttt{SGCN++}). 
\texttt{SGCN++} can be integrated with different sampling-based methods to significantly reduce both zeroth- and first-order variance, and resulting in a faster convergence rate and better generalization. \vspace{2mm}
\item We theoretically analyze the convergence of \texttt{SGCN++} and obtain an $\mathcal{O}(1/T)$ rate, which significantly improves the best known bound $\mathcal{O}(1/\sqrt{T})$. 
We empirically verify \texttt{SGCN++} through various experiments on several real-world datasets and different sampling methods, where it demonstrates significant improvements over the original sampling methods.  
\end{itemize}

\paragraph{Organization.} The paper is organized as follows. We discuss further related works in Section~\ref{section:related_works}. In Section~\ref{section:problem_formulation} we formally state the problem of training GCNs in a semi-supervised multi-class classification setting.  In  Section~\ref{section:SGCN}, we introduce sampling-based methods for scalable training of GCNs  (\texttt{SGCN}), and establish their convergence rate. We propose our  zeroth-order variance reduction  (\texttt{SGCN+}) and  a doubly variance reduction  (\texttt{SGCN++}) algorithms in Section~\ref{section:SGCN+} and Section~\ref{section:SGCN++}, respectively and obtain their convergence rates. Finally, we empirically evaluate our algorithms in Section~\ref{section:experiments} and conclude the paper in Section~\ref{section:conclusion}. For ease of exposition, we defer the proof of theoretical results to the appendix.


%% file: section/2-related.tex

\section{Related works}\label{section:related_works}

\paragraph{Training GCNs via sampling.} The full-batch training of a  typical GCN is employed in~\cite{semigcn} which necessitates keeping the whole graph data and intermediate nodes' representations in the memory. 
This is the key bottleneck that hinders the scalability of full-batch GCN training. 
To overcome this issues, sampling-based GCN training methods \cite{graphsage,control_variacne_gcn,clustergcn,fastgcn,adaptivegcn} are proposed to train GCNs based on mini-batch of nodes, and only aggregate the embeddings of a sampled subset of neighbors of nodes in the mini-batch. 
For example, 
\texttt{GraphSAGE}~\cite{graphsage} restricts the computation complexity by uniformly sampling a fixed number of neighbors from the previous layer nodes. 
However, a significant computational overhead is introduced when GCN goes deep.
\texttt{VRGCN}~\cite{control_variacne_gcn} further reduces the neighborhood size and uses history activation of the previous layer to reduce variance. 
However, they require to perform a full-batch graph convolutional operation on history activation during each forward propagation, which is computationally expensive.
Another direction 
applies layerwise importance sampling to reduce variance. 
For example, \texttt{FastGCN}~\cite{fastgcn} independently samples a constant number of nodes in all layers using importance sampling. However, the sampled nodes are too sparse to achieve high accuracy.
\texttt{LADIES}~\cite{ladies} further restrict the candidate nodes in the union of the neighborhoods of the sampled nodes in the upper layer. However, significant overhead may be incurred due to the expensive sampling algorithm.
In addition, subgraph sampling methods such as \texttt{GraphSAINT}~\cite{graphsaint} construct mini-batches by importance sampling, and apply normalization techniques to eliminate bias and reduce variance.
However, the sampled subgraphs are usually sparse and require a large sampling size to guarantee the performance.

\paragraph{Theoretical analysis.}
Despite many algorithmic progresses over the years, the theoretical understanding of the convergence for \texttt{SGCN}s training method is still limited.
\texttt{VRGCN} provides a convergence analysis under a strong assumption that the stochastic gradient due to sampling is unbiased and achieved a convergence rate of $\mathcal{O}(1/\sqrt{T})$. 
However, the convergence analysis is limited to \texttt{VRGCN}, and the assumption is not true due to the composite structure of the training objective as will be elaborated.
~\cite{biased_but_consistent} provides another convergence analysis for \texttt{FastGCN} under a strong assumption that the stochastic gradient of GCN converges to the consistent gradient exponentially fast with respect to the sample size, and results in the same convergence rate as unbiased ones, i.e., $\mathcal{O}(1/\sqrt{T})$. 
\cite{huang2021wide,du2019graph} analyze the convergence of full-batch GNN training from the perspective of Neural Tangent Kernels, \cite{xu2021optimization} analyze the  convergence of full-batch linear GNN training with assumptions on the initialization.
Most recently, \cite{sato2020constant} provides PAC learning-style bounds on the node embedding and gradient estimation for \texttt{SGCN}s training. 
Another direction of theoretical research focuses on analyzing the expressive power of GCN~\cite{garg2020generalization,chen2019generalization,zhang2020efficient,xu2020neural}, which is not the focus of this paper and omitted for brevity.


\paragraph{Training GCNs and composite optimization.}
The proposed \emph{doubly variance reduction} shares the same spirit with the variance reduced composite optimization problem considered in ~\cite{zhang2019composite,hu2020biased,tran2020hybrid,zhang2019stochastic,zhang2019multi}, but we remark that there are few key differences that make the theoretical analysis significantly more challenging for the GCNs. Please see Remark~\ref{remark:composite} and Appendix~\ref{supp:connect_to_comp_opt} for detailed discussion.

%% file: section/3-SGCN.tex
\section{Problem formulation}
\label{section:problem_formulation}

We begin by introducing the basic mathematical formulation of training GCNs. 
In this paper, we consider training GCNs in a semi-supervised multi-class classification setting. 
Given an undirected graph $\mathcal{G}(\mathcal{V},\mathcal{E})$ with $N=|\mathcal{V} |$ and $| \mathcal{E} |$, edges and the adjacency matrix $\mathbf{A}\in \{0, 1\}^{N\times N}$, we assume that each node is associated with a feature vector $\mathbf{x}_i\in\mathbb{R}^d$ and categorical label $y_i \in\mathbb{Z}+$. 
We use $\mathbf{X} = \{ \mathbf{x}_i \}_{i=1}^N$ and  $\mathbf{y} = \{y_i\}_{i=1}^N$ to denote the node feature matrix and label vector, respectively. The Laplacian matrix is calculated as $\mathbf{L} = \mathbf{D}^{-1/2}\mathbf{A} \mathbf{D}^{-1/2}$ or $\mathbf{L} = \mathbf{D}^{-1}\mathbf{A}$ where  $\mathbf{D}\in\mathbb{R}^{N\times N}$ is the degree matrix.  We use $\bm{\theta} = \{ \mathbf{W}^{(\ell)} \}_{\ell=1}^L$ to denote the stacked weight parameters of a $L$-layer GCN. The training of full-batch GCN (\texttt{FullGCN}) as an empirical risk minimization problem aims at  minimizing the loss $\mathcal{L}(\bm{\theta})$ over all training data
\begin{equation}\label{eqn:errm:gcn}
    \mathcal{L}(\bm{\theta}) = \frac{1}{N}\sum_{i=1}^N \text{Loss}(\mathbf{h}^{(L)}_i, y_i),~~
    \mathbf{H}^{(\ell)} = \sigma(\underbrace{\mathbf{L} \mathbf{H}^{(\ell-1)} \mathbf{W}^{(\ell)}}_{\mathbf{Z}^{(\ell)}}),
\end{equation}
where $\mathbf{H}^{(0)} = \mathbf{X} \in \mathbb{R}^{N \times d}$ denotes the input node features matrix, $\mathbf{h}^{(\ell)}_i$ is the $i$th row of $\mathbf{H}^{(\ell)}$ that corresponds to representation of $i$th node at layer $\ell$, $\text{Loss}(\cdot,\cdot)$ is the  loss function (e.g., cross-entropy loss) to measure the discrepancy between the prediction of the GCN and its ground truth label, and $\sigma(\cdot)$ is the activation function (e.g., ReLU function).

\input{supplementary/algorithm_sgcn_complete}
\section{A tight analysis of sampling based stochastic GCN training} \label{section:SGCN}

In this section, we start by introducing the challenges in efficiently training GCNs and introduce sampling-based GCN strategies to alleviate the issue (in Section~\ref{section:sampling_related}), and finally provide a tight analysis of the convergence rate of the sampling-based strategies for  training GCNs (in Section~\ref{section:sampling_analysis}).

\subsection{Scalable training via sampling} \label{section:sampling_related}
To efficiently solve empirical risk in problem~(Eq.~\ref{eqn:errm:gcn}), a simple idea to apply Stochastic Gradient Descent (SGD), where at each iteration we utilize a mini-batch $\mathcal{V}_{\mathcal{B}}$ of all nodes  to update the model parameters.  However, due to the interdependence of nodes in a graph, training GNNs on large-scale graphs remains a big challenge.
More specifically, in GNNs, the representation (embedding) of a node is obtained by gathering the embeddings of its neighbors from the previous layers. 
Unlike other neural networks that the final output and gradient can be perfectly decomposed over individual data samples, in GNNs, the embedding of a given node depends recursively on all its neighbor’s embedding, and such dependency grows exponentially with respect to the number of layers, a phenomenon known as \textit{neighbor explosion}, which prevents their application to large-scale graphs.

One practical solution to alleviate this issue is to leverage sampling-based GCN training strategy (\texttt{SGCN} in Algorithm~\ref{algorithm:sgcn_complete}) to sample a subset of nodes and construct a sparser normalized
Laplacian matrix $\widetilde{\mathbf{L}}^{(\ell)}$ for each layer with $\text{supp}(\widetilde{\mathbf{L}}^{(\ell)}) \ll \text{supp}(\mathbf{L})$, and perform forward and backward propagation only based on the sampled Laplacian matrices. 
More specifically, during training,
we sample a mini-batch of nodes  $\mathcal{V}_\mathcal{B}\subseteq\mathcal{V}$ from all nodes with size $B=|\mathcal{V}_\mathcal{B}|$, and construct the set of sparser Laplacian matrices $\{\widetilde{\mathbf{L}}^{(\ell)}\}_{\ell=1}^L$   based on nodes sampled at each layer  and compute the stochastic gradient to update  parameters as
\footnote{We use a tilde symbol $\widetilde{\square}$ for their stochastic form}
\begin{equation}
    \nabla \widetilde{\mathcal{L}}(\bm{\theta}) = \frac{1}{B} \sum_{i\in\mathcal{V}_\mathcal{B}} \nabla \text{Loss}(\widetilde{\mathbf{h}}^{(L)}_i, y_i),~~
    \widetilde{\mathbf{H}}^{(\ell)} = \sigma(\underbrace{\widetilde{\mathbf{L}}^{(\ell)} \widetilde{\mathbf{H}}^{(\ell-1)} \mathbf{W}^{(\ell)}}_{\widetilde{\mathbf{Z}}^{(\ell)}}).
\end{equation}

The sparse Laplacian matrix construction algorithms can be roughly classified as \emph{nodewise} sampling, \emph{layerwise} sampling, and \emph{subgraph} sampling. 

\paragraph{Node-wide sampling.}
The main idea of node-wise sampling is to first sample all the nodes needed for the computation using neighbor sampling (NS), then train the GCN based on the sampled nodes. 
For each node in the $\ell$th GCN layer, NS randomly samples $s$ of its neighbors at the $(\ell-1)$th GCN layer and formulate $\widetilde{\mathbf{L}}^{(\ell)}$ by
\begin{equation}\label{eq:node_wise}
    \widetilde{L}_{i,j}^{(\ell)} = \begin{cases}
 \frac{|\mathcal{N}(i)|}{s} \times {L}_{i,j}, &\text{ if } j \in \widetilde{\mathcal{N}}^{(\ell)}(i) \\
 0, &\text{ otherwise }
\end{cases}
\end{equation}
where $\mathcal{N}(i)$ is the full set of $i$th node neighbor, $\widetilde{\mathcal{N}}^{(\ell)}(i) \subseteq \mathcal{N}(i)$ is the sampled neighbors of node $i$ for $\ell$th GCN layer.
\texttt{GraphSAGE}~\cite{graphsage} and \texttt{VRGCN}~\cite{control_variacne_gcn} follows the spirit of node-wise sampling where it performs uniform node sampling on the previous layer neighbors for a fixed number of nodes to bound the mini-batch computation complexity. 

\paragraph{Layer-wise sampling.}
To avoid the neighbor explosion issue, layer-wise sampling is introduced to control the size of sampled neighborhoods in each layer. 
For the $\ell$th GCN layer, layer-wise sampling methods sample a set of nodes $\mathcal{B}^{(\ell)}\subseteq\mathcal{V}$ of size $s$ under the distribution $\boldsymbol{p}$ to approximate the Laplacian by
\begin{equation}\label{eq:layer-wise}
    \widetilde{L}_{i,j}^{(\ell)} = \begin{cases}
 \frac{1}{s \times p_j}\times {L}_{i,j}, &\text{ if } j \in \mathcal{B}^{(\ell)} \\
 0, &\text{ otherwise }
\end{cases}
\end{equation}
Existing work \texttt{FastGCN}~\cite{fastgcn} and \texttt{LADIES}~\cite{ladies} follows the spirit of layer-wise sampling.
\texttt{FastGCN} performs independent node sampling for each layer and applies important sampling to reduce variance and results in a constant sample size in all layers. 
However, mini-batches potentially become too sparse to achieve high accuracy. 
\texttt{LADIES} improves \texttt{FastGCN} by layer-dependent sampling. 
Based on the sampled nodes in the upper layer, it selects their neighborhood nodes, constructs a bipartite subgraph, and computes the importance probability accordingly. 
Then, it samples a fixed number of nodes based on the calculated probability, and recursively conducts such a procedure per layer to construct the whole computation graph. 

\paragraph{Subgraph sampling.}
Subgraph sampling is similar to layer-wise sampling by restricting the sampled Laplacian matrices at each layer are identical
\begin{equation}\label{eq:graph-wise}
    \widetilde{L}_{i,j}^{(1)} = \ldots = \widetilde{L}_{i,j}^{(L)} = \begin{cases}
 \frac{1}{s \times p_j}\times {L}_{i,j}, &\text{ if } j \in \mathcal{B}^{(\ell)} \\
 0, &\text{ otherwise }
\end{cases}
\end{equation}
For example, \texttt{GraphSAINT}~\cite{graphsaint} can be viewed as a special case of layer-wise sampling algorithm \texttt{FastGCN} by restricting the nodes sampled at the $1$-st to the $(L-1)$th layer the same as the nodes sampled at the $L$th layer. 
However, \texttt{GraphSAINT} requires a significant large mini-batch size compared to other layer-wise sampling methods. 
We leave this as a potential future direction to explore.

Although sampling methods can alleviate the neighborhood explosion issue and make scalable training possible, sampling introduces a bias which degrades the convergence rate. 
In the next subsection, we provide a rigorous analysis of convergence rate to demonstrate the impact of sampling on convergence rate compared to SGD with no sampling being utilized.

\subsection{A tight analysis of convergence rate} \label{section:sampling_analysis}

Compared to vanilla SGD, the key challenge of theoretical understanding for \texttt{SGCN} training is the \textit{biasedness} of stochastic gradient due to sampling of nodes at inner layers. 
To see this, let denote \texttt{FullGCN}'s full-batch gradient as $\nabla \mathcal{L}(\bm{\theta}) = \{ \mathbf{G}^{(\ell)} \}_{\ell=1}^L,~\mathbf{G}^{(\ell)} = \frac{\partial \mathcal{L}(\bm{\theta})}{\partial \mathbf{W}^{(\ell)}}$ and \texttt{SGCN}'s stochastic gradient as $\nabla \widetilde{\mathcal{L}}(\bm{\theta}) = \{ \widetilde{\mathbf{G}}^{(\ell)} \}_{\ell=1}^L,~ \widetilde{\mathbf{G}}^{(\ell)} = \frac{\partial \widetilde{\mathcal{L}}(\bm{\theta})}{ \partial \mathbf{W}^{(\ell)}}$.
By the chain rule, we can compute the  full-batch gradient $\mathbf{G}_t^{(\ell)}$ w.r.t.  the $\ell$th layer weight matrix  $\mathbf{W}^{(\ell)}$ as
\begin{equation}\label{eq:full_batch_gradient}
\begin{aligned}
    \mathbf{G}_t^{(\ell)} &= [\mathbf{L} \mathbf{H}_t^{(\ell-1)}]^\top \big( \mathbf{D}_t^{(\ell+1)} \circ \sigma^\prime(\mathbf{Z}_t^{(\ell)}) \big), \\
    \mathbf{D}_t^{(\ell)} &= \mathbf{L}^\top \big( \mathbf{D}_t^{(\ell+1)} \circ \sigma^\prime(\mathbf{Z}_t^{(\ell)}) \big) \mathbf{W}_t^{(\ell)},~\text{and}~\mathbf{D}_t^{(L+1)} = \frac{\partial \mathcal{L}(\bm{\theta}_t)}{\partial \mathbf{H}^{(L)}},
\end{aligned}
\end{equation}
and compute stochastic gradient   $\widetilde{\mathbf{G}}_t^{(\ell)}$ utilized in \texttt{SGCN} for the $\ell$th layer w.r.t. $\mathbf{W}^{(\ell)}$ as 
\begin{equation}\label{eq:stochastic_gradient_sgcn}
\begin{aligned}
    \widetilde{\mathbf{G}}_t^{(\ell)} &= [\widetilde{\mathbf{L}}^{(\ell)} \widetilde{\mathbf{H}}_t^{(\ell-1)}]^\top \big( \widetilde{\mathbf{D}}_t^{(\ell+1)} \circ \sigma^\prime(\widetilde{\mathbf{Z}}_t^{(\ell)}) \big) ,\\
    \widetilde{\mathbf{D}}_t^{(\ell)} &= [\widetilde{\mathbf{L}}^{(\ell)}]^\top \big( \widetilde{\mathbf{D}}_t^{(\ell+1)} \circ \sigma^\prime(\widetilde{\mathbf{Z}}_t^{(\ell)}) \big) \mathbf{W}_t^{(\ell)},~\text{and}~\widetilde{\mathbf{D}}_t^{(L+1)} = \frac{\partial \widetilde{\mathcal{L}}(\bm{\theta}_t)}{\partial \widetilde{\mathbf{H}}^{(L)}}.
\end{aligned}
\end{equation}
For any layer $\ell\in[L]$, the stochastic gradient $\widetilde{\mathbf{G}}_t^{(\ell)}$ is a biased estimator of full-batch gradient $\mathbf{G}_t^{(\ell)}$, as it is computed from $\mathbf{H}^{(L)}_t$ and $\mathbf{Z}_t^{(\ell)}$, which are not available in \texttt{SGCN} since $\widetilde{\mathbf{H}}_t^{(L)}$ and $\widetilde{\mathbf{Z}}_t^{(\ell)}$ are used as an approximation 
during training. Recently,~\cite{control_variacne_gcn} established a convergence rate under the strong assumption that the stochastic gradient of \texttt{SGCN} is unbiased and \cite{biased_but_consistent} provided another analysis under the strong assumption that the stochastic gradient converges to the consistent gradient exponentially fast as the number of sampled nodes increases. While both studies establish the same convergence rate of $\mathcal{O}(1/\sqrt{T})$, however,  these assumptions do not hold in reality  due to the composite structure of the training objectives and sampling of nodes at inner layers. Motivated by this, we aim at providing a tight analysis without the aforementioned strong assumptions on the stochastic gradient. Our analysis is inspired by the bias and variance decomposition of the mean-square error of stochastic gradient, which has been previously used in~\cite{cong2020minimal} to analyze the variance of  stochastic gradient in GCN.\footnote{In this paper, we use a slightly different formulation of bias-variance decomposition as in~\cite{cong2020minimal}. Please refer to Section~\ref{supp:connect_to_comp_opt} for details.}
Formally, we can decompose mean-square error of stochastic gradient as
\begin{equation}
\begin{aligned}
    &\mathbb{E}[\|\nabla \widetilde{\mathcal{L}}(\bm{\theta}) - \nabla\mathcal{L}(\bm{\theta}) \|_\mathrm{F}^2]= \underbrace{\mathbb{E}[\|\mathbb{E}[ \nabla \widetilde{\mathcal{L}}(\bm{\theta})] - \nabla \mathcal{L}(\bm{\theta}) \|_\mathrm{F}^2]}_{\text{Bias}~\mathbb{E}[\|\mathbf{b}\|_\mathrm{F}^2]} + \underbrace{\mathbb{E}[\|\nabla \widetilde{\mathcal{L}}(\bm{\theta}) - \mathbb{E}[\nabla \widetilde{\mathcal{L}}(\bm{\theta})] \|_\mathrm{F}^2]}_{\text{Variance}~\mathbb{E}[\|\mathbf{n}\|_\mathrm{F}^2]},
\end{aligned}
\end{equation}
where the bias terms $\mathbb{E}[\|\mathbf{b}\|_\mathrm{F}^2]$ is mainly due to the node embedding approximation variance (\emph{zeroth-order variance}) during forward propagation and the variance term $\mathbb{E}[\|\mathbf{n}\|_\mathrm{F}^2]$ is mainly due to the layerwise gradient variance (\emph{first-order variance}) during backward propagation. Please refer to Figure~\ref{fig:illustrate_variance} for an illustration about the two types of variances. Before proceeding to analysis, we make the following standard assumptions on the Lipschitz-continuity and smoothness of the loss function $\text{Loss}(\cdot,\cdot)$ and activation function $\sigma(\cdot)$.

\begin{figure}[t]
    \centering
    \includegraphics[width=1.0\textwidth]{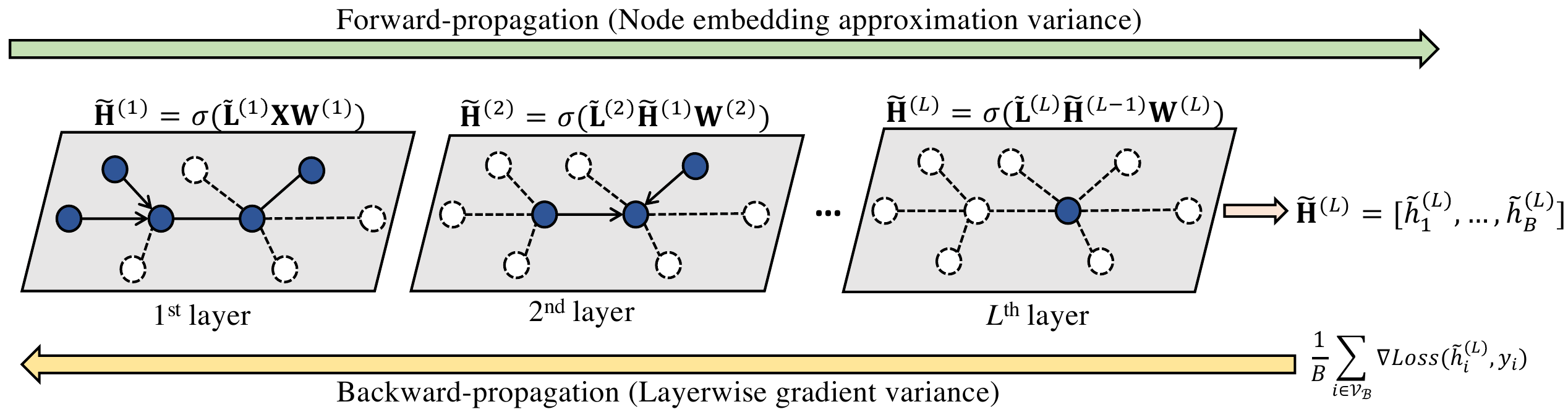}
    \caption{Relationship between the two types of variance with the training process, where embedding approximation variance (zeroth-order variance) happens during forward-propagation and layerwise gradient variance (first-order variance) happens during backward-propagation.}
    \label{fig:illustrate_variance}
\end{figure}


\begin{assumption}\label{assumption:loss_lip_smooth}
Let assume there exist constant $C_\text{loss}, L_\text{loss}$ such that the following inequality holds: 
\begin{equation}
\|\text{Loss}(\mathbf{h}^{(L)},y) - \text{Loss}({\mathbf{h}^\prime}^{(L)},y) \|_2 \leq C_\text{loss}\|\mathbf{h}^{(L)} - {\mathbf{h}^\prime}^{(L)}\|_2,
\end{equation}
and 
\begin{equation}
\|\nabla \text{Loss}(\mathbf{h}^{(L)},y) - \nabla \text{Loss}({\mathbf{h}^\prime}^{(L)},y) \|_2 \leq L_\text{loss}\|\mathbf{h}^{(L)} - {\mathbf{h}^\prime}^{(L)}\|_2.
\end{equation}
\end{assumption}

\begin{assumption}\label{assumption:sigma_lip_smooth}
Let assume there exist constant $C_\sigma, L_\sigma$ such that the following inequality holds: $\|\sigma(\mathbf{z}^{(\ell)}) - \sigma({\mathbf{z}^\prime}^{(\ell)}) \|_2 \leq C_\sigma\|\mathbf{z}^{(\ell)} - {\mathbf{z}^\prime}^{(\ell)}\|_2$ and $\|\sigma^\prime(\mathbf{z}^{(\ell)}) - \sigma^\prime({\mathbf{z}^\prime}^{(\ell)}) \|_2 \leq L_\sigma\|\mathbf{z}^{(\ell)} - {\mathbf{z}^\prime}^{(\ell)}\|_2$.
\end{assumption}

We also make the following customary assumptions on the norm of weight matrices, Laplacian matrices, and node feature matrix, which are also used in the generalization analysis of GNNs~\cite{garg2020generalization,liao2020pac}. 
\begin{assumption}\label{assumption:bound_norm}
For any $\ell\in[L]$, the norm of weight matrices, Laplacian matrices, node features are bounded $\|\mathbf{W}^{(\ell)}\|_{_\mathrm{F}} \leq B_W$, $\|\widetilde{\mathbf{L}}^{(\ell)}\|_{_\mathrm{F}} \leq B_{LA}$, $\|\mathbf{L}\|_{_\mathrm{F}} \leq B_{LA}$, and $ \|\mathbf{X}\|_{_\mathrm{F}} \leq B_X$, respectively.
\end{assumption}


Before presenting the convergence of \texttt{SGCN}, we introduce the notation of \emph{propagation matrices} $\{ \mathbf{P}^{(\ell)} \}_{\ell=1}^L $, which are defined as the column-wise expectation of the sparser Laplacian matrices.
For example, let consider a general form of the sampled Laplacian matrix $\widetilde{\mathbf{L}}^{(\ell)} \in\mathbb{R}^{N\times N}$ as
\begin{equation}
    \widetilde{L}_{i,j}^{(\ell)} = \begin{cases}
    \frac{L_{i,j}}{\alpha_{i,j}} & \text{ if } i\in\mathcal{B}^{(\ell)}~\text{and}~j\in\mathcal{B}^{(\ell-1)} \\ 
    0 & \text{ otherwise } 
    \end{cases},
\end{equation}
where $\alpha_{i,j}$ is the weighted constant depends on the sampling algorithms.Then, the propagation matrix $\mathbf{P}^{(\ell)}\in\mathbb{R}^{N\times N}$ is denoted as $P_{i,j}^{(\ell)} = \mathbb{E}_{i\in\mathcal{B}^{(\ell)}} \Big[ \widetilde{L}_{i,j}^{(\ell)}~|~i\in\mathcal{B}^{(\ell)} \Big]$,  
where the expectation is taken over row indices $i$. 
Note that this notation is only for presenting the theoretical results, and are not used in the practical training algorithms.
By doing so, we can decompose the difference between  $\widetilde{\mathbf{L}}^{(\ell)}$ and  $\mathbf{L}$ as the summation of  difference $\|\widetilde{\mathbf{L}}^{(\ell)} - \mathbf{P}^{(\ell)}\|_\mathrm{F}^2$ and difference $\|\mathbf{P}^{(\ell)} - \mathbf{L}\|_\mathrm{F}^2$.


In the following theorem, we show that the upper bound of the bias and variance of stochastic gradient is closely related to the expectation of  difference $\mathbb{E}[\|\widetilde{\mathbf{L}}^{(\ell)} - \mathbf{P}^{(\ell)}\|_\mathrm{F}^2]$ and  difference $\mathbb{E}[\|\mathbf{P}^{(\ell)} - \mathbf{L}\|_\mathrm{F}^2]$ which can significantly impact the convergence of \texttt{SGCN}. 
\begin{theorem} [Convergence of \texttt{SGCN}] \label{theorem:convergence_of_sgcn}
Suppose Assumptions~\ref{assumption:loss_lip_smooth}, \ref{assumption:sigma_lip_smooth}, \ref{assumption:bound_norm} hold and apply \texttt{SGCN} with learning rate chosen as $\eta=\min\{1/L_\mathrm{F}, 1/\sqrt{T}\}$ where $L_\mathrm{F}$ is the smoothness constant. 
Let $\Delta_\mathbf{n}$ and $\Delta_\mathbf{b}$ denote the upper bound on the variance and bias of stochastic gradients as:
\begin{equation}
    \Delta_\mathbf{n} = \sum_{\ell=1}^L \mathcal{O}(\mathbb{E}[\| \widetilde{\mathbf{L}}^{(\ell)} - \mathbf{P}^{(\ell)} \|_{_\mathrm{F}}^2]) + \mathcal{O}(\mathbb{E}[\| \mathbf{P}^{(\ell)} - \mathbf{L} \|_\mathrm{F}^2]),~
    \Delta_\mathbf{b} = \sum_{\ell=1}^L \mathcal{O}(\mathbb{E}[\|\mathbf{P}^{(\ell)} - \mathbf{L}\|_{_\mathrm{F}}^2]).
\end{equation}
Then, the output of \texttt{SGCN} satisfies
\begin{equation}
    \min_{t\in[T]} \mathbb{E}[\|\nabla \mathcal{L}(\bm{\theta}_t)\|_{\mathrm{F}}^2]
    \leq \frac{2(\mathcal{L}(\bm{\theta}_1) - \mathcal{L}(\bm{\theta}^\star))}{\sqrt{T}} + \frac{L_\mathrm{F} \Delta_\mathbf{n}}{\sqrt{T}} + \Delta_\mathbf{b}.
\end{equation}
\end{theorem}

The proof is deferred to Appendix~\ref{section:proof_of_thm1} where we also compute  the exact value of key parameters $L_{\mathrm{F}}$, $\Delta_\mathbf{n}$, and $\Delta_\mathbf{b}$  in Lemma~\ref{lemma:smoothness_L_layer}, Lemma~\ref{lemma:upper-bound-sgcn-vars}, and Lemma~\ref{lemma:upper-bound-sgcn-bias}, respectively.

From the rate obtained in Theorem~\ref{theorem:convergence_of_sgcn}, we can observe that  after $T$ iterations the gradient norm of
\texttt{SGCN} is at most $\mathcal{O}(\Delta_\mathbf{n}/\sqrt{T})+\Delta_\mathbf{b}$, which suffers from a constant residual error $\Delta_\mathbf{b}$ that is  not decreasing as the number of iterations $T$ increases. 
Without the bias\footnote{We have $\Delta_\mathbf{b}=0$ if all neighbors are used to calculate the exact node embeddings, i.e., $\mathbf{P}^{(\ell)} = \mathbf{L},~\forall \ell\in[L]$.} we recover the convergence of vanilla SGD.  Of course, this type of convergence is only useful if $\Delta_\mathbf{b}$ and $\Delta_\mathbf{n}$ are small enough.
We note that existing \texttt{SGCN} algorithms propose to reduce $\Delta_\mathbf{b}$ by increasing the number of neighbors sampled at each layer (e.g., \texttt{GraphSAGE}), or  applying importance sampling (e.g., \texttt{FastGCN}, \texttt{LADIES} and \texttt{GraphSAINT}).

\begin{remark}\label{remark:composite} Although we formulate sampling-based GCNs as a special case of the composite optimization problem, it is worth noting that compared to the classical composite optimization, there are a few key differences that make the utilization variance reduced composite optimization non-trivial: 
(a) different objective function that makes the GCN analysis challenging; 
(b) different gradient computation, analysis, and algorithm which make directly applying multi-level variance reduction methods such as SPIDER~\cite{zhang2019multi} nontrivial; 
(c) different theoretical results and novel intuition for sampling-based GCN training.
Due to the space limit, we defer the detail discussion to the Appendix~\ref{supp:connect_to_comp_opt}.
\end{remark}



%% file: supplementary/algorithm_sgcn_complete.tex
\begin{algorithm}[t]
  \caption{\texttt{SGCN}: Vanilla sampling-based GCN training method }
\begin{algorithmic}[1]
\label{algorithm:sgcn_complete}
  \STATE {\bfseries Input:} Learning rate $\eta>0$, a sampling strategy
    \FOR{$t = 1,\ldots,T$}
        \STATE Sample a mini-batch $\mathcal{V}_\mathcal{B}\subset \mathcal{V}$ of all nodes with size $B$ \\
        \STATE Sample a  subset of nodes $\mathcal{B}^{(\ell)}$ for each layer $\ell = 1, 2, \ldots, L$ based on given sampling strategy 
        \STATE Compute per-layer sparse Laplacian matrices $\widetilde{\mathbf{L}}^{(\ell)}$ based on $\mathcal{B}^{(\ell)}$ for all layers $\ell = 1, 2, \ldots, L$
        \STATE Calculate node embeddings using 
        \begin{equation}
            \widetilde{\mathbf{H}}^{(\ell)} = \sigma(\widetilde{\mathbf{L}}^{(\ell)} \widetilde{\mathbf{H}}^{(\ell-1)} \mathbf{W}^{(\ell)}),\text{ where }\widetilde{\mathbf{H}}^{(0)} = \mathbf{X},
        \end{equation}\vspace{-5pt}
        \STATE Calculate loss as $\widetilde{\mathcal{L}}(\bm{\theta}_t) = \frac{1}{B}\sum_{i\in\mathcal{V}_\mathcal{B}} \text{Loss}(\widetilde{\bm{h}}_i^{(L)}, y_i)$
        \STATE Calculate stochastic gradient $\nabla \widetilde{\mathcal{L}}(\bm{\theta}_t) = \{ \widetilde{\mathbf{G}}^{(\ell)}\}_{\ell=1}^L $ as
        \begin{equation}
            \begin{aligned}
            \widetilde{\mathbf{G}}_t^{(\ell)} &:= [\widetilde{\mathbf{L}}^{(\ell)} \widetilde{\mathbf{H}}_t^{(\ell-1)}]^\top \Big( \widetilde{\mathbf{D}}_t^{(\ell+1)} \circ \nabla \sigma(\widetilde{\mathbf{Z}}_t^{(\ell)}) \Big),~\\
            \widetilde{\mathbf{D}}_t^{(\ell)} &:= [\widetilde{\mathbf{L}}^{(\ell)}]^\top \Big( \widetilde{\mathbf{D}}_t^{(\ell+1)} \circ \nabla \sigma(\widetilde{\mathbf{Z}}_t^{(\ell)}) \Big) \mathbf{W}_t^{(\ell)},~
            \widetilde{\mathbf{D}}_t^{(L+1)} = \frac{\partial \widetilde{\mathcal{L}}(\bm{\theta}_t)}{\partial \widetilde{\mathbf{H}}^{(L)}}
            \end{aligned}
        \end{equation}
        \STATE Update parameters as $\bm{\theta}_{t+1} = \bm{\theta}_t - \eta \nabla \widetilde{\mathcal{L}}(\bm{\theta}_t)$ 
    \ENDFOR
\STATE {\bfseries Output:} Model with parameter $\bm{\theta}_{T+1}$ \\
\end{algorithmic}
\end{algorithm}

%% file: section/4-SGCN+.tex
\section{\text{SGCN+}: Zeroth-order Variance Reduction} \label{section:SGCN+}


An important question to answer is: \textit{can we eliminate the residual error without using all neighbors during forward-propagation?}
A remarkable attempt to answer this question has been recently made in \texttt{VRGCN}~\cite{control_variacne_gcn} 
where they propose to use historical node embeddings as an approximation to estimate the true node embeddings.
More specifically, the graph convolution in \texttt{VRGCN} is defined as $\widetilde{\mathbf{H}}^{(\ell)}_t = \sigma\big( \mathbf{L} \widetilde{\mathbf{H}}^{(\ell-1)}_{t-1} \mathbf{W}^{(\ell)} + \widetilde{\mathbf{L}}^{(\ell)} (\widetilde{\mathbf{H}}^{(\ell-1)}_t - \widetilde{\mathbf{H}}^{(\ell-1)}_{t-1}) \mathbf{W}^{(\ell)} \big)$.
Taking advantage of historical node embeddings, \texttt{VRGCN} requires less sampled neighbors and results in significant less computation overhead during gradient computation.
Although \texttt{VRGCN} achieves significant speed up and better performance compared to other \texttt{SGCN}s, it involves using the full Laplacian matrix at each iteration, which can be computationally prohibitive. 
Moreover, since both \texttt{SGCN}s and \texttt{VRGCN} are approximating the exact node embeddings calculated using all neighbors, it is still not clear why \texttt{VRGCN} achieves a better convergence result than \texttt{SGCN}s using historical node embeddings. 

\input{section/algorithm_SGCN+}

To fill in these gaps, we introduce \emph{zeroth-order variance reduced} sampling-based GCN training method dubbed as \texttt{SGCN+}. 
As shown in Algorithm~\ref{algorithm:sgcn+}, \texttt{SGCN+}\footnote{For ease the exposition, we include a higher level version of the algorithm here and defer the detailed version to the Algorithm~\ref{algorithm:sgcn_plus_complete} in Appendix~\ref{supp:detail_algorithm}.} has two types of forward propagation: the forward propagation at the \emph{snapshot steps} and the forward propagation at the \emph{regular steps}.
At the snapshot step (lines~\ref{line:sgcn+snapshot_start}-\ref{line:sgcn+snapshot_end} in Algorithm~\ref{algorithm:sgcn+}), a full Laplacian matrix is utilized:
\begin{equation} \label{eq:sgcn_plus_snapshot}
    \mathbf{H}^{(\ell)}_t = \sigma(\mathbf{Z}^{(\ell)}_t),~
    \mathbf{Z}^{(\ell)}_t = \mathbf{L} \mathbf{H}_t^{(\ell-1)} \mathbf{W}^{(\ell)}_t,~
    \widetilde{\mathbf{Z}}^{(\ell)}_t \leftarrow \mathbf{Z}^{(\ell)}_t
\end{equation}

During the regular steps  (lines~\ref{line:sgcn+regular_start}-\ref{line:sgcn+regular_end} in Algorithm~\ref{algorithm:sgcn+}), the sampled Laplacian matrix is utilized:
\begin{equation} \label{eq:sgcn_plus_regular}
    \widetilde{\mathbf{H}}^{(\ell)}_t = \sigma(\widetilde{\mathbf{Z}}^{(\ell)}),~
    \widetilde{\mathbf{Z}}^{(\ell)}_t = \widetilde{\mathbf{Z}}^{(\ell)}_{t-1} + \widetilde{\mathbf{L}}^{(\ell)} \widetilde{\mathbf{H}}_t^{(\ell-1)} \mathbf{W}_t^{(\ell)} - \widetilde{\mathbf{L}}^{(\ell)} \widetilde{\mathbf{H}}_{t-1}^{(\ell-1)} \mathbf{W}_{t-1}^{(\ell)}
\end{equation}

Besides, due to the aforementioned recursive update rule (Eq.~\ref{eq:sgcn_plus_regular}), the norm of node embeddings are not guaranteed to be bounded as in vanilla \texttt{SGCN}. 
Notice that the unbounded norm embedding potentially results in gradient explosion that unstabilize the training process, and this issue also exists in \texttt{VRGCN} (Proposition~\ref{proposition:matrix_norm_bound}).
To overcome the issue, we introduce an early stop criterion (line~\ref{line:sgcn+early_stop}) by checking the relative scale between the norm of the current node embedding to the snapshot one, and immediately start another snapshot step if the condition is violated. 
By properly chosen $\alpha$ and snapshot gap $K$, the proposed \texttt{SGCN+} only requires one full Laplacian graph convolution operation at most every $K$ iterations, which significantly reduce the computation burden of \texttt{VRGCN}.

In the following theorem, we introduce the convergence result of \texttt{SGCN+}.
Recall that the node embedding approximation variance (\emph{zeroth-order variance}) determines the bias of stochastic gradient $\mathbb{E}[\|\mathbf{b}\|_\mathrm{F}^2]$. 
Applying \texttt{SGCN+} can significantly reduce the bias of stochastic gradients, such that its value is small enough that it will not deteriorate the convergence.

\begin{theorem} [Convergence of \texttt{SGCN}+]\label{theorem:convergence_of_sgcn_plus}
Suppose Assumptions~\ref{assumption:loss_lip_smooth}, \ref{assumption:sigma_lip_smooth}, \ref{assumption:bound_norm} hold and apply \texttt{SGCN+} with learning rate chosen as $\eta=\min\{1/L_\mathrm{F}, 1/\sqrt{T}\}$ where $L_\mathrm{F}$ is the smoothness constant.
Let $\Delta_\mathbf{n}$ and $\Delta_\mathbf{b}^+$ denote the upper bound for the variance and bias of stochastic gradient as:
\begin{equation}
    \begin{aligned}
        \Delta_\mathbf{n} &= \sum_{\ell=1}^L \mathcal{O}(\mathbb{E}[\| \widetilde{\mathbf{L}}^{(\ell)} - \mathbf{P}^{(\ell)} \|_{_\mathrm{F}}^2]) + \mathcal{O}(\mathbb{E}[\| \mathbf{P}^{(\ell)} - \mathbf{L} \|_\mathrm{F}^2]), \\
        \Delta_\mathbf{b}^+ &= \eta^2 \Delta_\mathbf{b}^{+\prime},~
        \Delta_\mathbf{b}^{+\prime} = \mathcal{O}\Big( \alpha^4 K\sum_{\ell=1}^L |\mathbb{E}[\| \mathbf{P}^{(\ell)} \|_\mathrm{F}^2] - \|\mathbf{L}\|_\mathrm{F}^2| \Big)
    \end{aligned}
\end{equation}

Then, the output of \texttt{SGCN+} satisfies 
\begin{equation}
\begin{aligned}
    \min_{t\in[T]} \mathbb{E}[\|\nabla \mathcal{L}(\bm{\theta}_t)\|_{\mathrm{F}}^2]
    &\leq \frac{2(\mathcal{L}(\bm{\theta}_1) - \mathcal{L}(\bm{\theta}^\star))}{\sqrt{T}} + \frac{L_\mathrm{F} \Delta_\mathbf{n}}{\sqrt{T}} + \frac{\Delta_\mathbf{b}^{+\prime}}{T}.
    \end{aligned}
\end{equation}
\end{theorem}

\vspace{-5pt}

The proof of theorem is deferred to Appendix~\ref{section:proof_of_thm2} and the exact value of key parameters $L_\mathrm{F}$, $\Delta_\mathbf{n}$, and $\Delta_\mathbf{b}^+$ are computed in Lemma~\ref{lemma:smoothness_L_layer}, Lemma~\ref{lemma:upper-bound-sgcn-vars}, and Lemma~\ref{lemma:upper-bound-sgcn-plus-bias},  respectively, and can be found in Appendices~\ref{section:proof_of_thm1} and~\ref{section:proof_of_thm2}.

An immediate implication of Theorem~\ref{theorem:convergence_of_sgcn_plus} is that  that after $T$ iterations the gradient norm of \texttt{SGCN+} is at most $\mathcal{O
}(\Delta_\mathbf{n}/\sqrt{T}) + \mathcal{O}(\Delta_\mathbf{b}^{+\prime}/T)$. 
When using all neighbors for calculating the exact node embeddings, we have $\mathbf{P}^{(\ell)} = \mathbf{L}$ such that $\Delta_\mathbf{b}^{+\prime} = 0$, which leads to convergence rate of SGD. 
Compared with vallina \texttt{SGCN}, the bias of \texttt{SGCN+} is scaled by learning rate $\eta$. 
Therefore, we can reduce the negative effect of bias by choosing the learning rate as $\eta=\mathcal{O}(1/\sqrt{T})$. 
This also explains why  \texttt{SGCN+} achieves a significantly better convergence rate compared to \texttt{SGCN}.

%% file: section/algorithm_SGCN+.tex
\begin{algorithm}[tb]
  \caption{\texttt{SGCN+}: Zeroth-order variance reduction (Detailed version in Algorithm~\ref{algorithm:sgcn_plus_complete})}
  \label{algorithm:sgcn+}
\begin{algorithmic}[1]
  \STATE {\bfseries Input:} Learning rate $\eta>0$, snapshot gap $K\geq 1$, $t_0=1$ and $s=1$, staleness factor $\alpha \geq 1$
    \FOR{$t = 1,\ldots,T$}
        \IF{$(t - t_{s-1})~\text{mod}~K=0$}
            \STATE Calculate node embeddings using Eq.~\ref{eq:sgcn_plus_snapshot} \label{line:sgcn+snapshot_start}
            \STATE Calculate full-batch gradient $\nabla \mathcal{L}(\bm{\theta}_t)$ as Eq.~\ref{eq:full_batch_gradient} and update as $\bm{\theta}_{t+1} = \bm{\theta}_t - \eta \nabla \mathcal{L}(\bm{\theta}_t)$ 
            \STATE Set $t_s = t$ and $s = s + 1$ \label{line:sgcn+snapshot_end}
        \ELSE
            \STATE Calculate node embeddings using Eq.~\ref{eq:sgcn_plus_regular} \label{line:sgcn+regular_start}
            \IF{$\|\widetilde{\mathbf{H}}^{(\ell)}_{t-1}\|_\mathrm{F} \geq \alpha \| \mathbf{H}^{(\ell)}_{t_{s-1}} \|_\mathrm{F}$ for any $\ell \in [L]$} \label{line:sgcn+early_stop}
                \STATE Go to line~\ref{line:sgcn+snapshot_start}
            \ENDIF
            \STATE Calculate stochastic gradient $\nabla \widetilde{\mathcal{L}}(\bm{\theta}_t)$ as Eq.~\ref{eq:stochastic_gradient_sgcn} and update as $\bm{\theta}_{t+1} = \bm{\theta}_t - \eta \nabla \widetilde{\mathcal{L}}(\bm{\theta}_t)$  \label{line:sgcn+regular_end}
        \ENDIF
    \ENDFOR
\STATE {\bfseries Output:} Model with parameter $\bm{\theta}_{T+1}$ \\
\end{algorithmic}
\end{algorithm}

%% file: section/5-SGCN++.tex
\section{\text{SGCN++}: Doubly Variance Reduction} \label{section:SGCN++}
Algorithm~\ref{algorithm:sgcn+} applies zeroth-order variance reduction on node embedding matrices and results in a faster convergence. 
However, both \texttt{SGCN} and \texttt{SGCN+} suffer from the same stochastic gradient variance $\Delta_\mathbf{n}$, which can be only reduced either by increasing the mini-batch size of \texttt{SGCN} or applying variance reduction on stochastic gradient.
An interesting question that arises is: \emph{can we further accelerate the convergence by simultaneously employing zeroth-order variance reduction on node embeddings and first-order variance reduction on layerwise gradient?}
To answer this question, we propose \emph{doubly variance reduction} algorithm \texttt{SGCN++}, that extends the variance reduction algorithm from node embedding approximation to layerwise gradient estimation. 

\input{section/algorithm_SGCN++}

As shown in Algorithm~\ref{algorithm:sgcn++}, the main idea of \texttt{SGCN++}~\footnote{For ease the exposition, we include a higher level version of the algorithm here and defer the detailed version to the Algorithm~\ref{algorithm:sgcn_plus_plus_complete} in Appendix~\ref{supp:detail_algorithm}.} is to use the historical gradient as control variants for current layerwise gradient estimation.
More specifically, similar to \texttt{SGCN+} that has two types of forward propagation steps, \texttt{SGCN++} also has two types of backward propagation: 
at the \emph{snapshot steps} and at the \emph{regular steps}. 
The snapshot steps (lines~\ref{line:sgcn++snapshot_start}-\ref{line:sgcn++snapshot_end}) backward propagation are full-batch gradient computation as is defined in Eq.~\ref{eq:full_batch_gradient}, and the computed full-batch gradient are saved as control variants for the following regular steps.
The backward propagation (lines~\ref{line:sgcn++regular_start}-\ref{line:sgcn++regular_end}) at  regular steps are defined as
\begin{equation}\label{eq:sgcn_plus_plus_regular}
\begin{aligned}
    \widetilde{\mathbf{G}}_t^{(\ell)} 
    &= \widetilde{\mathbf{G}}_{t-1}^{(\ell)} + [\widetilde{\mathbf{L}}^{(\ell)} \widetilde{\mathbf{H}}_t^{(\ell-1)}]^\top \Big(\widetilde{\mathbf{D}}_t^{(\ell+1)} \circ  \sigma^\prime(\widetilde{\mathbf{Z}}_t)\Big) - [\widetilde{\mathbf{L}}^{(\ell)} \widetilde{\mathbf{H}}_{t-1}^{(\ell-1)}]^\top \Big(\widetilde{\mathbf{D}}_{t-1}^{(\ell+1)} \circ  \sigma^\prime(\widetilde{\mathbf{Z}}_{t-1})\Big)  \\
    \widetilde{\mathbf{D}}_t^{(\ell)} 
    &= \widetilde{\mathbf{D}}_{t-1}^{(\ell)}  + [\widetilde{\mathbf{L}}^{(\ell)}]^\top \Big(\widetilde{\mathbf{D}}_t^{(\ell+1)} \circ \sigma^\prime(\widetilde{\mathbf{Z}}_t) \Big) [\mathbf{W}_t^{(\ell)}] - [\widetilde{\mathbf{L}}^{(\ell)}]^\top \Big(\widetilde{\mathbf{D}}_{t-1}^{(\ell+1)} \circ \sigma^\prime(\widetilde{\mathbf{Z}}_{t-1}) \Big) [\mathbf{W}_{t-1}^{(\ell)}]
\end{aligned}
\end{equation}
\vspace{-5pt}

Besides, similar to the discussion we had for \texttt{SGCN+}, the norm of node embeddings and gradient are not guaranteed to be bounded as in vanilla \texttt{SGCN} (Proposition~\ref{proposition:matrix_norm_bound}) due to the aforementioned recursive update rule. 
To overcome the issue, we introduce an early stop criterion (line~\ref{line:sgcn++early_stop}) by checking the relative scale between the norm of the current node embedding and gradient to the snapshot one, and immediately start another snapshot step if the condition is violated. 

Next, in the following theorem, we establish the convergence rate of \texttt{SGCN++}. Recall that the mean-square error of the stochastic gradient can be decomposed into bias $\mathbb{E}[\|\mathbf{b}\|_\mathrm{F}^2]$ that is due to node embedding approximation and variance $\mathbb{E}[\|\mathbf{n}\|_\mathrm{F}^2]$ that is due to layerwise gradient estimation. 
Applying doubly variance reduction on node embedding and layerwise gradient simultaneously can significantly reduce mean-square error of stochastic gradient and speed up convergence.
\begin{theorem} [Convergence of \texttt{SGCN}++] \label{theorem:convergence_of_sgcn_plus_plus}
Suppose Assumptions~\ref{assumption:loss_lip_smooth}, \ref{assumption:sigma_lip_smooth}, \ref{assumption:bound_norm} hold, and denote  $L_\mathrm{F}$ as the smoothness constant
and $\Delta^{++}_{\mathbf{n}+\mathbf{b}}$ as the upper-bound of mean-square error of stochastic gradient 
\begin{equation}
    \Delta_{\mathbf{n}+\mathbf{b}}^{++} = \eta^2 \Delta_{\mathbf{n}+\mathbf{b}}^{++\prime} = \eta^2 \mathcal{O}\Big((\alpha^2 + 1)(\alpha^2 + \beta^2 + \alpha^2 \beta^2) K \sum_{\ell=1}^L |\mathbb{E}[\| \widetilde{\mathbf{L}}^{(\ell)} \|_\mathrm{F}^2] - \|\mathbf{L}\|_\mathrm{F}^2| \Big)
\end{equation}
Apply \texttt{SGCN++} in Algorithm~\ref{algorithm:sgcn++} with learning rate as $\eta=\frac{2}{L_\mathrm{F}+\sqrt{L_\mathrm{F}^2+4\Delta^{++\prime}_{\mathbf{n}+\mathbf{b}}}}$. Then it holds that
\begin{equation}
\frac{1}{T}\sum_{t=1}^T \mathbb{E}[ \|\nabla \mathcal{L}(\bm{\theta}_t)\|^2 ] \leq \frac{1}{T} \Big(L_\mathrm{F}+\sqrt{L_\mathrm{F}^2+4\Delta_{\mathbf{n}+\mathbf{b}}^{++\prime}} \Big) \Big( \mathcal{L}(\bm{\theta}_1) - \mathcal{L}(\bm{\theta}^\star) \Big).
\end{equation}
\end{theorem}
The proof of theorem is deferred to Appendix~\ref{section:proof_of_thm3} and the exact value of key parameter $L_\mathrm{F}$ and $\Delta^{++\prime}_{\mathbf{n}+\mathbf{b}}$ are computed in Lemma~\ref{lemma:smoothness_L_layer} and Lemma~\ref{lemma:upper-bound-sgcn-plus-plus-mse}, respectively, and can be found in Appendices~\ref{section:proof_of_thm1} and~\ref{section:proof_of_thm3}.


Theorem~\ref{theorem:convergence_of_sgcn_plus_plus} implies that applying doubly variance reduction can scale the mean-square error $\mathcal{O}(\eta^2 (\alpha^2 + 1)(\alpha^2 + \beta^2 + \alpha^2 \beta^2) K)$ times smaller. As a result, after $T$ iterations the norm of gradient of solution obtained by \texttt{SGCN++} is at most $\mathcal{O}(\Delta_{\mathbf{n}+\mathbf{b}}^{++\prime}/T)$, which enjoys the same rate as  vanilla variance reduced SGD~\cite{reddi2016stochastic,fang2018spider}, and a similar storage and computation overhead as \texttt{VRGCN}.
Please notice that our algorithm  only requires the historical node embeddings and stochastic gradient for variance reduction, therefore \texttt{SGCN++} is also applicable to other GCN variants~\cite{li2019deepgcns,velivckovic2017graph} and adaptive sampling algorithms~\cite{liu2020bandit}. Our proof follows the proof of SARAH~\cite{nguyen2019finite} by selecting the learning rate reverse proportional to the bias and variance, thus minimize the effect on convergence.
Notice that since the objective function here is significantly more complex comparing to classical variance reduced composite optimization, for expressiveness we use $\mathcal{O}(\cdot)$ for both bias and variance. As a result, we cannot exactly derive the optimal snapshot gap size $K$. However, our empirical results suggests that choosing $K=10$ and then selecting $\alpha=\beta=1.1$ such that algorithm at most restart once at early stage works well on most datasets.


\paragraph{Scalability of \texttt{SGCN++}.}
One might doubt \emph{whether the computation at the snapshot step with full-batch gradient will hinder the scalability of \texttt{SGCN++} for extremely large graphs}? 
Heuristically, we can approximate the full-batch gradient with the gradient calculated on a large-batch using all neighbors. 
The intuition stems from matrix Bernstein inequality~\cite{gross2011recovering}, where the probability of the approximation error violating the desired accuracy decreases exponentially as the number of samples increase.
Please refer to Algorithm~\ref{algorithm:sgcn_plus_plus_complete_no_full_batch} for the full-batch free \texttt{SGCN++} and explanation on why large-batch approximation is feasible using tools from matrix concentration. 
We remark that large-batch approximation can also be utilized in \texttt{SGCN+} to further reduce the memory requirement for historical node embeddings.


\paragraph{Effect of using dropout and data augmentations.}
Exploring the effect of data augmentation and dropout on variance reduction algorithms is interesting. 
Since data augmentation is usually applied to the feature matrices before training, it will not affect the result. 
However, applying variance reduction with Dropout will reduce the randomness introduced by Dropout, potentially leads to a more stable and faster convergence, but lessen the effect of dropout.
Applying batch normalization (not commonly used in GCN) on variance reduction algorithms is non-trivial because batch normalization has trainable parameters that are changing at each iteration.

\paragraph{Relation between two types of variance.}
We illustrate in Figure~\ref{fig:illustrate_variance} the relationship between node embedding approximation variance and layerwise gradient variance to the forward- and backward-propagation. We remark that zeroth-order variance reduction \texttt{SGCN+} is only applied during forward-propagation, and doubly (zeroth- and first-order) variance reduction \texttt{SGCN++} is applied during forward- and backward-propagation, simultaneously.

%% file: section/algorithm_SGCN++.tex
\begin{algorithm}[tb]
  \caption{\texttt{SGCN++}: Doubly variance reduction (Detailed version in Algorithm~\ref{algorithm:sgcn_plus_plus_complete}) }
  \label{algorithm:sgcn++}
\begin{algorithmic}[1]
  \STATE {\bfseries Input:} Learning rate $\eta>0$, snapshot gap $K\geq 1$, $t_0 = 1$ and $s=1$, staleness factor $\alpha, \beta \geq 1$
    \FOR{$t = 1,\ldots,T$}
        \IF{$(t-t_{s-1})~\text{mod}~K=0$} \label{line:sgcn++snapshot_start} 
            \STATE Calculate node embeddings using  Eq.~\ref{eq:sgcn_plus_snapshot} \label{line:sgcn++snapshot_start}
            \STATE Calculate full-batch gradient $\nabla \mathcal{L}(\bm{\theta}_t)$ using Eq.~\ref{eq:full_batch_gradient} and update as $\bm{\theta}_{t+1} = \bm{\theta}_t - \eta \nabla \mathcal{L}(\bm{\theta}_t)$ 
            \STATE Save the per layerwise gradient $\widetilde{\mathbf{G}}_t^{(\ell)} \leftarrow \mathbf{G}_t^{(\ell)},~     \widetilde{\mathbf{D}}_t^{(\ell)} \leftarrow \mathbf{D}_t^{(\ell)},~\forall \ell\in[L]$
            \STATE Set $t_s = t$ and $s = s+1$    \label{line:sgcn++snapshot_end}         
        \ELSE
            \STATE Calculate node embeddings using Eq.~\ref{eq:sgcn_plus_regular} \label{line:sgcn++regular_start}
            \IF{$\|\widetilde{\mathbf{H}}^{(\ell)}_{t-1}\|_\mathrm{F} \geq \alpha \| \mathbf{H}^{(\ell)}_{t_{s-1}} \|_\mathrm{F}$ or $\|\widetilde{\mathbf{D}}^{(\ell)}_{t-1}\|_\mathrm{F} \geq \beta \| \mathbf{D}^{(\ell)}_{t_{s-1}} \|_\mathrm{F}$ for any $\ell \in [L]$} \label{line:sgcn++early_stop}
                \STATE Go to line~\ref{line:sgcn++snapshot_start}
            \ENDIF
            \STATE Calculate stochastic gradient $\nabla \widetilde{\mathcal{L}}(\bm{\theta}_t)$ usng Eq.~\ref{eq:sgcn_plus_plus_regular} and update as $\bm{\theta}_{t+1} = \bm{\theta}_t - \eta \nabla \widetilde{\mathcal{L}}(\bm{\theta}_t)$ \label{line:sgcn++regular_end}
        \ENDIF
    \ENDFOR
\STATE {\bfseries Output:} Model with parameter $\bm{\theta}_{T+1}$ \\
\end{algorithmic}
\end{algorithm}

%% file: section/6-experiment.tex
\section{Experiments} \label{section:experiments}

\begin{table}[t]
\caption{Summary of dataset statistics. 
\textbf{m} stands for \textbf{m}ulti-class classification, and \textbf{s} stands for \textbf{s}ingle-class.
}
\label{table:datasets}
\centering
\begin{tabular}{rcccccc}
\toprule
\textbf{Dataset}   & \textbf{Nodes} & \textbf{Edges} & \textbf{Degree} & \textbf{Feature} & \textbf{Classes} & \textbf{Train/Val/Test} \\ \midrule
\textbf{PPI}       & $14,755$         & $225,270$        & $15$              & $50$               & $121$(\textbf{m})          & $66\%/12\%/22\%$          \\ 
\textbf{PPI-Large} & $56,944$         & $818,716$        & $14$              & $50$               & $121$(\textbf{m})          & $79\%/11\%/10\%$          \\ 
\textbf{Flickr}    & $89,250$         & $899,756$        & $10$              & $500$              & $7$ (\textbf{s})               & $50\%/25\%/25\%$          \\ 
\textbf{Reddit}    & $232,965$        & $11,606,919$     & $50$              & $602$              & $41$(\textbf{s})               & $66\%/10\%/24\%$          \\ 
\textbf{Yelp}      & $716,847$        & $6,977,410$      & $10$              & $300$              & $100$(\textbf{m})          & $75\%/10\%/15\% $         \\ 
\bottomrule
\end{tabular}
\end{table}

\paragraph{Experimental setup.}
We evaluate our proposed methods in semi-supervised learning setting on various classification datasets. We summarize the dataset statistics in Table~\ref{table:datasets}.
In addition to different sampling mechanisms, we introduce \texttt{Exact} sampling that takes all neighbors for node embedding computation during mini-batch training (no zeroth-order variance), which can be used to explicitly test the importance of first-order variance reduction.
We add \texttt{SGCN+(Zeroth)} and \texttt{SGCN++(Doubly)} on top of each sampling method to illustrate how zeroth-order and doubly variance reduction affect GCN training.
By default, we train $2$-layer GCNs with hidden dimension of $256$, snapshot gap $K=10$, staleness factors $\alpha=\beta=1.1$. 
We use all nodes for snapshot computation on Flickr, PPI, PPI-large datasets, and employ snapshot large-batch approximation for both \texttt{SGCN+} and \texttt{SGCN++} by randomly selecting $50\%$ of nodes for Reddit and $15\%$ of nodes for Yelp dataset.
We update the model with a mini-batch size of $B=512$ and Adam optimizer with a learning rate of $\eta=0.01$.
We conduct training $3$ times for $200$ epochs and report the average results.
We choose the model with the lowest validation error as the convergence point. 
Code to reproduce the experiment results can be found at \href{https://github.com/CongWeilin/SGCN}{here}.

\input{supplementary/experiment_configs}

\paragraph{Implementation details.}
To demonstrate the effectiveness of doubly variance reduction, we modified the PyTorch implementation of {GCN}\footnote{\url{https://github.com/tkipf/pygcn}}~\cite{semigcn} to add \texttt{LADIES}~\cite{ladies}, \texttt{FastGCN}~\cite{fastgcn}, \texttt{GraphSAGE}~\cite{graphsage}, \texttt{GraphSAINT}~\cite{graphsaint}, \texttt{VRGCN}~\cite{control_variacne_gcn}, and \texttt{Exact} sampling mechanism. 
Then, we implement \texttt{SGCN+} and \texttt{SGCN++} on the top of each sampling method to illustrate how zeroth-order variance reduction and doubly variance reduction help for GCN training. 
By default, we train $2$-layers GCNs with hidden state dimension of $256$, element-wise ELU as the activation function and symmetric normalized Laplacian matrix $\mathbf{L} = \mathbf{D}^{-1/2} \mathbf{A} \mathbf{D}^{-1/2}$. 
We use \emph{mean-aggregation} for single-class classification task and \emph{concatenate-aggregation} for multi-class classification.
The default mini-batch batch size and sampled node size are summarized in Table~\ref{table:train_config}.

During training, we update the model using Adam optimizer with a learning rate of $0.01$. 
For \texttt{SGCN++}, historical node embeddings are first calculated on GPUs and transfer to CPU memory using PyTorch command \texttt{Tensor.to(device)}. 
Therefore, no extra GPU memory is required when training with \texttt{SGCN++}. 
To balance the staleness of snapshot model and the computational efficiency, as default we choose snapshot gap $K=10$ and early stop inner-loop if the Euclidean distance between current step gradient to snapshot gradient is larger than $0.002$ times the norm of snapshot gradient. For each epoch we construct $10$ mini-batches in parallel using Python package \texttt{multiprocessing} and perform training on the sampled $10$ mini-batches. 
To achieves a fair comparison of different sampling strategies in terms of sampling complexity, we implement all sampling algorithms using \texttt{numpy.random} and \texttt{scipy.sparse} package. 

We have to emphasize that, in order to better observe the impact of sampling on convergence, we have not use any augmentation methods (e.g., ``layer normalization'', ``skip-connection'', and ``attention''), which have been proven to impact the GCN performance in~\cite{cai2020graphnorm,dwivedi2020benchmarkgnns}.
Notice that we are not criticizing the usage of these augmentations. 
Instead, we use the most primitive network structure to better explore the impact of sampling and variance reduction on convergence.

\paragraph{Hardware specification and environment.}
We run our experiments on a single machine with Intel i$9$-$10850$K, Nvidia RTX $3090$ GPU, and $32$GB RAM memory. 
The code is written in Python $3.7$ and we use PyTorch $1.4$ on CUDA $10.1$ to train the model on GPU.

\input{section/experiment_table}

\begin{figure}[h]
    \centering
    \includegraphics[width=0.95\textwidth]{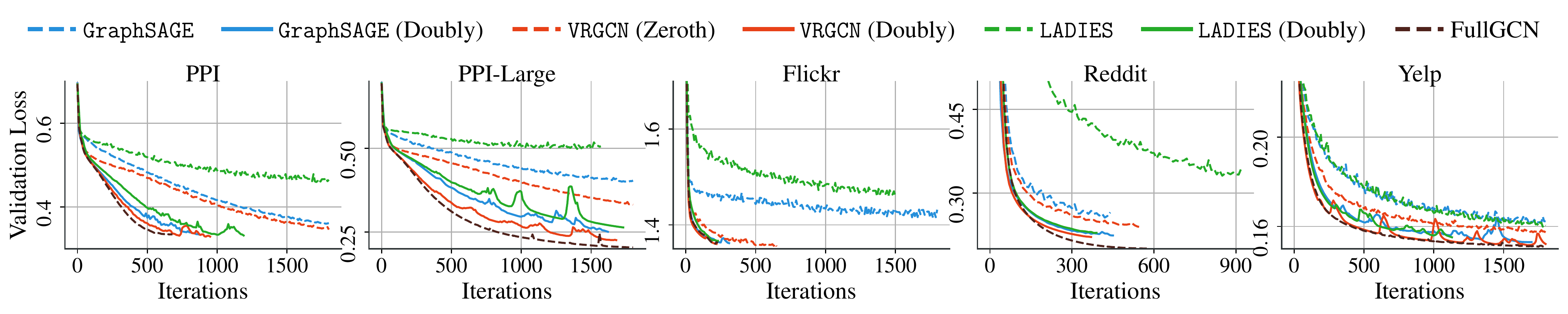}
    \caption{Comparing the validation loss of \texttt{SGCN} and \texttt{SGCN++} on real world datasets.}
    \label{fig:convergence_result}
\end{figure}

\begin{figure}[h]
    \centering
    \includegraphics[width=0.9\textwidth]{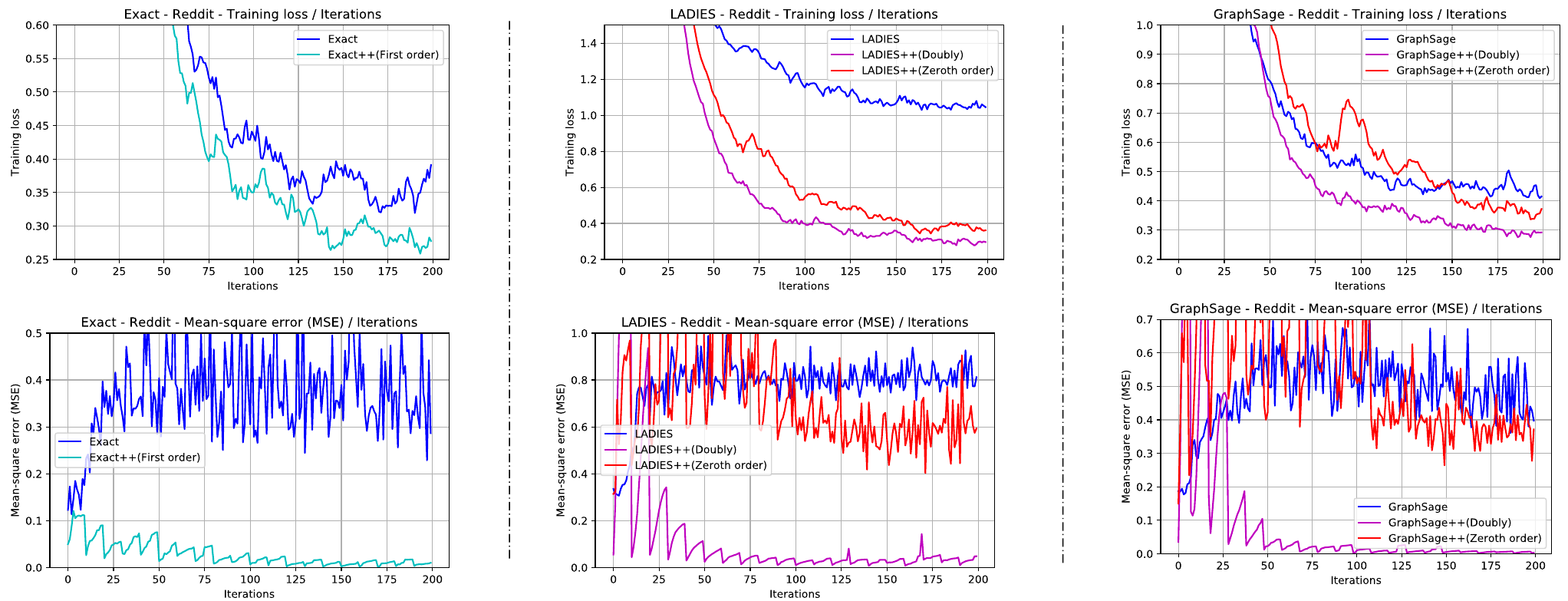}
    \caption{Comparing the mean-square error of stochastic gradient to full gradient and training loss of \texttt{SGCN}, \texttt{SGCN+}, \texttt{SGCN++} in the first $200$ iterations of training process on \texttt{Reddit} dataset.}
    \label{fig:mse_compare}
\end{figure}

\paragraph{Overall results.}
In Table~\ref{table:f1_results} and Figure~\ref{fig:convergence_result}, we show the accuracy and convergence comparison of \texttt{SGCN}, \texttt{SGCN+}, and \texttt{SGCN++}. In Figure~\ref{fig:mse_compare}, we evaluate the effect of variance reduction on the mean-square error of stochastic gradient and its convergence.
We remark that multi-class classification tasks prefer a more stable node embedding and gradient than single-class classification tasks. 
Therefore, even the vanilla \texttt{Exact}, \texttt{GraphSAGE} and \texttt{VRGCN} already outperforms other baseline methods on PPI, PPI-large, and Yelp. 
Applying variance reductions can significantly reduce the MSE of stochastic gradient full gradient, thus further improve its performance.
In addition, we observe that the effect of variance reduction depends on its base sampling algorithms. 
Even though the performance of base sampling algorithm various significantly, the doubly variance reduction can bring their performance to a similar level.
Moreover, we can observe from the loss curves that \texttt{SGCN}s suffers an residual error as discussed in Theorem~\ref{theorem:convergence_of_sgcn}, and the residual error is proportional to node embedding approximation variance (zeroth-order variance), where \texttt{VRGCN} has less variance than \texttt{GraphSAGE} because of its zeroth-order variance reduction, and \texttt{GraphSAGE} has less variance than \texttt{LADIES} because more nodes are sampled for node embedding approximation. 

\paragraph{GPU memory usage.}
In Figure~\ref{fig:gpu_utils}, We compare the GPU memory usage of \texttt{SGCN} and \texttt{SGCN++}.
We calculate the allocated memory by \texttt{torch.cuda.memory\_allocated}, which is the current GPU memory occupied by tensors in bytes for a given device. 
We calculate the maximum allocated memory by \texttt{torch.cuda.max\_memory\_allocated}, which is the maximum GPU memory occupied by tensors in bytes for a given device.
From Figure~\ref{fig:gpu_utils}, we observe that neither running full-batch GCN nor saving historical node embeddings and gradients will significantly increase the computation overhead during training.
Besides, since all historical activations are stored outside GPU, we see that \texttt{SGCN++} only requires several megabytes to transfer data between GPU memory to the host, which can be ignored compared to the memory usage of calculation itself. 

\begin{figure}[h]
    \centering
    \includegraphics[width=0.7\textwidth]{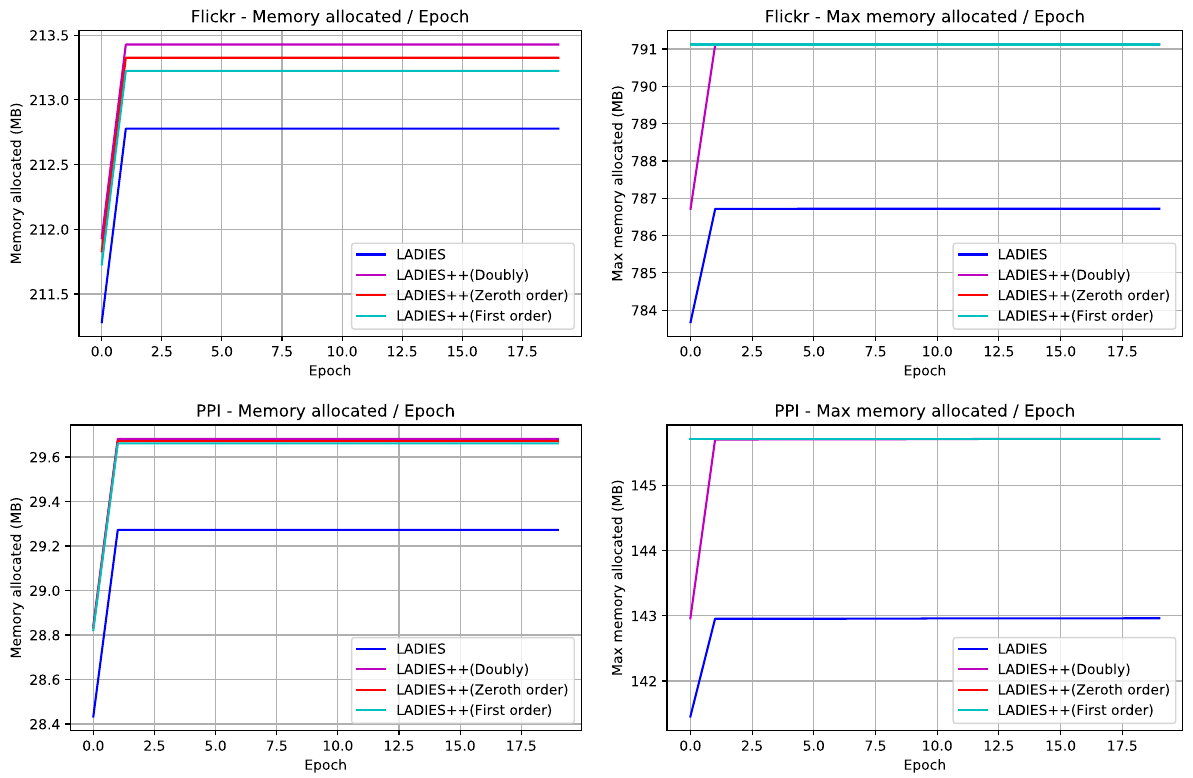}
    \caption{Comparison of GPU memory usage of \texttt{SGCN} and \texttt{SGCN++} on Flickr and PPI dataset.}
    \label{fig:gpu_utils}
\end{figure}

\begin{figure}
    \centering
    \includegraphics[width=1.0\textwidth]{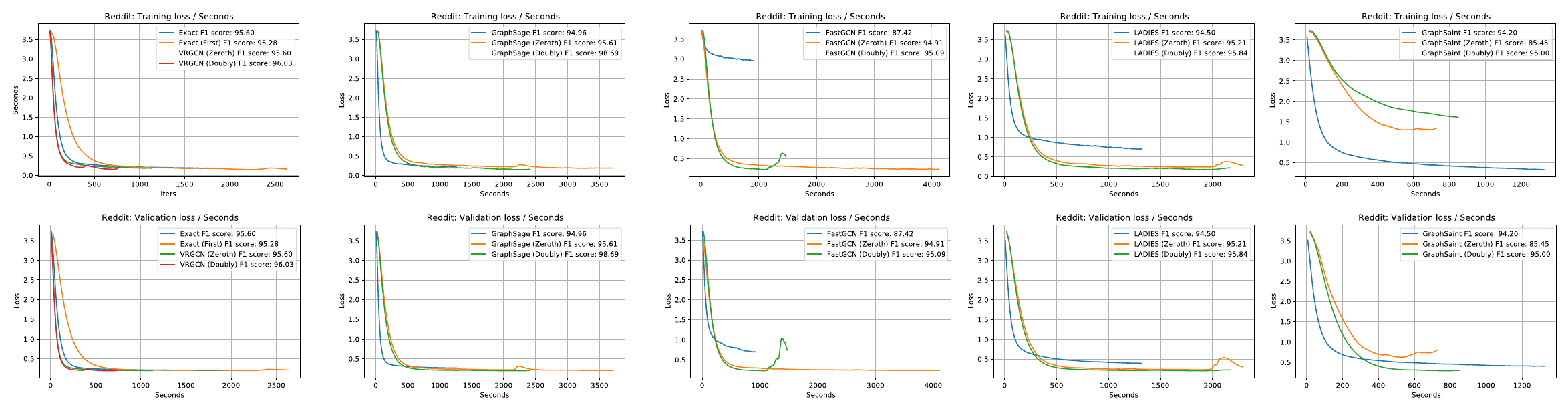}
    \caption{Comparison of training and validation loss of \texttt{SGCN}, \texttt{SGCN+}, \texttt{SGCN++} using wall clock time on \texttt{Reddit} dataset. }
    \label{fig:reddit-wall-clock}
\end{figure}

\begin{figure}
    \centering
    \includegraphics[width=1.0\textwidth]{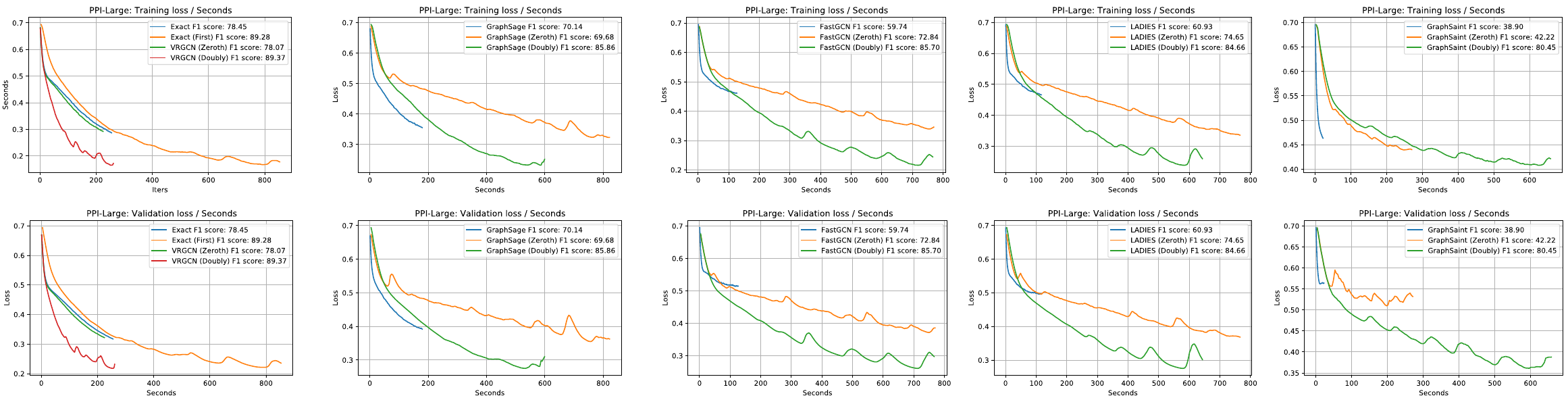}
    \caption{Comparison of training and validation loss of \texttt{SGCN}, \texttt{SGCN+}, \texttt{SGCN++} using wall clock time on \texttt{PPI-Large} dataset. }
    \label{fig:ppi-large-wall-clock}
\end{figure}
\paragraph{Evaluation of total time.}
In Table~\ref{table:doubly_ladies_time} and Table~\ref{table:ladies_time}, we report the average time of doubly variance reduced \texttt{LADIES++} and vanilla \texttt{LADIES}. We classify the wall clock time during the training process into five categories:
\begin{itemize}
    \item \textit{Snapshot step sampling time}: The time used to construct the snapshot full-batch or the snapshot large-batch. In practice, we directly use full-batch training for the smaller datasets (e.g., PPI, PPI-large, and Flickr) and use sampled snapshot large-batch for large datasets (e.g., Reddit and Yelp). When constructing snapshot large-batch, the \texttt{Exact} sampler has to go through all neighbors of each node using for-loops based on the graph structure, such that it is time-consuming.
    \item \textit{Snapshot step transfer time}: The time required to transfer the sampled snapshot batch nodes and Laplacian matrices to the GPUs.
    \item \textit{Regular step sampling time}: The time used to construct the mini-batches using layerwise \texttt{LADIES} sampler. 
    \item \textit{Regular step transfer time}: The time required to transfer the sampled mini-batch nodes and Laplacian matrices to GPUs, and the time to transfer the historical node embeddings and the stochastic gradient between GPUs and CPUs.
    \item \textit{Computation time}: The time used for forward- and backward-propagation.
\end{itemize}

Notice that we are reporting the total time per iteration because the vanilla sampling-based method cannot reach the same accuracy as the doubly variance reduced algorithm (due to the residual error as shown in Theorem~\ref{theorem:convergence_of_sgcn}).

From Table~\ref{table:doubly_ladies_time} and Table~\ref{table:ladies_time}, we can observe that the most time-consuming process in sampling-based GCN training is data sampling and data transfer.
The extra computation time introduces by employing the snapshot step is negligible when comparing to the mini-batch sampling time during each regular step.
Therefore, a promising future direction for large-scale graph training is developing a provable sampling algorithm with low sampling complexity.
Besides, in Figure~\ref{fig:reddit-wall-clock} and Figure~\ref{fig:ppi-large-wall-clock}, we compare the wall clock time of all methods on \texttt{PPI-large} and \texttt{Reddit dataset} with our default setup. From Figure~\ref{fig:reddit-wall-clock} and Figure~\ref{fig:ppi-large-wall-clock}, we can observe that although \texttt{SGCN+} and \texttt{SGCN++} sometime require more wall-clock time, they can get a better training and validation result than vanilla \texttt{SGCN}.

\input{supplementary/table_time_summary}

\paragraph{Evaluation of snapshot gap for \texttt{SGCN+} and \texttt{SGCN++}.}
Doubly variance reduced \texttt{SGCN++} requires performing full-batch (large-batch) training periodically to calculate the snapshot node embeddings and gradients. 
A larger snapshot gap $K$ can make training faster, but also might make the snapshot node embeddings and gradients too stale for variance reduction. 
In this experiment, we evaluate the effect of snapshot grap on training by choosing mini-batch size as $B=512$, disable the early stop criterion, and change the inner-loop intervals from $K=5$ mini-batches to $K=20$ mini-batches.
In Figure~\ref{fig:change_of_K} and Figure~\ref{fig:change_k_zeroth_result}, we show the comparison of training loss and validation loss with different number of inner-loop intervals for \texttt{SGCN++} and \texttt{SGCN+} on \texttt{Reddit} dataset, respectively.
We can observe that the model with a smaller snapshot gap requires less number of iterations to reach the same training and validation loss, and gives us a better generalization performance (F1-score).

\begin{figure}[h]
    \centering
    \includegraphics[width=1\textwidth]{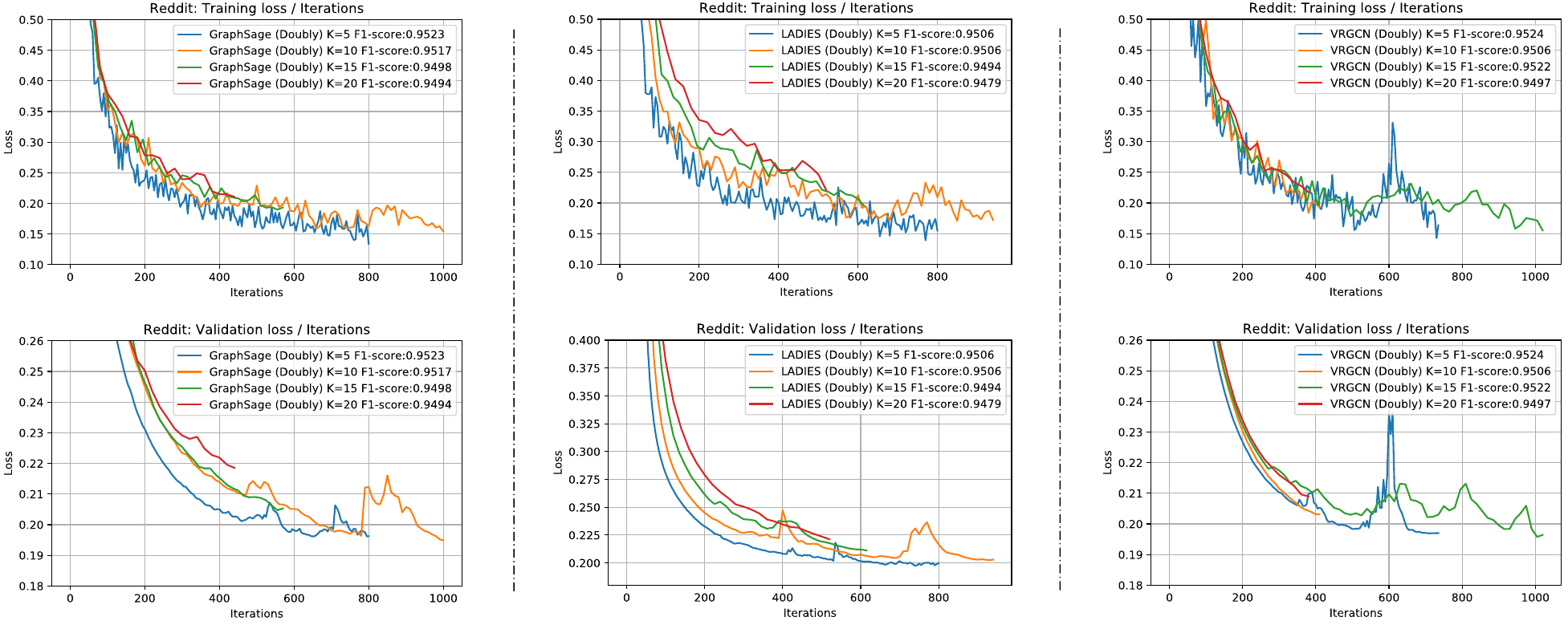}
    \caption{Comparison of training loss, validation loss, and F1-score of \texttt{SGCN++} with different snapshot gap on \texttt{Reddit} dataset. }
    \label{fig:change_of_K}
\end{figure}

\begin{figure}
    \centering
    \includegraphics[width=1\textwidth]{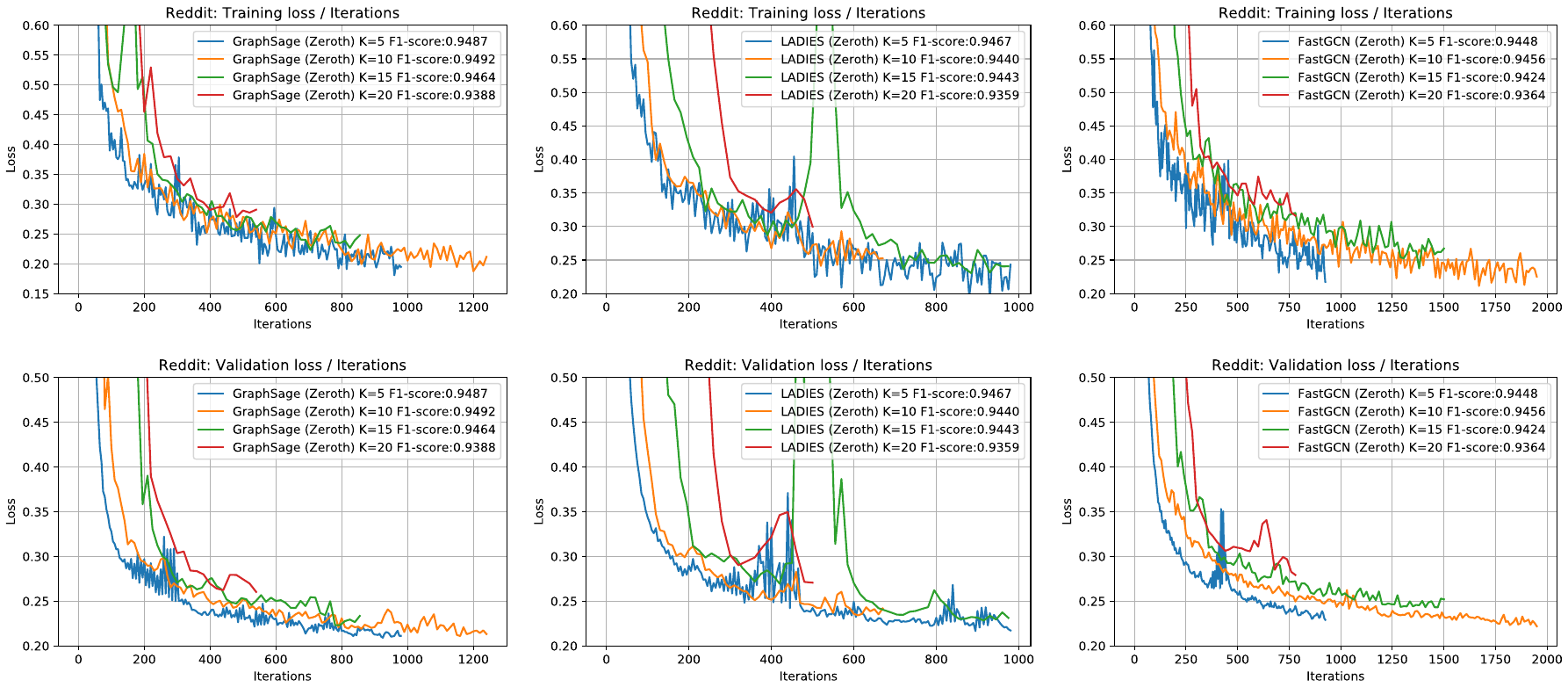}
    \caption{Comparison of training loss, validation loss, and F1-score of \texttt{SGCN+} with different snapshot gap on \texttt{Reddit} dataset.}
    \label{fig:change_k_zeroth_result}
\end{figure}

\paragraph{Evaluation of large-batch size for \texttt{SGCN+} and \texttt{SGCN++}.}
The full-batch gradient calculation at each snapshot step is computationally expensive. 
Heuristically, we can approximate the full-batch gradient by using the gradient computed on a large-batch of nodes. 
Besides, it is worth noting that large-batch approximation can be also used for the node embedding approximation in zeroth-order variance reduction. 
In \texttt{SGCN+}, saving the historical node embeddings for all nodes in an extreme large graph can be computationally prohibitive. 
An alternative strategy is sampling a large-batch during the snapshot step, computing the node embeddings for all nodes in the large-batch, and saving the freshly computed node embeddings on the storage. After that, mini-batch nodes are sampled from the large-batch during the regular steps.
Let denote $B^\prime$ as the snapshot step large-batch size and $B$ denote the regular step mini-batch size. 
By default, we choose snapshot gap as $K=10$, disable the early stop criterion, fix the regular step batch size as $B=512$, and change the snapshot step batch size $B^\prime$ from $20,000$ (20K) to $80,000$ (80K).
In Figure~\ref{fig:chang_B_prime} and Figure~\ref{fig:change_K_zeroth}, we show the comparison of training loss and validation loss with different snapshot step large-batch size $B^\prime$ for \texttt{SGCN++} and \texttt{SGCN+}, respectively.

\begin{figure}[h]
    \centering
    \includegraphics[width=1.0\textwidth]{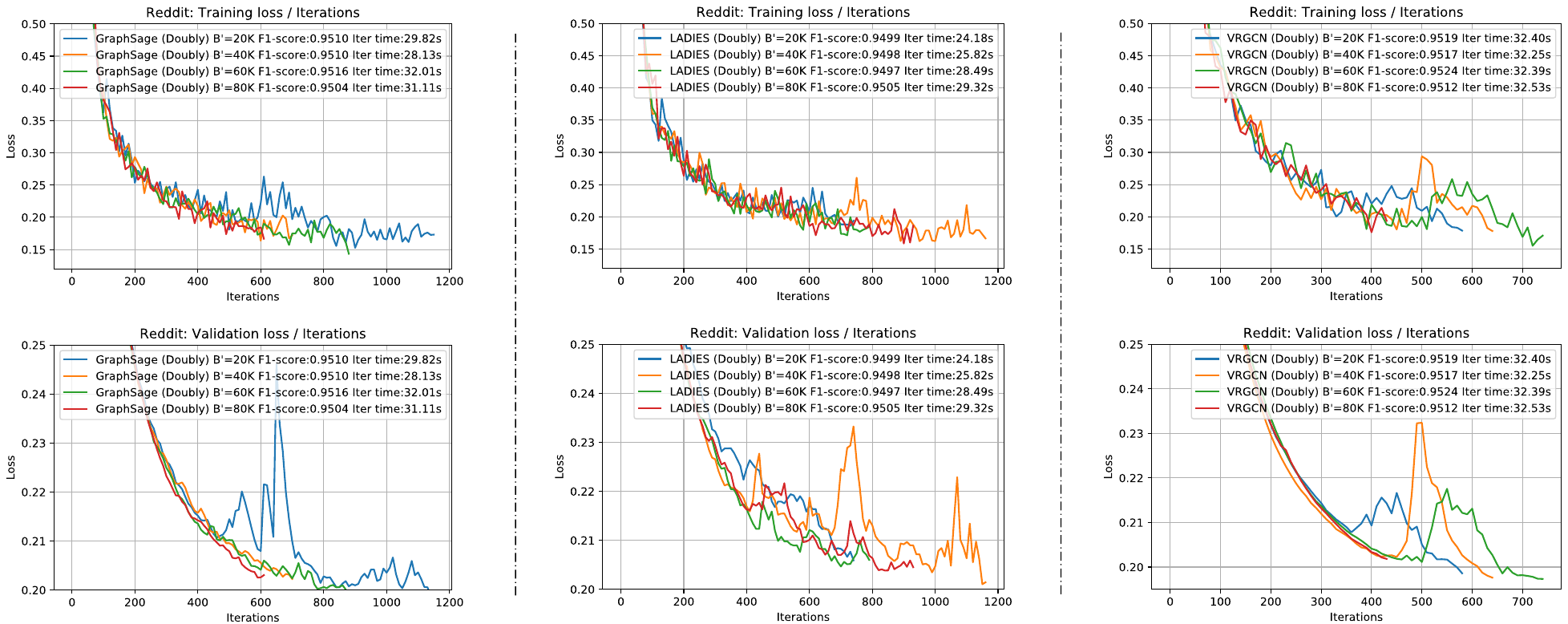}
    \caption{{Comparison of training loss, validation loss, and F1-score of \texttt{SGCN++} with different snapshot large-batch size on \texttt{Reddit} dataset.}}
    \label{fig:chang_B_prime}
\end{figure}

\begin{figure}
    \centering
    \includegraphics[width=1.0\textwidth]{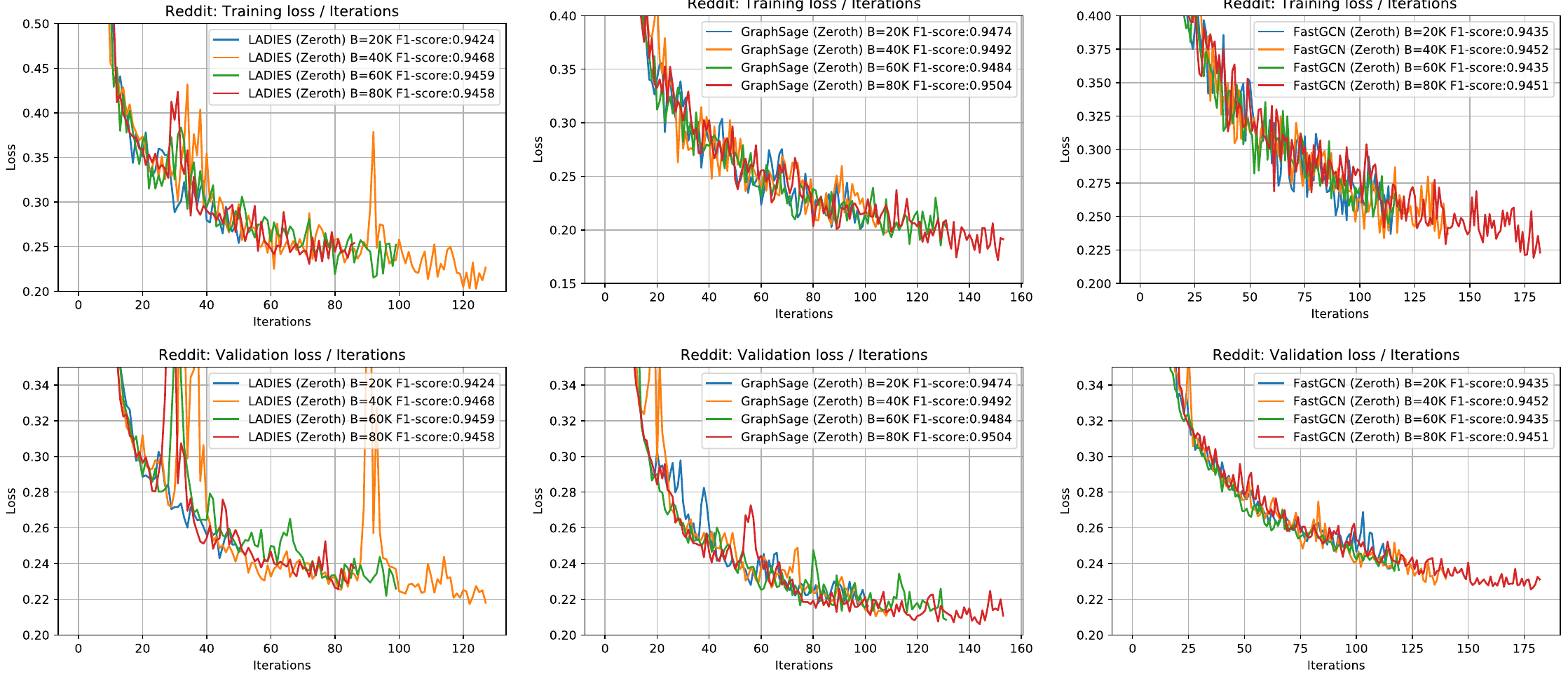}
    \caption{{Comparison of training loss, validation loss, and F1-score of \texttt{SGCN+} with different snapshot large-batch size on \texttt{Reddit} dataset.}}
    \label{fig:change_K_zeroth}
\end{figure}




\paragraph{Evaluation of increasing snapshot gap.}
Snapshot gap size $K$ serves as a budge hyper-parameter that balances between training speed and the quality of variance reduction. 
During training, as the number of iterations increases, the GCN models convergences to a saddle point. Therefore, it is interesting to explore whether increasing the snapshot gap $K$ during the training process can obtain a speed boost. In Figure~\ref{fig:change_K_valid_loss}, we show the comparison of validation loss of fixed snapshot gap $K=10$ and gradually increasing snapshot gap $K=10+0.1\times s, s=1,2,\ldots$, where $s$ is the number of snapshot steps has been computed. 
We disable the early stop criterion so that we can run the desired number of inner-loop steps without being affected by early stopping.
Recall that the key bottleneck for \texttt{SGCN++} is memory budget and sampling complexity, rather than snapshot computing. Dynamically increasing snapshot gap can reduce the number of snapshot steps, but cannot significantly reduce the training time but might lead to a performance drop.

\begin{figure}
    \centering
    \includegraphics[width=1.0\textwidth]{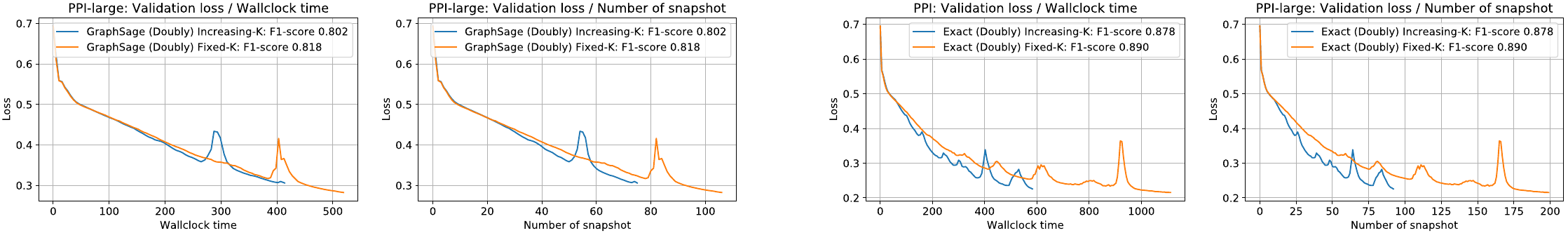}
    \caption{{Effectiveness of gradually increasing snapshot gap $K$ during training on wallclock time (second) and accuracy on \texttt{PPI} dataset. We choose snapshot gap $K=10$ for fixed-$K$. For increasing $K$, we choose snapshot gap $K=10+0.1\times s, s=1,2,\ldots$, where $s$ is the number of snapshot steps.}}
    \label{fig:change_K_valid_loss}
\end{figure}

\begin{figure}[h]
    \centering
    \includegraphics[width=1.0\textwidth]{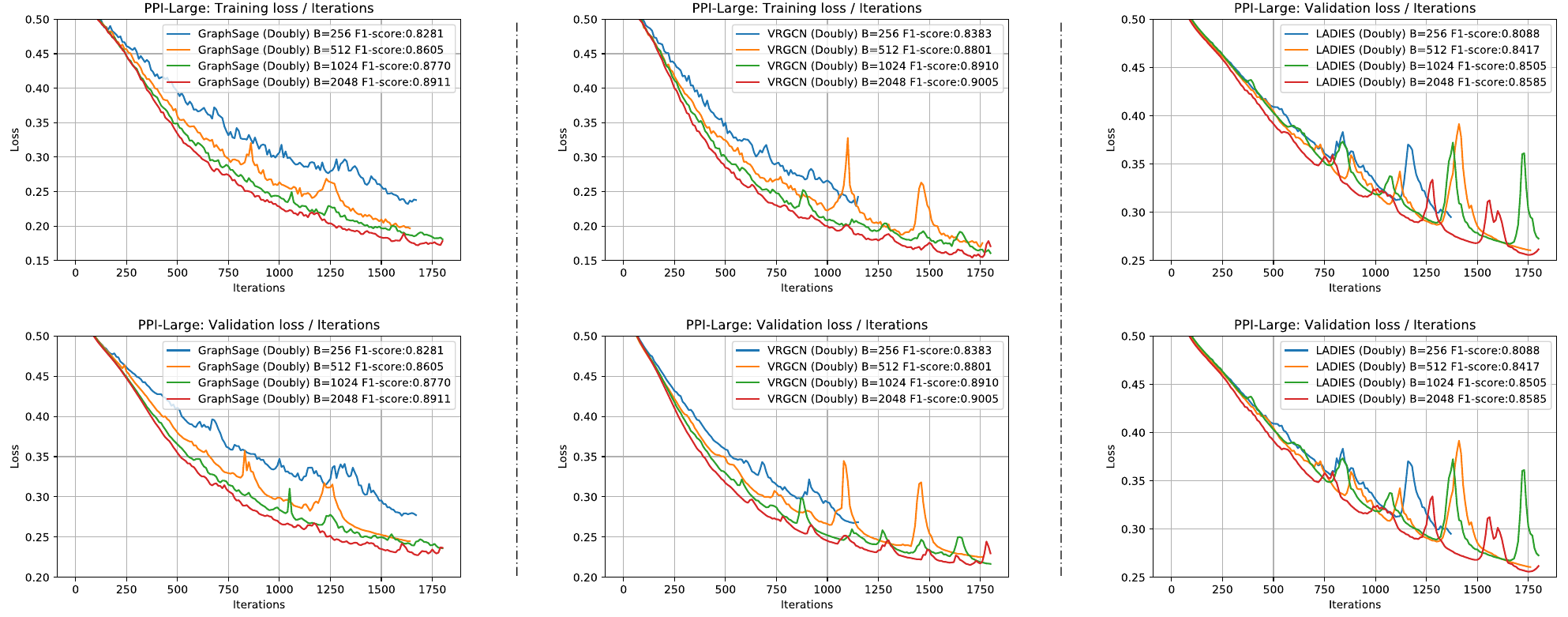}
    \caption{Comparison of training loss, validation loss, and F1-score of \texttt{SGCN++} with different mini-batch size on \texttt{PPI-Large} dataset.}
    \vspace{10pt}
    \label{fig:change_mini_batch}
\end{figure}

\paragraph{The effect of mini-batch size.}
In Figure~\ref{fig:change_mini_batch}, we show the comparsion of training loss and validation loss with different regular step mini-batch size. 
By default, we choose the snapshot gap as $K=10$, fix the snapshot step batch size as $B^\prime = 80,000$, disable the early stop criterion, and change the regular step mini-batch size $B$ from $256$ to $2,048$.
Besides, we note that subgraph sampling algorithm \texttt{GraphSAINT} requires an extreme large mini-batch size every iterations. In Figure~\ref{fig:graphsaint_batch_size}, we explicitly compare the effectiveness of mini-batch size on doubly variance reduced \texttt{GraphSAINT++} and vanilla \texttt{GraphSAINT}, and show that a smaller mini-batch is required by \texttt{GraphSAINT++}.

\begin{figure}[h]
    \centering \vspace{-0.5cm}
    \includegraphics[width=0.5\textwidth]{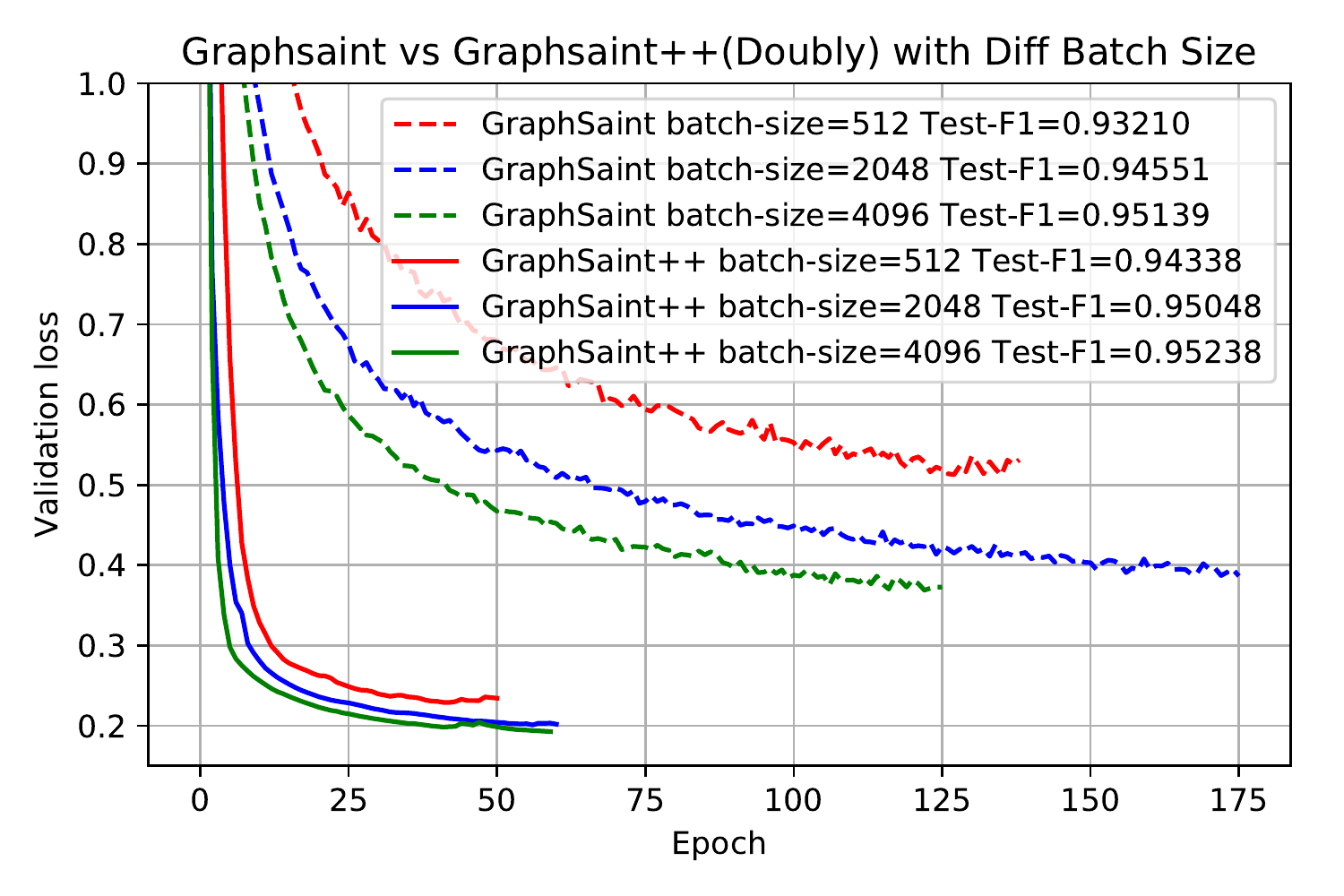}
    \caption{Comparing the validation loss and F1-score of \texttt{GraphSAINT} and \texttt{GraphSAINT++} with different mini-batch size on \texttt{Reddit} dataset}
    \label{fig:graphsaint_batch_size}
\end{figure}


\paragraph{Comparison of SGD and Adam.}
It is worth noting that Adam optimizer is used as the default optimizer during training. We choose Adam optimizer over SGD optimizer for the following reasons: 
\begin{itemize}
    \item [(a)] Baseline methods training with SGD cannot converge when using a constant learning rate due to the bias and variance in stochastic gradient (Adam has some implicit variance reduction effect, which can alleviate the issue). 
    The empirical result of SGD trained baseline models has a huge performance gap to the one trained with Adam, which makes the comparison meaningless. 
    For example in Figure~\ref{fig:sgd_vs_adam_ppi}, we compare Adam and SGD optimizer on PPI dataset. For Adam optimizer we use PyTorch's default learning rate $0.01$, and for SGD optimizer we choose learning rate as $0.1$, which is selected as the most stable learning rate from range $[0.01, 1]$ for this dataset. Although the SGD is using a learning rate $10$ times larger than Adam, it requires $100$ times more iterations than Adam to reach the early stop point (valid loss do not decrease for $200$ iterations), and suffers a giant performance gap when comparing to Adam optimizer. 
    \item [(b)] Most public implementation of GCNs, including all implementations in PyTorch Geometric and DGL packages, use Adam optimizer instead of SGD optimizer.
    \item [(c)] In this paper, we mainly focus on how to estimate a stabilized stochastic gradient, instead of how to take the existing gradient for weight update. 
    We employ Adam optimizer for all algorithms during experiment, which lead to a fair comparison.
\end{itemize}

\begin{figure}
    \centering
    \includegraphics[width=0.8\textwidth]{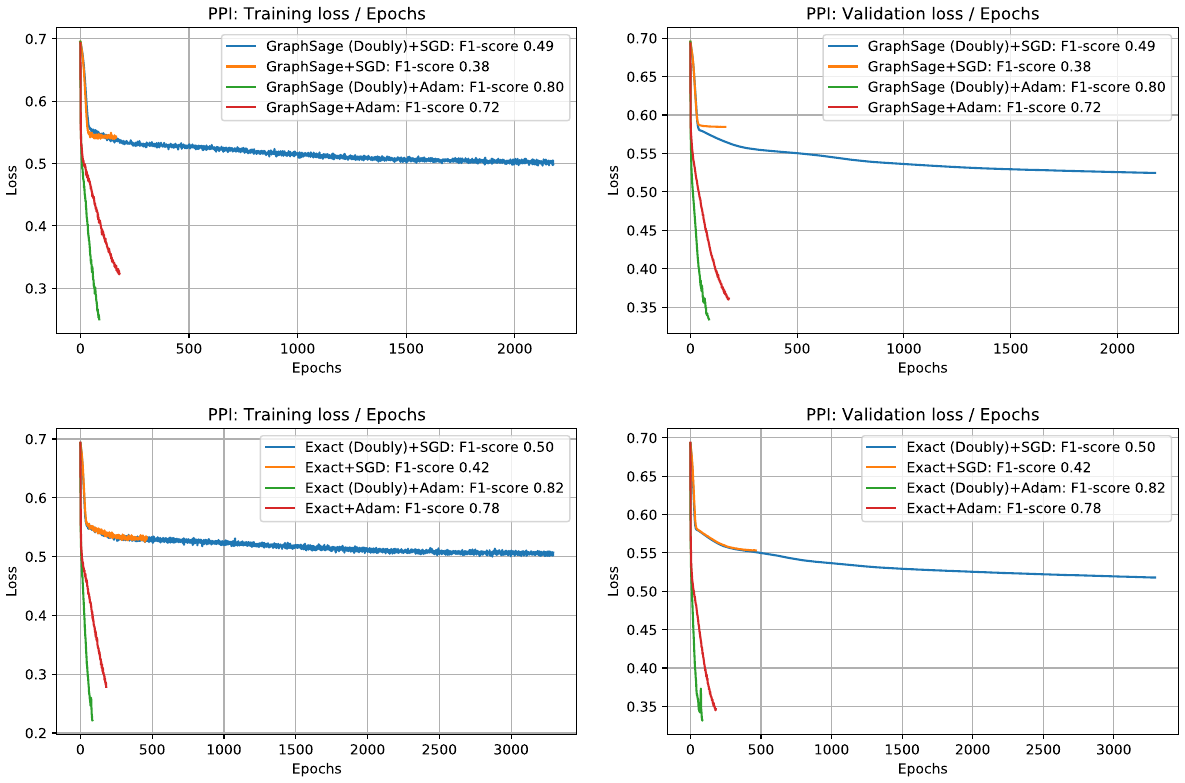}
    \caption{Comparison of doubly variance reduction and vanilla sampling-based GCN training on PPI dataset with SGD (learning rate $0.1$) and Adam optimizer (learning rate $0.01$). All other configurations are as default.}
    \label{fig:sgd_vs_adam_ppi} \vspace{-0.5cm}
\end{figure}

%% file: supplementary/experiment_configs.tex
\begin{table}[h]
\caption{Configuration of different sampling algorithms during training}
\label{table:train_config}
\centering
\scalebox{0.99}{
\begin{tabular}{lcccccc}
\toprule
                                 & \texttt{GraphSAGE} & \texttt{VRGCN} & \texttt{Exact} & \texttt{FastGCN} & \texttt{LADIES} & \texttt{GraphSAINT} \\ \midrule
\textbf{Mini-batch size}         & 512                & 512            & 512            & 512             & 512             & 2048                \\
\textbf{Sampled neighbors}       & 5                  & 2              & All            & -                & -               & -                   \\ 
\textbf{Samples in layerwise}    & -                  & -              & -              & 4096             & 512             & 2048                \\ \bottomrule
\end{tabular}}
\end{table}

%% file: section/experiment_table.tex
\renewcommand{\arraystretch}{1.0}
\begin{table*}[t]
\centering
\caption{Comparison of the accuracy (F1-score) of \texttt{SGCN}, \texttt{SGCN+}, and \texttt{SGCN++}. }
\label{table:f1_results}
\begin{small}\begin{sc}
\resizebox{.99\linewidth}{!}{
\begin{tabular}{llccccc}
\toprule
\multicolumn{2}{c}{\textbf{Method}}                    & \textbf{PPI}   & \textbf{PPI-Large} & \textbf{Flickr} & \textbf{Reddit}  & \textbf{Yelp}             \\\midrule
\multirow{2}{*}{\textbf{Exact}}      & SGCN            & $78.14\pm0.14$            & $77.90\pm0.67$            & $52.30\pm0.49$          & $95.07\pm0.01$          & $59.99\pm0.72$              \\ 
                                     & SGCN++ (Doubly)  & $\mathbf{81.66}\pm0.45$   & $\mathbf{89.07}\pm0.18$   & $\mathbf{52.88}\pm0.97$ & $\mathbf{95.17}\pm0.25$ & $\mathbf{62.09}\pm0.25$     \\\midrule
\multirow{2}{*}{\textbf{VRGCN}}      & SGCN+ (Zeroth)  & $77.65\pm0.23$            & $77.81\pm0.57$            & $52.57\pm0.77$          & $95.17\pm0.38$          & $61.29\pm0.11$              \\
                                     & SGCN++ (Doubly) & $\mathbf{82.50}\pm0.42$   & $\mathbf{88.65}\pm 0.73$   & $52.53\pm0.47$          & $\mathbf{95.17}\pm0.31$          & $\mathbf{62.64}\pm0.80$     \\\midrule
\multirow{3}{*}{\textbf{GraphSAGE}}  & SGCN            & $72.19\pm0.29$            & $69.51\pm0.25$            & $51.13\pm0.54$          & $94.73\pm0.12$          & $59.58\pm0.60$              \\
                                     & SGCN+ (Zeroth)  & $70.82\pm0.35$            & $69.67\pm0.73$            & $51.13\pm0.80$          & $94.89\pm0.02$          & $58.62\pm0.98$              \\  
                                     & SGCN++ (Doubly) & $\mathbf{80.30}\pm0.81$   & $\mathbf{85.41}\pm0.39$   & $\mathbf{52.65}\pm0.97$ & $\mathbf{95.18}\pm0.58$ & $\mathbf{61.75}\pm0.62$     \\\midrule
\multirow{3}{*}{\textbf{FastGCN}}    & SGCN            & $64.21\pm0.35$            & $59.60\pm0.32$            & $50.74\pm0.93$          & $87.36\pm0.18$          & $55.75\pm0.81$              \\
                                     & SGCN+ (Zeroth)  & $71.78\pm0.66$            & $72.30\pm0.89$            & $51.07\pm0.81$          & $94.54\pm0.36$          & $56.86\pm0.35$              \\  
                                     & SGCN++ (Doubly) & $\mathbf{80.25}\pm0.73$   & $\mathbf{85.04}\pm0.99$   & $\mathbf{52.57}\pm0.84$ & $\mathbf{94.99}\pm0.35$ & $\mathbf{60.63}\pm0.50$     \\\midrule
\multirow{3}{*}{\textbf{LADIES}}     & SGCN            & $61.80\pm0.87$            & $60.74\pm0.99$            & $50.29\pm0.86$          & $94.11\pm0.57$          & $59.84\pm0.29$              \\ 
                                     & SGCN+ (Zeroth)  & $71.00\pm0.45$            & $74.03\pm0.56$            & $51.83\pm0.34$          & $94.39\pm0.43$          & $57.06\pm0.26$              \\ 
                                     & SGCN++ (Doubly) & $\mathbf{82.10}\pm0.15$   & $\mathbf{84.18}\pm0.08$   & $\mathbf{52.09}\pm0.14$ & $\mathbf{95.05}\pm0.29$ & $\textbf{60.96}\pm0.66$     \\\midrule
\multirow{3}{*}{\textbf{GraphSAINT}} & SGCN            & $61.15\pm0.16$            & $38.68\pm0.99$            & $50.10\pm0.44$          & $93.68\pm0.23$          & $54.65\pm0.34$              \\
                                     & SGCN+ (Zeroth)  & $74.90\pm0.52$            & $41.40\pm0.17$            & $50.66\pm0.63$          & $84.61\pm0.95$          & $55.42\pm0.46$              \\ 
                                     & SGCN++ (Doubly) & $\mathbf{79.45}\pm0.71$   & $\mathbf{79.71}\pm0.05$   & $\mathbf{50.94}\pm0.94$ & $\mathbf{94.18}\pm0.57$ & $\mathbf{57.03}\pm0.39$     \\ \midrule
\multirow{1}{*}{\textbf{FullGCN}}    & N/A              & $82.14\pm0.12$            & $90.62\pm0.13$            & $52.99\pm0.26$          & $95.15\pm0.04$          & $62.77\pm0.10$              \\ 
                                     \bottomrule
\end{tabular}}
\end{sc}\end{small}
\end{table*}

%% file: supplementary/table_time_summary.tex
\begin{table}[h]
\caption{Comparison of average time (1 snapshot step and 10 regular steps) of doubly variance reduced \texttt{LADIES++} with regular step batch size as $512$. Full-batch is used for snapshot step on PPI, PPI-Large, and Flickr. $50\%$ training set nodes are sampled for the snapshot step on Reddit, and $15\%$ training set nodes are sampled for the snapshot step on Yelp.}
\label{table:doubly_ladies_time}
\centering
\scalebox{0.99}{
\begin{tabular}{lccccc}
\toprule
\textbf{Time (second)} & \textbf{PPI} & \textbf{PPI-Large} & \textbf{Flickr} & \textbf{Reddit} & \textbf{Yelp} \\ \midrule
Snapshot step sampling      & $0.182$      & $0.355$            & $0.221$          & $18.446$         &  $21.909$      \\ 
Snapshot step transfer      & $0.035$      & $0.070$            & $0.036$          & $0.427$         & $0.176$        \\ 
Regular step sampling       & $1.128$      & $1.322$            & $0.899$          & $9.499$         & $9.102$       \\ 
Regular step transfer       & $0.393$      & $0.459$            & $0.250$          & $0.550$         & $0.372$       \\ 
Computation            & $0.215$      & $0.196$             & $0.136$          &  $0.399$        & $0.139$       \\ \midrule
Total time             & $1.954$      & $2.377$            & $1.442$          &   $29.321$       & $31.697$      \\ 
\bottomrule
\end{tabular}
}
\end{table}

\begin{table}[h]
\caption{Comparison of average time (10 regular steps) of \texttt{LADIES} with regular step batch size as $512$.}
\label{table:ladies_time}
\centering
\scalebox{0.99}{
\begin{tabular}{lccccc}
\toprule
\textbf{Time (second)} & \textbf{PPI} & \textbf{PPI-Large} & \textbf{Flickr} & \textbf{Reddit} & \textbf{Yelp} \\ \midrule
Regular step sampling       & $1.042$      &  $1.077$       & $0.977$         & $9.856$ & $9.155$    \\
Regular step transfer       & $0.036$      &  $0.047$       & $0.016$         & $0.496$ & $0.041$    \\
Computation            & $0.077$     &   $0.068$         & $0.034$      &  $0.029$  &  $0.082$      \\ \midrule
Total time             & $1.156$     &   $1.192$         & $1.028$      &  $10.381$  &  $9.278$      \\ 
\bottomrule
\end{tabular}
}
\end{table}

%% file: supplementary/algorithm_summarize.tex
\section{Detailed algorithms} \label{supp:detail_algorithm}


In order to help readers better compare the difference of different algorithms, we summarize 
the zeroth-order variance reduced algorithm \texttt{SGCN+} in Algorithm~\ref{algorithm:sgcn_plus_complete}, and doubly variance reduced algorithm \texttt{SGCN++} in Algorithm~\ref{algorithm:sgcn_plus_plus_complete}.

\subsection{\texttt{SGCN+}}
The main idea of \texttt{SGCN+} is to use historical node features to reduce the node embedding approximation variance due to neighbor sampling. The detailed descriptions are summarized in Algorithm~\ref{algorithm:sgcn_plus_complete}.

\vspace{0pt}
\input{supplementary/algorithm_sgcn_plus_complete}
\subsection{\texttt{SGCN++}}
The main idea of \texttt{SGCN++} is to use historical node features and historical gradient to reduce the node embedding approximation variance due to neighbor sampling and the stochastic gradient variance due to mini-batch training. The detailed descriptions are summarized in Algorithm~\ref{algorithm:sgcn_plus_plus_complete}.

\vspace{20pt}
\input{supplementary/algorithm_sgcn_plus_plus_complete}

\subsection{\texttt{SGCN++} without full-batch}
Furthermore, in Algorithm~\ref{algorithm:sgcn_plus_plus_complete_no_full_batch}, we provide an alternative version of \texttt{SGCN++} that does not require full-batch forward- and backward-propagation at the snapshot step.
The basic idea is to approximate the full-batch gradient by sampling a large mini-batch $\mathcal{V}_\mathcal{B}^\prime$ of size $B^\prime=|\mathcal{V}_\mathcal{B}^\prime|$ using \texttt{Exact} sampling, then compute the node embedding matrices and stochastic gradients on the sampled large-batch $\mathcal{V}_\mathcal{B}^\prime$.

\input{supplementary/algorithm_sgcn_plus_plus_complete_fullbatch_less}


The intuition of snapshot step large-batch approximation stems from matrix Bernstein inequality~\cite{gross2011recovering}. More specifically, suppose given $\widetilde{\mathbf{G}}_i \in \mathbb{R}^{d\times d}$ be the stochastic gradient computed by using the $i$th node with \texttt{Exact} sampling (all neighbors are used to calculate the exact node embeddings). 
Suppose the different between $\widetilde{\mathbf{G}}_i$ and full-gradient $\mathbb{E} [ \widetilde{\mathbf{G}}_i ]$ is uniformly bounded and the variance is bounded:
\begin{equation}
    \| \widetilde{\mathbf{G}}_i - \mathbb{E} [ \widetilde{\mathbf{G}}_i ] \|\mathrm{F} \leq \mu,~ \mathbb{E}[\| \widetilde{\mathbf{G}}_i - \mathbb{E} [ \widetilde{\mathbf{G}}_i ] \|\mathrm{F}^2] \leq \sigma^2.
\end{equation}
Let $\widetilde{\mathbf{G}}^\prime$ as the snapshot step gradient computed on the sampled large batch
\begin{equation}
    \widetilde{\mathbf{G}}^\prime = \frac{1}{B^\prime} \sum_{i\in\mathcal{V}_\mathcal{B}^\prime} \widetilde{\mathbf{G}}_i.
\end{equation}
By matrix Bernstein inequality, we know the probability of $\| \widetilde{\mathbf{G}}^\prime-\mathbb{E} [ \widetilde{\mathbf{G}}_i ] \|\mathrm{F}$ larger than some constant $\epsilon$ decreases exponentially as the size of the sampled large-batch size $B^\prime$ increase, i.e.,
\begin{equation}
    \Pr(\|\widetilde{\mathbf{G}}^\prime - \mathbb{E} [ \widetilde{\mathbf{G}}_i ] \|\mathrm{F} \geq \epsilon) \leq 2d \exp\left(-n\cdot \min \left\{\frac{\epsilon^2}{4\sigma^2},\frac{\epsilon}{2\mu} \right\}\right).
\end{equation}
Therefore, by choosing a large enough snapshot step batch size $B^\prime$, we can obtain a good approximation of full-gradient.

%% file: supplementary/algorithm_sgcn_plus_complete.tex
\begin{breakablealgorithm}
  \caption{\texttt{SGCN+}: Zeroth-order variance reduction (Detailed version of Algorithm~\ref{algorithm:sgcn_plus_complete})}
  \label{algorithm:sgcn_plus_complete}
\begin{algorithmic}[1]
  \STATE {\bfseries Input:} Learning rate $\eta>0$, snapshot gap $K\geq 1$, $t_0=1$ ans $s=1$, staleness factor $\alpha \geq 1$.
    \FOR{$t = 1,\ldots,T$}
        \IF{$(t-t_s)~\text{mod}~K=0$}
            \STATE \texttt{\% Snapshot steps} 
            \STATE Calculate node embeddings and update historical node embeddings using \label{line:algorithm_full_sgcn+_snapshot_start}
            \begin{equation}
                \mathbf{Z}^{(\ell)}_t = \mathbf{L} \mathbf{H}_t^{(\ell-1)} \mathbf{W}^{(\ell)}_t,~
                \mathbf{H}^{(\ell)}_t=\sigma(\mathbf{Z}^{(\ell)}_t),~
                \widetilde{\mathbf{Z}}^{(\ell)}_t \leftarrow \mathbf{Z}^{(\ell)}_t
            \end{equation}
            \STATE Calculate loss as $\mathcal{L}(\bm{\theta}_t) = \frac{1}{N}\sum_{i=1}^N \text{Loss}(\bm{h}_i^{(L)}, y_i)$ 
            \STATE Calculate full-batch gradient $\nabla \mathcal{L}(\bm{\theta}_t) = \{\mathbf{G}^{(\ell)}\}_{\ell=1}^L $ as 
            \begin{equation}
                \begin{aligned}
                \mathbf{G}_t^{(\ell)} &:= [\mathbf{L} \mathbf{H}_t^{(\ell-1)}]^\top \Big( \mathbf{D}_t^{(\ell+1)} \circ \nabla \sigma(\mathbf{Z}_t^{(\ell)}) \Big),\\
                ~\mathbf{D}_t^{(\ell)} &:= \mathbf{L}^\top \Big( \mathbf{D}_t^{(\ell+1)} \circ \nabla \sigma(\mathbf{Z}_t^{(\ell)}) \Big) \mathbf{W}_t^{(\ell)},~ 
                \mathbf{D}_t^{(L+1)} = \frac{\partial \mathcal{L}(\bm{\theta}_t)}{\partial \mathbf{H}^{(L)}} 
                \end{aligned}
            \end{equation}
            \STATE Update parameters as $\bm{\theta}_{t+1} = \bm{\theta}_t - \eta \nabla \mathcal{L}(\bm{\theta}_t)$, set $t_s = t$ and $s = s+1$
        \ELSE
            \STATE \texttt{\% Regular steps}
            \STATE Sample mini-batch $\mathcal{V}_\mathcal{B}\subset \mathcal{V}$
            \STATE Calculate node embeddings using
            \begin{equation}
                \widetilde{\mathbf{Z}}^{(\ell)}_t = \widetilde{\mathbf{Z}}^{(\ell)}_{t-1} + \widetilde{\mathbf{L}}^{(\ell)} \widetilde{\mathbf{H}}_t^{(\ell-1)} \mathbf{W}_t^{(\ell)} - \widetilde{\mathbf{L}}^{(\ell)} \widetilde{\mathbf{H}}_{t-1}^{(\ell-1)} \mathbf{W}_{t-1}^{(\ell)},~
                \widetilde{\mathbf{H}}^{(\ell)}_t=\sigma(\widetilde{\mathbf{Z}}^{(\ell)})
            \end{equation}
            \IF{$\|\widetilde{\mathbf{H}}^{(\ell)}_{t}\|_\mathrm{F} \geq \alpha \| \mathbf{H}^{(\ell)}_{t_{s-1}} \|_\mathrm{F}$ for any $\ell \in [L]$}
                \STATE Go to line~\ref{line:algorithm_full_sgcn+_snapshot_start}
            \ELSE
            
            \STATE Calculate loss as $\widetilde{\mathcal{L}}(\bm{\theta}_t) = \frac{1}{B}\sum_{i\in\mathcal{V}_\mathcal{B}} \text{Loss}(\widetilde{\bm{h}}_i^{(L)}, y_i)$ 
            \STATE Calculate the stochastic gradient $\nabla \widetilde{\mathcal{L}}(\bm{\theta}_t) = \{ \widetilde{\mathbf{G}}^{(\ell)}\}_{\ell=1}^L $ as
            \begin{equation}
                \begin{aligned}
                \widetilde{\mathbf{G}}_t^{(\ell)} &:= [\widetilde{\mathbf{L}}^{(\ell)} \widetilde{\mathbf{H}}_t^{(\ell-1)}]^\top \Big( \widetilde{\mathbf{D}}_t^{(\ell+1)} \circ \nabla \sigma(\widetilde{\mathbf{Z}}_t^{(\ell)}) \Big),~ \\
                \widetilde{\mathbf{D}}_t^{(\ell)} &:= [\widetilde{\mathbf{L}}^{(\ell)}]^\top \Big( \widetilde{\mathbf{D}}_t^{(\ell+1)} \circ \nabla \sigma(\widetilde{\mathbf{Z}}_t^{(\ell)}) \Big) \mathbf{W}_t^{(\ell)},~
                \widetilde{\mathbf{D}}_t^{(L+1)} = \frac{\partial \widetilde{\mathcal{L}}(\bm{\theta}_t)}{\partial \widetilde{\mathbf{H}}^{(L)}}
                \end{aligned}
            \end{equation}
            \STATE Update parameters as $\bm{\theta}_{t+1} = \bm{\theta}_t - \eta \nabla \widetilde{\mathcal{L}}(\bm{\theta}_t)$ 
            \ENDIF
            \ENDIF
    \ENDFOR
\STATE {\bfseries Output:} Model with parameter $\bm{\theta}_{T+1}$ \\
\end{algorithmic}
\end{breakablealgorithm}

%% file: supplementary/algorithm_sgcn_plus_plus_complete.tex
\begin{breakablealgorithm}
  \caption{\texttt{SGCN++}: Doubly variance reduction (Detailed version of Algorithm~\ref{algorithm:sgcn_plus_plus_complete})}
  \label{algorithm:sgcn_plus_plus_complete}
\begin{algorithmic}[1]
  \STATE {\bfseries Input:} Learning rate $\eta>0$, snapshot gap $K\geq 1$, $t_0=1$ ans $s=1$, staleness factor $\alpha \geq 1$.
    \FOR{$t = 1,\ldots,T$}
        \IF{$(t-t_{s-1})~\text{mod}~K=0$}
            \STATE \texttt{\% Snapshot steps}
            \STATE Calculate node embeddings and update historical node embeddings using \label{line:sgcn++_full_snapshot_v1}
            \begin{equation}
                \mathbf{Z}^{(\ell)}_t = \mathbf{L} \mathbf{H}_t^{(\ell-1)} \mathbf{W}^{(\ell)}_t,~
                \mathbf{H}^{(\ell)}_t=\sigma(\mathbf{Z}^{(\ell)}_t),~
                \widetilde{\mathbf{Z}}^{(\ell)}_t \leftarrow \mathbf{Z}^{(\ell)}_t
            \end{equation}
            \STATE Calculate loss as $\mathcal{L}(\bm{\theta}_t) = \frac{1}{N}\sum_{i=1}^N \text{Loss}(\bm{h}_i^{(L)}, y_i)$ 
            \STATE Calculate the full-batch gradient $\nabla \mathcal{L}(\bm{\theta}_t) = \{\mathbf{G}^{(\ell)}\}_{\ell=1}^L $ as
            \begin{equation}
                \begin{aligned}
                \mathbf{G}_t^{(\ell)} &:= [\mathbf{L} \mathbf{H}_t^{(\ell-1)}]^\top \Big( \mathbf{D}_t^{(\ell)} \circ \nabla \sigma(\mathbf{Z}_t^{(\ell)}) \Big),~\\
                \mathbf{D}_t^{(\ell)} &:= \mathbf{L}^\top \Big( \mathbf{D}_t^{(\ell+1)} \circ \nabla \sigma(\mathbf{Z}_t^{(\ell)}) \Big) \mathbf{W}_t^{(\ell)},~
                \mathbf{D}_t^{(L+1)} = \frac{\partial \mathcal{L}(\bm{\theta}_t)}{\partial \mathbf{H}^{(L)}} 
                \end{aligned}
            \end{equation}
            \STATE Save the per layerwise gradient $\widetilde{\mathbf{G}}_t^{(\ell)} \leftarrow \mathbf{G}_t^{(\ell)},~     \widetilde{\mathbf{D}}_t^{(\ell)} \leftarrow \mathbf{D}_t^{(\ell)}$ for all $\ell\in[L]$
            \STATE Update parameters as $\bm{\theta}_{t+1} = \bm{\theta}_t - \eta \nabla \mathcal{L}(\bm{\theta}_t)$, set $t_s = t$ and $s = s+1$
        \ELSE
            \STATE \texttt{\% Regular steps}
            \STATE Sample mini-batch $\mathcal{V}_\mathcal{B}\subset \mathcal{V}$
            \STATE Calculate node embeddings using
            \begin{equation}
                \widetilde{\mathbf{Z}}^{(\ell)}_t = \widetilde{\mathbf{Z}}^{(\ell)}_{t-1} + \widetilde{\mathbf{L}}^{(\ell)} \widetilde{\mathbf{H}}_t^{(\ell-1)} \mathbf{W}_t^{(\ell)} - \widetilde{\mathbf{L}}^{(\ell)} \widetilde{\mathbf{H}}_{t-1}^{(\ell-1)} \mathbf{W}_{t-1}^{(\ell)},~
                \widetilde{\mathbf{H}}^{(\ell)}_t=\sigma(\widetilde{\mathbf{Z}}^{(\ell)})
            \end{equation}
            \IF{$\|\widetilde{\mathbf{H}}^{(\ell)}_{t-1}\|_\mathrm{F} \geq \alpha \| \mathbf{H}^{(\ell)}_{t_{s-1}} \|_\mathrm{F}$ or $\|\widetilde{\mathbf{D}}^{(\ell)}_{t-1}\|_\mathrm{F} \geq \beta \| \mathbf{D}^{(\ell)}_{t_{s-1}} \|_\mathrm{F}$ for any $\ell \in [L]$}
            \STATE Go to line~\ref{line:sgcn++_full_snapshot_v1}
            \ELSE
            \STATE Calculate loss as $\widetilde{\mathcal{L}}(\bm{\theta}_t) = \frac{1}{B}\sum_{i\in\mathcal{V}_\mathcal{B}} \text{Loss}(\widetilde{\bm{h}}_i^{(L)}, y_i)$ 
            \STATE Calculate the stochastic gradient $\nabla \widetilde{\mathcal{L}}(\bm{\theta}_t) = \{ \widetilde{\mathbf{G}}^{(\ell)}\}_{\ell=1}^L $ as
            \begin{equation}
                \begin{aligned}
                \widetilde{\mathbf{G}}_t^{(\ell)} 
                &= \widetilde{\mathbf{G}}_{t-1}^{(\ell)} + [\widetilde{\mathbf{L}}^{(\ell)} \widetilde{\mathbf{H}}_t^{(\ell-1)}]^\top \Big(\widetilde{\mathbf{D}}_t^{(\ell+1)} \circ  \nabla \sigma(\widetilde{\mathbf{Z}}_t)\Big) \\
                &\qquad - [\widetilde{\mathbf{L}}^{(\ell)} \widetilde{\mathbf{H}}_{t-1}^{(\ell-1)}]^\top \Big(\widetilde{\mathbf{D}}_{t-1}^{(\ell+1)} \circ  \nabla \sigma(\widetilde{\mathbf{Z}}_{t-1})\Big)  \\
                \widetilde{\mathbf{D}}_t^{(\ell)} 
                &= \widetilde{\mathbf{D}}_{t-1}^{(\ell)} + [\widetilde{\mathbf{L}}^{(\ell)}]^\top \Big(\widetilde{\mathbf{D}}_t^{(\ell+1)} \circ \nabla \sigma(\widetilde{\mathbf{Z}}_t) \Big) [\mathbf{W}_t^{(\ell)}] \\
                &\qquad - [\widetilde{\mathbf{L}}^{(\ell)}]^\top \Big(\widetilde{\mathbf{D}}_{t-1}^{(\ell+1)} \circ \nabla \sigma(\widetilde{\mathbf{Z}}_{t-1}) \Big) [\mathbf{W}_{t-1}^{(\ell)}],~\widetilde{\mathbf{D}}_t^{(L+1)} = \frac{\partial \widetilde{\mathcal{L}}(\bm{\theta}_t)}{\partial \widetilde{\mathbf{H}}_t^{(L)}}
                \end{aligned}
            \end{equation}
            \STATE Update parameters as $\bm{\theta}_{t+1} = \bm{\theta}_t - \eta \nabla \widetilde{\mathcal{L}}(\bm{\theta}_t)$ 
            \ENDIF\ENDIF
    \ENDFOR
\STATE {\bfseries Output:} Model with parameter $\bm{\theta}_{T+1}$ \\
\end{algorithmic}
\end{breakablealgorithm}

%% file: supplementary/algorithm_sgcn_plus_plus_complete_fullbatch_less.tex
\begin{breakablealgorithm}
  \caption{\texttt{SGCN++} (without full-batch): Doubly variance reduction }
  \label{algorithm:sgcn_plus_plus_complete_no_full_batch}
\begin{algorithmic}[1]
  \STATE {\bfseries Input:} Learning rate $\eta>0$, snapshot gap $K>0$
    \FOR{$t = 1,\ldots,T$}
        \IF{$(t-t_{s-1})~\text{mod}~K=0$}
            \STATE \texttt{\% Snapshot steps}
            \STATE Sample a large-batch $\mathcal{V}_\mathcal{B}^\prime$ of size $B^\prime$ and construct the Laplacian matrices $\mathbf{L}^{(\ell)}$ for each layer using all neighbors, i.e., \label{line:sgcn++_full_snapshot_v2}
            \begin{equation}
                L_{i,j}^{(\ell)} = 
                \begin{cases}
                 {L}_{i,j}, &\text{ if } j \in \mathcal{N}^{(\ell)}(i) \\
                 0, &\text{ otherwise }
                \end{cases}
            \end{equation}
            \STATE Calculate node embeddings and update historical node embeddings using 
            \begin{equation}
                \mathbf{Z}^{(\ell)}_t = \mathbf{L}^{(\ell)} \mathbf{H}_t^{(\ell-1)} \mathbf{W}^{(\ell)}_t,~
                \mathbf{H}^{(\ell)}_t=\sigma(\mathbf{Z}^{(\ell)}_t),~
                \widetilde{\mathbf{Z}}^{(\ell)}_t \leftarrow \mathbf{Z}^{(\ell)}_t
            \end{equation}
            \STATE Calculate loss as $\mathcal{L}(\bm{\theta}_t) = \frac{1}{B^\prime}\sum_{i\in\mathcal{V}_\mathcal{B}^\prime} \text{Loss}(\bm{h}_i^{(L)}, y_i)$ 
            \STATE Calculate the approximated snapshot gradient $\nabla \mathcal{L}(\bm{\theta}_t) = \{\mathbf{G}^{(\ell)}\}_{\ell=1}^L $ as
            \begin{equation}
                \begin{aligned}
                \mathbf{G}_t^{(\ell)} &:= [\mathbf{L}^{(\ell)} \mathbf{H}_t^{(\ell-1)}]^\top \Big( \mathbf{D}_t^{(\ell+1)} \circ \nabla \sigma(\mathbf{Z}_t^{(\ell)}) \Big),~\\
                \mathbf{D}_t^{(\ell)} &:= [\mathbf{L}^{(\ell)}]^\top \Big( \mathbf{D}_t^{(\ell+1)} \circ \nabla \sigma(\mathbf{Z}_t^{(\ell)}) \Big) \mathbf{W}_t^{(\ell)},~
                \mathbf{D}_t^{(L+1)} = \frac{\partial \mathcal{L}(\bm{\theta}_t)}{\partial \mathbf{H}^{(L)}} 
                \end{aligned}
            \end{equation}
            \STATE Save the per layerwise gradient $\widetilde{\mathbf{G}}_t^{(\ell)} \leftarrow \mathbf{G}_t^{(\ell)},~     \widetilde{\mathbf{D}}_t^{(\ell)} \leftarrow \mathbf{D}_t^{(\ell)},~\forall \ell\in[L]$
            \STATE Update parameters as $\bm{\theta}_{t+1} = \bm{\theta}_t - \eta \nabla \mathcal{L}(\bm{\theta}_t)$ , set $t_s = t$ and $s = s + 1$
        \ELSE
            \STATE \texttt{\% Regular steps}
            \STATE Sample mini-batch $\mathcal{V}_\mathcal{B}\subset \mathcal{V}_\mathcal{B}^\prime$
            \STATE Calculate node embeddings using
            \begin{equation}
                \widetilde{\mathbf{Z}}^{(\ell)}_t = \widetilde{\mathbf{Z}}^{(\ell)}_{t-1} + \widetilde{\mathbf{L}}^{(\ell)} \widetilde{\mathbf{H}}_t^{(\ell-1)} \mathbf{W}_t^{(\ell)} - \widetilde{\mathbf{L}}^{(\ell)} \widetilde{\mathbf{H}}_{t-1}^{(\ell-1)} \mathbf{W}_{t-1}^{(\ell)},~
                \widetilde{\mathbf{H}}^{(\ell)}_t=\sigma(\widetilde{\mathbf{Z}}^{(\ell)})
            \end{equation}
            \IF{$\|\widetilde{\mathbf{H}}^{(\ell)}_{t-1}\|_\mathrm{F} \geq \alpha \| \mathbf{H}^{(\ell)}_{t_{s-1}} \|_\mathrm{F}$ or $\|\widetilde{\mathbf{D}}^{(\ell)}_{t-1}\|_\mathrm{F} \geq \beta \| \mathbf{D}^{(\ell)}_{t_{s-1}} \|_\mathrm{F}$ for any $\ell \in [L]$}
            \STATE Go to line~\ref{line:sgcn++_full_snapshot_v2}
            \ELSE
            \STATE Calculate loss as $\widetilde{\mathcal{L}}(\bm{\theta}_t) = \frac{1}{B}\sum_{i\in\mathcal{V}_\mathcal{B}} \text{Loss}(\widetilde{\bm{h}}_i^{(L)}, y_i)$ 
            \STATE Calculate the stochastic gradient $\nabla \widetilde{\mathcal{L}}(\bm{\theta}_t) = \{ \widetilde{\mathbf{G}}^{(\ell)}\}_{\ell=1}^L $ as
            \begin{equation}
                \begin{aligned}
                &\widetilde{\mathbf{G}}_t^{(\ell)} = \widetilde{\mathbf{G}}_{t-1}^{(\ell)} + [\widetilde{\mathbf{L}}^{(\ell)} \widetilde{\mathbf{H}}_t^{(\ell-1)}]^\top \Big(\widetilde{\mathbf{D}}_t^{(\ell+1)} \circ  \nabla \sigma(\widetilde{\mathbf{Z}}_t)\Big) - [\widetilde{\mathbf{L}}^{(\ell)} \widetilde{\mathbf{H}}_{t-1}^{(\ell-1)}]^\top \Big(\widetilde{\mathbf{D}}_{t-1}^{(\ell+1)} \circ  \nabla \sigma(\widetilde{\mathbf{Z}}_{t-1})\Big)  \\
                &\widetilde{\mathbf{D}}_t^{(\ell)} = \widetilde{\mathbf{D}}_{t-1}^{(\ell)} + [\widetilde{\mathbf{L}}^{(\ell)}]^\top \Big(\widetilde{\mathbf{D}}_t^{(\ell+1)} \circ \nabla \sigma(\widetilde{\mathbf{Z}}_t) \Big) [\mathbf{W}_t^{(\ell)}] - [\widetilde{\mathbf{L}}^{(\ell)}]^\top \Big(\widetilde{\mathbf{D}}_{t-1}^{(\ell+1)} \circ \nabla \sigma(\widetilde{\mathbf{Z}}_{t-1}) \Big) [\mathbf{W}_{t-1}^{(\ell)}],~\\
                &\widetilde{\mathbf{D}}_t^{(L+1)} = \frac{\partial \widetilde{\mathcal{L}}(\bm{\theta}_t)}{\partial \widetilde{\mathbf{H}}_t^{(L)}}
                \end{aligned}
                \vspace{-10pt}
            \end{equation}
            \STATE Update parameters as $\bm{\theta}_{t+1} = \bm{\theta}_t - \eta \nabla \widetilde{\mathcal{L}}(\bm{\theta}_t)$ 
            \ENDIF\ENDIF
    \ENDFOR
\STATE {\bfseries Output:} Model with parameter $\bm{\theta}_{T+1}$ \\
\end{algorithmic}
\end{breakablealgorithm}

%% file: appendix/proof_sgcn_thm_1.tex
\section{Proof of Theorem~\ref{theorem:convergence_of_sgcn}}\label{section:proof_of_thm1}


Before processing the proof of Theorem~\ref{theorem:convergence_of_sgcn}, let first recall the definition and notation on the forward and backward propagation of GCNs.

\subsection{Notations for gradient computation}
We introduce the following notations to simplify the representation and make it easier for readers to understand. 
Let formulate each GCN layer in \texttt{FullGCN} as a function
\begin{equation}
    \mathbf{H}^{(\ell)} = f^{(\ell)}(\mathbf{H}^{(\ell-1)}, \mathbf{W}^{(\ell)}) = \sigma( \underbrace{\mathbf{L} \mathbf{H}^{(\ell-1)} \mathbf{W}^{(\ell)}}_{\mathbf{Z}^{(\ell)}})  \in \mathbb{R}^{N\times d_\ell},
\end{equation}
and its gradient w.r.t. the input node embedding matrix $\mathbf{D}^{(\ell)} \in \mathbb{R}^{N \times d_{\ell-1}}$ is computed as
\begin{equation}
    \mathbf{D}^{(\ell)} = \nabla_H f^{(\ell)}(\mathbf{D}^{(\ell+1)}, \mathbf{H}^{(\ell-1)}, \mathbf{W}^{(\ell)})= [\mathbf{L}]^\top \Big(\mathbf{D}^{(\ell+1)} \circ \sigma^\prime(\mathbf{L}\mathbf{H}^{(\ell-1)} \mathbf{W}^{(\ell)}) \Big) [\mathbf{W}^{(\ell)}]^\top,
\end{equation}
and its gradient w.r.t. the weight matrix $\mathbf{G}^{(\ell)} \in \mathbb{R}^{d_{\ell-1} \times d_\ell}$ is computed as
\begin{equation}
    \mathbf{G}^{(\ell)} = \nabla_W f^{(\ell)}(\mathbf{D}^{(\ell+1)}, \mathbf{H}^{(\ell-1)}, \mathbf{W}^{(\ell)}) =  [\mathbf{L} \mathbf{H}^{(\ell-1)}]^\top \Big(\mathbf{D}^{(\ell+1)} \circ  \sigma^\prime(\mathbf{L}^{(\ell)}\mathbf{H}^{(\ell-1)}\mathbf{W}^{(\ell)})\Big) \Big].
\end{equation}

Similarly, we can formulate the calculation of node embedding matrix $\widetilde{\mathbf{H}}^{(\ell)} \in \mathbb{R}^{N\times d_\ell}$ at each GCN layer in \texttt{SGCN} as 
\begin{equation}
    \widetilde{\mathbf{H}}^{(\ell)} = \widetilde{f}^{(\ell)}(\widetilde{\mathbf{H}}^{(\ell-1)}, \mathbf{W}^{(\ell)}) = \sigma( \underbrace{ \widetilde{\mathbf{L}}^{(\ell)} \widetilde{\mathbf{H}}^{(\ell-1)} \mathbf{W}^{(\ell)}}_{\widetilde{\mathbf{Z}}^{(\ell)}} ), 
\end{equation}
and its gradient w.r.t. the input node embedding matrix $\widetilde{\mathbf{D}}^{(\ell)} \in \mathbb{R}^{N \times d_{\ell-1}}$ is computed as 
\begin{equation}
    \widetilde{\mathbf{D}}^{(\ell)} = \nabla_H \widetilde{f}^{(\ell)}(\widetilde{\mathbf{D}}^{(\ell+1)}, \widetilde{\mathbf{H}}^{(\ell-1)}, \mathbf{W}^{(\ell)}) = [\widetilde{\mathbf{L}}^{(\ell)} ]^\top \Big(\widetilde{\mathbf{D}}^{(\ell+1)} \circ \sigma^\prime(\widetilde{\mathbf{L}}^{(\ell)}\widetilde{\mathbf{H}}^{(\ell-1)}\mathbf{W}^{(\ell)}) \Big) [\mathbf{W}^{(\ell)}]^\top \Big],
\end{equation}
and its gradient w.r.t. the weight matrix $\widetilde{\mathbf{G}}^{(\ell)} \in \mathbb{R}^{d_{\ell-1} \times d_\ell}$ is computed as 
\begin{equation}
    \widetilde{\mathbf{G}}^{(\ell)} = \Big[\nabla_W \widetilde{f}^{(\ell)}(\widetilde{\mathbf{D}}^{(\ell+1)}, \widetilde{\mathbf{H}}^{(\ell-1)}, \mathbf{W}^{(\ell)}):=  [\widetilde{\mathbf{L}}^{(\ell)} \widetilde{\mathbf{H}}^{(\ell-1)}]^\top \Big(\widetilde{\mathbf{D}}^{(\ell+1)} \circ  \sigma^\prime(\widetilde{\mathbf{L}}^{(\ell)}\widetilde{\mathbf{H}}^{(\ell-1)}\mathbf{W}^{(\ell)})\Big) \Big].
\end{equation}

Let us denote the gradient of loss w.r.t. the final node embedding matrix  as
\begin{equation}
    \begin{aligned}
    \mathbf{D}^{(L+1)} &= \frac{\partial \text{Loss}(\mathbf{H}^{(L)}, \mathbf{y}) }{\partial \mathbf{H}^{(L)}} \in\mathbb{R}^{N \times d_L} 
    , [\mathbf{D}^{(L+1)}]_i = \frac{1}{N}\frac{\partial \text{Loss}(\mathbf{h}_i^{(L)}, y_i)}{\partial \mathbf{h}_i^{(L)}} \in \mathbb{R}^{d_L}, \\
    \widetilde{\mathbf{D}}^{(L+1)} &= \frac{\partial \text{Loss}(\widetilde{\mathbf{H}}^{(L)}, \mathbf{y})}{\partial \widetilde{\mathbf{H}}^{(L)}} \in\mathbb{R}^{N \times d_L}
    , [\widetilde{\mathbf{D}}^{(L+1)}]_i = \frac{1}{B} \mathbf{1}_{\{i\in\mathcal{V}_\mathcal{B}\}} \frac{\partial \text{Loss}(\widetilde{\mathbf{h}}_i^{(L)}, y_i)}{\partial \widetilde{\mathbf{h}}_i^{(L)}} \in \mathbb{R}^{d_L} .
    \end{aligned}
\end{equation}
Notice that $\widetilde{\mathbf{D}}^{(L+1)}$ is a $N\times d_L$ matrix with the row number $i\in\mathcal{V}_\mathcal{B}$ are non-zero vectors.
Then we can write the gradient for the $\ell$th weight matrix in \texttt{FullGCN} and \texttt{SGCN} as
\begin{equation}
    \begin{aligned}
    \mathbf{G}^{(\ell)} = \nabla_W f^{(\ell)}( \nabla_H f^{(\ell+1)}( \ldots \nabla_H f^{(L)}( \mathbf{D}^{(L+1)}, \mathbf{H}^{(L-1)}, \mathbf{W}^{(L)}) \ldots,  \mathbf{H}^{(\ell)}, \mathbf{W}^{(\ell+1)}) , \mathbf{H}^{(\ell-1)}, \mathbf{W}^{(\ell)}), \\
    \widetilde{\mathbf{G}}^{(\ell)} = \nabla_W \widetilde{f}^{(\ell)}( \nabla_H \widetilde{f}^{(\ell+1)}( \ldots \nabla_H \widetilde{f}^{(L)}( \widetilde{\mathbf{D}}^{(L+1)}, \widetilde{\mathbf{H}}^{(L-1)}, \mathbf{W}^{(L)}) \ldots, \widetilde{\mathbf{H}}^{(\ell)}, \mathbf{W}^{(\ell+1)}) , \widetilde{\mathbf{H}}^{(\ell-1)}, \mathbf{W}^{(\ell)}).
    \end{aligned}
\end{equation}

\subsection{Upper bounds on the node embedding matrices and layerwise gradients}

Based on the Assumption~\ref{assumption:bound_norm}, we first derive the upper-bound on the node embedding matrices and the gradient passing from $\ell$th layer node embedding matrix to the $(\ell-1)$th layer node embedding matrix.
\begin{proposition} \label{proposition:matrix_norm_bound}
    For any $\ell\in[L]$, the Frobenius norm of node embedding matrices, gradient passing from the $\ell$th layer node embeddings to the $(\ell-1)$th are bounded 
    \begin{equation}
        \|\mathbf{H}^{(\ell)}\|_{\mathrm{F}} \leq B_H,~
        \|\widetilde{\mathbf{H}}^{(\ell)}\|_{\mathrm{F}} \leq B_H,~
        \| \frac{\partial \sigma(\mathbf{L} \mathbf{H}^{(\ell-1)} \mathbf{W}^{(\ell)})}{\partial \mathbf{H}^{(\ell-1)}} \|_{\mathrm{F}} \leq B_D,~
        \| \frac{\partial \sigma(\widetilde{\mathbf{L}}^{(\ell)} \widetilde{\mathbf{H}}^{(\ell-1)} \mathbf{W}^{(\ell)})}{\partial \widetilde{\mathbf{H}}^{(\ell-1)}} \|_{\mathrm{F}} \leq B_D,
    \end{equation}
    where 
    \begin{equation}
        B_H = \max\Big\{1, ( C_\sigma B_{LA} B_{W})^L\Big\} B_{X},~
        B_D = \max\Big\{ 1, (B_{LA} C_\sigma B_{W})^{L-\ell} \Big\} C_\text{loss}
    \end{equation}
\end{proposition}
\begin{proof}
Let first upper bounded the node embeddings,
\begin{equation}
    \begin{aligned}
    \|\mathbf{H}^{(\ell)}\|_{\mathrm{F}} &= \|\sigma(\mathbf{L} \mathbf{H}^{(\ell-1)} \mathbf{W}^{(\ell)})\|_{\mathrm{F}} \\
    &\leq C_\sigma B_{LA} B_W \| \mathbf{H}^{(\ell-1)} \|_{\mathrm{F}} \\
    &\leq ( C_\sigma B_{LA} B_{W})^\ell \| \mathbf{X} \|_{\mathrm{F}} \\
    &\leq \max\Big\{1, ( C_\sigma B_{LA} B_{W})^L\Big\} B_{X}.
    \end{aligned}
\end{equation}

Then, we derive the upper bound of gradient passing between layers, 
\begin{equation}
    \begin{aligned}
    \| \mathbf{D}^{(\ell)}\|_{\mathrm{F}} 
    &= \Big\| [\mathbf{L}]^\top \Big(\mathbf{D}^{(\ell+1)} \circ \sigma^\prime(\mathbf{L}\mathbf{H}^{(\ell-1)} \mathbf{W}^{(\ell)}) \Big) [\mathbf{W}^{(\ell)}]^\top \Big\|_{\mathrm{F}} \\
    &\leq B_{LA} C_\sigma B_W \| \mathbf{D}^{(\ell+1)}\|_{\mathrm{F}} \\
    &\leq \Big( B_{LA} C_\sigma B_{W} \Big)^{L-\ell} \| \mathbf{D}^{(L+1)}\|_{\mathrm{F}} \\
    &\leq \max\Big\{ 1, (B_{LA} C_\sigma B_{W})^L\Big\} C_\text{loss}.
    \end{aligned}
\end{equation}

A similar strategy can be used to derive the upper bound of its stochastic version.
\end{proof}

\subsection{Lipschitz continuity and smoothness property of graph convolution layers}
Then, we derive the Lipschitz continuity of $\widetilde{f}^{(\ell)}(\cdot, \cdot)$ and its gradient. Notice that the following result hold for deterministic function $f^{(\ell)}(\cdot, \cdot)$ as well.

\begin{proposition}
    $\widetilde{f}^{(\ell)}(\cdot, \cdot)$ is $C_H$-Lipschitz continuous w.r.t. the input node embedding matrix where $C_H=C_\sigma B_{LA} B_W$.
\end{proposition}
\begin{proof}
\begin{equation}
        \begin{aligned}
        &\|\widetilde{f}^{(\ell)}(\mathbf{H}_1^{(\ell-1)}, \mathbf{W}^{(\ell)}) - \widetilde{f}^{(\ell)}(\mathbf{H}_2^{(\ell-1)}, \mathbf{W}^{(\ell)})\|_{\mathrm{F}} \\
        &=\| \sigma (\widetilde{\mathbf{L}}^{(\ell)} \mathbf{H}_1^{(\ell-1)} \mathbf{W}^{(\ell)}) - \sigma (\widetilde{\mathbf{L}}^{(\ell)} \mathbf{H}_2^{(\ell-1)} \mathbf{W}^{(\ell)}) \|_{\mathrm{F}} \\
        &\leq C_\sigma \|\widetilde{\mathbf{L}}^{(\ell)}\|_{\mathrm{F}} \|\mathbf{H}_1^{(\ell-1)} - \mathbf{H}_2^{(\ell-1)}\|_{\mathrm{F}} \|\mathbf{W}^{(\ell)}\|_{\mathrm{F}} \\
        &\leq C_\sigma B_{LA} B_W \|\mathbf{H}_1^{(\ell-1)} - \mathbf{H}_2^{(\ell-1)}\|_{\mathrm{F}}.
        \end{aligned}
    \end{equation}
\end{proof}

\begin{proposition}
    $\widetilde{f}^{(\ell)}(\cdot, \cdot)$ is $C_W$-Lipschitz continuous w.r.t. the weight matrix where $C_W=C_\sigma B_{LA} B_H$.
\end{proposition}
\begin{proof}
\begin{equation}
        \begin{aligned}
        &\|\widetilde{f}^{(\ell)}(\mathbf{H}^{(\ell-1)}, \mathbf{W}_1^{(\ell)}) - \widetilde{f}^{(\ell)}(\mathbf{H}^{(\ell-1)}, \mathbf{W}_2^{(\ell)})\|_{\mathrm{F}} \\
        &=\| \sigma (\widetilde{\mathbf{L}}^{(\ell)} \mathbf{H}^{(\ell-1)} \mathbf{W}_1^{(\ell)}) - \sigma (\widetilde{\mathbf{L}}^{(\ell)} \mathbf{H}^{(\ell-1)} \mathbf{W}_2^{(\ell)}) \|_{\mathrm{F}} \\
        &\leq C_\sigma \|\widetilde{\mathbf{L}}^{(\ell)}\|_{\mathrm{F}} \|\mathbf{H}^{(\ell-1)}\|_{\mathrm{F}} \|\mathbf{W}_1^{(\ell)} - \mathbf{W}_2^{(\ell)}\|_{\mathrm{F}} \\
        &\leq C_\sigma B_{LA} B_H \|\mathbf{W}_1^{(\ell)} - \mathbf{W}_2^{(\ell)}\|_{\mathrm{F}}.
        \end{aligned}
    \end{equation}
\end{proof}
\begin{proposition}
$\nabla_H \widetilde{f}^{(\ell)}(\cdot, \cdot, \cdot)$ is $L_H$-Lipschitz continuous where 
    \begin{equation}
        L_H = \max\Big\{ B_{LA} C_\sigma B_W,  B_{LA}^2 B_D B_W^2 L_\sigma, B_{LA} B_D C_\sigma + B_{LA}^2 B_D B_W L_\sigma B_H \Big\}.
    \end{equation}
\end{proposition}
\begin{proof}
\begin{equation}
        \begin{aligned}
        &\| \nabla_H \widetilde{f}^{(\ell)}(\widetilde{\mathbf{D}}_1^{(\ell+1)}, \widetilde{\mathbf{H}}^{(\ell-1)}, \mathbf{W}^{(\ell)}) - \nabla_H \widetilde{f}^{(\ell)}(\widetilde{\mathbf{D}}_2^{(\ell+1)}, \widetilde{\mathbf{H}}^{(\ell-1)}, \mathbf{W}^{(\ell)}) \|_{\mathrm{F}} \\
        &\leq \| [\widetilde{\mathbf{L}}^{(\ell)} ]^\top \Big(\widetilde{\mathbf{D}}_1^{(\ell+1)} \circ \sigma^\prime(\widetilde{\mathbf{L}}^{(\ell)}\widetilde{\mathbf{H}}^{(\ell-1)}\mathbf{W}^{(\ell)}) \Big) [\mathbf{W}^{(\ell)}]^\top \\
        &\qquad - [\widetilde{\mathbf{L}}^{(\ell)} ]^\top \Big(\widetilde{\mathbf{D}}_2^{(\ell+1)} \circ \sigma^\prime(\widetilde{\mathbf{L}}^{(\ell)}\widetilde{\mathbf{H}}^{(\ell-1)}\mathbf{W}^{(\ell)}) \Big) [\mathbf{W}^{(\ell)}]^\top \|_{\mathrm{F}} \\
        &\leq B_{LA} C_\sigma B_W \| \widetilde{\mathbf{D}}_1^{(\ell+1)} - \widetilde{\mathbf{D}}_2^{(\ell+1)}\|_{\mathrm{F}},
        \end{aligned}
    \end{equation}
    \begin{equation}
        \begin{aligned}
        &\| \nabla_H \widetilde{f}^{(\ell)}(\widetilde{\mathbf{D}}^{(\ell+1)}, \widetilde{\mathbf{H}}_1^{(\ell-1)}, \mathbf{W}^{(\ell)}) - \nabla_H \widetilde{f}^{(\ell)}(\widetilde{\mathbf{D}}^{(\ell+1)}, \widetilde{\mathbf{H}}_2^{(\ell-1)}, \mathbf{W}^{(\ell)}) \|_{\mathrm{F}} \\
        &\leq \| [\widetilde{\mathbf{L}}^{(\ell)} ]^\top \Big(\widetilde{\mathbf{D}}^{(\ell+1)} \circ \sigma^\prime(\widetilde{\mathbf{L}}^{(\ell)}\widetilde{\mathbf{H}}_1^{(\ell-1)}\mathbf{W}^{(\ell)}) \Big) [\mathbf{W}^{(\ell)}]^\top \\
        &\qquad - [\widetilde{\mathbf{L}}^{(\ell)} ]^\top \Big(\widetilde{\mathbf{D}}^{(\ell+1)} \circ \sigma^\prime(\widetilde{\mathbf{L}}^{(\ell)}\widetilde{\mathbf{H}}_2^{(\ell-1)}\mathbf{W}^{(\ell)}) \Big) [\mathbf{W}^{(\ell)}]^\top \|_{\mathrm{F}} \\
        &\leq B_{LA}^2 B_D B_W^2 L_\sigma \| \widetilde{\mathbf{H}}_1^{(\ell-1)} - \widetilde{\mathbf{H}}_2^{(\ell-1)}\|_{\mathrm{F}},
        \end{aligned}
    \end{equation}
    \begin{equation}
        \begin{aligned}
        &\| \nabla_H \widetilde{f}^{(\ell)}(\widetilde{\mathbf{D}}^{(\ell+1)}, \widetilde{\mathbf{H}}^{(\ell-1)}, \mathbf{W}_1^{(\ell)}) - \nabla_H \widetilde{f}^{(\ell)}(\widetilde{\mathbf{D}}^{(\ell+1)}, \widetilde{\mathbf{H}}^{(\ell-1)}, \mathbf{W}_2^{(\ell)}) \|_{\mathrm{F}} \\
        &\leq \| [\widetilde{\mathbf{L}}^{(\ell)} ]^\top \Big(\widetilde{\mathbf{D}}^{(\ell+1)} \circ \sigma^\prime(\widetilde{\mathbf{L}}^{(\ell)}\widetilde{\mathbf{H}}^{(\ell-1)}\mathbf{W}_1^{(\ell)}) \Big) [\mathbf{W}_1^{(\ell)}]^\top \\
        &\qquad - [\widetilde{\mathbf{L}}^{(\ell)} ]^\top \Big(\widetilde{\mathbf{D}}^{(\ell+1)} \circ \sigma^\prime(\widetilde{\mathbf{L}}^{(\ell)}\widetilde{\mathbf{H}}^{(\ell-1)}\mathbf{W}_2^{(\ell)}) \Big) [\mathbf{W}_2^{(\ell)}]^\top \|_{\mathrm{F}} \\
        &\leq (B_{LA} B_D C_\sigma + B_{LA}^2 B_D B_W L_\sigma B_H) \| \mathbf{W}_1^{(\ell)} - \mathbf{W}_2^{(\ell)}\|_{\mathrm{F}}.
        \end{aligned}
    \end{equation}
\end{proof}
\begin{proposition}
$\nabla_W \widetilde{f}^{(\ell)}(\cdot, \cdot, \cdot)$ is $L_W$-Lipschitz continuous where 
    \begin{equation}
        L_W = \max\Big\{B_{LA} B_H C_\sigma, B_{LA}^2 B_H^2 H_D L_\sigma, B_{LA} B_D C_\sigma + B_{LA}^2 B_H^2 B_D L_\sigma \Big\}.
    \end{equation}    
\end{proposition}
\begin{proof}
\begin{equation}
    \begin{aligned}
    &\| \nabla_W \widetilde{f}^{(\ell)}(\widetilde{\mathbf{D}}_1^{(\ell+1)}, \widetilde{\mathbf{H}}^{(\ell-1)}, \mathbf{W}^{(\ell)}) - \nabla_W \widetilde{f}^{(\ell)}(\widetilde{\mathbf{D}}_2^{(\ell+1)}, \widetilde{\mathbf{H}}^{(\ell-1)}, \mathbf{W}^{(\ell)}) \|_{\mathrm{F}} \\
    &\leq \| [\widetilde{\mathbf{L}}^{(\ell)} \widetilde{\mathbf{H}}^{(\ell-1)}]^\top \Big(\widetilde{\mathbf{D}}_1^{(\ell+1)} \circ  \sigma^\prime(\widetilde{\mathbf{L}}^{(\ell)}\widetilde{\mathbf{H}}^{(\ell-1)}\mathbf{W}^{(\ell)})\Big) \\
    &\qquad - [\widetilde{\mathbf{L}}^{(\ell)} \widetilde{\mathbf{H}}^{(\ell-1)}]^\top \Big(\widetilde{\mathbf{D}}_2^{(\ell+1)} \circ  \sigma^\prime(\widetilde{\mathbf{L}}^{(\ell)}\widetilde{\mathbf{H}}^{(\ell-1)}\mathbf{W}^{(\ell)})\Big) \|_{\mathrm{F}} \\
    &\leq B_{LA} B_H C_\sigma \| \widetilde{\mathbf{D}}_1^{(\ell+1)} - \widetilde{\mathbf{D}}_2^{(\ell+1)} \|_{\mathrm{F}},
    \end{aligned}
\end{equation}
\begin{equation}
    \begin{aligned}
    &\| \nabla_W \widetilde{f}^{(\ell)}(\widetilde{\mathbf{D}}^{(\ell+1)}, \widetilde{\mathbf{H}}^{(\ell-1)}_1, \mathbf{W}^{(\ell)}) - \nabla_W \widetilde{f}^{(\ell)}(\widetilde{\mathbf{D}}^{(\ell+1)}, \widetilde{\mathbf{H}}^{(\ell-1)}_2, \mathbf{W}^{(\ell)}) \|_{\mathrm{F}} \\
    &\leq \| [\widetilde{\mathbf{L}}^{(\ell)} \widetilde{\mathbf{H}}_1^{(\ell-1)}]^\top \Big(\widetilde{\mathbf{D}}^{(\ell+1)} \circ  \sigma^\prime(\widetilde{\mathbf{L}}^{(\ell)}\widetilde{\mathbf{H}}_1^{(\ell-1)}\mathbf{W}^{(\ell)})\Big) \\
    &\qquad - [\widetilde{\mathbf{L}}^{(\ell)} \widetilde{\mathbf{H}}_2^{(\ell-1)}]^\top \Big(\widetilde{\mathbf{D}}^{(\ell+1)} \circ  \sigma^\prime(\widetilde{\mathbf{L}}^{(\ell)}\widetilde{\mathbf{H}}_2^{(\ell-1)}\mathbf{W}^{(\ell)})\Big) \|_{\mathrm{F}} \\
    &\leq (B_{LA}^2 B_D^2 C_\sigma + B_{LA}^2 B_H 
    B_D L_\sigma B_W) \| \widetilde{\mathbf{H}}_1^{(\ell-1)} - \widetilde{\mathbf{H}}_2^{(\ell-1)} \|_{\mathrm{F}},
    \end{aligned}
\end{equation}
\begin{equation}
    \begin{aligned}
    &\| \nabla_W \widetilde{f}^{(\ell)}(\widetilde{\mathbf{D}}^{(\ell+1)}, \widetilde{\mathbf{H}}^{(\ell-1)}, \mathbf{W}_1^{(\ell)}) - \nabla_W \widetilde{f}^{(\ell)}(\widetilde{\mathbf{D}}^{(\ell+1)}, \widetilde{\mathbf{H}}^{(\ell-1)}, \mathbf{W}_2^{(\ell)}) \|_{\mathrm{F}} \\
    &\leq \| [\widetilde{\mathbf{L}}^{(\ell)} \widetilde{\mathbf{H}}^{(\ell-1)}]^\top \Big(\widetilde{\mathbf{D}}^{(\ell+1)} \circ  \sigma^\prime(\widetilde{\mathbf{L}}^{(\ell)}\widetilde{\mathbf{H}}^{(\ell-1)}\mathbf{W}_1^{(\ell)})\Big) \\
    &\qquad - [\widetilde{\mathbf{L}}^{(\ell)} \widetilde{\mathbf{H}}^{(\ell-1)}]^\top \Big(\widetilde{\mathbf{D}}^{(\ell+1)} \circ  \sigma^\prime(\widetilde{\mathbf{L}}^{(\ell)}\widetilde{\mathbf{H}}^{(\ell-1)}\mathbf{W}_2^{(\ell)})\Big) \|_{\mathrm{F}} \\
    &\leq B_{LA}^2 B_H^2 H_D L_\sigma \| \mathbf{W}_1^{(\ell)} - \mathbf{W}_2^{(\ell)} \|_{\mathrm{F}}.
    \end{aligned}
\end{equation}
\end{proof}

\subsection{Lipschitz continouity of the gradient of graph convolutional network}
Let first recall the parameters and gradients of a $L$-layer GCN is defined as
\begin{equation}
\bm{\theta} = \{ \mathbf{W}^{(1)}, \ldots, \mathbf{W}^{(L)} \},~
\nabla \mathcal{L}(\bm{\theta}) = \{\mathbf{G}^{(1)}, \ldots, \mathbf{G}^{(L)}\},
\end{equation}
where $\mathbf{G}^{(\ell)}$ is defined as the gradient w.r.t. the $\ell$th layer weight matrix. Let us slightly abuse of notation and define the distance between two set of parameters $\bm{\theta}_1, \bm{\theta}_2$ and its gradient as
\begin{equation}
    \| \bm{\theta}_1-\bm{\theta}_2\|_{\mathrm{F}} = \sum_{\ell=1}^L \| \mathbf{W}_1^{(\ell)} - \mathbf{W}_2^{(\ell)}\|_{\mathrm{F}},~
    \| \nabla\mathcal{L}( \bm{\theta}_1)-\nabla \mathcal{L}(\bm{\theta}_2)\|_{\mathrm{F}} = \sum_{\ell=1}^L \| \mathbf{G}_1^{(\ell)} - \mathbf{G}_2^{(\ell)}\|_{\mathrm{F}}.
\end{equation}

Then, we derive the Lipschitz continuous constant of the gradient of a $L$-layer graph convolutonal network. 
Notice that the above result also hold for sampling-based GCN training.

\begin{lemma} \label{lemma:smoothness_L_layer}
The gradient of an $L$-layer GCN is $L_{\mathrm{F}}$-Lipschitz continuous with $L_{\mathrm{F}}=L(L U_{\max L}^2 U_{\max C}^2 + U_{\max L}^2 )$, i.e.,
\begin{equation}
    \|\nabla \mathcal{L}(\bm{\theta}_1) - \nabla \mathcal{L}(\bm{\theta}_2)\|_{\mathrm{F}}^2 \leq L_{\mathrm{F}} \| \bm{\theta}_1-\bm{\theta}_2\|_{\mathrm{F}}^2,
\end{equation}
where
\begin{equation}
\begin{aligned}
U_{\max C} &= \max\{1, C_H\}^L C_W,~ \\
U_{\max L} &= L_W \max\{ 1, L_H\}^L \times \max\{1, L_\text{loss}\}.
\end{aligned}
\end{equation}

\end{lemma}
\begin{proof}

We first consider the gradient w.r.t. the $\ell$th graph convolutional layer weight matrix
\begin{equation}
    \begin{aligned}
    &\| \mathbf{G}_1^{(\ell)} - \mathbf{G}_2^{(\ell)}\|_{\mathrm{F}} \\
    &= \|\nabla_W f^{(\ell)}( \nabla_H f^{(\ell+1)}( \ldots \nabla_H f^{(L)}( \mathbf{D}^{(L+1)}_1, \mathbf{H}_1^{(L-1)}, \mathbf{W}_1^{(L)}) \ldots, \mathbf{H}_1^{(\ell)}, \mathbf{W}_1^{(\ell+1)}) , \mathbf{H}_1^{(\ell-1)}, \mathbf{W}_1^{(\ell)}) \\
    &\quad - \nabla_W f^{(\ell)}( \nabla_H f^{(\ell+1)}( \ldots \nabla_H f^{(L)}( \mathbf{D}^{(L+1)}_2, \mathbf{H}_2^{(L-1)}, \mathbf{W}_2^{(L)}) \ldots,\mathbf{H}_2^{(\ell)} \mathbf{W}_2^{(\ell+1)}) , \mathbf{H}_2^{(\ell-1)}, \mathbf{W}_2^{(\ell)}) \|_{\mathrm{F}} \\
    &\leq \| \nabla_W f^{(\ell)}( \nabla_H f^{(\ell+1)}( \ldots \nabla_H f^{(L)}( \mathbf{D}^{(L+1)}_1, \mathbf{H}_1^{(L-1)}, \mathbf{W}_1^{(L)}) \ldots, \mathbf{H}_1^{(\ell)}, \mathbf{W}_1^{(\ell+1)}) , \mathbf{H}_1^{(\ell-1)}, \mathbf{W}_1^{(\ell)}) \\
    &\quad - \nabla_W f^{(\ell)}( \nabla_H f^{(\ell+1)}( \ldots \nabla_H f^{(L)}( \mathbf{D}^{(L+1)}_2, \mathbf{H}_1^{(L-1)}, \mathbf{W}_1^{(L)}) \ldots, \mathbf{H}_1^{(\ell)}, \mathbf{W}_1^{(\ell+1)}) , \mathbf{H}_1^{(\ell-1)}, \mathbf{W}_1^{(\ell)}) \|_{\mathrm{F}}  \\
    &\quad + \| \nabla_W f^{(\ell)}( \nabla_H f^{(\ell+1)}( \ldots \nabla_H f^{(L)}( \mathbf{D}^{(L+1)}_2, \mathbf{H}_1^{(L-1)}, \mathbf{W}_1^{(L)}) \ldots, \mathbf{H}_1^{(\ell)}, \mathbf{W}_1^{(\ell+1)}) , \mathbf{H}_1^{(\ell-1)}, \mathbf{W}_1^{(\ell)}) \\
    &\quad - \nabla_W f^{(\ell)}( \nabla_H f^{(\ell+1)}( \ldots \nabla_H f^{(L)}( \mathbf{D}^{(L+1)}_2, \mathbf{H}_2^{(L-1)}, \mathbf{W}_1^{(L)}) \ldots, \mathbf{H}_1^{(\ell)}, \mathbf{W}_1^{(\ell+1)}) , \mathbf{H}_1^{(\ell-1)}, \mathbf{W}_1^{(\ell)}) \|_{\mathrm{F}} \\ 
    &\quad + \| \nabla_W f^{(\ell)}( \nabla_H f^{(\ell+1)}( \ldots \nabla_H f^{(L)}( \mathbf{D}^{(L+1)}_2, \mathbf{H}_2^{(L-1)}, \mathbf{W}_1^{(L)}) \ldots, \mathbf{H}_1^{(\ell)}, \mathbf{W}_1^{(\ell+1)}) , \mathbf{H}_1^{(\ell-1)}, \mathbf{W}_1^{(\ell)}) \\
    &\quad - \nabla_W f^{(\ell)}( \nabla_H f^{(\ell+1)}( \ldots \nabla_H f^{(L)}( \mathbf{D}^{(L+1)}_2, \mathbf{H}_2^{(L-1)}, \mathbf{W}_2^{(L)}) \ldots, \mathbf{H}_1^{(\ell)}, \mathbf{W}_1^{(\ell+1)}) , \mathbf{H}_1^{(\ell-1)}, \mathbf{W}_1^{(\ell)}) \|_{\mathrm{F}} + \ldots \\
    &\quad + \| \nabla_W f^{(\ell)}( \mathbf{D}_2^{(\ell+1)}, \mathbf{H}_1^{(\ell-1)}, \mathbf{W}_1^{(\ell)}) - \nabla_W f^{(\ell)}( \mathbf{D}_2^{(\ell+1)}, \mathbf{H}_2^{(\ell-1)}, \mathbf{W}_1^{(\ell)}) \|_{\mathrm{F}} \\ 
    &\quad + \| \nabla_W f^{(\ell)}( \mathbf{D}_2^{(\ell+1)}, \mathbf{H}_2^{(\ell-1)}, \mathbf{W}_1^{(\ell)})  - \nabla_W f^{(\ell)}( \mathbf{D}_2^{(\ell+1)}, \mathbf{H}_2^{(\ell-1)}, \mathbf{W}_2^{(\ell)}) \|_{\mathrm{F}}.
    \end{aligned}
\end{equation}

By the Lipschitz continuity of $\nabla_W f^{(\ell)}(\cdot)$ and $\nabla_H f^{(\ell)}(\cdot)$
\begin{equation}\label{eq:G_ell_lip_cont}
    \begin{aligned}
    \| \mathbf{G}_1^{(\ell)} - \mathbf{G}_2^{(\ell)}\|_{\mathrm{F}} 
    &\leq L_W L_H^{L-\ell-1} L_\text{loss}\|\mathbf{H}^{(L)}_1 - \mathbf{H}^{(L)}_2\|_{\mathrm{F}} \\
    &\quad + L_W L_H^{L-\ell-1} (\|\mathbf{H}_1^{(L-1)} - \mathbf{H}_2^{(L-1)}\|_{\mathrm{F}} + \|\mathbf{W}_1^{(L)} - \mathbf{W}_2^{(L)}\|_{\mathrm{F}}) + \ldots \\
    &\quad + L_W L_H (\|\mathbf{H}_1^{(\ell)} - \mathbf{H}_2^{(\ell)}\|_{\mathrm{F}} + \|\mathbf{W}_1^{(\ell+1)} - \mathbf{W}_2^{(\ell+1)}\|_{\mathrm{F}}) \\
    &\quad + L_W (\|\mathbf{H}_1^{(\ell-1)} - \mathbf{H}_2^{(\ell-1)}\|_{\mathrm{F}} + \|\mathbf{W}_1^{(\ell)} - \mathbf{W}_2^{(\ell)}\|_{\mathrm{F}}).
    \end{aligned}
\end{equation}

Let define $U_{\max L}$ as
\begin{equation}
   U_{\max L} = 
L_W \cdot \max\{ 1, L_H\}^L \cdot \max\{ 1, L_\text{loss}\},
\end{equation}
then we can rewrite the above equation as
\begin{equation}
    \| \mathbf{G}_1^{(\ell)} - \mathbf{G}_2^{(\ell)}\|_{\mathrm{F}} \leq U_{\max L} \Big( \sum_{j=1}^L \| \mathbf{H}_1^{(j)} - \mathbf{H}_2^{(j)}\|_{\mathrm{F}} \Big) + U_{\max L} \Big( \sum_{j=1}^L \| \mathbf{W}_1^{(j)} - \mathbf{W}_2^{(j)}\|_{\mathrm{F}} \Big).
\end{equation}

Then, let consider the upper bound of $\|\mathbf{H}_1^{(\ell)} - \mathbf{H}_2^{(\ell)}\|_{\mathrm{F}}$
\begin{equation}
    \begin{aligned}
    \|\mathbf{H}_1^{(\ell)} - \mathbf{H}_2^{(\ell)}\|_{\mathrm{F}} &= \| f^{(\ell)}(f^{(\ell-1)}(\ldots f^{(1)}(\mathbf{X}, \mathbf{W}_1^{(1)}) \ldots, \mathbf{W}_1^{(\ell-1)}) , \mathbf{W}_1^{(\ell)}) \\
    &\quad - f^{(\ell)}(f^{(\ell-1)}(\ldots f^{(1)}(\mathbf{X}, \mathbf{W}_2^{(1)}) \ldots, \mathbf{W}_2^{(\ell-1)}) , \mathbf{W}_2^{(\ell)}) \|_{\mathrm{F}} \\
    &\leq \| f^{(\ell)}(f^{(\ell-1)}(\ldots f^{(1)}(\mathbf{X}, \mathbf{W}_1^{(1)}) \ldots, \mathbf{W}_1^{(\ell-1)}) , \mathbf{W}_1^{(\ell)}) \\
    &\quad - f^{(\ell)}(f^{(\ell-1)}(\ldots f^{(1)}(\mathbf{X}, \mathbf{W}_2^{(1)}) \ldots, \mathbf{W}_1^{(\ell-1)}) , \mathbf{W}_1^{(\ell)}) \|_{\mathrm{F}} + \ldots \\
    &\quad + \| f^{(\ell)}(f^{(\ell-1)}(\ldots f^{(1)}(\mathbf{X}, \mathbf{W}_2^{(1)}) \ldots, \mathbf{W}_2^{(\ell-1)}) , \mathbf{W}_1^{(\ell)}) \\
    &\quad - f^{(\ell)}(f^{(\ell-1)}(\ldots f^{(1)}(\mathbf{X}, \mathbf{W}_2^{(1)}) \ldots, \mathbf{W}_2^{(\ell-1)}) , \mathbf{W}_1^{(\ell)}) \|_{\mathrm{F}}.
    \end{aligned}
\end{equation}
By the Lipschitz continuity of $f^{(\ell)}(\cdot, \cdot)$ we have
\begin{equation}
    \begin{aligned}
    \|\mathbf{H}_1^{(\ell)} - \mathbf{H}_2^{(\ell)}\|_{\mathrm{F}} &\leq C_H^{L-1}C_W \|\mathbf{W}_1^{(1)} - \mathbf{W}_2^{(1)}\|_{\mathrm{F}} + \ldots + C_W \|\mathbf{W}_1^{(\ell)} - \mathbf{W}_2^{(\ell)}\|_{\mathrm{F}} \\
    &\leq U_{\max C} \Big( \sum_{j=1}^L \|\mathbf{W}_1^{(j)} - \mathbf{W}_2^{(j)}\|_{\mathrm{F}}\Big),
    \end{aligned}
\end{equation}
where $U_{\max C}=\max\{1, C_H\}^L C_W$.

Plugging it back we have
\begin{equation}
    \| \mathbf{G}_1^{(\ell)} - \mathbf{G}_2^{(\ell)}\|_{\mathrm{F}} \leq (L U_{\max L} U_{\max C} + U_{\max L} ) \Big( \sum_{j=1}^L \| \mathbf{W}_1^{(j)} - \mathbf{W}_2^{(j)}\|_{\mathrm{F}} \Big) .
\end{equation}
Summing both size from $\ell=1$ to $\ell=L$ we have
\begin{equation}
    \begin{aligned}
    \| \nabla \mathcal{L}(\bm{\theta}_1) - \nabla \mathcal{L}(\bm{\theta}_2) \|_{\mathrm{F}} 
    &= \sum_{\ell=1}^L \| \mathbf{G}_1^{(\ell)} - \mathbf{G}_2^{(\ell)}\|_{\mathrm{F}} \\
    &\leq L(L U_{\max L} U_{\max C} + U_{\max L} ) \Big( \sum_{j=1}^L \| \mathbf{W}_1^{(j)} - \mathbf{W}_2^{(j)}\|_{\mathrm{F}} \Big) \\
    &\leq L(L U_{\max L} U_{\max C} + U_{\max L} ) \|\bm{\theta}_1 - \bm{\theta}_2\|_{\mathrm{F}}.
    \end{aligned}
\end{equation}

\end{proof}

\subsection{Mean-square error of stochastic gradient}

By bias-variance decomposition, we can decompose the mean-square error of stochastic gradient as
\begin{equation}
    \sum_{\ell=1}^L \mathbb{E}[\|\widetilde{\mathbf{G}}^{(\ell)} - \mathbf{G}^{(\ell)} \|_{\mathrm{F}}^2] = \sum_{\ell=1}^L \Big[ \underbrace{\mathbb{E}[\|\mathbb{E}[\widetilde{\mathbf{G}}^{(\ell)}] - \mathbf{G}^{(\ell)} \|_{\mathrm{F}}^2]}_{\text{bias}~\mathbb{E}[\|\mathbf{b}\|_{\mathrm{F}}^2]} + \underbrace{\mathbb{E}[\|\widetilde{\mathbf{G}}^{(\ell)} - \mathbb{E}[\widetilde{\mathbf{G}}^{(\ell)}] \|_{\mathrm{F}}^2]}_{\text{variance}~\mathbb{E}[\|\mathbf{n}\|_{\mathrm{F}}^2]} \Big].
\end{equation}
Therefore, we have to explicitly define the computation of $\mathbb{E}[\widetilde{\mathbf{G}}^{(\ell)}]$, which requires computing $\bar{\mathbf{D}}^{(L+1)} = \mathbb{E}[\widetilde{\mathbf{D}}^{(L+1)}]$, $\bar{\mathbf{D}}^{(\ell)} = \mathbb{E}[\widetilde{\mathbf{D}}^{(\ell)}]$, and $\bar{\mathbf{G}}^{(\ell)} = \mathbb{E}[\widetilde{\mathbf{G}}^{(\ell)}]$.

Let defined a general form of the sampled Laplacian matrix $\widetilde{\mathbf{L}}^{(\ell)} \in\mathbb{R}^{N\times N}$ as
\begin{equation}
    \widetilde{L}_{i,j}^{(\ell)} = \begin{cases}
    \frac{L_{i,j}}{\alpha_{i,j}} & \text{ if } i\in\mathcal{B}^{(\ell)}~\text{and}~j\in\mathcal{B}^{(\ell-1)} \\ 
    0 & \text{ otherwise } 
    \end{cases},
\end{equation}
where $\alpha_{i,j}$ is the weighted constant depends on the sampling algorithms.

The expectation of $\widetilde{L}_{i,j}^{(\ell)}$ is computed as
\begin{equation}
    \mathbb{E}[\widetilde{L}_{i,j}^{(\ell)}] = \mathbb{E}_{i\in\mathcal{B}^{(\ell)}} \Big[ \mathbb{E}_{j\in\mathcal{B}^{(\ell-1)}}[\widetilde{L}_{i,j}^{(\ell)}~|~i\in\mathcal{B}^{(\ell)}] \Big].
\end{equation}

In order to compute the expectation of \texttt{SGCN}'s node embedding matrices,
let define the propagation matrix $\mathbf{P}^{(\ell)}\in\mathbb{R}^{N\times N}$ as
\begin{equation}
    P_{i,j}^{(\ell)} = \mathbb{E}_{i\in\mathcal{B}^{(\ell)}} \Big[ \widetilde{L}_{i,j}^{(\ell)}~|~i\in\mathcal{B}^{(\ell)} \Big],
\end{equation}
where the expectation is taken over row indices $i$. The above equation implies that under the condition that knowing the $i$th node is in $\mathcal{B}^{(\ell)}$, we have $P_{i,j}^{(\ell)} = \widetilde{L}_{i,j},~\forall j=\{1,\ldots, N\}$. Let consider the mean-aggregation for the $i$th node as
\begin{equation}
    \mathbf{x}_i^{(\ell)} = \sigma\Big(\sum_{j=1}^N \widetilde{L}^{(\ell)}_{i,j} \mathbf{x}_j^{(\ell-1)}\Big).
\end{equation}
Then, under the condition $i$th node is in $\mathcal{B}^{(\ell)}$, we can replace $\widetilde{L}^{(\ell)}_{i,j}$ by $P^{(\ell)}_{i,j}$, which gives us
\begin{equation}
    \mathbf{x}_i^{(\ell)} = \sigma\Big(\sum_{j=1}^N P^{(\ell)}_{i,j} \mathbf{x}_j^{(\ell-1)}\Big)
\end{equation}
As a result, we can write the expectation of $\mathbf{x}_i^{(\ell)}$ with respect to the indices $i$ as
\begin{equation}
    \begin{aligned}
    \mathbb{E}_{i\in\mathcal{B}^{(\ell)}}[ \mathbf{x}_i^{(\ell)} ~|~i\in\mathcal{B}^{(\ell)}] 
    &= \mathbb{E}_{i\in\mathcal{B}^{(\ell)}}\Big[ \sigma\Big(\sum_{j=1}^N \widetilde{L}^{(\ell)}_{i,j} \mathbf{x}_j^{(\ell-1)}\Big) ~|~i\in\mathcal{B}^{(\ell)}\Big] \\
    &= \mathbb{E}_{i\in\mathcal{B}^{(\ell)}}\Big[ \sigma\Big(\sum_{j=1}^N P^{(\ell)}_{i,j} \mathbf{x}_j^{(\ell-1)}\Big) ~|~i\in\mathcal{B}^{(\ell)}\Big] \\
    &= \sigma\Big(\sum_{j=1}^N P^{(\ell)}_{i,j} \mathbf{x}_j^{(\ell-1)}\Big).
    \end{aligned}
\end{equation}


Then define $\bar{\mathbf{H}}^{(\ell)} \in \mathbb{R}^{N \times d_\ell} $ as the node embedding of using full-batch but a subset of neighbors for neighbor aggregation, i.e.,
\begin{equation}
    \bar{\mathbf{H}}^{(\ell)} = \sigma(\mathbf{P}^{(\ell)} \bar{\mathbf{H}}^{(\ell-1)} \mathbf{W}^{(\ell)}),
\end{equation}
where all rows in $\bar{\mathbf{H}}^{(\ell)}$ are non-zero vectors. 

Using the notations defined above, we can compute $\bar{\mathbf{D}}^{(L+1)} \in\mathbb{R}^{N \times d_L}$, $\bar{\mathbf{G}}^{(\ell)} \in \mathbb{R}^{d_{\ell-1} \times d_\ell}$, and $\bar{\mathbf{D}}^{(\ell)} \in \mathbb{R}^{N \times d_{\ell-1}}$ as
\begin{equation}
    \bar{\mathbf{D}}^{(L+1)} = \mathbb{E} \Big[ \frac{\partial \text{Loss}(\bar{\mathbf{H}}^{(L)}) }{\partial \bar{\mathbf{H}}^{(L)}} \Big] \in\mathbb{R}^{N \times d_L}
    , \bar{\mathbf{d}}_i = \frac{1}{N} \frac{\partial \text{Loss}(\bar{\mathbf{h}}_i^{(L)}, y_i)}{\partial \bar{\mathbf{h}}_i^{(L)}} \in \mathbb{R}^{d_L},
\end{equation}
and
\begin{equation}
    \bar{\mathbf{D}}^{(\ell)} = \Big[\nabla_H \bar{f}^{(\ell)}(\bar{\mathbf{D}}^{(\ell+1)}, \bar{\mathbf{H}}^{(\ell-1)}, \mathbf{W}^{(\ell)}) = [\mathbf{L}]^\top \Big(\bar{\mathbf{D}}^{(\ell+1)} \circ \sigma^\prime(\mathbf{P}^{(\ell)} \bar{\mathbf{H}}^{(\ell-1)} \mathbf{W}^{(\ell)}) \Big)  [\mathbf{W}^{(\ell)}]^\top \Big] ,
\end{equation}
and
\begin{equation}
    \bar{\mathbf{G}}^{(\ell)} = \Big[\nabla_W \bar{f}^{(\ell)}(\bar{\mathbf{D}}^{(\ell+1)}, \bar{\mathbf{H}}^{(\ell-1)}, \mathbf{W}^{(\ell)}) =  [\mathbf{L} \bar{\mathbf{H}}^{(\ell-1)}]^\top \Big(\bar{\mathbf{D}}^{(\ell+1)} \circ \sigma^\prime(\mathbf{P}^{(\ell)} \bar{\mathbf{H}}^{(\ell-1)} \mathbf{W}^{(\ell)}) \Big) \Big] .
\end{equation}
As a result, we can represent $\bar{\mathbf{G}}^{(\ell)} = \mathbb{E}[\widetilde{\mathbf{G}}^{(\ell)}]$ as 
\begin{equation} \label{eq:G_bar_SGCN_plus}
    \begin{aligned}
    \bar{\mathbf{G}}^{(\ell)} = \nabla_W \bar{f}^{(\ell)}( \nabla_H \bar{f}^{(\ell+1)}( \ldots \nabla_H \bar{f}^{(L)}( \bar{\mathbf{D}}^{(L+1)}, \bar{\mathbf{H}}^{(L-1)}, \mathbf{W}^{(L)}) \ldots, \bar{\mathbf{H}}^{(\ell)}, \mathbf{W}^{(\ell+1)}) , \bar{\mathbf{H}}^{(\ell-1)}, \mathbf{W}^{(\ell)})).
    \end{aligned}
\end{equation}

\subsection{Supporting lemmas}
We derive the upper-bound of the bias and variance of the stochastic gradient in the following lemmas.
\begin{lemma} [Upper-bound on variance] \label{lemma:upper-bound-sgcn-vars}
We can upper-bound the variance of stochastic gradient in \texttt{SGCN} as
\begin{equation}
    \sum_{\ell=1}^L \mathbb{E}[\|\widetilde{\mathbf{G}}^{(\ell)} - \mathbb{E}[\widetilde{\mathbf{G}}^{(\ell)}] \|_{\mathrm{F}}^2]
    \leq \sum_{\ell=1}^L \mathcal{O}(\mathbb{E}[\| \widetilde{\mathbf{L}}^{(\ell)} - \mathbf{P}^{(\ell)} \|_{\mathrm{F}}^2]) + \mathcal{O}(\mathbb{E}[ \|\mathbf{P}^{(\ell)} - \mathbf{L}\|_{\mathrm{F}}^2] )
\end{equation}
\end{lemma}
\begin{proof}
By definition, we can write the variance in \texttt{SGCN} as
\begin{equation}
    \begin{aligned}
    &\mathbb{E}[ \| \widetilde{\mathbf{G}}^{(\ell)} - \mathbb{E}[\widetilde{\mathbf{G}}^{(\ell)}] \|_{\mathrm{F}}^2 ] \\
    &= \mathbb{E}[ \| \nabla_W \widetilde{f}^{(\ell)}( \nabla_H \widetilde{f}^{(\ell+1)}( \ldots \nabla_H \widetilde{f}^{(L)}( \widetilde{\mathbf{D}}^{(L+1)}, \widetilde{\mathbf{H}}^{(L-1)} \mathbf{W}^{(L)}) \ldots, \widetilde{\mathbf{H}}^{(\ell)}, \mathbf{W}^{(\ell+1)}) , \widetilde{\mathbf{H}}^{(\ell-1)}, \mathbf{W}^{(\ell)}) \\
    &\quad - \nabla_W \bar{f}^{(\ell)}( \nabla_H \bar{f}^{(\ell+1)}( \ldots \nabla_H \bar{f}^{(L)}( \bar{\mathbf{D}}^{(L+1)}, \bar{\mathbf{H}}^{(L-1)}, \mathbf{W}^{(L)}) \ldots, \bar{\mathbf{H}}^{(\ell)}, \mathbf{W}^{(\ell+1)}) , \bar{\mathbf{H}}^{(\ell-1)}), \mathbf{W}^{(\ell)} \|_{\mathrm{F}}^2 ] \\
    &\leq (L+1) \mathbb{E}[  \| \nabla_W \widetilde{f}^{(\ell)}( \nabla_H \widetilde{f}^{(\ell+1)}( \ldots \nabla_H \widetilde{f}^{(L)}( \widetilde{\mathbf{D}}^{(L+1)}, \widetilde{\mathbf{H}}^{(L-1)} \mathbf{W}^{(L)}) \ldots, \widetilde{\mathbf{H}}^{(\ell)}, \mathbf{W}^{(\ell+1)}) , \widetilde{\mathbf{H}}^{(\ell-1)}, \mathbf{W}^{(\ell)}) \\
    &\qquad - \nabla_W \widetilde{f}^{(\ell)}( \nabla_H \widetilde{f}^{(\ell+1)}( \ldots \nabla_H \widetilde{f}^{(L)}( \bar{\mathbf{D}}^{(L+1)}, \widetilde{\mathbf{H}}^{(L-1)} \mathbf{W}^{(L)}) \ldots, \bar{\mathbf{H}}^{(\ell)}, \mathbf{W}^{(\ell+1)}) , \bar{\mathbf{H}}^{(\ell-1)}, \mathbf{W}^{(\ell)}) \|_{\mathrm{F}}^2 ]\\
    &\quad + (L+1) \mathbb{E}[ \| \nabla_W \widetilde{f}^{(\ell)}( \nabla_H \widetilde{f}^{(\ell+1)}( \ldots \nabla_H \widetilde{f}^{(L)}( \bar{\mathbf{D}}^{(L+1)}, \widetilde{\mathbf{H}}^{(L-1)} \mathbf{W}^{(L)}) \ldots, \widetilde{\mathbf{H}}^{(\ell)}, \mathbf{W}^{(\ell+1)}) , \widetilde{\mathbf{H}}^{(\ell-1)}, \mathbf{W}^{(\ell)}) \\
    &\qquad - \nabla_W \widetilde{f}^{(\ell)}( \nabla_H \widetilde{f}^{(\ell+1)}( \ldots \nabla_H \bar{f}^{(L)}( \bar{\mathbf{D}}^{(L+1)}, \bar{\mathbf{H}}^{(L-1)} \mathbf{W}^{(L)}) \ldots, \widetilde{\mathbf{H}}^{(\ell)}, \mathbf{W}^{(\ell+1)}) , \widetilde{\mathbf{H}}^{(\ell-1)}, \mathbf{W}^{(\ell)}) \|_{\mathrm{F}}^2 ] + \ldots \\
    &\quad + (L+1) \mathbb{E}[ \| \nabla_W \widetilde{f}^{(\ell)}( \nabla_H \widetilde{f}^{(\ell+1)}( \bar{\mathbf{D}}^{(\ell+2)}, \widetilde{\mathbf{H}}^{(\ell)}, \mathbf{W}^{(\ell+1)}) , \widetilde{\mathbf{H}}^{(\ell-1)}, \mathbf{W}^{(\ell)})  \\
    &\qquad - \nabla_W \widetilde{f}^{(\ell)}( \nabla_H \bar{f}^{(\ell+1)}( \bar{\mathbf{D}}^{(\ell+2)}, \bar{\mathbf{H}}^{(\ell)}, \mathbf{W}^{(\ell+1)}) , \widetilde{\mathbf{H}}^{(\ell-1)}, \mathbf{W}^{(\ell)}) \|_{\mathrm{F}}^2 ] \\
    &\quad + (L+1) \mathbb{E}[ \| \nabla_W \widetilde{f}^{(\ell)}( \bar{\mathbf{D}}^{(\ell+1)}, \widetilde{\mathbf{H}}^{(\ell-1)}, \mathbf{W}^{(\ell)}) - \nabla_W \bar{f}^{(\ell)}( \bar{\mathbf{D}}^{(\ell+1)}, \bar{\mathbf{H}}^{(\ell-1)}, \mathbf{W}^{(\ell)}) \|_{\mathrm{F}}^2 ] \\
    &\leq (L+1) L_W^2 L_H^{2(L-\ell-1)} \mathbb{E}[ \|\widetilde{\mathbf{D}}^{(L+1)} - \bar{\mathbf{D}}^{(L+1)} \|_{\mathrm{F}}^2 ] \\
    &\qquad + (L+1) L_W^2 L_H^{2(L-\ell-2)} \mathbb{E}[ \| \nabla_H \widetilde{f}^{(L)}( \bar{\mathbf{D}}^{(L+1)}, \widetilde{\mathbf{H}}^{(L-1)} \mathbf{W}^{(L)}) - \nabla_H \bar{f}^{(L)}( \bar{\mathbf{D}}^{(L+1)}, \bar{\mathbf{H}}^{(L-1)} \mathbf{W}^{(L)}) \|_{\mathrm{F}}^2 ] + \ldots \\
    &\qquad + (L+1) L_W^2 \mathbb{E}[ \| \nabla_H \widetilde{f}^{(\ell+1)}( \bar{\mathbf{D}}^{(\ell+2)}, \widetilde{\mathbf{H}}^{(\ell)}, \mathbf{W}^{(\ell+1)}) - \nabla_H \bar{f}^{(\ell+1)}( \bar{\mathbf{D}}^{(\ell+2)}, \bar{\mathbf{H}}^{(\ell)}, \mathbf{W}^{(\ell+1)})\|_{\mathrm{F}}^2 ] \\
    &\qquad + (L+1) \mathbb{E}[ \| \nabla_W \widetilde{f}^{(\ell)}( \bar{\mathbf{D}}^{(\ell+1)}, \widetilde{\mathbf{H}}^{(\ell-1)}, \mathbf{W}^{(\ell)}) - \nabla_W \bar{f}^{(\ell)}( \bar{\mathbf{D}}^{(\ell+1)}, \bar{\mathbf{H}}^{(\ell-1)}, \mathbf{W}^{(\ell)}) \|_{\mathrm{F}}^2 ] .
    \end{aligned}
\end{equation}

From the previous equation, we know that there are three key factors that will affect the variance:
\begin{itemize}
    \item The difference of gradient with respect to the last layer node representations 
    \begin{equation}\label{eq:sgcn_vars_important_eq1}
        \mathbb{E}[ \|\widetilde{\mathbf{D}}^{(L+1)} - \bar{\mathbf{D}}^{(L+1)} \|_{\mathrm{F}}^2 ].
    \end{equation}
    \item The difference of gradient with respect to the input node embedding matrix at each graph convolutional layer 
    \begin{equation}\label{eq:sgcn_vars_important_eq2}
        \mathbb{E}[ \| \nabla_H \widetilde{f}^{(\ell+1)}( \bar{\mathbf{D}}^{(\ell+2)}, \widetilde{\mathbf{H}}^{(\ell)}, \mathbf{W}^{(\ell+1)}) - \nabla_H \bar{f}^{(\ell+1)}( \bar{\mathbf{D}}^{(\ell+2)}, \bar{\mathbf{H}}^{(\ell)}, \mathbf{W}^{(\ell+1)})\|_{\mathrm{F}}^2 ].
    \end{equation}
    \item The difference of gradient with respect to the weight matrix at each graph convolutional layer 
    \begin{equation}\label{eq:sgcn_vars_important_eq3}
        \mathbb{E}[ \| \nabla_W \widetilde{f}^{(\ell)}( \bar{\mathbf{D}}^{(\ell+1)}, \widetilde{\mathbf{H}}^{(\ell-1)}, \mathbf{W}^{(\ell)}) - \nabla_W \bar{f}^{(\ell)}( \bar{\mathbf{D}}^{(\ell+1)}, \bar{\mathbf{H}}^{(\ell-1)}, \mathbf{W}^{(\ell)}) \|_{\mathrm{F}}^2 ].
    \end{equation}
\end{itemize}

First, let consider the upper-bound of Eq.~\ref{eq:sgcn_vars_important_eq1}.
\begin{equation}
    \begin{aligned}
    \mathbb{E} [ \| \widetilde{\mathbf{D}}^{(L+1)} - \bar{\mathbf{D}}^{(L+1)} \|_{\mathrm{F}}^2 ] &= \mathbb{E} \left[ \left\| \frac{\partial \text{Loss}(\widetilde{\mathbf{H}}^{(L)}, \mathbf{y}) }{\partial \widetilde{\mathbf{H}}^{(L)}} - \frac{\partial \text{Loss}(\bar{\mathbf{H}}^{(L)}, \mathbf{y}) }{\partial \bar{\mathbf{H}}^{(L)}} \right\|_{\mathrm{F}}^2 \right] \\
    &\leq L_{loss}^2 \mathbb{E}[ \| \widetilde{\mathbf{H}}^{(L)} -\bar{\mathbf{H}}^{(L)} \|_{\mathrm{F}}^2 ] \\
    &\leq L_{loss}^2 \mathbb{E}[ \| \sigma(\widetilde{\mathbf{L}}^{(L)} \widetilde{\mathbf{H}}^{(L-1)} \mathbf{W}^{(L)}) - \sigma(\mathbf{P}^{(\ell)}  \bar{\mathbf{H}}^{(L-1)} \mathbf{W}^{(L)}) \|_{\mathrm{F}}^2 ] \\
    &\leq L_{loss}^2 C_\sigma^2 B_W^2 \mathbb{E}[ \| \widetilde{\mathbf{L}}^{(L)} \widetilde{\mathbf{H}}^{(L-1)} - \mathbf{P}^{(L)}  \bar{\mathbf{H}}^{(L-1)}\|_{\mathrm{F}}^2 ] \\
    &\leq  L_{loss}^2 C_\sigma^2 B_W^2 B_H^2 \mathbb{E}[ \| \widetilde{\mathbf{L}}^{(L)} - \mathbf{P}^{(L)} \|_{\mathrm{F}}^2 ].
    \end{aligned}
\end{equation}

Then, let consider the upper-bound of Eq.~\ref{eq:sgcn_vars_important_eq2}.
\begin{equation}
    \begin{aligned}
    &\mathbb{E}[\| \nabla_H \widetilde{f}^{(\ell)}(\bar{\mathbf{D}}^{(\ell+1)}, \widetilde{\mathbf{H}}^{(\ell-1)}, \mathbf{W}^{(\ell)}) - \nabla_H \bar{f}^{(\ell)}(\bar{\mathbf{D}}^{(\ell+1)}, \bar{\mathbf{H}}^{(\ell-1)}, \mathbf{W}^{(\ell)})] \|_{\mathrm{F}}^2] \\
    &=\mathbb{E}[\| [\widetilde{\mathbf{L}}^{(\ell)} ]^\top \Big(\bar{\mathbf{D}}^{(\ell+1)} \circ \sigma^\prime(\widetilde{\mathbf{L}}^{(\ell)} \widetilde{\mathbf{H}}^{(\ell-1)} \mathbf{W}^{(\ell)}) \Big) [\mathbf{W}^{(\ell)}]^\top - [\mathbf{L} ]^\top \Big(\bar{\mathbf{D}}^{(\ell+1)} \circ \sigma^\prime(\mathbf{P}^{(\ell)} \bar{\mathbf{H}}^{(\ell-1)} \mathbf{W}^{(\ell)} ) \Big) [\mathbf{W}^{(\ell)}]^\top \|_{\mathrm{F}}^2] \\
    &\leq 2\mathbb{E}[\| [\widetilde{\mathbf{L}}^{(\ell)} ]^\top \Big(\bar{\mathbf{D}}^{(\ell+1)} \circ \sigma^\prime(\widetilde{\mathbf{L}}^{(\ell)} \widetilde{\mathbf{H}}^{(\ell-1)} \mathbf{W}^{(\ell)}) \Big) [\mathbf{W}^{(\ell)}]^\top - [\widetilde{\mathbf{L}}^{(\ell)} ]^\top \Big(\bar{\mathbf{D}}^{(\ell+1)} \circ \sigma^\prime(\mathbf{P}^{(\ell)} \bar{\mathbf{H}}^{(\ell-1)} \mathbf{W}^{(\ell)} ) \Big) [\mathbf{W}^{(\ell)}]^\top \|_{\mathrm{F}}^2] \\
    &\quad + 2\mathbb{E}[\| [\widetilde{\mathbf{L}}^{(\ell)} ]^\top \Big(\bar{\mathbf{D}}^{(\ell+1)} \circ \sigma^\prime(\mathbf{P}^{(\ell)} \bar{\mathbf{H}}^{(\ell-1)} \mathbf{W}^{(\ell)} ) \Big) [\mathbf{W}^{(\ell)}]^\top - [\mathbf{L} ]^\top \Big(\bar{\mathbf{D}}^{(\ell+1)} \circ \sigma^\prime(\mathbf{P}^{(\ell)} \bar{\mathbf{H}}^{(\ell-1)} \mathbf{W}^{(\ell)} ) \Big) [\mathbf{W}^{(\ell)}]^\top \|_{\mathrm{F}}^2] \\
    &\leq 2B_{LA}^2 B_D^2 B_W^4 L_\sigma^2 \mathbb{E}[\| \widetilde{\mathbf{L}}^{(\ell)} \widetilde{\mathbf{H}}^{(\ell-1)} - \mathbf{P}^{(\ell)} \bar{\mathbf{H}}^{(\ell-1)} \|_{\mathrm{F}}^2] + 2B_D^2 C_\sigma^2 B_W^2 \mathbb{E}[\| \widetilde{\mathbf{L}}^{(\ell)} - \mathbf{L} \|_{\mathrm{F}}^2] \\
    &\leq 2B_{LA}^2 B_D^2 B_H^2 B_W^4 L_\sigma^2 \mathbb{E}[\| \widetilde{\mathbf{L}}^{(\ell)} - \mathbf{P}^{(\ell)}  \|_{\mathrm{F}}^2] + 2B_D^2 C_\sigma^2 B_W^2 \mathbb{E}[\| \widetilde{\mathbf{L}}^{(\ell)} - \mathbf{P}^{(\ell)} + \mathbf{P}^{(\ell)} - \mathbf{L} \|_{\mathrm{F}}^2] \\
    &\leq 2\Big( B_{LA}^2 B_D^2 B_H^2 B_W^4 L_\sigma^2 + 2B_D^2 C_\sigma^2 B_W^2 \Big) \mathbb{E}[\| \widetilde{\mathbf{L}}^{(\ell)} - \mathbf{P}^{(\ell)}  \|_{\mathrm{F}}^2] + 4B_D^2 C_\sigma^2 B_W^2 \mathbb{E}[\| \mathbf{P}^{(\ell)} - \mathbf{L} \|_{\mathrm{F}}^2] \\ 
    &\leq \mathcal{O}(\mathbb{E}[\| \widetilde{\mathbf{L}}^{(\ell)} - \mathbf{P}^{(\ell)}  \|_{\mathrm{F}}^2]) + \mathcal{O}(\mathbb{E}[\| \mathbf{P}^{(\ell)} - \mathbf{L} \|_{\mathrm{F}}^2]).
    \end{aligned}
\end{equation}

Finally, let consider the upper-bound of Eq.~\ref{eq:sgcn_vars_important_eq3}.
\begin{equation}
    \begin{aligned}
    &\mathbb{E}[\| \nabla_W \widetilde{f}^{(\ell)}(\bar{\mathbf{D}}^{(\ell+1)}, \widetilde{\mathbf{H}}^{(\ell-1)}, \mathbf{W}^{(\ell)}) - \nabla_W \bar{f}^{(\ell)}(\bar{\mathbf{D}}^{(\ell+1)}, \bar{\mathbf{H}}^{(\ell-1)}, \mathbf{W}^{(\ell)}) \|_{\mathrm{F}}^2] \\
    &\leq \mathbb{E}[\| [\widetilde{\mathbf{L}}^{(\ell)} \widetilde{\mathbf{H}}^{(\ell-1)}]^\top \Big(\bar{\mathbf{D}}^{(\ell+1)} \circ \sigma^\prime(\widetilde{\mathbf{L}}^{(\ell)} \widetilde{\mathbf{H}}^{(\ell-1)} \mathbf{W}^{(\ell)} ) \Big) - [\mathbf{L} \bar{\mathbf{H}}^{(\ell-1)}]^\top \Big(\bar{\mathbf{D}}^{(\ell+1)} \circ \sigma^\prime(\mathbf{P}^{(\ell)} \bar{\mathbf{H}}^{(\ell-1)} \mathbf{W}^{(\ell)} ) \Big) \|_{\mathrm{F}}^2] \\
    &\leq 2\mathbb{E}[\| [\widetilde{\mathbf{L}}^{(\ell)} \widetilde{\mathbf{H}}^{(\ell-1)}]^\top \Big(\bar{\mathbf{D}}^{(\ell+1)} \circ \sigma^\prime(\widetilde{\mathbf{L}}^{(\ell)} \widetilde{\mathbf{H}}^{(\ell-1)} \mathbf{W}^{(\ell)} ) \Big) - [\widetilde{\mathbf{L}}^{(\ell)} \widetilde{\mathbf{H}}^{(\ell-1)}]^\top \Big(\bar{\mathbf{D}}^{(\ell+1)} \circ \sigma^\prime(\mathbf{P}^{(\ell)} \bar{\mathbf{H}}^{(\ell-1)} \mathbf{W}^{(\ell)} ) \Big) \|_{\mathrm{F}}^2] \\
    &\quad + 2\mathbb{E}[\| [\widetilde{\mathbf{L}}^{(\ell)} \widetilde{\mathbf{H}}^{(\ell-1)}]^\top \Big(\bar{\mathbf{D}}^{(\ell+1)} \circ \sigma^\prime(\mathbf{P}^{(\ell)} \bar{\mathbf{H}}^{(\ell-1)} \mathbf{W}^{(\ell)} ) \Big) - [\mathbf{L} \bar{\mathbf{H}}^{(\ell-1)}]^\top \Big(\bar{\mathbf{D}}^{(\ell+1)} \circ \sigma^\prime(\mathbf{P}^{(\ell)} \bar{\mathbf{H}}^{(\ell-1)} \mathbf{W}^{(\ell)} ) \Big)\|_{\mathrm{F}}^2] \\
    &\leq 2B_{LA}^2 B_H^2 B_D^2 B_W^2 L_\sigma^2 \mathbb{E}[\| \widetilde{\mathbf{L}}^{(\ell)} \widetilde{\mathbf{H}}^{(\ell-1)} - \mathbf{P}^{(\ell)} \bar{\mathbf{H}}^{(\ell-1)} \|_{\mathrm{F}}^2] + 2B_D^2 C_\sigma^2 \mathbb{E}[\|\widetilde{\mathbf{L}}^{(\ell)} \widetilde{\mathbf{H}}^{(\ell-1)} - \mathbf{L}\bar{\mathbf{H}}^{(\ell-1)} \|_{\mathrm{F}}^2] \\
    &\leq 2\Big( B_{LA}^2 B_H^2 B_D^2 B_W^2 L_\sigma^2 + B_D^2 C_\sigma^2 \Big) \mathbb{E}[\| \widetilde{\mathbf{L}}^{(\ell)} \widetilde{\mathbf{H}}^{(\ell-1)} - \mathbf{P}^{(\ell)} \bar{\mathbf{H}}^{(\ell-1)} \|_{\mathrm{F}}^2] \\
    &\qquad + 2B_D^2 B_H^2 C_\sigma^2 \mathbb{E}[\| \mathbf{P}^{(\ell)} - \mathbf{L} \|_{\mathrm{F}}^2] \\
    &\leq 2\Big( B_{LA}^2 B_H^4 B_D^2 B_W^2 L_\sigma^2 + B_H^2 B_D^2 C_\sigma^2 \Big) \mathbb{E}[\| \widetilde{\mathbf{L}}^{(\ell)} - \mathbf{P}^{(\ell)} \|_{\mathrm{F}}^2] \\
    &\qquad + 2B_D^2 B_H^2 C_\sigma^2 \mathbb{E}[\| \mathbf{P}^{(\ell)} - \mathbf{L} \|_{\mathrm{F}}^2] \\
    &\leq \mathcal{O}(\mathbb{E}[\| \widetilde{\mathbf{L}}^{(\ell)} - \mathbf{P}^{(\ell)} \|_{\mathrm{F}}^2]) + \mathcal{O}(\mathbb{E}[\| \mathbf{P}^{(\ell)} - \mathbf{L} \|_{\mathrm{F}}^2]).
    \end{aligned}
\end{equation}

Combining the result from Eq.~\ref{eq:sgcn_vars_important_eq1}, \ref{eq:sgcn_vars_important_eq2}, \ref{eq:sgcn_vars_important_eq3} we have
\begin{equation}
    \begin{aligned}
    \mathbb{E}[ \| \widetilde{\mathbf{G}}^{(\ell)} - \mathbb{E}[\widetilde{\mathbf{G}}^{(\ell)}] \|_{\mathrm{F}}^2 ] 
    &\leq  \mathcal{O}(\mathbb{E}[\| \widetilde{\mathbf{L}}^{(\ell)} - \mathbf{P}^{(\ell)} \|_{\mathrm{F}}^2]) + \ldots + \mathcal{O}(\mathbb{E}[\| \widetilde{\mathbf{L}}^{(L)} - \mathbf{P}^{(L)} \|_{\mathrm{F}}^2]) \\
    &\qquad + \mathcal{O}(\mathbb{E}[\| \mathbf{P}^{(\ell)} - \mathbf{L} \|_{\mathrm{F}}^2 ]) +\ldots + \mathcal{O}(\mathbb{E}[\| \mathbf{P}^{(L)} - \mathbf{L} \|_{\mathrm{F}}^2 ] ).
    \end{aligned}
\end{equation}
\end{proof}

\begin{lemma} [Upper-bound on bias] \label{lemma:upper-bound-sgcn-bias}
We can upper-bound the bias of stochastic gradient in \texttt{SGCN} as
\begin{equation}
    \sum_{\ell=1}^L \mathbb{E}[\|\mathbb{E}[\widetilde{\mathbf{G}}^{(\ell)}] - \mathbf{G}^{(\ell)} \|_{\mathrm{F}}^2] 
    \leq \sum_{\ell=1}^L \mathcal{O}(\|\mathbf{P}^{(\ell)} - \mathbf{L}\|_{\mathrm{F}}^2) .
\end{equation}
\end{lemma}
\begin{proof}
By definition, we can write the bias of stochastic gradient in \texttt{SGCN} as
\begin{equation}\label{eq:decompose_bias_into_layers}
    \begin{aligned}
    &\mathbb{E}[ \| \mathbb{E}[\widetilde{\mathbf{G}}^{(\ell)}] - \mathbf{G}^{(\ell)}\|_{\mathrm{F}}^2 ]\\
    &\leq \mathbb{E}[ \| \nabla_W \bar{f}^{(\ell)}( \nabla_H \bar{f}^{(\ell+1)}( \ldots \nabla_H \bar{f}^{(L)}( \bar{\mathbf{D}}^{(L+1)}, \bar{\mathbf{H}}^{(L-1)}, \mathbf{W}^{(L)}) \ldots, \bar{\mathbf{H}}^{(\ell)}, \mathbf{W}^{(\ell+1)}) , \bar{\mathbf{H}}^{(\ell-1)}, \mathbf{W}^{(\ell)}) \\
    &\qquad - \nabla_W f^{(\ell)}( \nabla_H f^{(\ell+1)}( \ldots \nabla_H f^{(L)}( \mathbf{D}^{(L+1)}, \mathbf{H}^{(L-1)}, \mathbf{W}^{(L)}) \ldots,  \mathbf{H}^{(\ell)}, \mathbf{W}^{(\ell+1)}) , \mathbf{H}^{(\ell-1)}, \mathbf{W}^{(\ell)}) \|_{\mathrm{F}}^2 ]\\
    &\leq (L+1) \mathbb{E}[ \| \nabla_W \bar{f}^{(\ell)}( \nabla_H \bar{f}^{(\ell+1)}( \ldots \nabla_H \bar{f}^{(L)}( \bar{\mathbf{D}}^{(L+1)}, \bar{\mathbf{H}}^{(L-1)}, \mathbf{W}^{(L)}) \ldots, \bar{\mathbf{H}}^{(\ell)}, \mathbf{W}^{(\ell+1)}) , \bar{\mathbf{H}}^{(\ell-1)}, \mathbf{W}^{(\ell)}) \\
    &\qquad - \nabla_W \bar{f}^{(\ell)}( \nabla_H \bar{f}^{(\ell+1)}( \ldots \nabla_H \bar{f}^{(L)}( \mathbf{D}^{(L+1)}, \bar{\mathbf{H}}^{(L-1)}, \mathbf{W}^{(L)}) \ldots, \bar{\mathbf{H}}^{(\ell)}, \mathbf{W}^{(\ell+1)}) , \bar{\mathbf{H}}^{(\ell-1)}, \mathbf{W}^{(\ell)})  \|_{\mathrm{F}}^2 ]\\
    &\quad + (L+1) \mathbb{E}[ \| \nabla_W \bar{f}^{(\ell)}( \nabla_H \bar{f}^{(\ell+1)}( \ldots \nabla_H \bar{f}^{(L)}( \mathbf{D}^{(L+1)}, \bar{\mathbf{H}}^{(L-1)}, \mathbf{W}^{(L)}) \ldots, \bar{\mathbf{H}}^{(\ell)}, \mathbf{W}^{(\ell+1)}) , \bar{\mathbf{H}}^{(\ell-1)}, \mathbf{W}^{(\ell)}) \\
    &\qquad - \nabla_W \bar{f}^{(\ell)}( \nabla_H \bar{f}^{(\ell+1)}( \ldots \nabla_H f^{(L)}( \mathbf{D}^{(L+1)}, \mathbf{H}^{(L-1)}, \mathbf{W}^{(L)}) \ldots, \bar{\mathbf{H}}^{(\ell)}, \mathbf{W}^{(\ell+1)}) , \bar{\mathbf{H}}^{(\ell-1)}, \mathbf{W}^{(\ell)}) \|_{\mathrm{F}}^2 ] + \ldots \\
    &\quad + (L+1) \mathbb{E}[ \| \nabla_W \bar{f}^{(\ell)}( \nabla_H \bar{f}^{(\ell+1)}( \mathbf{D}^{(\ell+2)}, \bar{\mathbf{H}}^{(\ell)}, \mathbf{W}^{(\ell+1)}) , \bar{\mathbf{H}}^{(\ell-1)}, \mathbf{W}^{(\ell)}) \\
    &\qquad - \nabla_W \bar{f}^{(\ell)}( \nabla_H f^{(\ell+1)}( \mathbf{D}^{(\ell+2)}, \mathbf{H}^{(\ell)}, \mathbf{W}^{(\ell+1)}) , \bar{\mathbf{H}}^{(\ell-1)}, \mathbf{W}^{(\ell)}) \|_{\mathrm{F}}^2 ] \\
    &\quad + (L+1) \mathbb{E}[ \| \nabla_W \bar{f}^{(\ell)}( \mathbf{D}^{(\ell+1)}, \bar{\mathbf{H}}^{(\ell-1)}, \mathbf{W}^{(\ell)}) - \nabla_W f^{(\ell)}( \mathbf{D}^{(\ell+1)}, \mathbf{H}^{(\ell-1)}, \mathbf{W}^{(\ell)}) \|_{\mathrm{F}}^2 ]\\
    &\leq (L+1) L_W^2 L_H^{2(L-\ell-1)} \mathbb{E}[ \| \bar{\mathbf{D}}^{(L+1)} -\mathbf{D}^{(L+1)} \|_{\mathrm{F}}^2 ]\\
    &\quad + (L+1) L_W^2 L_H^{2(L-\ell-2)} \mathbb{E}[ \| \nabla_H \bar{f}^{(L)}( \mathbf{D}^{(L+1)}, \bar{\mathbf{H}}^{(L-1)}, \mathbf{W}^{(L)}) - \nabla_H f^{(L)}( \mathbf{D}^{(L+1)}, \mathbf{H}^{(L-1)}, \mathbf{W}^{(L)})\|_{\mathrm{F}}^2 ] + \ldots \\
    &\quad + (L+1) L_W^2 \mathbb{E}[ \| \nabla_H \bar{f}^{(\ell+1)}( \mathbf{D}^{(\ell+2)}, \bar{\mathbf{H}}^{(\ell)}, \mathbf{W}^{(\ell+1)}) - \nabla_H f^{(\ell+1)}( \mathbf{D}^{(\ell+2)}, \mathbf{H}^{(\ell)}, \mathbf{W}^{(\ell+1)})\|_{\mathrm{F}}^2 ] \\
    &\quad + (L+1) \mathbb{E}[ \| \nabla_W \bar{f}^{(\ell)}( \mathbf{D}^{(\ell+1)}, \bar{\mathbf{H}}^{(\ell-1)}, \mathbf{W}^{(\ell)}) - \nabla_W f^{(\ell)}( \mathbf{D}^{(\ell+1)}, \mathbf{H}^{(\ell-1)}, \mathbf{W}^{(\ell)})\|_{\mathrm{F}}^2 ] .
    \end{aligned}
\end{equation}

From the previous equation, we know that there are three key factors that will affect the bias:
\begin{itemize}
    \item The difference of gradient with respect to the last layer node representations 
    \begin{equation}\label{eq:sgcn_bias_important_eq1}
        \mathbb{E}[ \| \bar{\mathbf{D}}^{(L+1)} -\mathbf{D}^{(L+1)} \|_{\mathrm{F}}^2 ].
    \end{equation}
    \item The difference of gradient with respect to the input node embedding matrix at each graph convolutional layer 
    \begin{equation}\label{eq:sgcn_bias_important_eq2}
        \mathbb{E}[ \| \nabla_H \bar{f}^{(\ell+1)}( \mathbf{D}^{(\ell+2)}, \bar{\mathbf{H}}^{(\ell)}, \mathbf{W}^{(\ell+1)}) - \nabla_H f^{(\ell+1)}( \mathbf{D}^{(\ell+2)}, \mathbf{H}^{(\ell)}, \mathbf{W}^{(\ell+1)})\|_{\mathrm{F}}^2 ].
    \end{equation}
    \item The difference of gradient with respect to the weight matrix at each graph convolutional layer 
    \begin{equation}\label{eq:sgcn_bias_important_eq3}
        \mathbb{E}[ \| \nabla_W \bar{f}^{(\ell)}( \mathbf{D}^{(\ell+1)}, \bar{\mathbf{H}}^{(\ell-1)}, \mathbf{W}^{(\ell)}) - \nabla_W f^{(\ell)}( \mathbf{D}^{(\ell+1)}, \mathbf{H}^{(\ell-1)}, \mathbf{W}^{(\ell)})\|_{\mathrm{F}}^2 ] .
    \end{equation}
\end{itemize}

Firstly, let consider the upper-bound of Eq.~\ref{eq:sgcn_bias_important_eq1}.
\begin{equation}
    \begin{aligned}
    \mathbb{E}[ \| \bar{\mathbf{D}}^{(L+1)} -\mathbf{D}^{(L+1)} \|_{\mathrm{F}}^2 ] &= \mathbb{E}[ \| \frac{\partial \text{Loss}(\bar{\mathbf{H}}^{(L)}, \mathbf{y}) }{\partial \bar{\mathbf{H}}^{(L)}} - \frac{\partial \text{Loss}(\mathbf{H}^{(L)}, \mathbf{y}) }{\partial \mathbf{H}^{(L)}} \|_{\mathrm{F}}^2 ] \\
    &\leq L_{loss}^2 \mathbb{E}[ \| \bar{\mathbf{H}}^{(L)} -\mathbf{H}^{(L)} \|_{\mathrm{F}}^2 ] .
    \end{aligned}
\end{equation}
The upper-bound for $\mathbb{E}[\| \bar{\mathbf{H}}^{(\ell)} - \mathbf{H}^{(\ell)}\|_{\mathrm{F}}^2]$ as
\begin{equation}
    \begin{aligned}
    \mathbb{E}[ \| \bar{\mathbf{H}}^{(\ell)} - \mathbf{H}^{(\ell)} \|_{\mathrm{F}}^2 ]
    &= \mathbb{E}[ \| \sigma( \mathbf{P}^{(\ell)} \bar{\mathbf{H}}^{(\ell-1)} \mathbf{W}^{(\ell)}) - \sigma( \mathbf{L} \mathbf{H}^{(\ell-1)} \mathbf{W}^{(\ell)}) \|_{\mathrm{F}}^2 ] \\
    &\leq C_\sigma^2 B_W^2 \mathbb{E}[ \|\mathbf{P}^{(\ell)} \bar{\mathbf{H}}^{(\ell-1)} - \mathbf{L} \bar{\mathbf{H}}^{(\ell-1)} + \mathbf{L} \bar{\mathbf{H}}^{(\ell-1)} -  \mathbf{L} \mathbf{H}^{(\ell-1)}\|_{\mathrm{F}}^2 ] \\
    &\leq 2 C_\sigma^2 B_W^2 B_H^2 \mathbb{E}[ \| \mathbf{P}^{(\ell)} - \mathbf{L} \|_{\mathrm{F}}^2 ] + 2 C_\sigma^2 B_W^2 B_{LA}^2 \mathbb{E}[ \| \bar{\mathbf{H}}^{(\ell-1)} - \mathbf{H}^{(\ell-1)}  \|_{\mathrm{F}}^2 ] \\
    &\leq \mathcal{O}(\mathbb{E}[\| \mathbf{P}^{(1)} - \mathbf{L} \|_{\mathrm{F}}^2] ) + \ldots + \mathcal{O}(\mathbb{E}[ \| \mathbf{P}^{(\ell)} - \mathbf{L} \|_{\mathrm{F}}^2 ] ).
    \end{aligned}
\end{equation}
Therefore, we have
\begin{equation}
    \mathbb{E}[ \| \bar{\mathbf{D}}^{(L+1)} -\mathbf{D}^{(L+1)} \|_{\mathrm{F}}^2 ]  \leq \mathcal{O}(\mathbb{E}[\| \mathbf{P}^{(1)} - \mathbf{L} \|_{\mathrm{F}}^2 ]) + \ldots + \mathcal{O}( \mathbb{E}[\| \mathbf{P}^{(L)} - \mathbf{L} \|_{\mathrm{F}}^2 ]).
\end{equation}

Then, let consider the upper-bound of Eq.~\ref{eq:sgcn_bias_important_eq2}.
\begin{equation}
    \begin{aligned}
    & \mathbb{E}[ \| \nabla_H \bar{f}^{(\ell)}(\mathbf{D}^{(\ell+1)}, \bar{\mathbf{H}}^{(\ell-1)}, \mathbf{W}^{(\ell)}) - \nabla_H f^{(\ell)}(\mathbf{D}^{(\ell+1)}, \mathbf{H}^{(\ell-1)}, \mathbf{W}^{(\ell)}) \|_{\mathrm{F}}^2 ] \\
    &= \mathbb{E}[ \| [\mathbf{L}]^\top \Big(\mathbf{D}^{(\ell+1)} \circ \sigma^\prime(\mathbf{P}^{(\ell)} \bar{\mathbf{H}}^{(\ell-1)} \mathbf{W}^{(\ell)}) \Big) [\mathbf{W}^{(\ell)}]^\top - [\mathbf{L}]^\top \Big(\mathbf{D}^{(\ell+1)} \circ \sigma^\prime(\mathbf{L} \mathbf{H}^{(\ell-1)} \mathbf{W}^{(\ell)}) \Big) [\mathbf{W}^{(\ell)}]^\top \|_{\mathrm{F}}^2 ] \\
    &\leq B_{LA}^2 B_D^2 B_W^4 L_\sigma^2 \mathbb{E}[ \| \mathbf{P}^{(\ell)} \bar{\mathbf{H}}^{(\ell-1)} - \mathbf{L}\bar{\mathbf{H}}^{(\ell-1)} + \mathbf{L}\bar{\mathbf{H}}^{(\ell-1)} - \mathbf{L} \mathbf{H}^{(\ell-1)} \|_{\mathrm{F}}^2 ] \\
    &\leq 2B_{LA}^2 B_D^2 B_W^4 L_\sigma^2 B_H^2 \mathbb{E}[ \| \mathbf{P}^{(\ell)} - \mathbf{L} \|_{\mathrm{F}}^2 ] + 2B_{LA}^4 B_D^2 B_W^4 L_\sigma^2 \mathbb{E}[ \| \bar{\mathbf{H}}^{(\ell-1)} - \mathbf{H}^{(\ell-1)} \|_{\mathrm{F}}^2 ] \\
    &\leq \mathcal{O}(\mathbb{E}[\| \mathbf{P}^{(1)} - \mathbf{L} \|_{\mathrm{F}}^2] ) + \ldots + \mathcal{O}(\mathbb{E}[ \| \mathbf{P}^{(\ell)} - \mathbf{L} \|_{\mathrm{F}}^2] ).
    \end{aligned}
\end{equation}

Finally, let consider the upper-bound of Eq.~\ref{eq:sgcn_bias_important_eq3}.
\begin{equation}
    \begin{aligned}
    &\mathbb{E}[ \| \nabla_W \bar{f}^{(\ell)}(\mathbf{D}^{(\ell+1)}, \bar{\mathbf{H}}^{(\ell-1)}, \mathbf{W}^{(\ell)}) - \nabla_W f^{(\ell)}(\mathbf{D}^{(\ell+1)}, \mathbf{H}^{(\ell-1)}, \mathbf{W}^{(\ell)}) \|_{\mathrm{F}}^2 ] \\
    &= \mathbb{E}[ \| [\mathbf{L} \bar{\mathbf{H}}^{(\ell-1)}]^\top \Big(\mathbf{D}^{(\ell+1)} \circ \sigma^\prime(\mathbf{P}^{(\ell)} \bar{\mathbf{H}}^{(\ell-1)} \mathbf{W}^{(\ell)} ) \Big) - [\mathbf{L} \mathbf{H}^{(\ell-1)}]^\top \Big(\mathbf{D}^{(\ell+1)} \circ  \sigma^\prime(\mathbf{L}^{(\ell)}\mathbf{H}^{(\ell-1)}\mathbf{W}^{(\ell)})\Big) \|_{\mathrm{F}}^2 ]\\
    &\leq 2\mathbb{E}[ \| [\mathbf{L} \bar{\mathbf{H}}^{(\ell-1)}]^\top \Big(\mathbf{D}^{(\ell+1)} \circ \sigma^\prime(\mathbf{P}^{(\ell)} \bar{\mathbf{H}}^{(\ell-1)} \mathbf{W}^{(\ell)} ) \Big) - [\mathbf{L} \mathbf{H}^{(\ell-1)}]^\top \Big(\mathbf{D}^{(\ell+1)} \circ \sigma^\prime(\mathbf{P}^{(\ell)} \bar{\mathbf{H}}^{(\ell-1)} \mathbf{W}^{(\ell)} ) \Big) \|_{\mathrm{F}}^2 ]\\
    &\qquad + 2\mathbb{E}[ \| [\mathbf{L} \mathbf{H}^{(\ell-1)}]^\top \Big(\mathbf{D}^{(\ell+1)} \circ \sigma^\prime(\mathbf{P}^{(\ell)} \bar{\mathbf{H}}^{(\ell-1)} \mathbf{W}^{(\ell)} ) \Big) - [\mathbf{L} \mathbf{H}^{(\ell-1)}]^\top \Big(\mathbf{D}^{(\ell+1)} \circ  \sigma^\prime(\mathbf{L}^{(\ell)}\mathbf{H}^{(\ell-1)}\mathbf{W}^{(\ell)})\Big) \|_{\mathrm{F}}^2 ] \\
    &\leq 2B_D^2 C_\sigma^2 B_{LA}^2 \mathbb{E}[ \| \bar{\mathbf{H}}^{(\ell-1)} - \mathbf{H}^{(\ell-1)} \|_{\mathrm{F}}^2 ] \\
    &\qquad + 2B_{LA}^2 B_H^2 B_D^2 L_\sigma^2 B_W^2 \mathbb{E}[ \| \mathbf{P}^{(\ell)} \bar{\mathbf{H}}^{(\ell-1)} - \mathbf{L}\bar{\mathbf{H}}^{(\ell-1)} + \mathbf{L}\bar{\mathbf{H}}^{(\ell-1)} - \mathbf{L} \mathbf{H}^{(\ell-1)}\|_{\mathrm{F}}^2 ]\\
    &\leq 2\Big( B_D^2 C_\sigma^2 B_{LA}^2  + B_{LA}^4 B_H^2 B_D^2 L_\sigma^2 B_W^2 \Big) \mathbb{E}[ \| \bar{\mathbf{H}}^{(\ell-1)} - \mathbf{H}^{(\ell-1)} \|_{\mathrm{F}}^2 ]\\
    &\qquad + 2B_{LA}^2 B_H^4 B_D^2 L_\sigma^2 B_W^2 \mathbb{E}[ \| \mathbf{P}^{(\ell)} - \mathbf{L} \|_{\mathrm{F}}^2 ] \\
    &\leq \mathcal{O}(\mathbb{E}[\| \mathbf{P}^{(1)} - \mathbf{L} \|_{\mathrm{F}}^2] ) + \ldots + \mathcal{O}(\mathbb{E}[\| \mathbf{P}^{(\ell)} - \mathbf{L} \|_{\mathrm{F}}^2] ).
    \end{aligned}
\end{equation}

Combining the result from Eq.~\ref{eq:sgcn_bias_important_eq1}, \ref{eq:sgcn_bias_important_eq2}, \ref{eq:sgcn_bias_important_eq3} we have
\begin{equation}
    \mathbb{E}[ \| \mathbb{E}[\widetilde{\mathbf{G}}^{(\ell)}] - \mathbf{G}^{(\ell)}\|_{\mathrm{F}}^2 ]
    \leq \mathcal{O}(\mathbb{E}[ \| \mathbf{P}^{(1)} - \mathbf{L} \|_{\mathrm{F}}^2 ])  + \ldots + \mathcal{O}(\mathbb{E}[ \| \mathbf{P}^{(L)} - \mathbf{L} \|_{\mathrm{F}}^2 ]) .
\end{equation}
\end{proof}

\subsection{Remaining steps toward Theorem~\ref{theorem:convergence_of_sgcn}}

By the smoothness of $\mathcal{L}(\bm{\theta}_t)$, we have 
\begin{equation}
    \begin{aligned}
    \mathcal{L}(\bm{\theta}_{t+1}) &\leq \mathcal{L}(\bm{\theta}_t) + \langle \nabla \mathcal{L}(\bm{\theta}_t), \bm{\theta}_{t+1} - \bm{\theta}_t \rangle + \frac{L_f}{2}\|\bm{\theta}_{t+1} - \bm{\theta}_t\|_{\mathrm{F}}^2 \\
    &= \mathcal{L}(\bm{\theta}_t) - \eta \langle \nabla \mathcal{L}(\bm{\theta}_t), \nabla\widetilde{\mathcal{L}}(\bm{\theta}_t) \rangle + \frac{\eta^2 L_f }{2}\|\nabla\widetilde{\mathcal{L}}(\bm{\theta}_t)\|_{\mathrm{F}}^2 .
    \end{aligned}
\end{equation}

Let $\mathcal{F}_t = \{ \{\mathcal{B}_1^{(\ell)}\}_{\ell=1}^L ,\ldots, \{\mathcal{B}_{t-1}^{(\ell)}\}_{\ell=1}^L\}$.
Note that the weight parameters $\bm{\theta}_t$ is a function of history of the generated random process and hence is random.
Taking expectation on both sides condition on $\mathcal{F}_t$ and using $\eta<1/L_f$ we have
\begin{equation}
    \begin{aligned}
    &\mathbb{E}[\mathcal{L}(\bm{\theta}_{t+1}) | \mathcal{F}_t] \\
    &\leq \mathcal{L}(\bm{\theta}_t) - \eta \langle \nabla \mathcal{L}(\bm{\theta}_t), \mathbb{E}[\nabla\widetilde{\mathcal{L}}(\bm{\theta}_t)|\mathcal{F}_t] \rangle + \frac{\eta^2 L_f }{2} \Big(\mathbb{E}[\|\nabla\widetilde{\mathcal{L}}(\bm{\theta}_t) - \mathbb{E}[\nabla\widetilde{\mathcal{L}}(\bm{\theta}_t)|\mathcal{F}_t]\|_{\mathrm{F}}^2 | \mathcal{F}_t] + \mathbb{E}[\|\mathbb{E}[\nabla\widetilde{\mathcal{L}}(\bm{\theta}_t)|\mathcal{F}_t]\|_{\mathrm{F}}^2 | \mathcal{F}_t]\Big) \\
    &= \mathcal{L}(\bm{\theta}_t) - \eta \langle \nabla \mathcal{L}(\bm{\theta}_t), \nabla \mathcal{L}(\bm{\theta}_t) + \mathbb{E}[\mathbf{b}_t|\mathcal{F}_t] \rangle + \frac{\eta^2 L_f }{2} \Big(\mathbb{E}[\|\mathbf{n}_t\|_{\mathrm{F}}^2 | \mathcal{F}_t] + \|\nabla \mathcal{L}(\bm{\theta}_t) + \mathbb{E}[\mathbf{b}_t | \mathcal{F}_t] \|_{\mathrm{F}}^2 \Big) \\
    &\leq \mathcal{L}(\bm{\theta}_t) + \frac{\eta}{2}\Big( - 2 \langle \nabla \mathcal{L}(\bm{\theta}_t), \nabla \mathcal{L}(\bm{\theta}_t) + \mathbb{E}[\mathbf{b}_t|\mathcal{F}_t] \rangle + \|\nabla \mathcal{L}(\bm{\theta}_t) + \mathbb{E}[\mathbf{b}_t|\mathcal{F}_t] \|_{\mathrm{F}}^2 \Big) + \frac{\eta^2 L_f}{2} \mathbb{E}[\|\mathbf{n}_t\|_{\mathrm{F}}^2|\mathcal{F}_t] \\
    &= \mathcal{L}(\bm{\theta}_t) + \frac{\eta}{2}\Big( - \| \nabla \mathcal{L}(\bm{\theta}_t)\|_{\mathrm{F}}^2 + \mathbb{E}[\|\mathbf{b}_t \|_{\mathrm{F}}^2|\mathcal{F}_t] \Big) + \frac{\eta^2 L_f}{2} \mathbb{E}[\|\mathbf{n}_t\|_{\mathrm{F}}^2|\mathcal{F}_t]
    \end{aligned}
\end{equation}

Denote $\Delta_\mathbf{b}$ as the upper bound of bias of stochasitc gradient as shown in Lemma~\ref{lemma:upper-bound-sgcn-bias} and Denote $\Delta_\mathbf{n}$ as the upper bound of bias of stochastic gradient as shown in Lemma~\ref{lemma:upper-bound-sgcn-vars}.
Plugging in the upper bound of bias and variance, taking expectation over $\mathcal{F}_t$, and rearranging the term we have
\begin{equation}
    \mathbb{E}[\| \nabla \mathcal{L}(\bm{\theta}_t)\|_{\mathrm{F}}^2] \leq \frac{2}{\eta} \Big( \mathbb{E}[\mathcal{L}(\bm{\theta}_t)] - \mathbb{E}[\mathcal{L}(\bm{\theta}_{t+1})]\Big)  + \eta L_f \Delta_\mathbf{n} + \Delta_\mathbf{b} 
\end{equation}

Summing up from $t=1$ to $T$, rearranging we have
\begin{equation}
\begin{aligned}
    \frac{1}{T} \sum_{t=1}^T \mathbb{E}[\|\nabla \mathcal{L}(\bm{\theta}_t)\|_{\mathrm{F}}^2] &\leq \frac{2}{\eta T}\sum_{t=1}^T (\mathbb{E}[\mathcal{L}(\bm{\theta}_t)] - \mathbb{E}[\mathcal{L}(\bm{\theta}_{t+1})])  + \eta L_f \Delta_\mathbf{n} + \Delta_\mathbf{b}\\
    &\underset{(a)}{\leq} \frac{2}{\eta T}( \mathcal{L}(\bm{\theta}_1) - \mathcal{L}(\bm{\theta}^\star) ) + \eta L_f \Delta_\mathbf{n} + \Delta_\mathbf{b}
\end{aligned}
\end{equation}
where the inequality $(a)$ is due to $\mathcal{L}(\bm{\theta}^\star) \leq \mathbb{E}[\mathcal{L}(\bm{\theta}_{T+1})]$.


By selecting learning rate as $\eta=1/\sqrt{T}$, we have
\begin{equation}
    \begin{aligned}
    \frac{1}{T} \sum_{t=1}^T \mathbb{E}[ \|\nabla \mathcal{L}(\bm{\theta}_t)\|_{\mathrm{F}}^2 ] &\leq \frac{2(\mathcal{L}(\bm{\theta}_1) - \mathcal{L}(\bm{\theta}^\star) )}{\sqrt{T}} +  \frac{L_f \Delta_\mathbf{n}}{\sqrt{T}} +  \Delta_\mathbf{b}\\
    \end{aligned}
\end{equation}

%% file: appendix/proof_sgcn_thm_2.tex
\section{Proof of Theorem~\ref{theorem:convergence_of_sgcn_plus}}\label{section:proof_of_thm2}

\subsection{Upper-bounded on the node embedding matrices and layerwise gradients}

When using variance reduction algorithm, we cannot use upper bound on the node embedding matrices and layerwise gradients as derived in Proposition~\ref{proposition:matrix_norm_bound}.
To see this, let consider the upper bound on the node embedding matrix before the activation function $\widetilde{\mathbf{Z}}^{(\ell)}_t$, we have
\begin{equation}
    \begin{aligned}
    \| \widetilde{\mathbf{Z}}^{(\ell)}_t \|_\mathrm{F} 
    &= \| \widetilde{\mathbf{Z}}^{(\ell)}_{t-1} + \widetilde{\mathbf{L}}_t^{(\ell)} \widetilde{\mathbf{H}}^{(\ell-1)}_t \mathbf{W}_t^{(\ell)} - \widetilde{\mathbf{L}}_t^{(\ell)} \widetilde{\mathbf{H}}^{(\ell-1)}_{t-1} \mathbf{W}_{t-1}^{(\ell)}  \|_\mathrm{F} \\
    &\leq  \| \widetilde{\mathbf{Z}}^{(\ell)}_{t-1} \|_\mathrm{F}+  \| \widetilde{\mathbf{L}}_t^{(\ell)} \widetilde{\mathbf{H}}^{(\ell-1)}_t \mathbf{W}_t^{(\ell)} - \widetilde{\mathbf{L}}_t^{(\ell)} \widetilde{\mathbf{H}}^{(\ell-1)}_{t-1} \mathbf{W}_{t-1}^{(\ell)} \|_\mathrm{F} \\
    \end{aligned}
\end{equation}

Therefore, in the worst case, the norm of node embedding matrices is growing as the inner loop size. To control the growth on matrix norm, we introduced an early stopping criterion by checking the norm of node embeddings, and immediately start another snapshot step if the condition is triggered. 

Let define $t_s$ as the stopping time of the inner loop in the $s$th outer loop,
\begin{equation}
    e_s = e_{s-1} + \min \Big\{ \max_{k \geq e_{s-1}} \{ \| \widetilde{\mathbf{H}}_k^{(\ell)}\|_\mathrm{F} \leq \alpha \| \widetilde{\mathbf{H}}_{e_{s-1}}^{(\ell)} \|_\mathrm{F} \}, K \Big\},~
    t_0 = 1,
\end{equation}
and $E_s = e_s - e_{s-1} \leq K$ as the snapshot gap at the $s$th inner-loop and $S$ as the total number of outer loops. 
By doing so, we can upper bounded the node embedding matrix by $\| \mathbf{H}^{(\ell)}_t \|_\mathrm{F} \leq \alpha B_H,~\forall t \in [T]$.

\subsection{Supporting lemmas}
In the following lemma, we derive the upper-bound on the node embedding approximation error of each GCN layer in \texttt{SGCN+}. 
This upper-bound plays an important role in the analysis of the upper-bound of the bias term for the stochastic gradient. 
Suppose the input node embedding matrix for the $\ell$ GCN layer as $\bar{\mathbf{H}}_t^{(\ell-1)}$, the forward propagation for the $\ell$th layer in \texttt{SGCN+} is defined as
\begin{equation}
    \bar{\mathbf{Z}}_t^{(\ell)} = \widetilde{ f}^{(\ell)}(\bar{\mathbf{H}}_t^{(\ell-1)}, \mathbf{W}^{(\ell)}_t) = \mathbf{P}^{(\ell)}_t \bar{\mathbf{H}}_t^{(\ell)} \mathbf{W}^{(\ell)}_t ,
\end{equation}
and the forward propagation for the $\ell$th layer in \texttt{FullGCN} is defined as 
\begin{equation}
    f^{(\ell)}(\bar{\mathbf{H}}_t^{(\ell-1)}, \mathbf{W}^{(\ell)}_t) = \mathbf{L} \bar{\mathbf{H}}_t^{(\ell-1)} \mathbf{W}^{(\ell)}_t.
\end{equation}
In the following, we derive the upper-bound of 
\begin{equation}
    \mathbb{E}[\| \widetilde{ f}^{(\ell)}(\bar{\mathbf{H}}_t^{(\ell-1)}, \mathbf{W}^{(\ell)}_t ) - f^{(\ell)}(\bar{\mathbf{H}}_t^{(\ell-1)}, \mathbf{W}^{(\ell)}_t ) \|_{\mathrm{F}}^2] = \mathbb{E}[\| \mathbf{L} \bar{\mathbf{H}}_t^{(\ell-1)} \mathbf{W}_t^{(\ell)} - \bar{\mathbf{Z}}_t^{(\ell)} \|_{\mathrm{F}}^2 ].
\end{equation}

\begin{lemma}\label{lemma:upper-bound-sgcn-plus-bias-1}
Let consider the $s$th epoch, let $t$ be the current step. Therefore, for any $t\in\{e_{s-1}, \ldots, e_s \}$, we have
\begin{equation}
    \mathbb{E}[\| \mathbf{L} \bar{\mathbf{H}}_t^{(\ell-1)} \mathbf{W}_t^{(\ell)} - \bar{\mathbf{Z}}_t^{(\ell)} \|_{\mathrm{F}}^2 ] \leq \eta^2 E_s \times \mathcal{O}\left( \alpha^4 \sum_{j=1}^\ell \big| \mathbb{E}[\| \mathbf{P}^{(j)}_t \|_{\mathrm{F}}^2] - \|\mathbf{L}\|_{\mathrm{F}}^2 \big|  \right).
\end{equation}
\end{lemma}
\begin{proof}
\begin{equation} \label{eq:upper-bound-sgcn-plus-bias-1-eq1}
    \begin{aligned}
    &\| \mathbf{L} \bar{\mathbf{H}}_t^{(\ell-1)} \mathbf{W}_t^{(\ell)} - \bar{\mathbf{Z}}_t^{(\ell)} \|_{\mathrm{F}}^2 \\
    &= \| [\mathbf{L} \bar{\mathbf{H}}_t^{(\ell-1)} \mathbf{W}_t^{(\ell)} - \mathbf{L} \bar{\mathbf{H}}_{t-1}^{(\ell-1)} \mathbf{W}_{t-1}^{(\ell)}] + [ \mathbf{L} \bar{\mathbf{H}}_{t-1}^{(\ell-1)} \mathbf{W}_{t-1}^{(\ell)} - \bar{\mathbf{Z}}_{t-1}^{(\ell)}] - [ \bar{\mathbf{Z}}_t^{(\ell)} - \bar{\mathbf{Z}}_{t-1}^{(\ell)}] \|_{\mathrm{F}}^2 \\
    &= \| \mathbf{L} \bar{\mathbf{H}}_t^{(\ell-1)} \mathbf{W}_t^{(\ell)} - \mathbf{L} \bar{\mathbf{H}}_{t-1}^{(\ell-1)} \mathbf{W}_{t-1}^{(\ell)} \|_{\mathrm{F}}^2 + \| \mathbf{L} \bar{\mathbf{H}}_{t-1}^{(\ell-1)} \mathbf{W}_{t-1}^{(\ell)} - \bar{\mathbf{Z}}_{t-1}^{(\ell)} \|_{\mathrm{F}}^2 + \| \bar{\mathbf{Z}}_t^{(\ell)} - \bar{\mathbf{Z}}_{t-1}^{(\ell)} \|_{\mathrm{F}}^2 \\
    &\quad + 2 \langle \mathbf{L} \bar{\mathbf{H}}_t^{(\ell-1)} \mathbf{W}_t^{(\ell)} - \mathbf{L} \bar{\mathbf{H}}_{t-1}^{(\ell-1)} \mathbf{W}_{t-1}^{(\ell)},  \mathbf{L} \bar{\mathbf{H}}_{t-1}^{(\ell-1)} \mathbf{W}_{t-1}^{(\ell)} - \bar{\mathbf{Z}}_{t-1}^{(\ell)} \rangle \\
    &\quad - 2 \langle \mathbf{L} \bar{\mathbf{H}}_t^{(\ell-1)} \mathbf{W}_t^{(\ell)} - \mathbf{L} \bar{\mathbf{H}}_{t-1}^{(\ell-1)} \mathbf{W}_{t-1}^{(\ell)}, \bar{\mathbf{Z}}_t^{(\ell)} - \bar{\mathbf{Z}}_{t-1}^{(\ell)} \rangle \\
    &\quad - 2 \langle \mathbf{L} \bar{\mathbf{H}}_{t-1}^{(\ell-1)} \mathbf{W}_{t-1}^{(\ell)} - \bar{\mathbf{Z}}_{t-1}^{(\ell)}, \bar{\mathbf{Z}}_t^{(\ell)} - \bar{\mathbf{Z}}_{t-1}^{(\ell)} \rangle.
    \end{aligned}
\end{equation}

Recall that by the update rule, we have
\begin{equation}
\bar{\mathbf{Z}}_t^{(\ell)} - \bar{\mathbf{Z}}_{t-1}^{(\ell)} = \mathbf{P}_t^{(\ell)} \bar{\mathbf{H}}_t^{(\ell-1)} \mathbf{W}_t^{(\ell)} - \mathbf{P}_t^{(\ell)} \bar{\mathbf{H}}_{t-1}^{(\ell-1)} \mathbf{W}_{t-1}^{(\ell)},
\end{equation}
and 
\begin{equation}
\mathbb{E}[ \bar{\mathbf{Z}}_t^{(\ell)} - \bar{\mathbf{Z}}_{t-1}^{(\ell)} | \mathcal{F}_t] = \mathbf{L} \bar{\mathbf{H}}_t^{(\ell-1)} \mathbf{W}_t^{(\ell)} - \mathbf{L} \bar{\mathbf{H}}_{t-1}^{(\ell-1)} \mathbf{W}_{t-1}^{(\ell)}.
\end{equation}

Taking expectation on both side of Eq.~\ref{eq:upper-bound-sgcn-plus-bias-1-eq1} condition on $\mathcal{F}_t$, we have
\begin{equation}
    \begin{aligned}
    \mathbb{E}[\| \mathbf{L} \bar{\mathbf{H}}_t^{(\ell-1)} \mathbf{W}_t^{(\ell)} - \bar{\mathbf{Z}}_t^{(\ell)} \|_{\mathrm{F}}^2 | \mathcal{F}_t ] 
    &\leq \| \mathbf{L} \bar{\mathbf{H}}_{t-1}^{(\ell-1)} \mathbf{W}_{t-1}^{(\ell)} - \bar{\mathbf{Z}}_{t-1}^{(\ell)} \|_{\mathrm{F}}^2 + \mathbb{E}[\| \bar{\mathbf{Z}}_t^{(\ell)} - \bar{\mathbf{Z}}_{t-1}^{(\ell)} \|_{\mathrm{F}}^2 | \mathcal{F}_t ] \\
    &\quad - \| \mathbf{L} \bar{\mathbf{H}}_t^{(\ell-1)} \mathbf{W}_t^{(\ell)} - \mathbf{L} \bar{\mathbf{H}}_{t-1}^{(\ell-1)} \mathbf{W}_{t-1}^{(\ell)} \|_{\mathrm{F}}^2. 
    \end{aligned}
\end{equation}

Then take the expectation on both side of Eq.~\ref{eq:upper-bound-sgcn-plus-bias-1-eq1} w.r.t. $\mathcal{F}_t$, we have
\begin{equation} \label{eq:upper-bound-sgcn-plus-bias-1-eq2}
    \begin{aligned}
    \mathbb{E}[\| \mathbf{L} \bar{\mathbf{H}}_t^{(\ell-1)} \mathbf{W}_t^{(\ell)} - \bar{\mathbf{Z}}_t^{(\ell)} \|_{\mathrm{F}}^2 ] 
    &\leq \| \mathbf{L} \bar{\mathbf{H}}_{t-1}^{(\ell-1)} \mathbf{W}_{t-1}^{(\ell)} - \bar{\mathbf{Z}}_{t-1}^{(\ell)} \|_{\mathrm{F}}^2 + \mathbb{E}[\| \bar{\mathbf{Z}}_t^{(\ell)} - \bar{\mathbf{Z}}_{t-1}^{(\ell)} \|_{\mathrm{F}}^2 ] \\
    &\quad - \| \mathbf{L} \bar{\mathbf{H}}_t^{(\ell-1)} \mathbf{W}_t^{(\ell)} - \mathbf{L} \bar{\mathbf{H}}_{t-1}^{(\ell-1)} \mathbf{W}_{t-1}^{(\ell)} \|_{\mathrm{F}}^2. 
    \end{aligned}
\end{equation}

Since we know $t\in \{e_{s-1},\ldots, e_s\} $, we can denote $t=e_{s-1}+k,~k\leq E_s$ such that formulate Eq.~\ref{eq:upper-bound-sgcn-plus-bias-1-eq2} as
\begin{equation} \label{eq:upper-bound-sgcn-plus-bias-1-eq3}
    \begin{aligned}
    &\mathbb{E}[\| \mathbf{L} \bar{\mathbf{H}}_t^{(\ell-1)} \mathbf{W}_t^{(\ell)} - \bar{\mathbf{Z}}_t^{(\ell)} \|_{\mathrm{F}}^2 ] \\
    &= \mathbb{E}[\| \mathbf{L} \bar{\mathbf{H}}_{e_{s-1}+k}^{(\ell-1)} \mathbf{W}_{e_{s-1}+k}^{(\ell)} - \bar{\mathbf{Z}}_{e_{s-1}+k}^{(\ell)} \|_{\mathrm{F}}^2 ] \\
    &=  \underbrace{\mathbb{E}[\| \mathbf{L} \bar{\mathbf{H}}_{e_{s-1}}^{(\ell-1)} \mathbf{W}_{e_{s-1}}^{(\ell)} - \bar{\mathbf{Z}}_{e_{s-1}}^{(\ell)} \|_{\mathrm{F}}^2 ]}_{(A)} \\
    &\quad + \sum_{t=e_{s-1}+1}^{e_s} \Big( \mathbb{E}[\| \bar{\mathbf{Z}}_t^{(\ell)} - \bar{\mathbf{Z}}_{t-1}^{(\ell)} \|_{\mathrm{F}}^2 ] - \| \mathbf{L} \bar{\mathbf{H}}_t^{(\ell-1)} \mathbf{W}_t^{(\ell)} - \mathbf{L} \bar{\mathbf{H}}_{t-1}^{(\ell-1)} \mathbf{W}_{t-1}^{(\ell)} \|_{\mathrm{F}}^2 \Big).
    \end{aligned}
\end{equation}

Knowing that we are using all neighbors at the snapshot step $t \in\{ e_{0}, \ldots, e_S \}$, we have $(A)=0$ in Eq.~\ref{eq:upper-bound-sgcn-plus-bias-1-eq3}. As a result, we have
\begin{equation} \label{eq:upper-bound-sgcn-plus-bias-1-eq5}
    \begin{aligned}
    &\mathbb{E}[\| \mathbf{L} \bar{\mathbf{H}}_{e_{s-1}+k}^{(\ell-1)} \mathbf{W}_{e_{s-1}+k}^{(\ell)} - \bar{\mathbf{Z}}_{e_{s-1}+k}^{(\ell)} \|_{\mathrm{F}}^2 ] \\
    &\leq \sum_{t=e_{s-1}+1}^{e_s} \Big( \underbrace{\mathbb{E}[\| \bar{\mathbf{Z}}_t^{(\ell)} - \bar{\mathbf{Z}}_{t-1}^{(\ell)} \|_{\mathrm{F}}^2 ] - \| \mathbf{L} \bar{\mathbf{H}}_t^{(\ell-1)} \mathbf{W}_t^{(\ell)} - \mathbf{L} \bar{\mathbf{H}}_{t-1}^{(\ell-1)} \mathbf{W}_{t-1}^{(\ell)} \|_{\mathrm{F}}^2}_{(B)} \Big).
    \end{aligned}
\end{equation}

Let take a closer look at term $(B)$. 
\begin{equation}
    \begin{aligned}
    &\mathbb{E}[\| \bar{\mathbf{Z}}_t^{(\ell)} - \bar{\mathbf{Z}}_{t-1}^{(\ell)} \|_{\mathrm{F}}^2 ] - \| \mathbf{L} \bar{\mathbf{H}}_t^{(\ell-1)} \mathbf{W}_t^{(\ell)} - \mathbf{L} \bar{\mathbf{H}}_{t-1}^{(\ell-1)} \mathbf{W}_{t-1}^{(\ell)} \|_{\mathrm{F}}^2 \\
    &= \mathbb{E}[\| \mathbf{P}_t ^{(\ell)} \bar{\mathbf{H}}_t^{(\ell-1)} \mathbf{W}_t^{(\ell)} - \mathbf{P}_t^{(\ell)} \bar{\mathbf{H}}_{t-1}^{(\ell-1)} \mathbf{W}_{t-1}^{(\ell)} \|_{\mathrm{F}}^2] - \| \mathbf{L} \bar{\mathbf{H}}_t^{(\ell-1)} \mathbf{W}_t^{(\ell)} - \mathbf{L} \bar{\mathbf{H}}_{t-1}^{(\ell-1)} \mathbf{W}_{t-1}^{(\ell)} \|_{\mathrm{F}}^2 \\
    &\leq \mathbb{E}[\| \mathbf{P}_t^{(\ell)} (\bar{\mathbf{H}}_t^{(\ell-1)} \mathbf{W}_t^{(\ell)} - \bar{\mathbf{H}}_{t-1}^{(\ell-1)} \mathbf{W}_{t-1}^{(\ell)}) \|_{\mathrm{F}}^2] - \| \mathbf{L} (\bar{\mathbf{H}}_t^{(\ell-1)} \mathbf{W}_t^{(\ell)} - \bar{\mathbf{H}}_{t-1}^{(\ell-1)} \mathbf{W}_{t-1}^{(\ell)})\|_{\mathrm{F}}^2\\
    &\leq \Big(\mathbb{E}[\| \mathbf{P}_t^{(\ell)} \|_{\mathrm{F}}^2] - \|\mathbf{L}\|_{\mathrm{F}}^2 \Big) \underbrace{\mathbb{E}[\| \bar{\mathbf{H}}_t^{(\ell-1)} \mathbf{W}_t^{(\ell)} - \bar{\mathbf{H}}_{t-1}^{(\ell-1)} \mathbf{W}_{t-1}^{(\ell)} \|_{\mathrm{F}}^2] }_{(C)}.
    \end{aligned}
\end{equation}
Let take a closer look at term $(C)$.
\begin{equation} \label{eq:upper-bound-sgcn-plus-bias-1-eq4}
    \begin{aligned}
    &\mathbb{E}[\| \bar{\mathbf{H}}_t^{(\ell-1)} \mathbf{W}_t^{(\ell)} - \bar{\mathbf{H}}_{t-1}^{(\ell-1)} \mathbf{W}_{t-1}^{(\ell)} \|_{\mathrm{F}}^2] \\
    &=\mathbb{E}[\| \bar{\mathbf{H}}_t^{(\ell-1)} \mathbf{W}_t^{(\ell)} - \bar{\mathbf{H}}_t^{(\ell-1)} \mathbf{W}_{t-1}^{(\ell)} + \bar{\mathbf{H}}_t^{(\ell-1)} \mathbf{W}_{t-1}^{(\ell)} - \bar{\mathbf{H}}_{t-1}^{(\ell-1)} \mathbf{W}_{t-1}^{(\ell)} \|_{\mathrm{F}}^2] \\
    &\underset{(a)}{\leq}  2 \alpha^2 B_{H}^2 \mathbb{E}[\| \mathbf{W}_t^{(\ell)} - \mathbf{W}_{t-1}^{(\ell)}\|_{\mathrm{F}}^2] + 2\mathbb{E}[\| \bar{\mathbf{H}}_t^{(\ell-1)} \mathbf{W}_{t-1}^{(\ell)} - \bar{\mathbf{H}}_{t-1}^{(\ell-1)} \mathbf{W}_{t-1}^{(\ell)} \|_{\mathrm{F}}^2],
    \end{aligned}
\end{equation}
where $(a)$ is due to the early stop criterion in our algorithm.

By induction, we can formulate Eq.~\ref{eq:upper-bound-sgcn-plus-bias-1-eq4} by
\begin{equation} \label{eq:upper-bound-sgcn-plus-bias-1-eq7}
    \begin{aligned}
    &\mathbb{E}[\| \bar{\mathbf{H}}_t^{(\ell-1)} \mathbf{W}_t^{(\ell)} - \bar{\mathbf{H}}_{t-1}^{(\ell-1)} \mathbf{W}_{t-1}^{(\ell)} \|_{\mathrm{F}}^2] \\
    &\leq  2 \alpha^2 B_H^2 \mathbb{E}[\| \mathbf{W}_t^{(\ell)} - \mathbf{W}_{t-1}^{(\ell)}\|_{\mathrm{F}}^2] + 2^2 \alpha^2 B_{H}^2 \mathbb{E}[\| \mathbf{W}_t^{(\ell-1)} - \mathbf{W}_{t-1}^{(\ell-1)}\|_{\mathrm{F}}^2] + \ldots \\
    &\qquad + 2^\ell \alpha^{2} B_{H}^{2}  \mathbb{E}[\| \mathbf{W}_t^{(1)} - \mathbf{W}_{t-1}^{(1)}\|_{\mathrm{F}}^2].
    \end{aligned}
\end{equation}


By the update rule of weight matrices, we know
\begin{equation}
    \mathbb{E}[\| \mathbf{W}_t^{(\ell)} - \mathbf{W}_{t-1}^{(\ell)} \|_{\mathrm{F}}^2] = \eta^2 \mathbb{E}[\| \bar{\mathbf{G}}_{t-1}^{(\ell)} \|_{\mathrm{F}}^2].
\end{equation}
By the definition of $\bar{\mathbf{G}}_{t-1}^{(\ell)}$ in Eq.~\ref{eq:G_bar_SGCN_plus}, we have that
\begin{equation} \label{eq:upper-bound-sgcn-plus-bias-1-eq6}
    \begin{aligned}
    \mathbb{E}[\| \bar{\mathbf{G}}_{t-1}^{(\ell)} \|_{\mathrm{F}}^2] 
    &\leq \alpha^2 B_{LA}^2 B_H^2 B_D^2 C_\sigma^2.
    \end{aligned}
\end{equation}

Plugging the results back, we have the upper bound of Eq.~\ref{eq:upper-bound-sgcn-plus-bias-1-eq3} as
\begin{equation} \label{eq:upper-bound-sgcn-plus-bias-1-eq7}
    \begin{aligned}
    \mathbb{E}[\| \mathbf{L} \bar{\mathbf{H}}_t^{(\ell-1)} \mathbf{W}_t^{(\ell)} - \bar{\mathbf{Z}}_t^{(\ell)} \|_{\mathrm{F}}^2 ] 
    &= \mathbb{E}[\| \mathbf{L} \bar{\mathbf{H}}_{e_{s-1}+k}^{(\ell-1)} \mathbf{W}_{e_{s-1}+k}^{(\ell)} - \bar{\mathbf{Z}}_{e_{s-1}+k}^{(\ell)} \|_{\mathrm{F}}^2 ] \\
    &\leq \eta^2 E_s \times \mathcal{O}\Big( \alpha^2 \sum_{j=1}^\ell | \mathbb{E}[\| \mathbf{P}_t^{(j)} \|_{\mathrm{F}}^2] - \|\mathbf{L}\|_{\mathrm{F}}^2 | \times \| \mathbf{W}_t^{(\ell)} - \mathbf{W}_{t-1}^{(\ell)} \|^2_\mathrm{F}\Big) \\
    &\leq \eta^2 E_s \times \mathcal{O}\Big( \alpha^4 \sum_{j=1}^\ell | \mathbb{E}[\| \mathbf{P}_t^{(j)} \|_{\mathrm{F}}^2] - \|\mathbf{L}\|_{\mathrm{F}}^2 | \Big).
    \end{aligned}
\end{equation}

\end{proof}

Based on the upper-bound of node embedding approximation error of each graph convolutional layer,  we derived the upper-bound on the bias of stochastic gradient in \texttt{SGCN}.

\begin{lemma} [Upper-bound on bias] \label{lemma:upper-bound-sgcn-plus-bias}
Let consider the $s$th epoch, let $t$ be the current step. Therefore, for any $t\in\{e_{s-1}, \ldots, e_s \}$,
we can upper-bound the bias of stochastic gradient in \texttt{SGCN+} as
\begin{equation}
    \sum_{\ell=1}^L \mathbb{E}[\|\mathbb{E}[\widetilde{\mathbf{G}}_t^{(\ell)}] - \mathbf{G}_t^{(\ell)} \|_{\mathrm{F}}^2] 
    \leq \eta^2 E_s \times \mathcal{O} \left( \alpha^4 \sum_{j=1}^\ell | \mathbb{E}[\| \mathbf{P}_t^{(j)} \|_{\mathrm{F}}^2] - \|\mathbf{L}\|_{\mathrm{F}}^2 | \right)
\end{equation}
\end{lemma}
\begin{proof}
From the decomposition of bias as shown in previously in Eq.~\ref{eq:decompose_bias_into_layers}, we have
\begin{equation} \label{eq:lemma:upper-bound-sgcn-plus-bias-eq1}
    \begin{aligned}
    &\mathbb{E}[ \| \mathbb{E}[\widetilde{\mathbf{G}}_t^{(\ell)}] - \mathbf{G}^{(\ell)}_t \|_{\mathrm{F}}^2 ]\\
    &\leq (L+1) L_W^2 L_H^{2(L-\ell-1)} \mathbb{E}[ \| \bar{\mathbf{D}}^{(L+1)}_t -\mathbf{D}_t^{(L+1)} \|_{\mathrm{F}}^2 ]\\
    &\quad + (L+1) L_W^2 L_H^{2(L-\ell-2)} \mathbb{E}[ \| \nabla_H \bar{f}^{(L)}( \mathbf{D}_t^{(L+1)}, \bar{\mathbf{H}}_t^{(L-1)}, \mathbf{W}_t^{(L)}) - \nabla_H f^{(L)}( \mathbf{D}_t^{(L+1)}, \mathbf{H}_t^{(L-1)}, \mathbf{W}_t^{(L)})\|_{\mathrm{F}}^2 ] + \ldots \\
    &\quad + (L+1) L_W^2 \mathbb{E}[ \| \nabla_H \bar{f}^{(\ell+1)}( \mathbf{D}_t^{(\ell+2)}, \bar{\mathbf{H}}_t^{(\ell)}, \mathbf{W}_t^{(\ell+1)}) - \nabla_H f^{(\ell+1)}( \mathbf{D}_t^{(\ell+2)}, \mathbf{H}_t^{(\ell)}, \mathbf{W}_t^{(\ell+1)})\|_{\mathrm{F}}^2 ] \\
    &\quad + (L+1) \mathbb{E}[ \| \nabla_W \bar{f}^{(\ell)}( \mathbf{D}_t^{(\ell+1)}, \bar{\mathbf{H}}_t^{(\ell-1)}, \mathbf{W}_t^{(\ell)}) - \nabla_W f^{(\ell)}( \mathbf{D}_t^{(\ell+1)}, \mathbf{H}_t^{(\ell-1)}, \mathbf{W}_t^{(\ell)})\|_{\mathrm{F}}^2 ].
    \end{aligned}
\end{equation}

From the previous equation, we know that there are three key factors that will affect the bias:
\begin{itemize}
    \item The difference of gradient with respect to the last layer node representations 
    \begin{equation}\label{eq:sgcn_plus_bias_important_eq1}
        \mathbb{E}[ \| \bar{\mathbf{D}}^{(L+1)}_t -\mathbf{D}^{(L+1)}_t \|_{\mathrm{F}}^2 ].
    \end{equation}
    \item The difference of gradient with respect to the input node embedding matrix at each graph convolutional layer 
    \begin{equation}\label{eq:sgcn_plus_bias_important_eq2}
        \mathbb{E}[ \| \nabla_H \bar{f}^{(\ell+1)}( \mathbf{D}^{(\ell+2)}_t, \bar{\mathbf{H}}^{(\ell)}_t, \mathbf{W}^{(\ell+1)}_t) - \nabla_H f^{(\ell+1)}( \mathbf{D}^{(\ell+2)}_t, \mathbf{H}^{(\ell)}_t, \mathbf{W}^{(\ell+1)}_t)\|_{\mathrm{F}}^2 ].
    \end{equation}
    \item The difference of gradient with respect to the weight matrix at each graph convolutional layer 
    \begin{equation}\label{eq:sgcn_plus_bias_important_eq3}
        \mathbb{E}[ \| \nabla_W \bar{f}^{(\ell)}( \mathbf{D}^{(\ell+1)}_t, \bar{\mathbf{H}}^{(\ell-1)}_t, \mathbf{W}^{(\ell)}_t) - \nabla_W f^{(\ell)}( \mathbf{D}^{(\ell+1)}_t, \mathbf{H}^{(\ell-1)}_t, \mathbf{W}^{(\ell)}_t)\|_{\mathrm{F}}^2 ] .
    \end{equation}
\end{itemize}

Firstly, let consider the upper-bound of Eq.~\ref{eq:sgcn_plus_bias_important_eq1}.
\begin{equation} \label{eq:sgcn_plus_bias_important_eq4}
    \begin{aligned}
    \mathbb{E}[ \| \bar{\mathbf{D}}^{(L+1)}_t -\mathbf{D}^{(L+1)}_t \|_{\mathrm{F}}^2 ] &= \mathbb{E}[ \| \frac{\partial \text{Loss}(\bar{\mathbf{H}}^{(L)}_t, \mathbf{y}) }{\partial \bar{\mathbf{H}}^{(L)}_t} - \frac{\partial \text{Loss}(\mathbf{H}^{(L)}_t, \mathbf{y}) }{\partial \mathbf{H}^{(L)}_t} \|_{\mathrm{F}}^2 ] \\
    &\leq L_{loss}^2 \mathbb{E}[ \| \bar{\mathbf{H}}^{(L)}_t -\mathbf{H}^{(L)}_t \|_{\mathrm{F}}^2 ] \\
    &\leq L_{loss}^2 C_\sigma^2 \mathbb{E}[ \| \bar{\mathbf{Z}}^{(L)}_t -\mathbf{Z}^{(L)}_t \|_{\mathrm{F}}^2 ]
    \end{aligned}
\end{equation}
We can decompose $\mathbb{E}[ \| \bar{\mathbf{Z}}^{(L)}_t -\mathbf{Z}^{(L)}_t \|_{\mathrm{F}}^2 ]$ as
\begin{equation}
    \begin{aligned}
    &\mathbb{E}[\| \bar{\mathbf{Z}}^{(L)}_t - \mathbf{Z}^{(L)}_t \|_{\mathrm{F}}^2] \\
    &= \mathbb{E}[\| \bar{\mathbf{Z}}^{(L)}_t - \mathbf{L}\mathbf{H}^{(L-1)}_t \mathbf{W}^{(L)}_t \|_{\mathrm{F}}^2] \\
    &\leq 2 \mathbb{E}[\| \bar{\mathbf{Z}}^{(L)}_t - \mathbf{L}\bar{\mathbf{H}}^{(L-1)}_t \mathbf{W}^{(L)}_t \|_{\mathrm{F}}^2]  + 2 \mathbb{E}[\| \mathbf{L}\bar{\mathbf{H}}^{(L-1)}_t \mathbf{W}^{(L)}_t - \mathbf{L}\mathbf{H}^{(L-1)}_t \mathbf{W}^{(L)}_t \|_{\mathrm{F}}^2] \\
    &\leq 2 \mathbb{E}[\| \bar{\mathbf{Z}}^{(L)}_t - \mathbf{L}\bar{\mathbf{H}}^{(L-1)}_t \mathbf{W}^{(L)}_t \|_{\mathrm{F}}^2]   + 2 B_{LA}^2 B_W^2 C_\sigma^2 \mathbb{E}[\| \bar{\mathbf{Z}}^{(L-1)}_t - \mathbf{Z}^{(L-1)}_t \|_{\mathrm{F}}^2] \\
    &\leq \sum_{\ell=1}^L \mathcal{O}(\mathbb{E}[\| \bar{\mathbf{Z}}^{(\ell)}_t - \mathbf{L}\bar{\mathbf{H}}_t^{(\ell-1)} \mathbf{W}^{(\ell)}_t \|_{\mathrm{F}}^2]). 
    \end{aligned}
\end{equation}
Using result from Lemma~\ref{lemma:upper-bound-sgcn-plus-bias-1}, we can upper bound Eq.~\ref{eq:sgcn_plus_bias_important_eq4} by
\begin{equation}
    \mathbb{E}[ \| \bar{\mathbf{D}}^{(L+1)}_t -\mathbf{D}^{(L+1)}_t \|_{\mathrm{F}}^2 ] \leq  \eta^2 E_s \times \mathcal{O}\left( \alpha^4 \sum_{\ell=1}^L  \Big| \mathbb{E}[\| \mathbf{P}^{(\ell)}_t \|_{\mathrm{F}}^2] - \|\mathbf{L}\|_{\mathrm{F}}^2 \Big| \right). 
\end{equation}

Then, let consider the upper-bound of Eq.~\ref{eq:sgcn_plus_bias_important_eq2}.
\begin{equation} \label{eq:sgcn_plus_bias_important_eq5}
    \begin{aligned}
    & \mathbb{E}[ \| \nabla_H \bar{f}^{(\ell)}(\mathbf{D}^{(\ell+1)}_t, \bar{\mathbf{H}}^{(\ell-1)}_t, \mathbf{W}^{(\ell)}_t ) - \nabla_H f^{(\ell)}(\mathbf{D}^{(\ell+1)}_t, \mathbf{H}^{(\ell-1)}_t, \mathbf{W}^{(\ell)}_t ) \|_{\mathrm{F}}^2 ] \\
    &= \mathbb{E}[ \| [\mathbf{L}]^\top \Big(\mathbf{D}^{(\ell+1)}_t \circ \sigma^\prime(\mathbf{P}_t^{(\ell)} \bar{\mathbf{H}}^{(\ell-1)}_t \mathbf{W}^{(\ell)}_t ) \Big) [\mathbf{W}^{(\ell)}_t ]^\top - [\mathbf{L}]^\top \Big(\mathbf{D}^{(\ell+1)}_t \circ \sigma^\prime(\mathbf{L} \mathbf{H}^{(\ell-1)}_t \mathbf{W}^{(\ell)}_t ) \Big) [\mathbf{W}^{(\ell)}_t ]^\top \|_{\mathrm{F}}^2 ] \\
    &\leq B_{LA}^2 B_D^2 B_W^2 L_\sigma^2 \mathbb{E}[ \| \bar{\mathbf{Z}}_t^{(\ell)} - \mathbf{L} \mathbf{H}^{(\ell-1)}_t \mathbf{W}^{(\ell)}_t \|_{\mathrm{F}}^2 ] \\
    &\leq 2 B_{LA}^2 B_D^2 B_W^2 L_\sigma^2 \mathbb{E}[ \| \bar{\mathbf{Z}}_t^{(\ell)} - \mathbf{L} \bar{\mathbf{H}}^{(\ell-1)}_t \mathbf{W}^{(\ell)}_t \|_{\mathrm{F}}^2 ] \\
    &\qquad + 2 B_{LA}^2 B_D^2 B_W^2 L_\sigma^2 \underbrace{\mathbb{E}[ \| \mathbf{L} \bar{\mathbf{H}}^{(\ell-1)}_t \mathbf{W}^{(\ell)}_t - \mathbf{L} \mathbf{H}^{(\ell-1)}_t \mathbf{W}^{(\ell)}_t \|_{\mathrm{F}}^2 ]}_{(A)},
    \end{aligned}
\end{equation}
where $\bar{\mathbf{Z}}_t^{(\ell)} = \bar{\mathbf{Z}}_{t-1}^{(\ell)} + \mathbf{P}^{(\ell)} \bar{\mathbf{H}}_t^{(\ell-1)} \mathbf{W}^{(\ell)}_t - \mathbf{P}^{(\ell)} \bar{\mathbf{H}}_{t-1}^{(\ell-1)} \mathbf{W}^{(\ell)}_{t-1}$.

Let take a closer look at term $(A)$, we have
\begin{equation}\label{eq:upper_bound_sgcn_plus_bias_eq1}
    \begin{aligned}
    &\mathbb{E}[ \| \mathbf{L} \bar{\mathbf{H}}^{(\ell-1)}_t \mathbf{W}^{(\ell)}_t - \mathbf{L} \mathbf{H}^{(\ell-1)}_t \mathbf{W}^{(\ell)}_t \|_{\mathrm{F}}^2 ] \\
    &\leq B_{LA}^2 B_W^2 C_\sigma^2 \mathbb{E}[\| \bar{\mathbf{Z}}^{(\ell-2)}_t  - \mathbf{L} \mathbf{H}^{(\ell-2)}_t \mathbf{W}^{(\ell-1)}_t  \|_{\mathrm{F}}^2] \\
    &\leq 2B_{LA}^2 B_W^2 C_\sigma^2 \mathbb{E}[\| \bar{\mathbf{Z}}^{(\ell-2)}_t - \mathbf{L} \bar{\mathbf{H}}^{(\ell-2)}_t \mathbf{W}^{(\ell-1)}_t  \|_{\mathrm{F}}^2] \\
    &\qquad + 2B_{LA}^2 B_W^2 C_\sigma^2 \mathbb{E}[\| \mathbf{L} \bar{\mathbf{H}}^{(\ell-2)}_t \mathbf{W}^{(\ell-1)}_t - \mathbf{L} \mathbf{H}^{(\ell-2)}_t \mathbf{W}^{(\ell-1)}_t  \|_{\mathrm{F}}^2].
    \end{aligned}
\end{equation}

Therefore, by induction we can upper bound Eq.~\ref{eq:sgcn_plus_bias_important_eq5} by
\begin{equation}
    \begin{aligned}
    & \mathbb{E}[ \| \nabla_H \bar{f}^{(\ell)}(\mathbf{D}^{(\ell+1)}_t, \bar{\mathbf{H}}^{(\ell-1)}_t, \mathbf{W}^{(\ell)}_t ) - \nabla_H f^{(\ell)}(\mathbf{D}^{(\ell+1)}_t, \mathbf{H}^{(\ell-1)}_t, \mathbf{W}^{(\ell)}_t ) \|_{\mathrm{F}}^2 ] \\
    &\leq \mathcal{O}(\mathbb{E}[ \| \bar{\mathbf{Z}}^{(\ell)}_t - \mathbf{L} \bar{\mathbf{H}}^{(\ell-1)}_t \mathbf{W}^{(\ell)}_t \|_{\mathrm{F}}^2 ]) + \mathcal{O}(\mathbb{E}[\| \bar{\mathbf{Z}}^{(\ell-1)}_t - \mathbf{L} \bar{\mathbf{H}}^{(\ell-2)}_t \mathbf{W}^{(\ell-1)}_t  \|_{\mathrm{F}}^2]) + \ldots \\
    &\qquad + \mathcal{O}(\mathbb{E}[\| \bar{\mathbf{Z}}^{(2)}_t - \mathbf{L} \bar{\mathbf{H}}^{(1)}_t \mathbf{W}^{(2)}_t  \|_{\mathrm{F}}^2]) + \mathcal{O}(\mathbb{E}[\| \bar{\mathbf{Z}}^{(1)}_t - \mathbf{L} \mathbf{X} \mathbf{W}^{(1)}_t  \|_{\mathrm{F}}^2]).
    \end{aligned}
\end{equation}

Using result from Lemma~\ref{lemma:upper-bound-sgcn-plus-bias-1}, we have
\begin{equation}
    \begin{aligned}
    & \mathbb{E}[ \| \nabla_H \bar{f}^{(\ell)}(\mathbf{D}^{(\ell+1)}_t, \bar{\mathbf{H}}^{(\ell-1)}_t, \mathbf{W}^{(\ell)}_t) - \nabla_H f^{(\ell)}(\mathbf{D}^{(\ell+1)}_t, \mathbf{H}^{(\ell-1)}_t, \mathbf{W}^{(\ell)}_t) \|_{\mathrm{F}}^2 ] \\
    &\leq \eta^2 E_s \times \mathcal{O} \left( \alpha^4 \sum_{j=1}^\ell | \mathbb{E}[\| \mathbf{P}_t^{(j)} \|_{\mathrm{F}}^2] - \|\mathbf{L}\|_{\mathrm{F}}^2 | \right).
    \end{aligned}
\end{equation}

Finally, let consider the upper-bound of Eq.~\ref{eq:sgcn_plus_bias_important_eq3}.
\begin{equation}
    \begin{aligned}
    &\mathbb{E}[ \| \nabla_W \bar{f}^{(\ell)}(\mathbf{D}_t^{(\ell+1)}, \bar{\mathbf{H}}_t^{(\ell-1)}, \mathbf{W}^{(\ell)}_t) - \nabla_W f^{(\ell)}(\mathbf{D}_t^{(\ell+1)}, \mathbf{H}^{(\ell-1)}_t, \mathbf{W}^{(\ell)}_t) \|_{\mathrm{F}}^2 ] \\
    &= \mathbb{E}[ \| [\mathbf{L} \bar{\mathbf{H}}_t^{(\ell-1)}]^\top \Big(\mathbf{D}_t^{(\ell+1)} \circ \sigma^\prime(\bar{\mathbf{Z}}^{(\ell)}_t) \Big) - [\mathbf{L} \mathbf{H}^{(\ell-1)}_t]^\top \Big(\mathbf{D}_t^{(\ell+1)} \circ  \sigma^\prime(\mathbf{L}^{(\ell)}\mathbf{H}^{(\ell-1)}_t\mathbf{W}^{(\ell)}_t)\Big) \|_{\mathrm{F}}^2 ]\\
    &\leq 2\mathbb{E}[ \| [\mathbf{L} \bar{\mathbf{H}}_t^{(\ell-1)}]^\top \Big(\mathbf{D}_t^{(\ell+1)} \circ \sigma^\prime(\bar{\mathbf{Z}}^{(\ell)}_t ) \Big) - [\mathbf{L} \mathbf{H}^{(\ell-1)}_t]^\top \Big(\mathbf{D}_t^{(\ell+1)} \circ \sigma^\prime(\bar{\mathbf{Z}}^{(\ell)}_t) \Big) \|_{\mathrm{F}}^2 ]\\
    &\qquad + 2\mathbb{E}[ \| [\mathbf{L} \mathbf{H}^{(\ell-1)}_t]^\top \Big(\mathbf{D}_t^{(\ell+1)} \circ \sigma^\prime(\bar{\mathbf{Z}}^{(\ell)}_t ) \Big) - [\mathbf{L} \mathbf{H}^{(\ell-1)}_t]^\top \Big(\mathbf{D}_t^{(\ell+1)} \circ  \sigma^\prime(\mathbf{L}^{(\ell)}\mathbf{H}^{(\ell-1)}_t\mathbf{W}^{(\ell)}_t)\Big) \|_{\mathrm{F}}^2 ] \\
    &\leq 2 B_D^2 C_\sigma^2 B_{LA}^2 \underbrace{\mathbb{E}[ \| \bar{\mathbf{H}}_t^{(\ell-1)} - \mathbf{H}^{(\ell-1)}_t \|_{\mathrm{F}}^2 ]}_{(B)} \\
    &\quad + 4B_{LA}^2 \alpha^2 B_{H}^2 B_D^2 L_\sigma^2 \Big( \mathbb{E}[ \| \bar{\mathbf{Z}}^{(\ell)}_t - \mathbf{L} \bar{\mathbf{H}}_t^{(\ell-1)} \mathbf{W}^{(\ell)}_t \|_{\mathrm{F}}^2 ] + \mathbb{E}[ \| \mathbf{L} \bar{\mathbf{H}}_t^{(\ell-1)} \mathbf{W}^{(\ell)}_t - \mathbf{L} \mathbf{H}^{(\ell-1)}_t \mathbf{W}^{(\ell)}_t \|_{\mathrm{F}}^2 ] \Big).
    \end{aligned}
\end{equation}
By definition, we can write the term $(B)$ as 
\begin{equation}
    \begin{aligned}
    \mathbb{E}[ \| \bar{\mathbf{H}}_t^{(\ell-1)} - \mathbf{H}^{(\ell-1)}_t \|_{\mathrm{F}}^2 ] 
    &\leq C_\sigma^2 \mathbb{E}[\| \bar{\mathbf{Z}}^{(\ell-1)}_t - \mathbf{L} \mathbf{H}^{(\ell-2)}_t \mathbf{W}^{(\ell-1)}_t \|_{\mathrm{F}}^2] \\
    &\leq 2C_\sigma^2 \mathbb{E}[\| \bar{\mathbf{Z}}^{(\ell-1)}_t - \mathbf{L} \bar{\mathbf{H}}_t^{(\ell-2)} \mathbf{W}^{(\ell-1)}_t \|_{\mathrm{F}}^2] \\
    &\qquad + 2C_\sigma^2 \mathbb{E}[\| \mathbf{L} \bar{\mathbf{H}}_t^{(\ell-2)} \mathbf{W}^{(\ell)}_t - \mathbf{L} \mathbf{H}^{(\ell-2)}_t \mathbf{W}^{(\ell-1)}_t \|_{\mathrm{F}}^2].
    \end{aligned}
\end{equation}
Plugging term $(B)$ back and using Eq.~\ref{eq:upper_bound_sgcn_plus_bias_eq1} and Lemma~\ref{lemma:upper-bound-sgcn-plus-bias-1}, we have
\begin{equation}
    \begin{aligned}
    & \mathbb{E}[ \| \nabla_W \bar{f}^{(\ell)}(\mathbf{D}_t^{(\ell+1)}, \bar{\mathbf{H}}_t^{(\ell-1)}, \mathbf{W}^{(\ell)}_t) - \nabla_W f^{(\ell)}(\mathbf{D}_t^{(\ell+1)}, \mathbf{H}^{(\ell-1)}_t, \mathbf{W}^{(\ell)}_t) \|_{\mathrm{F}}^2 ] \\
    &\leq \eta^2 E_s \times \mathcal{O} \left( \alpha^4 \sum_{j=1}^\ell | \mathbb{E}[\| \mathbf{P}_t^{(j)} \|_{\mathrm{F}}^2] - \|\mathbf{L}\|_{\mathrm{F}}^2 | \right). 
    \end{aligned}
\end{equation}

Combining the result from Eq.~\ref{eq:sgcn_plus_bias_important_eq1}, \ref{eq:sgcn_plus_bias_important_eq2}, \ref{eq:sgcn_plus_bias_important_eq3} we have
\begin{equation}
    \mathbb{E}[ \| \mathbb{E}[\widetilde{\mathbf{G}}_t^{(\ell)}] - \mathbf{G}_t^{(\ell)}\|_{\mathrm{F}}^2 ] \leq \eta^2 E_s \times \mathcal{O} \left( \alpha^4 \sum_{j=1}^\ell | \mathbb{E}[\| \mathbf{P}_t^{(j)} \|_{\mathrm{F}}^2] - \|\mathbf{L}\|_{\mathrm{F}}^2 | \right).
\end{equation}

\end{proof}

\subsection{Remaining steps toward Theorem~\ref{theorem:convergence_of_sgcn_plus}}
Now we are ready to prove Theorem~\ref{theorem:convergence_of_sgcn_plus}.
By the smoothness of $\mathcal{L}(\bm{\theta}_t)$, we have 
\begin{equation}
    \begin{aligned}
    \mathcal{L}(\bm{\theta}_{t+1}) &\leq \mathcal{L}(\bm{\theta}_t) + \langle \nabla \mathcal{L}(\bm{\theta}_t), \bm{\theta}_{t+1} - \bm{\theta}_t \rangle + \frac{L_{\mathrm{F}}}{2}\|\bm{\theta}_{t+1} - \bm{\theta}_t\|^2 \\
    &= \mathcal{L}(\bm{\theta}_t) - \eta \langle \nabla \mathcal{L}(\bm{\theta}_t), \nabla \widetilde{\mathcal{L}}(\bm{\theta}_t) \rangle + \frac{\eta^2 L_{\mathrm{F}} }{2}\|\nabla \widetilde{\mathcal{L}}(\bm{\theta}_t)\|^2.
    \end{aligned}
\end{equation}

Let $\mathcal{F}_t = \{ \{\mathcal{B}_1^{(\ell)}\}_{\ell=1}^L ,\ldots, \{\mathcal{B}_{t-1}^{(\ell)}\}_{\ell=1}^L\}$.
Note that the weight parameters $\bm{\theta}_t$ is a function of history of the generated random process and hence is random.
Taking expectation on both sides condition on $\mathcal{F}_t$ and using $\eta<1/L_{\mathrm{F}}$ we have
\begin{equation}
    \begin{aligned}
    &\mathbb{E}[\nabla \mathcal{L}(\bm{\theta}_{t+1}) | \mathcal{F}_t] \\
    &\leq \mathcal{L}(\bm{\theta}_t) - \eta \langle \nabla \mathcal{L}(\bm{\theta}_t), \mathbb{E}[\nabla \widetilde{\mathcal{L}}(\bm{\theta}_t)|\mathcal{F}_t] \rangle + \frac{\eta^2 L_{\mathrm{F}} }{2} \Big(\mathbb{E}[\|\nabla \widetilde{\mathcal{L}}(\bm{\theta}_t) - \mathbb{E}[\nabla \widetilde{\mathcal{L}}(\bm{\theta}_t)|\mathcal{F}_t]\|^2 | \mathcal{F}_t] + \mathbb{E}[\|\mathbb{E}[\mathbf{g}|\mathcal{F}_t]\|^2 | \mathcal{F}_t]\Big) \\
    &= \mathcal{L}(\bm{\theta}_t) - \eta \langle \nabla \mathcal{L}(\bm{\theta}_t), \nabla \mathcal{L}(\bm{\theta}_t) + \mathbb{E}[\mathbf{b}_t|\mathcal{F}_t] \rangle + \frac{\eta^2 L_{\mathrm{F}} }{2} \Big(\mathbb{E}[\|\mathbf{n}_t\|^2 | \mathcal{F}_t] + \|\nabla \mathcal{L}(\bm{\theta}_t) + \mathbb{E}[\mathbf{b}_t | \mathcal{F}_t] \|^2 \Big) \\
    &\leq \mathcal{L}(\bm{\theta}_t) + \frac{\eta}{2}\Big( - 2 \langle \nabla \mathcal{L}(\bm{\theta}_t), \nabla \mathcal{L}(\bm{\theta}_t) + \mathbb{E}[\mathbf{b}_t|\mathcal{F}_t] \rangle + \|\nabla \mathcal{L}(\bm{\theta}_t) + \mathbb{E}[\mathbf{b}_t|\mathcal{F}_t] \|^2 \Big) + \frac{\eta^2 L_{\mathrm{F}}}{2} \mathbb{E}[\|\mathbf{n}_t\|^2|\mathcal{F}_t] \\
    &\leq \mathcal{L}(\bm{\theta}_t) + \frac{\eta}{2}\Big( - \| \nabla \mathcal{L}(\bm{\theta}_t)\|^2 + \mathbb{E}[\|\mathbf{b}_t \|^2|\mathcal{F}_t] \Big) + \frac{\eta^2 L_{\mathrm{F}}}{2} \mathbb{E}[\|\mathbf{n}_t\|^2|\mathcal{F}_t].
    \end{aligned}
\end{equation}

Plugging in the upper bound of bias and variance, taking expectation over $\mathcal{F}_t$, and rearranging the term we have
\begin{equation}
    \mathbb{E}[\| \nabla \mathcal{L}(\bm{\theta}_t)\|^2] \leq \frac{2}{\eta} \Big( \mathbb{E}[\mathcal{L}(\bm{\theta}_t)] - \mathbb{E}[\nabla \mathcal{L}(\bm{\theta}_{t+1})]\Big)  + \eta L_{\mathrm{F}} \mathbb{E}[\| \mathbf{n}_t \|^2] + \eta^2 \mathbb{E}[\| \mathbf{b}_t \|^2].
\end{equation}

Summing up from $t=1$ to $T$, rearranging we have
\begin{equation}
\begin{aligned}
    \frac{1}{T} \sum_{t=1}^T \mathbb{E}[ \|\nabla \mathcal{L}(\bm{\theta}_t)\|^2 ] &\leq \frac{2}{\eta T}\sum_{t=1}^T (\mathbb{E}[\mathcal{L}(\bm{\theta}_t)] - \mathbb{E}[\mathcal{L}(\bm{\theta}_{t+1})])  + \eta L_{\mathrm{F}} \Delta_\mathbf{n} + \eta^2 \Delta_\mathbf{b}^{+\prime} \\
    &\underset{(a)}{\leq} \frac{2}{\eta T}( \mathcal{L}(\bm{\theta}_1) - \mathcal{L}(\bm{\theta}^\star) ) + \eta L_{\mathrm{F}} \Delta_\mathbf{n} + \eta^2 \Delta_\mathbf{b}^{+\prime},
\end{aligned}
\end{equation}
where the inequality $(a)$ is due to $\mathcal{L}(\bm{\theta}^\star) \leq \mathbb{E}[\mathcal{L}(\bm{\theta}_{T+1})]$.

By selecting learning rate as $\eta=1/\sqrt{T}$, we have
\begin{equation}
    \begin{aligned}
    \frac{1}{T} \sum_{t=1}^T \mathbb{E}[ \|\nabla \mathcal{L}(\bm{\theta}_t)\|^2 ] &\leq \frac{2(\mathcal{L}(\bm{\theta}_1) - \mathcal{L}(\bm{\theta}^\star))}{\sqrt{T}} + \frac{L_{\mathrm{F}} \Delta_\mathbf{n}}{\sqrt{T}} + \frac{\Delta_\mathbf{b}^{+\prime}}{T}.
    \end{aligned}
\end{equation}

\subsection{Discussion on the elimination of bias term}
Recall that the existence of bias term $\Delta_b$ in Theorem 1 is due to inner-layer sampling. 
In the analysis of \texttt{SGCN}, since the bias term $\Delta_b$ is not multiplied by the learning rate $\eta$ (Eq. 110 - 111), the bias term is not decreasing as the $T$ increases when selecting $\eta = 1/\sqrt{T}$.
To resolve this non-decreasing bias term, \texttt{SGCN+} proposes to use historical node embeddings for variance reduction, such that the bias term $\Delta_b^+$ is controlled by the difference between the current model to the snapshot model. Because the difference between the current model to the snapshot model is affected by the selection of $\eta$ (Lemma 4), in the analysis of \texttt{SGCN+}, the bias term $\Delta_b^+$ is multiplied by $\eta^2$ (Eq. 151). Therefore, by selecting $\eta = 1/\sqrt{T}$ the bias term is decreasing as the number of iterations increases.

%% file: appendix/proof_sgcn_thm_3.tex
\section{Proof of Theorem~\ref{theorem:convergence_of_sgcn_plus_plus}}\label{section:proof_of_thm3}

\subsection{Upper-bounded on the node embedding matrices and layerwise gradients}

When using variance reduction algorithm, we cannot use upper bound on the node embedding matrices and layerwise gradients as derived in Proposition~\ref{proposition:matrix_norm_bound}.

To see this, let consider the upper bound on the node embedding matrix before the activation function $\widetilde{\mathbf{Z}}^{(\ell)}_t$, we have
\begin{equation}
    \begin{aligned}
    \| \widetilde{\mathbf{Z}}^{(\ell)}_t \|_\mathrm{F} 
    &= \| \widetilde{\mathbf{Z}}^{(\ell)}_{t-1} + \widetilde{\mathbf{L}}_t^{(\ell)} \widetilde{\mathbf{H}}^{(\ell-1)}_t \mathbf{W}_t^{(\ell)} - \widetilde{\mathbf{L}}_t^{(\ell)} \widetilde{\mathbf{H}}^{(\ell-1)}_{t-1} \mathbf{W}_{t-1}^{(\ell)}  \|_\mathrm{F} \\
    &\leq  \| \widetilde{\mathbf{Z}}^{(\ell)}_{t-1} \|_\mathrm{F}+  \| \widetilde{\mathbf{L}}_t^{(\ell)} \widetilde{\mathbf{H}}^{(\ell-1)}_t \mathbf{W}_t^{(\ell)} - \widetilde{\mathbf{L}}_t^{(\ell)} \widetilde{\mathbf{H}}^{(\ell-1)}_{t-1} \mathbf{W}_{t-1}^{(\ell)} \|_\mathrm{F}.
    \end{aligned}
\end{equation}

Similarly, we can compute the upper bound on the gradient passing from the $\ell+1$ to the $\ell$th layer as
\begin{equation}
    \begin{aligned}
        \| \widetilde{\mathbf{D}}_t^{(\ell)} \|_\mathrm{F}
    &= \left\| \widetilde{\mathbf{D}}_{t-1}^{(\ell)}  + [\widetilde{\mathbf{L}}^{(\ell)}]^\top \Big(\widetilde{\mathbf{D}}_t^{(\ell+1)} \circ \sigma^\prime(\widetilde{\mathbf{Z}}_t) \Big) [\mathbf{W}_t^{(\ell)}] - [\widetilde{\mathbf{L}}^{(\ell)}]^\top \Big(\widetilde{\mathbf{D}}_{t-1}^{(\ell+1)} \circ \sigma^\prime(\widetilde{\mathbf{Z}}_{t-1}) \Big) [\mathbf{W}_{t-1}^{(\ell)}] \right\|_\mathrm{F} \\
    &\leq \| \widetilde{\mathbf{D}}_{t-1}^{(\ell)} \|_\mathrm{F} + \left\| [\widetilde{\mathbf{L}}^{(\ell)}]^\top \Big(\widetilde{\mathbf{D}}_t^{(\ell+1)} \circ \sigma^\prime(\widetilde{\mathbf{Z}}_t) \Big) [\mathbf{W}_t^{(\ell)}] - [\widetilde{\mathbf{L}}^{(\ell)}]^\top \Big(\widetilde{\mathbf{D}}_{t-1}^{(\ell+1)} \circ \sigma^\prime(\widetilde{\mathbf{Z}}_{t-1}) \Big) [\mathbf{W}_{t-1}^{(\ell)}] \right\|_\mathrm{F}.
    \end{aligned}
\end{equation}

Similar to the discussion we had in the proof of Theorem~\ref{theorem:convergence_of_sgcn_plus}, in the worst case, the matrix norm are  growing as the inner loop size. To control the growth on matrix norm, we introduced an early stopping criterion by checking the norm of node embeddings and gradient, and immediately start another snapshot step if the condition is triggered. 

Let define $t_s$ as the stopping time of the inner loop in the $s$th outer loop,
\begin{equation}
    e_s = e_{s-1} + \min \Big\{ \max_{k \geq e_{s-1}} \{ \| \widetilde{\mathbf{H}}_k^{(\ell)}\|_\mathrm{F} \leq \alpha \| \widetilde{\mathbf{H}}_{e_{s-1}}^{(\ell)} \|_\mathrm{F}~\text{and}~\| \widetilde{\mathbf{D}}_k^{(\ell)}\|_\mathrm{F} \leq \beta \| \widetilde{\mathbf{D}}_{e_{s-1}}^{(\ell)} \|_\mathrm{F} \}, K \Big\},~
    t_0 = 1,
\end{equation}
and $E_s = e_s - e_{s-1} \leq K$ as the snapshot gap at the $s$th inner-loop and $S$ as the total number of outer loops. 
By doing so, we can upper bounded the node embedding matrix by $\| \mathbf{H}^{(\ell)}_t \|_\mathrm{F} \leq \alpha B_H, \| \mathbf{D}^{(\ell)}_t \|_\mathrm{F} \leq \beta B_D,~\forall t \in [T]$.

\subsection{Supporting lemmas}

In the following lemma, we decompose the mean-square error of stochastic gradient at the $\ell$th layer $\mathbb{E}[\|\widetilde{\mathbf{G}}^{(\ell)} - \mathbf{G}^{(\ell)}\|F^2]$ as the summation of 
\begin{itemize}
    \item The difference between the gradient with respect to the last layer node embedding matrix
    \begin{equation}\label{eq:sgcn_pplus_important_steps_1}
        \mathbb{E}[ \|\widetilde{\mathbf{D}}^{(L+1)} - \mathbf{D}^{(L+1)} \|_{\mathrm{F}}^2 ].
    \end{equation}
    \item The difference of gradient passing from the $(\ell+1)$th layer node embedding to the $\ell$th layer node embedding
    \begin{equation}\label{eq:sgcn_pplus_important_steps_2}
        \mathbb{E}[ \| \nabla_H \widetilde{f}^{(\ell+1)}( \mathbf{D}^{(\ell+2)}, \widetilde{\mathbf{H}}^{(\ell)}, \mathbf{W}^{(\ell+1)}) - \nabla_H f^{(\ell+1)}( \mathbf{D}^{(\ell+2)}, \mathbf{H}^{(\ell)}, \mathbf{W}^{(\ell+1)})\|_{\mathrm{F}}^2 ].
    \end{equation}
    \item The difference of gradient passing from the $\ell$th layer node embedding to the $\ell$th layer weight matrix
    \begin{equation}\label{eq:sgcn_pplus_important_steps_3}
        \mathbb{E}[ \| \nabla_W \widetilde{f}^{(\ell)}( \mathbf{D}^{(\ell+1)}, \widetilde{\mathbf{H}}^{(\ell-1)}, \mathbf{W}^{(\ell)}) - \nabla_W f^{(\ell)}( \mathbf{D}^{(\ell+1)}, \mathbf{H}^{(\ell-1)}, \mathbf{W}^{(\ell)}) \|_{\mathrm{F}}^2 ].
    \end{equation}
\end{itemize}

\begin{lemma}\label{lemma:sgcn_pplus_lemma1}
The mean-square error of stochastic gradient at the $\ell$th layer can be decomposed as 
    \begin{equation}
        \begin{aligned}
        &\mathbb{E}[\|\widetilde{\mathbf{G}}^{(\ell)} - \mathbf{G}^{(\ell)}\|F^2] \\
        &\leq \mathcal{O}(\mathbb{E}[ \|\widetilde{\mathbf{D}}^{(L+1)} - \mathbf{D}^{(L+1)} \|_{\mathrm{F}}^2 ]) \\
        &\qquad + \mathcal{O}(\mathbb{E}[ \| \nabla_H \widetilde{f}^{(L)}( \mathbf{D}^{(L+1)}, \widetilde{\mathbf{H}}^{(L-1)} \mathbf{W}^{(L)}) - \nabla_H f^{(L)}( \mathbf{D}^{(L+1)}, \mathbf{H}^{(L-1)} \mathbf{W}^{(L)}) \|_{\mathrm{F}}^2 ]) + \ldots\\
        &\qquad + \mathcal{O}(\mathbb{E}[ \| \nabla_H \widetilde{f}^{(\ell+1)}( \mathbf{D}^{(\ell+2)}, \widetilde{\mathbf{H}}^{(\ell)}, \mathbf{W}^{(\ell+1)}) - \nabla_H f^{(\ell+1)}( \mathbf{D}^{(\ell+2)}, \mathbf{H}^{(\ell)}, \mathbf{W}^{(\ell+1)})\|_{\mathrm{F}}^2 ]) \\
        &\qquad + \mathcal{O}(\mathbb{E}[ \| \nabla_W \widetilde{f}^{(\ell)}( \mathbf{D}^{(\ell+1)}, \widetilde{\mathbf{H}}^{(\ell-1)}, \mathbf{W}^{(\ell)}) - \nabla_W f^{(\ell)}( \mathbf{D}^{(\ell+1)}, \mathbf{H}^{(\ell-1)}, \mathbf{W}^{(\ell)}) \|_{\mathrm{F}}^2 ]).
        \end{aligned}
    \end{equation}
\end{lemma}
\begin{proof}
By definition, we can write down the mean-square error of stochastic gradient as
\begin{equation}
    \begin{aligned}
    &\mathbb{E}[ \| \widetilde{\mathbf{G}}^{(\ell)} - \mathbf{G}^{(\ell)} \|_{\mathrm{F}}^2 ] \\
    &= \mathbb{E}[ \| \nabla_W \widetilde{f}^{(\ell)}( \nabla_H \widetilde{f}^{(\ell+1)}( \ldots \nabla_H \widetilde{f}^{(L)}( \widetilde{\mathbf{D}}^{(L+1)}, \widetilde{\mathbf{H}}^{(L-1)} \mathbf{W}^{(L)}) \ldots, \widetilde{\mathbf{H}}^{(\ell)}, \mathbf{W}^{(\ell+1)}) , \widetilde{\mathbf{H}}^{(\ell-1)}, \mathbf{W}^{(\ell)}) \\
    &\quad - \nabla_W f^{(\ell)}( \nabla_H f^{(\ell+1)}( \ldots \nabla_H f^{(L)}( \mathbf{D}^{(L+1)}, \mathbf{H}^{(L-1)}, \mathbf{W}^{(L)}) \ldots, \mathbf{H}^{(\ell)}, \mathbf{W}^{(\ell+1)}) , \mathbf{H}^{(\ell-1)}), \mathbf{W}^{(\ell)} \|_{\mathrm{F}}^2 ] \\
    &\leq (L+1) \mathbb{E}[  \| \nabla_W \widetilde{f}^{(\ell)}( \nabla_H \widetilde{f}^{(\ell+1)}( \ldots \nabla_H \widetilde{f}^{(L)}( \widetilde{\mathbf{D}}^{(L+1)}, \widetilde{\mathbf{H}}^{(L-1)} \mathbf{W}^{(L)}) \ldots, \widetilde{\mathbf{H}}^{(\ell)}, \mathbf{W}^{(\ell+1)}) , \widetilde{\mathbf{H}}^{(\ell-1)}, \mathbf{W}^{(\ell)}) \\
    &\qquad - \nabla_W \widetilde{f}^{(\ell)}( \nabla_H \widetilde{f}^{(\ell+1)}( \ldots \nabla_H \widetilde{f}^{(L)}( \mathbf{D}^{(L+1)}, \widetilde{\mathbf{H}}^{(L-1)} \mathbf{W}^{(L)}) \ldots, \mathbf{H}^{(\ell)}, \mathbf{W}^{(\ell+1)}) , \mathbf{H}^{(\ell-1)}, \mathbf{W}^{(\ell)}) \|_{\mathrm{F}}^2 ]\\
    &\quad + (L+1) \mathbb{E}[ \| \nabla_W \widetilde{f}^{(\ell)}( \nabla_H \widetilde{f}^{(\ell+1)}( \ldots \nabla_H \widetilde{f}^{(L)}( \mathbf{D}^{(L+1)}, \widetilde{\mathbf{H}}^{(L-1)} \mathbf{W}^{(L)}) \ldots, \widetilde{\mathbf{H}}^{(\ell)}, \mathbf{W}^{(\ell+1)}) , \widetilde{\mathbf{H}}^{(\ell-1)}, \mathbf{W}^{(\ell)}) \\
    &\qquad - \nabla_W \widetilde{f}^{(\ell)}( \nabla_H \widetilde{f}^{(\ell+1)}( \ldots \nabla_H f^{(L)}( \mathbf{D}^{(L+1)}, \mathbf{H}^{(L-1)} \mathbf{W}^{(L)}) \ldots, \widetilde{\mathbf{H}}^{(\ell)}, \mathbf{W}^{(\ell+1)}) , \widetilde{\mathbf{H}}^{(\ell-1)}, \mathbf{W}^{(\ell)}) \|_{\mathrm{F}}^2 ] + \ldots \\
    &\quad + (L+1) \mathbb{E}[ \| \nabla_W \widetilde{f}^{(\ell)}( \nabla_H \widetilde{f}^{(\ell+1)}( \mathbf{D}^{(\ell+2)}, \widetilde{\mathbf{H}}^{(\ell)}, \mathbf{W}^{(\ell+1)}) , \widetilde{\mathbf{H}}^{(\ell-1)}, \mathbf{W}^{(\ell)})  \\
    &\qquad - \nabla_W \widetilde{f}^{(\ell)}( \nabla_H f^{(\ell+1)}( \mathbf{D}^{(\ell+2)}, \mathbf{H}^{(\ell)}, \mathbf{W}^{(\ell+1)}) , \widetilde{\mathbf{H}}^{(\ell-1)}, \mathbf{W}^{(\ell)}) \|_{\mathrm{F}}^2 ] \\
    &\quad + (L+1) \mathbb{E}[ \| \nabla_W \widetilde{f}^{(\ell)}( \mathbf{D}^{(\ell+1)}, \widetilde{\mathbf{H}}^{(\ell-1)}, \mathbf{W}^{(\ell)}) - \nabla_W f^{(\ell)}( \mathbf{D}^{(\ell+1)}, \mathbf{H}^{(\ell-1)}, \mathbf{W}^{(\ell)}) \|_{\mathrm{F}}^2 ] \\
    &\leq \mathcal{O}(\mathbb{E}[ \|\widetilde{\mathbf{D}}^{(L+1)} - \mathbf{D}^{(L+1)} \|_{\mathrm{F}}^2 ]) \\
    &\qquad + \mathcal{O}(\mathbb{E}[ \| \nabla_H \widetilde{f}^{(L)}( \mathbf{D}^{(L+1)}, \widetilde{\mathbf{H}}^{(L-1)} \mathbf{W}^{(L)}) - \nabla_H f^{(L)}( \mathbf{D}^{(L+1)}, \mathbf{H}^{(L-1)} \mathbf{W}^{(L)}) \|_{\mathrm{F}}^2 ]) + \ldots\\
    &\qquad + \mathcal{O}(\mathbb{E}[ \| \nabla_H \widetilde{f}^{(\ell+1)}( \mathbf{D}^{(\ell+2)}, \widetilde{\mathbf{H}}^{(\ell)}, \mathbf{W}^{(\ell+1)}) - \nabla_H f^{(\ell+1)}( \mathbf{D}^{(\ell+2)}, \mathbf{H}^{(\ell)}, \mathbf{W}^{(\ell+1)})\|_{\mathrm{F}}^2 ]) \\
    &\qquad + \mathcal{O}(\mathbb{E}[ \| \nabla_W \widetilde{f}^{(\ell)}( \mathbf{D}^{(\ell+1)}, \widetilde{\mathbf{H}}^{(\ell-1)}, \mathbf{W}^{(\ell)}) - \nabla_W f^{(\ell)}( \mathbf{D}^{(\ell+1)}, \mathbf{H}^{(\ell-1)}, \mathbf{W}^{(\ell)}) \|_{\mathrm{F}}^2 ]).
    \end{aligned}
\end{equation}
\end{proof}

Recall the definition of stochastic gradient for all model parameters $\nabla \widetilde{\mathcal{L}}(\bm{\theta}_t) = \{ \widetilde{\mathbf{G}}_t^{(\ell)}\}_{\ell=1}^L$ where $\widetilde{\mathbf{G}}_t^{(\ell)}$ is the gradient for the $\ell$th weight matrix, i.e.,
\begin{equation}
    \mathbf{W}^{(\ell)}_t = \mathbf{W}^{(\ell)}_{t-1} - \eta \widetilde{\mathbf{G}}_{t-1}^{(\ell)}.
\end{equation}

In the following lemma, we derive the upper-bound on the difference of the gradient passing from the $\ell$th to $(\ell-1)$th layer given the same inputs $\mathbf{D}_t^{(\ell+1)},~\widetilde{\mathbf{H}}_t^{(\ell-1)}$, where the backward propagation for the $\ell$th layer in \texttt{SGCN++} is defined as
\begin{equation}
    \begin{aligned}
    &\nabla_H \widetilde{f}^{(\ell)}(\mathbf{D}_t^{(\ell+1)}, \widetilde{\mathbf{H}}_t^{(\ell-1)}, \mathbf{W}_t^{(\ell)}) \\
    &=\widetilde{\mathbf{D}}_{t-1}^{(\ell)} + [\widetilde{\mathbf{L}}^{(\ell)}_t]^\top (\mathbf{D}_t^{(\ell+1)} \circ \sigma^\prime(\widetilde{\mathbf{Z}}_t^{(\ell)})) \mathbf{W}_t^{(\ell)} - [\widetilde{\mathbf{L}}^{(\ell)}_t]^\top (\mathbf{D}_{t-1}^{(\ell+1)} \circ \sigma^\prime(\widetilde{\mathbf{Z}}_{t-1}^{(\ell)}))  \mathbf{W}_{t-1}^{(\ell)}],
    \end{aligned}
\end{equation}
and  the backward propagation for the $\ell$th layer in \texttt{FullGCN} is defined as
\begin{equation}
    \nabla_H f^{(\ell)}(\mathbf{D}_t^{(\ell+1)}, \widetilde{\mathbf{H}}_t^{(\ell-1)}, \mathbf{W}_t^{(\ell)}) = \mathbf{L}^\top (\mathbf{D}_t^{(\ell+1)} \circ \sigma^\prime(\widetilde{\mathbf{Z}}_t^{(\ell)}))\mathbf{W}_t^{(\ell)}.
\end{equation}

\begin{lemma} \label{lemma:gradient_H_SGCN++}
Let consider the $s$th epoch, let $t$ be the current step. Therefore, for any $t\in\{e_{s-1}, \ldots, e_s\}$, we can upper-bound the difference of the gradient with respect to the input node embedding matrix at the $\ell$th graph convolutional layer given the same input $\mathbf{D}_t^{(\ell+1)}$ and $\widetilde{\mathbf{H}}_t^{(\ell-1)}$ by
\begin{equation}
    \begin{aligned}
    &\mathbb{E}[\|\nabla_H \widetilde{f}^{(\ell)}(\mathbf{D}_t^{(\ell+1)}, \widetilde{\mathbf{H}}_t^{(\ell-1)}, \mathbf{W}_t^{(\ell)}) - \nabla_H f^{(\ell)}(\mathbf{D}_t^{(\ell+1)}, \widetilde{\mathbf{H}}_t^{(\ell-1)}, \mathbf{W}_t^{(\ell)}) \|_{\mathrm{F}}^2] \\
    &\leq \sum_{t=e_{s-1}+1}^{e_{s}} \eta^2 \mathcal{O} \left( \Big( \sum_{\ell=1}^L | \mathbb{E}[\|\widetilde{\mathbf{L}}^{(\ell)}\|_{\mathrm{F}}^2] - \| \mathbf{L} \|_{\mathrm{F}}^2 | \Big) \times (\alpha^2 + \beta^2 + \alpha^2 \beta^2) \mathbb{E}[\|\nabla \widetilde{\mathcal{L}}(\bm{\theta}_{t-1})\|_{\mathrm{F}}^2] \right)
    \end{aligned}
\end{equation}
\end{lemma}
\begin{proof}
To simplify the presentation, let us denote $\widetilde{\mathbf{D}}_t^{(\ell)} = \nabla_H \widetilde{f}^{(\ell)}(\mathbf{D}_t^{(\ell+1)}, \widetilde{\mathbf{H}}_t^{(\ell-1)}, \mathbf{W}_t^{(\ell)})$. Then, by definition we have
\begin{equation}
    \widetilde{\mathbf{D}}_t^{(\ell)} = \widetilde{\mathbf{D}}_{t-1}^{(\ell)} + [\widetilde{\mathbf{L}}^{(\ell)}_t]^\top (\mathbf{D}_t^{(\ell+1)} \circ \sigma^\prime(\widetilde{\mathbf{Z}}_t^{(\ell)})) \mathbf{W}_t^{(\ell)} - [\widetilde{\mathbf{L}}^{(\ell)}_t]^\top (\mathbf{D}_{t-1}^{(\ell+1)} \circ \sigma^\prime(\widetilde{\mathbf{Z}}_{t-1}^{(\ell)}))  \mathbf{W}_{t-1}^{(\ell)}.
\end{equation}
Therefore, we know that
\begin{equation} \label{eq:gradient_H_SGCN++_eq1}
    \begin{aligned}
    &\|\mathbf{L}^\top (\mathbf{D}_t^{(\ell+1)} \circ \sigma^\prime(\widetilde{\mathbf{Z}}_t^{(\ell)})) \mathbf{W}_t^{(\ell)} - \widetilde{\mathbf{D}}^{(\ell)}_t \|_{\mathrm{F}}^2 \\
    &= \| [ \mathbf{L}^\top (\mathbf{D}_t^{(\ell+1)} \circ \sigma^\prime(\widetilde{\mathbf{Z}}_t^{(\ell)})) \mathbf{W}_t^{(\ell)} - \mathbf{L}^\top (\mathbf{D}_{t-1}^{(\ell+1)} \circ \sigma^\prime(\widetilde{\mathbf{Z}}_{t-1}^{(\ell)})) \mathbf{W}_{t-1}^{(\ell)} ] \\
    & \qquad + [ \mathbf{L}^\top (\mathbf{D}_{t-1}^{(\ell+1)} \circ \sigma^\prime(\widetilde{\mathbf{Z}}_{t-1}^{(\ell)})) \mathbf{W}_{t-1}^{(\ell)} - \widetilde{\mathbf{D}}_{t-1}^{(\ell)} ] - [ \widetilde{\mathbf{D}}_{t}^{(\ell)} - \widetilde{\mathbf{D}}_{t-1}^{(\ell)}] \|_{\mathrm{F}}^2 \\
    &\leq \| \underbrace{\mathbf{L}^\top (\mathbf{D}_t^{(\ell+1)} \circ \sigma^\prime(\widetilde{\mathbf{Z}}_t^{(\ell)})) \mathbf{W}_t^{(\ell)} - \mathbf{L}^\top (\mathbf{D}_{t-1}^{(\ell+1)} \circ \sigma^\prime(\widetilde{\mathbf{Z}}_{t-1}^{(\ell)})) \mathbf{W}_{t-1}^{(\ell)}}_{(A_1)} \|_{\mathrm{F}}^2 \\
    &\qquad + \| \underbrace{\mathbf{L}^\top (\mathbf{D}_{t-1}^{(\ell+1)} \circ \sigma^\prime(\widetilde{\mathbf{Z}}_{t-1}^{(\ell)})) \mathbf{W}_{t-1}^{(\ell)} - \widetilde{\mathbf{D}}_{t-1}^{(\ell)}}_{(A_2)} \|_{\mathrm{F}}^2 + \| \underbrace{\widetilde{\mathbf{D}}_{t}^{(\ell)} - \widetilde{\mathbf{D}}_{t-1}^{(\ell)}}_{(A_3)} \|_{\mathrm{F}}^2 \\
    &\qquad + 2\langle A_1, A_2\rangle - 2\langle A_1, A_3\rangle - 2\langle A_2, A_3\rangle.
    \end{aligned}
\end{equation}
Taking expectation condition on $\mathcal{F}_t$ on both side, and using the fact that 
\begin{equation}
    \begin{aligned}
    &\mathbb{E}[\widetilde{\mathbf{D}}_t^{(\ell)} - \widetilde{\mathbf{D}}_{t-1}^{(\ell)} | \mathcal{F}_t] \\
    &= \mathbb{E}[ [\widetilde{\mathbf{L}}^{(\ell)}_t]^\top (\mathbf{D}_t^{(\ell+1)} \circ \sigma^\prime(\widetilde{\mathbf{Z}}_t^{(\ell)})) \mathbf{W}_t^{(\ell)} - [\widetilde{\mathbf{L}}^{(\ell)}_t]^\top (\mathbf{D}_{t-1}^{(\ell+1)} \circ \sigma^\prime(\widetilde{\mathbf{Z}}_{t-1}^{(\ell)})) \mathbf{W}_{t-1}^{(\ell)} | \mathcal{F}_t ] \\
    &= \mathbf{L}^\top (\mathbf{D}_t^{(\ell+1)} \circ \sigma^\prime(\widetilde{\mathbf{Z}}_t^{(\ell)})) \mathbf{W}_t^{(\ell)} - \mathbf{L}^\top (\mathbf{D}_{t-1}^{(\ell+1)} \circ \sigma^\prime(\widetilde{\mathbf{Z}}_{t-1}^{(\ell)})) \mathbf{W}_{t-1}^{(\ell)}.
    \end{aligned}
\end{equation}

Therefore, we can write Eq.~\ref{eq:gradient_H_SGCN++_eq1} as
\begin{equation}
    \begin{aligned}
    &\mathbb{E}[\|\mathbf{L}^\top (\mathbf{D}_t^{(\ell+1)} \circ \sigma^\prime(\widetilde{\mathbf{Z}}_t^{(\ell)})) \mathbf{W}_t^{(\ell)} - \widetilde{\mathbf{D}}^{(\ell)}_t \|_{\mathrm{F}}^2 | \mathcal{F}_t ] \\
    &\leq \| \mathbf{L}^\top (\mathbf{D}_{t-1}^{(\ell+1)} \circ \sigma^\prime(\widetilde{\mathbf{Z}}_{t-1}^{(\ell)})) \mathbf{W}_t^{(\ell)} - \widetilde{\mathbf{D}}_{t-1}^{(\ell)} \|_{\mathrm{F}}^2 + \mathbb{E}[\| \widetilde{\mathbf{D}}_{t}^{(\ell)} - \widetilde{\mathbf{D}}_{t-1}^{(\ell)} \|_{\mathrm{F}}^2 | \mathcal{F}_t] \\
    &\qquad - \| \mathbf{L}^\top (\mathbf{D}_t^{(\ell+1)} \circ \sigma^\prime(\widetilde{\mathbf{Z}}_t^{(\ell)})) \mathbf{W}_t^{(\ell)} - \mathbf{L}^\top (\mathbf{D}_{t-1}^{(\ell+1)} \circ \sigma^\prime(\widetilde{\mathbf{Z}}_{t-1}^{(\ell)})) \mathbf{W}_{t-1}^{(\ell)} \|_{\mathrm{F}}^2.
    \end{aligned}
\end{equation}

Then, taking expectation over $\mathcal{F}_t$, we have
\begin{equation} \label{eq:gradient_H_SGCN++_eq2}
    \begin{aligned}
    &\mathbb{E}[\|\mathbf{L}^\top (\mathbf{D}_t^{(\ell+1)} \circ \sigma^\prime(\widetilde{\mathbf{Z}}_t^{(\ell)})) \mathbf{W}_t^{(\ell)} - \widetilde{\mathbf{D}}^{(\ell)}_t \|_{\mathrm{F}}^2 ] \\
    &\leq \mathbb{E}[ \| \mathbf{L}^\top (\mathbf{D}_{t-1}^{(\ell+1)} \circ \sigma^\prime(\widetilde{\mathbf{Z}}_{t-1}^{(\ell)})) \mathbf{W}_{t-1}^{(\ell)} - \widetilde{\mathbf{D}}_{t-1}^{(\ell)} \|_{\mathrm{F}}^2 ] + \mathbb{E}[\| \widetilde{\mathbf{D}}_{t}^{(\ell)} - \widetilde{\mathbf{D}}_{t-1}^{(\ell)} \|_{\mathrm{F}}^2 ] \\
    &\qquad - [ \| \mathbf{L}^\top (\mathbf{D}_t^{(\ell+1)} \circ \sigma^\prime(\widetilde{\mathbf{Z}}_t^{(\ell)})) \mathbf{W}_t^{(\ell)} - \mathbf{L}^\top (\mathbf{D}_{t-1}^{(\ell+1)} \circ \sigma^\prime(\widetilde{\mathbf{Z}}_{t-1}^{(\ell)})) \mathbf{W}_{t-1}^{(\ell)} \|_{\mathrm{F}}^2 ].
    \end{aligned}
\end{equation}

Since we know $t\in\{e_{s-1}, \ldots, e_s\}$, we can denote $t = e_{s-1} + k,~k \leq E_s$ such that we can write Eq.~\ref{eq:gradient_H_SGCN++_eq2} as
\begin{equation} \label{eq:gradient_H_SGCN++_eq3}
    \begin{aligned}
    &\mathbb{E}[\|\mathbf{L}^\top (\mathbf{D}_t^{(\ell+1)} \circ \sigma^\prime(\widetilde{\mathbf{Z}}_t^{(\ell)})) \mathbf{W}_t^{(\ell)} - \widetilde{\mathbf{D}}^{(\ell)}_t \|_{\mathrm{F}}^2 ] \\
    &= \mathbb{E}[\|\mathbf{L}^\top (\mathbf{D}_{e_{s-1}+k}^{(\ell+1)} \circ \sigma^\prime(\widetilde{\mathbf{Z}}_{e_{s-1}+k}^{(\ell)})) \mathbf{W}_{e_{s-1}+k}^{(\ell)} - \widetilde{\mathbf{D}}^{(\ell)}_{e_{s-1}+k} \|_{\mathrm{F}}^2 ] \\
    &\leq \mathbb{E} [ \| \mathbf{L}^\top (\mathbf{D}_{e_{s-1}}^{(\ell+1)} \circ \sigma^\prime(\widetilde{\mathbf{Z}}_{e_{s-1}}^{(\ell)})) \mathbf{W}_{e_{s-1}}^{(\ell)} - \widetilde{\mathbf{D}}_{e_{s-1}}^{(\ell)} \|_{\mathrm{F}}^2 ] + \sum_{t=e_{s-1}+1}^{e_{s}} \Big( \mathbb{E}[\| \widetilde{\mathbf{D}}_{t}^{(\ell)} - \widetilde{\mathbf{D}}_{t-1}^{(\ell)} \|_{\mathrm{F}}^2 ] \\
    &\qquad - [ \| \mathbf{L}^\top (\mathbf{D}_t^{(\ell+1)} \circ \sigma^\prime(\widetilde{\mathbf{Z}}_t^{(\ell)})) \mathbf{W}_t^{(\ell)} - \mathbf{L}^\top (\mathbf{D}_{t-1}^{(\ell+1)} \circ \sigma^\prime(\widetilde{\mathbf{Z}}_{t-1}^{(\ell)})) \mathbf{W}_{t-1}^{(\ell)} \|_{\mathrm{F}}^2 ] \Big) \\
    &\leq \mathbb{E}[\| \mathbf{L}^\top (\mathbf{D}_{e_{s-1}}^{(\ell+1)} \circ \sigma^\prime(\widetilde{\mathbf{Z}}_{e_{s-1}}^{(\ell)})) \mathbf{W}_{e_{s-1}}^{(\ell)} - \widetilde{\mathbf{D}}_{e_{s-1}}^{(\ell)} \|_{\mathrm{F}}^2 ] \\
    &\qquad + \sum_{t=e_{s-1}+1}^{e_{s}} \Big( \mathbb{E}[\| [\widetilde{\mathbf{L}}^{(\ell)}_t]^\top (\mathbf{D}_t^{(\ell+1)} \circ \sigma^\prime(\widetilde{\mathbf{Z}}_t^{(\ell)})) \mathbf{W}_t^{(\ell)} - [\widetilde{\mathbf{L}}^{(\ell)}_t]^\top (\mathbf{D}_{t-1}^{(\ell+1)} \circ \sigma^\prime(\widetilde{\mathbf{Z}}_{t-1}^{(\ell)}))  \mathbf{W}_{t-1}^{(\ell)} \|_{\mathrm{F}}^2 ] \\
    &\qquad - \mathbb{E}[ \| \mathbf{L}^\top (\mathbf{D}_t^{(\ell+1)} \circ \sigma^\prime(\widetilde{\mathbf{Z}}_t^{(\ell)})) \mathbf{W}_t^{(\ell)} - \mathbf{L}^\top (\mathbf{D}_{t-1}^{(\ell+1)} \circ \sigma^\prime(\widetilde{\mathbf{Z}}_{t-1}^{(\ell)})) \mathbf{W}_{t-1}^{(\ell)} \|_{\mathrm{F}}^2 ] \Big).
    \end{aligned}
\end{equation}

Knowing that we are taking full-batch gradient descent when $t=t_{s-1}$, we can write Eq.~\ref{eq:gradient_H_SGCN++_eq3} as
\begin{equation}
    \begin{aligned}
    &\mathbb{E}[\|\mathbf{L}^\top (\mathbf{D}_t^{(\ell+1)} \circ \sigma^\prime(\widetilde{\mathbf{Z}}_t^{(\ell)})) \mathbf{W}_t^{(\ell)} - \widetilde{\mathbf{D}}^{(\ell)}_t \|_{\mathrm{F}}^2 ] \\
    &\leq \sum_{t=e_{s-1}+1}^{e_{s}} 
    \Big( | \mathbb{E}[\|\widetilde{\mathbf{L}}^{(\ell)}\|_{\mathrm{F}}^2] - \| \mathbf{L} \|_{\mathrm{F}}^2 | \Big) \times \underbrace{\mathbb{E}[ \| (\mathbf{D}_t^{(\ell+1)} \circ \sigma^\prime(\widetilde{\mathbf{Z}}_t^{(\ell)})) \mathbf{W}_t^{(\ell)} - (\mathbf{D}_{t-1}^{(\ell+1)} \circ \sigma^\prime(\widetilde{\mathbf{Z}}_{t-1}^{(\ell)})) \mathbf{W}_{t-1}^{(\ell)} \|_{\mathrm{F}}^2 ]}_{(B)}.
    \end{aligned}
\end{equation}

Let take closer look at term $(B)$, we have
\begin{equation}\label{eq:sgcn_pplus_important_eq_1}
    \begin{aligned}
    & \mathbb{E}[ \| (\mathbf{D}_t^{(\ell+1)} \circ \sigma^\prime(\widetilde{\mathbf{Z}}_t^{(\ell)})) \mathbf{W}_t^{(\ell)} - (\mathbf{D}_{t-1}^{(\ell+1)} \circ \sigma^\prime(\widetilde{\mathbf{Z}}_{t-1}^{(\ell)})) \mathbf{W}_{t-1}^{(\ell)} \|_{\mathrm{F}}^2] \\
    &\leq 3 C_\sigma^2 B_W^2 \underbrace{\|\mathbf{D}_t^{(\ell+1)} - \mathbf{D}_{t-1}^{(\ell+1)} \|_{\mathrm{F}}^2}_{(C_1)} + 3 \beta^2 B_D^2 B_W^2 L_\sigma^2 \underbrace{\mathbb{E}[ \| \widetilde{\mathbf{Z}}_t^{(\ell)} - \widetilde{\mathbf{Z}}_{t-1}^{(\ell)} \|_{\mathrm{F}}^2 ]}_{(C_2)} \\
    &\quad + 3 \beta^2 C_\sigma^2 B_D^2  \mathbb{E}[ \|\mathbf{W}_t^{(\ell)} - \mathbf{W}_{t-1}^{(\ell)} \|_{\mathrm{F}}^2 ].
    \end{aligned}
\end{equation}

For term $(C_1)$ by definition we know 
\begin{equation}
    \begin{aligned}
    &\|\mathbf{D}_t^{(\ell+1)} - \mathbf{D}_{t-1}^{(\ell+1)} \|_{\mathrm{F}}^2 \\
    &= \|\Big( \mathbf{L}^\top (\mathbf{D}_t^{(\ell+2)} \circ \sigma^\prime(\mathbf{Z}_t^{(\ell+1)}))\mathbf{W}_t^{(\ell+1)} \Big) -  \Big( \mathbf{L}^\top (\mathbf{D}_{t-1}^{(\ell+2)} \circ \sigma^\prime(\mathbf{Z}_{t-1}^{(\ell+1)}))\mathbf{W}_{t-1}^{(\ell+1)} \Big)\|_{\mathrm{F}}^2 \\
    &\leq 3\| \Big( \mathbf{L}^\top (\mathbf{D}_t^{(\ell+2)} \circ \sigma^\prime(\mathbf{Z}_t^{(\ell+1)}))\mathbf{W}_t^{(\ell+1)} \Big) -  \Big( \mathbf{L}^\top (\mathbf{D}_{t-1}^{(\ell+2)} \circ \sigma^\prime(\mathbf{Z}_t^{(\ell+1)}))\mathbf{W}_t^{(\ell+1)} \Big) \|_{\mathrm{F}}^2 \\
    &\quad + 3\| \Big( \mathbf{L}^\top (\mathbf{D}_{t-1}^{(\ell+2)} \circ \sigma^\prime(\mathbf{Z}_t^{(\ell+1)}))\mathbf{W}_t^{(\ell+1)} \Big) - \Big( \mathbf{L}^\top (\mathbf{D}_{t-1}^{(\ell+2)} \circ \sigma^\prime(\mathbf{Z}_{t-1}^{(\ell+1)}))\mathbf{W}_t^{(\ell+1)} \Big) \|_{\mathrm{F}}^2 \\
    &\quad + 3\| \Big( \mathbf{L}^\top (\mathbf{D}_{t-1}^{(\ell+2)} \circ \sigma^\prime(\mathbf{Z}_{t-1}^{(\ell+1)}))\mathbf{W}_t^{(\ell+1)} \Big) - \Big( \mathbf{L}^\top (\mathbf{D}_{t-1}^{(\ell+2)} \circ \sigma^\prime(\mathbf{Z}_{t-1}^{(\ell+1)}))\mathbf{W}_{t-1}^{(\ell+1)} \Big) \|_{\mathrm{F}}^2 \\
    &\leq \mathcal{O}(\|\mathbf{D}_t^{(\ell+2)} - \mathbf{D}_{t-1}^{(\ell+2)} \|_{\mathrm{F}}^2) + \mathcal{O}(\beta^2 \|\mathbf{Z}_t^{(\ell+1)} - \mathbf{Z}_{t-1}^{(\ell+1)} \|_{\mathrm{F}}^2) + \mathcal{O}(\beta^2 \|\mathbf{W}_t^{(\ell+1)} - \mathbf{W}_{t-1}^{(\ell+1)} \|_{\mathrm{F}}^2).
    \end{aligned}
\end{equation}
By induction, we can upper bound $(C_1)$ in Eq.~\ref{eq:sgcn_pplus_important_eq_1} by
\begin{equation} \label{eq:sgcn_pplus_important_eq_4}
    \begin{aligned}
    \|\mathbf{D}_t^{(\ell+1)} - \mathbf{D}_{t-1}^{(\ell+1)} \|_{\mathrm{F}}^2 
    &\leq \underbrace{\mathcal{O}(\|\mathbf{D}^{(L+1)}_t - \mathbf{D}^{(L+1)}_{t-1}\|_{\mathrm{F}}^2)}_{(D_1)} \\
    &\quad + \underbrace{\mathcal{O}(\beta^2 \|\mathbf{Z}_t^{(\ell+1)} - \mathbf{Z}_{t-1}^{(\ell+1)} \|_{\mathrm{F}}^2)}_{(D_2)} + \ldots + \mathcal{O}(\beta^2 \|\mathbf{Z}_t^{(L)} - \mathbf{Z}_{t-1}^{(L)} \|_{\mathrm{F}}^2) \\
    &\quad + \mathcal{O}(\beta^2 \|\mathbf{W}_t^{(\ell+1)} - \mathbf{W}_{t-1}^{(\ell+1)} \|_{\mathrm{F}}^2) + \ldots + \mathcal{O}(\beta^2 \|\mathbf{W}_t^{(L)} - \mathbf{W}_{t-1}^{(L)} \|_{\mathrm{F}}^2).
    \end{aligned}
\end{equation}

For term $(D_1)$ in Eq.~\ref{eq:sgcn_pplus_important_eq_4} we have
\begin{equation}
    \begin{aligned}
    \| \mathbf{D}^{(L+1)}_t - \mathbf{D}^{(L+1)}_{t-1}\|_{\mathrm{F}}^2 
    &= \| \frac{\partial \mathcal{L}(\bm{\theta}_t) }{\partial \mathbf{W}_t^{(L)}} - \frac{\partial \mathcal{L}(\bm{\theta}_{t-1}) }{\partial \mathbf{W}_{t-1}^{(L)}} \|_{\mathrm{F}}^2 \\
    &\leq L_\text{loss}^2 C_\sigma^2 \| \mathbf{Z}^{(L)}_t - \mathbf{Z}^{(L)}_{t-1} \|_{\mathrm{F}}^2.
    \end{aligned}
\end{equation}
For term $(D_2)$ in Eq.~\ref{eq:sgcn_pplus_important_eq_4} we have
\begin{equation}
    \begin{aligned}
    \|\mathbf{Z}_t^{(\ell+1)} - \mathbf{Z}_{t-1}^{(\ell+1)} \|_{\mathrm{F}}^2 
    &\leq C_\sigma^2 \| \mathbf{L} \mathbf{H}^{(\ell)}_t \mathbf{W}^{(\ell+1)}_t - \mathbf{L} \mathbf{H}^{(\ell)}_{t-1} \mathbf{W}^{(\ell+1)}_{t-1} \|_{\mathrm{F}}^2 \\
    &\leq C_\sigma^2 B_{LA}^2 \| \mathbf{H}^{(\ell)}_t \mathbf{W}^{(\ell+1)}_t - \mathbf{H}^{(\ell)}_{t-1} \mathbf{W}^{(\ell+1)}_t + \mathbf{H}^{(\ell)}_{t-1} \mathbf{W}^{(\ell+1)}_t - \mathbf{H}^{(\ell)}_{t-1} \mathbf{W}^{(\ell+1)}_{t-1} \|_{\mathrm{F}}^2 \\
    &\leq 2C_\sigma^2 B_{LA}^2 \| \mathbf{H}^{(\ell)}_t \mathbf{W}^{(\ell+1)}_t - \mathbf{H}^{(\ell)}_{t-1} \mathbf{W}^{(\ell+1)}_t \|_{\mathrm{F}}^2 + 2C_\sigma^2 B_{LA}^2 \|\mathbf{H}^{(\ell)}_{t-1} \mathbf{W}^{(\ell+1)}_t - \mathbf{H}^{(\ell)}_{t-1} \mathbf{W}^{(\ell+1)}_{t-1} \|_{\mathrm{F}}^2 \\
    &\leq 2C_\sigma^4 B_{LA}^2 B_W^2 \| \mathbf{Z}^{(\ell)}_t  - \mathbf{Z}^{(\ell)}_{t-1} \|_{\mathrm{F}}^2 + 2\alpha^2 C_\sigma^2 B_{LA}^2 B_H^2 \| \mathbf{W}^{(\ell+1)}_t  - \mathbf{W}^{(\ell+1)}_{t-1} \|_{\mathrm{F}}^2.
    \end{aligned}
\end{equation}
By induction we can upper bound term $(D_2)$ in Eq.~\ref{eq:sgcn_pplus_important_eq_4} by
\begin{equation}
    \|\mathbf{Z}_t^{(\ell+1)} - \mathbf{Z}_{t-1}^{(\ell+1)} \|_{\mathrm{F}}^2 \leq \mathcal{O}(\alpha^2 \| \mathbf{W}^{(\ell+1)}_t  - \mathbf{W}^{(\ell+1)}_{t-1} \|_{\mathrm{F}}^2) + \ldots + \mathcal{O}(\alpha^2 \| \mathbf{W}^{(1)}_t  - \mathbf{W}^{(1)}_{t-1} \|_{\mathrm{F}}^2).
\end{equation}

For term $(C_2)$ in Eq.~\ref{eq:sgcn_pplus_important_eq_1} we have 
\begin{equation}
    \begin{aligned}
    \mathbb{E}[\| \widetilde{\mathbf{Z}}_t^{(\ell)} - \widetilde{\mathbf{Z}}_{t-1}^{(\ell)} \|_{\mathrm{F}}^2] 
    &= \mathbb{E}[\| \widetilde{\mathbf{L}}^{(\ell)}\widetilde{\mathbf{H}}_t^{(\ell-1)} \mathbf{W}^{(\ell)}_t - \widetilde{\mathbf{L}}^{(\ell)}\widetilde{\mathbf{H}}_{t-1}^{(\ell-1)} \mathbf{W}^{(\ell)}_{t-1} \|_{\mathrm{F}}^2 ] \\
    &\leq 2B_{LA}^2 B_W^2 C_\sigma^2 \mathbb{E}[\| \widetilde{\mathbf{Z}}_t^{(\ell-1)} - \widetilde{\mathbf{Z}}_{t-1}^{(\ell-1)} \|_{\mathrm{F}}^2] + 2\alpha^2 B_{LA}^2 B_H^2 \mathbb{E}[\|\mathbf{W}^{(\ell)}_t - \mathbf{W}^{(\ell)}_{t-1}\|_{\mathrm{F}}^2].
    \end{aligned}
\end{equation}
By induction term $(C_2)$ in Eq.~\ref{eq:sgcn_pplus_important_eq_4}
\begin{equation}\label{eq:diff_Z_wrt_W}
    \mathbb{E}[\| \widetilde{\mathbf{Z}}_t^{(\ell)} - \widetilde{\mathbf{Z}}_{t-1}^{(\ell)} \|_{\mathrm{F}}^2] \leq \mathcal{O}(\alpha^2 \mathbb{E}[\|\mathbf{W}^{(\ell)}_t - \mathbf{W}^{(\ell)}_{t-1}\|_{\mathrm{F}}^2] ) + \ldots + \mathcal{O}(\alpha^2 \mathbb{E}[\|\mathbf{W}^{(1)}_t - \mathbf{W}^{(1)}_{t-1}\|_{\mathrm{F}}^2] ).
\end{equation}

Plugging $(D_1),(D_2)$ back to $(C_1)$ and $(C_1),(C_2),(C_3)$ back to $(B)$, we have
\begin{equation}
    \begin{aligned}
    &\mathbb{E}[ \| (\mathbf{D}_t^{(\ell+1)} \circ \sigma^\prime(\widetilde{\mathbf{Z}}_t^{(\ell)})) \mathbf{W}_t^{(\ell)} - (\mathbf{D}_{t-1}^{(\ell+1)} \circ \sigma^\prime(\widetilde{\mathbf{Z}}_{t-1}^{(\ell)})) \mathbf{W}_{t-1}^{(\ell)} \|_{\mathrm{F}}^2] \\
    &\leq \sum_{t=e_{s-1}+1}^{e_{s}} \Big( \mathcal{O}\big((\alpha^2 + \beta^2 + \alpha^2 \beta^2) \mathbb{E}[\|\mathbf{W}_t^{(1)} - \mathbf{W}_{t-1}^{(1)}\|_{\mathrm{F}}^2] \big) +\ldots \\
    &\qquad + \mathcal{O}\big((\alpha^2 + \beta^2 + \alpha^2 \beta^2)\mathbb{E}[\|\mathbf{W}_t^{(L)} - \mathbf{W}_{t-1}^{(L)} \|_{\mathrm{F}}^2]\big) \Big) \\
    &= \sum_{t=e_{s-1}+1}^{e_{s}}  \eta^2 \mathcal{O}\Big( (\alpha^2 + \beta^2 + \alpha^2 \beta^2)\mathbb{E}[\|\nabla \widetilde{\mathcal{L}}(\bm{\theta}_{t-1})\|_{\mathrm{F}}^2] \Big).
    \end{aligned}
\end{equation}
Then plugging term $(B)$ back to Eq.~\ref{eq:sgcn_pplus_important_eq_1} we conclude the proof.

\end{proof}

Using the previous lemma, we provide the upper-bound of Eq.~\ref{eq:sgcn_pplus_important_steps_2}, which is one of the three key factors that affect the mean-square error of stochastic gradient at the $\ell$th layer.
\begin{lemma}\label{lemma:sgcn_pplus_support_lemma1}
Let suppose $t\in\{e_{s-1}+1,\ldots,e_{s}\}$. 
The upper-bound on the difference of the gradient with respect to the input node embedding matrix at the $\ell$th graph convolutional layer given the same input $\mathbf{D}_t^{(\ell+1)}$ but different input $\widetilde{\mathbf{H}}_t^{(\ell-1)}, \mathbf{H}_t^{(\ell-1)}$ is defined as
\begin{equation}
    \begin{aligned}
    &\mathbb{E}[ \| \nabla_H \widetilde{f}^{(\ell)}(\mathbf{D}_t^{(\ell+1)}, \widetilde{\mathbf{H}}_t^{(\ell-1)}, \mathbf{W}_t^{(\ell)}) - \nabla_H f^{(\ell)}(\mathbf{D}_t^{(\ell+1)}, \mathbf{H}_t^{(\ell-1)}, \mathbf{W}_t^{(\ell)}) \|_{\mathrm{F}}^2 ] \\
    &\leq \sum_{t=e_{s-1}+1}^{e_{s}} \eta^2 \mathcal{O} \left( \Big( \sum_{j=1}^\ell | \mathbb{E}[\|\widetilde{\mathbf{L}}^{(j)}\|_{\mathrm{F}}^2] - \| \mathbf{L} \|_{\mathrm{F}}^2 | \Big) \times (\alpha^2 + \beta^2 + \alpha^2 \beta^2) \mathbb{E}[\|\nabla \widetilde{\mathcal{L}}(\bm{\theta}_{t-1})\|_{\mathrm{F}}^2]  \right)
    \end{aligned}
\end{equation}

\end{lemma}
\begin{proof}
For the gradient w.r.t. the node embedding matrices, we have
\begin{equation} \label{eq:sgcn_pplus_support_lemma1_eq1}
    \begin{aligned}
    &\mathbb{E}[ \| \nabla_H \widetilde{f}^{(\ell)}(\mathbf{D}_t^{(\ell+1)}, \widetilde{\mathbf{H}}_t^{(\ell-1)}, \mathbf{W}_t^{(\ell)}) - \nabla_H f^{(\ell)}(\mathbf{D}_t^{(\ell+1)}, \mathbf{H}_t^{(\ell-1)}, \mathbf{W}_t^{(\ell)}) \|_{\mathrm{F}}^2 ] \\
    &= \mathbb{E}[ \| \Big( \widetilde{\mathbf{D}}_{t-1}^{(\ell)} + [\widetilde{\mathbf{L}}^{(\ell)}_t]^\top (\mathbf{D}_t^{(\ell+1)} \circ \sigma^\prime(\widetilde{\mathbf{Z}}_t^{(\ell)})) \mathbf{W}_t^{(\ell)} - [\widetilde{\mathbf{L}}^{(\ell)}_t]^\top (\mathbf{D}_{t-1}^{(\ell+1)} \circ \sigma^\prime(\widetilde{\mathbf{Z}}_{t-1}^{(\ell)}))  \mathbf{W}_{t-1}^{(\ell)}] \Big) \\
    &\qquad - \Big( \mathbf{L}^\top (\mathbf{D}_t^{(\ell+1)} \circ \sigma^\prime(\mathbf{Z}_t^{(\ell)}))\mathbf{W}_t^{(\ell)} \Big) \|_{\mathrm{F}}^2 ] \\
    &\leq 2\underbrace{\mathbb{E}[ \| \Big( \widetilde{\mathbf{D}}_{t-1}^{(\ell)} + [\widetilde{\mathbf{L}}^{(\ell)}_t]^\top (\mathbf{D}_t^{(\ell+1)} \circ \sigma^\prime(\widetilde{\mathbf{Z}}_t^{(\ell)})) \mathbf{W}_t^{(\ell)} - [\widetilde{\mathbf{L}}^{(\ell)}_t]^\top (\mathbf{D}_{t-1}^{(\ell+1)} \circ \sigma^\prime(\widetilde{\mathbf{Z}}_{t-1}^{(\ell)}))  \mathbf{W}_{t-1}^{(\ell)}] \Big)}_{(A)} \\
    &\qquad \underbrace{- \Big( \mathbf{L}^\top (\mathbf{D}_t^{(\ell+1)} \circ \sigma^\prime(\widetilde{\mathbf{Z}}_t^{(\ell)}))\mathbf{W}_t^{(\ell)} \Big) \|_{\mathrm{F}}^2 ]}_{(A)} \\
    &\qquad + 2\underbrace{\mathbb{E}[\| \Big( \mathbf{L}^\top (\mathbf{D}_t^{(\ell+1)} \circ \sigma^\prime(\widetilde{\mathbf{Z}}_t^{(\ell)}))\mathbf{W}_t^{(\ell)} \Big) - \Big( \mathbf{L}^\top (\mathbf{D}_t^{(\ell+1)} \circ \sigma^\prime(\mathbf{Z}_t^{(\ell)}))\mathbf{W}_t^{(\ell)} \Big) \|_{\mathrm{F}}^2]}_{(B)}.
    \end{aligned}
\end{equation}
Let first take a closer look at term $(A)$ in Eq.~\ref{eq:sgcn_pplus_support_lemma1_eq1}. 
Let suppose $t\in\{e_{s-1}+1,\ldots,e_{s}\}$. Then we can denote $t=e_{s-1}+k$ for some $k\leq K$.
By Lemma~\ref{lemma:gradient_H_SGCN++}, term $(A)$ can be bounded by
\begin{equation} \label{eq:sgcn_pplus_support_lemma1_eq3}
    (A) \leq \sum_{t=e_{s-1}+1}^{e_{s}} \eta^2 \mathcal{O}\left( \Big( \sum_{\ell=1}^L  |\mathbb{E}[\|\widetilde{\mathbf{L}}_t^{(\ell)}\|_{\mathrm{F}}^2] - \|\mathbf{L}\|_{\mathrm{F}}^2| \Big) \times (\alpha^2 + \beta^2 + \alpha^2 \beta^2 ) \mathbb{E}[\|\nabla \widetilde{\mathcal{L}}(\bm{\theta}_{t-1})\|_{\mathrm{F}}^2]. \right)
\end{equation}
Then we take a closer look at term $(B)$ in Eq.~\ref{eq:sgcn_pplus_support_lemma1_eq1}, we have 
\begin{equation} \label{eq:sgcn_pplus_support_lemma1_eq2}
    \begin{aligned}
    &\mathbb{E}[\| \Big( \mathbf{L}^\top (\mathbf{D}_t^{(\ell+1)} \circ \sigma^\prime(\widetilde{\mathbf{Z}}_t^{(\ell)}))\mathbf{W}_t^{(\ell)} \Big) - \Big( \mathbf{L}^\top (\mathbf{D}_t^{(\ell+1)} \circ \sigma^\prime(\mathbf{Z}_t^{(\ell)}))\mathbf{W}_t^{(\ell)} \Big) \|_{\mathrm{F}}^2] \\
    &\leq \beta^2 B_{LA}^2 B_D^2 B_W^2 L_\sigma^2 \underbrace{\mathbb{E}[\| \widetilde{\mathbf{Z}}_t^{(\ell)} - \mathbf{Z}_t^{(\ell)}\|_{\mathrm{F}}^2]}_{(C)}.
    \end{aligned}
\end{equation}

The term $(C)$ in Eq.~\ref{eq:sgcn_pplus_support_lemma1_eq2} can be decomposed as
\begin{equation}
    \begin{aligned}
    &\mathbb{E}[\| \widetilde{\mathbf{Z}}_t^{(\ell)} - \mathbf{Z}_t^{(\ell)}\|_{\mathrm{F}}^2] \\
    &= \mathbb{E}[\| \widetilde{\mathbf{Z}}_t^{(\ell)} - \mathbf{L}\mathbf{H}_t^{(\ell-1)} \mathbf{W}^{(\ell)}_t\|_{\mathrm{F}}^2] \\
    &\leq 2 \mathbb{E}[\| \widetilde{\mathbf{Z}}_t^{(\ell)} - \mathbf{L}\widetilde{\mathbf{H}}_t^{(\ell-1)} \mathbf{W}^{(\ell)}_t\|_{\mathrm{F}}^2]  + 2 \mathbb{E}[\| \mathbf{L}\widetilde{\mathbf{H}}_t^{(\ell-1)} \mathbf{W}^{(\ell)}_t - \mathbf{L}\mathbf{H}_t^{(\ell-1)} \mathbf{W}^{(\ell)}_t\|_{\mathrm{F}}^2] \\
    &\leq 2 \mathbb{E}[\| \widetilde{\mathbf{Z}}_t^{(\ell)} - \mathbf{L}\widetilde{\mathbf{H}}_t^{(\ell-1)} \mathbf{W}^{(\ell)}_t\|_{\mathrm{F}}^2]   + 2 B_{LA}^2 B_W^2 C_\sigma^2 \mathbb{E}[\| \widetilde{\mathbf{Z}}_t^{(\ell-1)} - \mathbf{Z}_t^{(\ell-1)} \|_{\mathrm{F}}^2].
    \end{aligned}
\end{equation}

By induction, we have
\begin{equation} \label{eq:sgcn_pplus_support_lemma1_eq5}
    \mathbb{E}[\| \widetilde{\mathbf{Z}}_t^{(\ell)} - \mathbf{Z}_t^{(\ell)}\|_{\mathrm{F}}^2] \leq  
    \mathcal{O}\Big( \underbrace{\mathbb{E}[\| \widetilde{\mathbf{Z}}_t^{(\ell)} - \mathbf{L}\widetilde{\mathbf{H}}_t^{(\ell-1)} \mathbf{W}^{(\ell)}_t\|_{\mathrm{F}}^2] }_{(D)} \Big) + \ldots + \mathcal{O}\Big(\mathbb{E}[\| \widetilde{\mathbf{Z}}_t^{(1)} - \mathbf{L}\mathbf{X} \mathbf{W}^{(1)}_t\|_{\mathrm{F}}^2] \Big).
\end{equation}

The upper-bound for term $(D)$ is similar to one we have in the proof of Lemma~\ref{lemma:upper-bound-sgcn-plus-bias-1}.
\begin{equation}
    \begin{aligned}
    &\| \mathbf{L}\widetilde{\mathbf{H}}_t^{(\ell-1)} \mathbf{W}^{(\ell)}_t  - \widetilde{\mathbf{Z}}_t^{(\ell)} \|_{\mathrm{F}}^2 \\
    &=\| [\mathbf{L}\widetilde{\mathbf{H}}_t^{(\ell-1)} \mathbf{W}^{(\ell)}_t - \mathbf{L}\widetilde{\mathbf{H}}_{t-1}^{(\ell-1)} \mathbf{W}^{(\ell)}_{t-1}] + [ \mathbf{L}\widetilde{\mathbf{H}}_{t-1}^{(\ell-1)} \mathbf{W}^{(\ell)}_{t-1} - \widetilde{\mathbf{Z}}_{t-1}^{(\ell)} ] - [ \widetilde{\mathbf{Z}}_t^{(\ell)} - \widetilde{\mathbf{Z}}_{t-1}^{(\ell)} ] \|_{\mathrm{F}}^2 \\
    &= \| \mathbf{L}\widetilde{\mathbf{H}}_t^{(\ell-1)} \mathbf{W}^{(\ell)}_t - \mathbf{L}\widetilde{\mathbf{H}}_{t-1}^{(\ell-1)} \mathbf{W}^{(\ell)}_{t-1} \|_{\mathrm{F}}^2 + \| \mathbf{L}\widetilde{\mathbf{H}}_{t-1}^{(\ell-1)} \mathbf{W}^{(\ell)}_{t-1} - \widetilde{\mathbf{Z}}_{t-1}^{(\ell)} \|_{\mathrm{F}}^2 + \| \widetilde{\mathbf{Z}}_t^{(\ell)} - \widetilde{\mathbf{Z}}_{t-1}^{(\ell)} \|_{\mathrm{F}}^2 \\
    &\qquad + 2\langle \mathbf{L}\widetilde{\mathbf{H}}_t^{(\ell-1)} \mathbf{W}^{(\ell)}_t, \mathbf{L}\widetilde{\mathbf{H}}_{t-1}^{(\ell-1)} \mathbf{W}^{(\ell)}_{t-1}\rangle - 2\langle \mathbf{L}\widetilde{\mathbf{H}}_t^{(\ell-1)} \mathbf{W}^{(\ell)}_t, \widetilde{\mathbf{Z}}_t^{(\ell)} - \widetilde{\mathbf{Z}}_{t-1}^{(\ell)} \rangle \\
    & \qquad - 2\langle \widetilde{\mathbf{Z}}_t^{(\ell)} - \widetilde{\mathbf{Z}}_{t-1}^{(\ell)}, \mathbf{L}\widetilde{\mathbf{H}}_{t-1}^{(\ell-1)} \mathbf{W}^{(\ell)}_{t-1} \rangle.
    \end{aligned}
\end{equation}
Taking expectation condition on $\mathcal{F}_t$ and using 
\begin{equation}
    \mathbb{E}[\widetilde{\mathbf{Z}}_t^{(\ell)} - \widetilde{\mathbf{Z}}_{t-1}^{(\ell)} | \mathcal{F}_t ] = \mathbf{L}\widetilde{\mathbf{H}}_t^{(\ell-1)} \mathbf{W}^{(\ell)}_t - \mathbf{L}\widetilde{\mathbf{H}}_{t-1}^{(\ell-1)} \mathbf{W}^{(\ell)}_{t-1},
\end{equation}
we have
\begin{equation}
    \begin{aligned}
    &\mathbb{E}[ \| \mathbf{L}\widetilde{\mathbf{H}}_t^{(\ell-1)} \mathbf{W}^{(\ell)}_t  - \widetilde{\mathbf{Z}}_t^{(\ell)} \|_{\mathrm{F}}^2 | \mathcal{F}_t ] \\
    &= \| \mathbf{L}\widetilde{\mathbf{H}}_{t-1}^{(\ell-1)} \mathbf{W}^{(\ell)}_{t-1} - \widetilde{\mathbf{Z}}_{t-1}^{(\ell)} \|_{\mathrm{F}}^2 +  \mathbb{E}[\| \widetilde{\mathbf{Z}}_t^{(\ell)} - \widetilde{\mathbf{Z}}_{t-1}^{(\ell)} \|_{\mathrm{F}}^2 | \mathcal{F}_t] - \| \mathbf{L}\widetilde{\mathbf{H}}_t^{(\ell-1)} \mathbf{W}^{(\ell)}_t - \mathbf{L}\widetilde{\mathbf{H}}_{t-1}^{(\ell-1)} \mathbf{W}^{(\ell)}_{t-1} \|_{\mathrm{F}}^2.
    \end{aligned}
\end{equation}

Take expectation over $\mathcal{F}_t$ we have
\begin{equation}
    \begin{aligned}
    &\mathbb{E}[ \| \mathbf{L}\widetilde{\mathbf{H}}_t^{(\ell-1)} \mathbf{W}^{(\ell)}_t  - \widetilde{\mathbf{Z}}_t^{(\ell)} \|_{\mathrm{F}}^2 ] \\
    &= \| \mathbf{L}\widetilde{\mathbf{H}}_{t-1}^{(\ell-1)} \mathbf{W}^{(\ell)}_{t-1} - \widetilde{\mathbf{Z}}_{t-1}^{(\ell)} \|_{\mathrm{F}}^2 + \mathbb{E}[\| \widetilde{\mathbf{Z}}_t^{(\ell)} - \widetilde{\mathbf{Z}}_{t-1}^{(\ell)} \|_{\mathrm{F}}^2 ] - \| \mathbf{L}\widetilde{\mathbf{H}}_t^{(\ell-1)} \mathbf{W}^{(\ell)}_t - \mathbf{L}\widetilde{\mathbf{H}}_{t-1}^{(\ell-1)} \mathbf{W}^{(\ell)}_{t-1} \|_{\mathrm{F}}^2.
    \end{aligned}
\end{equation}

Let suppose $t\in\{e_{s-1}+1,\ldots,e_{s}\}$. Then we can denote $t=e_{s-1}+k$ for some $k\leq K$ such that
\begin{equation} \label{eq:sgcn_pplus_support_lemma1_eq6}
    \begin{aligned}
    &\mathbb{E}[ \| \mathbf{L}\widetilde{\mathbf{H}}_t^{(\ell-1)} \mathbf{W}^{(\ell)}_t  - \widetilde{\mathbf{Z}}_t^{(\ell)} \|_{\mathrm{F}}^2 ] \\
    &= \mathbb{E}[ \| \mathbf{L}\widetilde{\mathbf{H}}_{e_{s-1}+k}^{(\ell-1)} \mathbf{W}^{(\ell)}_{e_{s-1}+k}  - \widetilde{\mathbf{Z}}_{e_{s-1}+k}^{(\ell)} \|_{\mathrm{F}}^2 ] \\
    &= \| \mathbf{L}\widetilde{\mathbf{H}}_{e_{s-1}}^{(\ell-1)} \mathbf{W}^{(\ell)}_{e_{s-1}} - \widetilde{\mathbf{Z}}_{e_{s-1}}^{(\ell)} \|_{\mathrm{F}}^2 + \sum_{t=e_{s-1}+1}^{e_{s}} \mathbb{E}[\| \widetilde{\mathbf{Z}}_t^{(\ell)} - \widetilde{\mathbf{Z}}_{t-1}^{(\ell)} \|_{\mathrm{F}}^2 ] - \| \mathbf{L}\widetilde{\mathbf{H}}_t^{(\ell-1)} \mathbf{W}^{(\ell)}_t - \mathbf{L}\widetilde{\mathbf{H}}_{t-1}^{(\ell-1)} \mathbf{W}^{(\ell)}_{t-1} \|_{\mathrm{F}}^2 \\
    &\underset{(a)}{\leq} \sum_{t=e_{s-1}+1}^{e_{s}} \mathcal{O}\Big(|\mathbb{E}[\|\widetilde{\mathbf{L}}^{(\ell)}_t \|_{\mathrm{F}}^2] - \|\mathbf{L}\|_{\mathrm{F}}^2| \times \alpha^{2} \mathbb{E}[\| \mathbf{W}_t^{(\ell)} - \mathbf{W}_{t-1}^{(\ell)} \|_{\mathrm{F}}^2 ]\Big).
    \end{aligned}
\end{equation}
where $(a)$ follows the deviation of Eq.~\ref{eq:upper-bound-sgcn-plus-bias-1-eq7}.

Plugging $(D)$ to $(C)$ and $(C)$ to $(B)$, we have
\begin{equation} \label{eq:sgcn_pplus_support_lemma1_eq4}
    \begin{aligned}
    (B) \leq  \sum_{t=e_{s-1}+1}^{e_{s}} \mathcal{O}\left( \Big( \sum_{\ell=1}^L  |\mathbb{E}[\|\widetilde{\mathbf{L}}_t^{(\ell)}\|_{\mathrm{F}}^2] - \|\mathbf{L}\|_{\mathrm{F}}^2| \Big) \times \alpha^2 \beta^2 \mathbb{E}[\| \mathbf{W}_t^{(\ell)} - \mathbf{W}_{t-1}^{(\ell)} \|_{\mathrm{F}}^2 ]\right).
    \end{aligned}
\end{equation}

Plugging Eq.~\ref{eq:sgcn_pplus_support_lemma1_eq3} and Eq.~\ref{eq:sgcn_pplus_support_lemma1_eq4} back to Eq.~\ref{eq:sgcn_pplus_support_lemma1_eq1}, we have
\begin{equation}
    \begin{aligned}
    & \mathbb{E}[ \| \nabla_H \widetilde{f}^{(\ell)}(\mathbf{D}_t^{(\ell+1)}, \widetilde{\mathbf{H}}_t^{(\ell-1)}, \mathbf{W}_t^{(\ell)}) - \nabla_H f^{(\ell)}(\mathbf{D}_t^{(\ell+1)}, \mathbf{H}_t^{(\ell-1)}, \mathbf{W}_t^{(\ell)}) \|_{\mathrm{F}}^2 ] \\
    &\leq \sum_{\ell=1}^L \sum_{t=e_{s-1}+1}^{e_{s}} \mathcal{O}\left(\Big( \sum_{\ell=1}^L  |\mathbb{E}[\|\widetilde{\mathbf{L}}_t^{(\ell)}\|_{\mathrm{F}}^2] - \|\mathbf{L}\|_{\mathrm{F}}^2| \Big) \times (\alpha^{2} + \beta^2 + \alpha^{2} \beta^2) \mathbb{E}[\| \mathbf{W}_t^{(\ell)} - \mathbf{W}_{t-1}^{(\ell)} \|_{\mathrm{F}}^2 ]\right) \\
    &= \sum_{t=e_{s-1}+1}^{e_{s}} \eta^2 \mathcal{O}\left( \Big( \sum_{\ell=1}^L  |\mathbb{E}[\|\widetilde{\mathbf{L}}_t^{(\ell)}\|_{\mathrm{F}}^2] - \|\mathbf{L}\|_{\mathrm{F}}^2| \Big) \times (\alpha^{2} + \beta^2 + \alpha^{2} \beta^2)  \mathbb{E}[\|\nabla \widetilde{\mathcal{L}}(\bm{\theta}_{t-1})\|_{\mathrm{F}}^2] \right).
    \end{aligned}
\end{equation}
\end{proof}
In the following lemma, we derive the upper-bound on the difference of the gradient with respect to the weight matrix at each graph convolutional layer. Suppose the input node embedding matrix for the $\ell$th GCN layer is defined as $\widetilde{\mathbf{H}}_t^{(\ell-1)}$, the gradient calculated for the $\ell$th weight matrix in \texttt{SGCN++} is defined as
\begin{equation}
    \begin{aligned}
    &\nabla_W \widetilde{f}^{(\ell)}(\mathbf{D}_t^{(\ell+1)}, \widetilde{\mathbf{H}}_t^{(\ell-1)}, \mathbf{W}_t^{(\ell)}) \\
    &=\widetilde{\mathbf{G}}_{t-1}^{(\ell)} + [\widetilde{\mathbf{L}}^{(\ell)}_t \widetilde{\mathbf{H}}^{(\ell-1)}_t ]^\top (\mathbf{D}_t^{(\ell+1)} \circ \sigma^\prime(\widetilde{\mathbf{Z}}_t^{(\ell)})) - [\widetilde{\mathbf{L}}^{(\ell)}_t \widetilde{\mathbf{H}}^{(\ell-1)}_{t-1} ]^\top (\mathbf{D}_{t-1}^{(\ell+1)} \circ \sigma^\prime(\widetilde{\mathbf{Z}}_{t-1}^{(\ell)})),
    \end{aligned}
\end{equation}
and  the backward propagation for the $\ell$th layer in \texttt{FullGCN} is defined as
\begin{equation}
    \nabla_W f^{(\ell)}(\mathbf{D}_t^{(\ell+1)}, \widetilde{\mathbf{H}}_t^{(\ell-1)}, \mathbf{W}_t^{(\ell)}) = [\mathbf{L} \widetilde{\mathbf{H}}_t^{(\ell-1)}]^\top (\mathbf{D}_t^{(\ell+1)} \circ \sigma^\prime(\widetilde{\mathbf{Z}}_t^{(\ell)})).
\end{equation}
\begin{lemma} \label{lemma:gradient_W_SGCN++}
Let suppose $t\in\{e_{s-1}+1,\ldots,e_{s}\}$. 
The upper-bound on the difference of the gradient with respect to the $\ell$th graph convolutional layer given the same input $\mathbf{D}_t^{(\ell+1)}$ and $\widetilde{\mathbf{H}}_t^{(\ell-1)}$ is defined as
\begin{equation}
    \begin{aligned}
    &\|\nabla_W \widetilde{f}^{(\ell)}(\mathbf{D}_t^{(\ell+1)}, \widetilde{\mathbf{H}}_t^{(\ell-1)}, \mathbf{W}_t^{(\ell)}) - \nabla_W f^{(\ell)}(\mathbf{D}_t^{(\ell+1)}, \widetilde{\mathbf{H}}_t^{(\ell-1)}, \mathbf{W}_t^{(\ell)}) \|_{\mathrm{F}}^2 \\
    &\leq \sum_{t=e_{s-1}+1}^{e_{s}} \eta^2 \mathcal{O}\left( \Big( \sum_{\ell=1}^L | \mathbb{E}[\|\widetilde{\mathbf{L}}^{(\ell)}_t\|_{\mathrm{F}}^2] - \| \mathbf{L} \|_{\mathrm{F}}^2 | \Big) \times (\alpha^2 \beta^2 +\alpha^4 \beta^2) \mathbb{E}[\|\nabla \widetilde{\mathcal{L}}(\bm{\theta}_{t-1})\|_{\mathrm{F}}^2] \right).
    \end{aligned}
\end{equation}
\end{lemma}
\begin{proof}
To simplify the presentation, let us denote $\widetilde{\mathbf{G}}_t^{(\ell)} = \nabla_W \widetilde{f}^{(\ell)}(\mathbf{D}_t^{(\ell+1)}, \widetilde{\mathbf{H}}_t^{(\ell-1)}, \mathbf{W}_t^{(\ell)})$. Then, by definition, we have
\begin{equation}
    \widetilde{\mathbf{G}}_t^{(\ell)} = \widetilde{\mathbf{G}}_{t-1}^{(\ell)} + [\widetilde{\mathbf{L}}^{(\ell)}_t \widetilde{\mathbf{H}}_t^{(\ell-1)}]^\top (\mathbf{D}_t^{(\ell+1)} \circ \sigma^\prime(\widetilde{\mathbf{Z}}_t^{(\ell)})) - [\widetilde{\mathbf{L}}^{(\ell)}_t \widetilde{\mathbf{H}}_{t-1}^{(\ell-1)}]^\top (\mathbf{D}_{t-1}^{(\ell+1)} \circ \sigma^\prime(\widetilde{\mathbf{Z}}_{t-1}^{(\ell)})).
\end{equation}

Therefore, we know that
\begin{equation} ~\label{eq:gradient_W_SGCN++eq1}
    \begin{aligned}
    &\|[\mathbf{L} \widetilde{\mathbf{H}}_t^{(\ell-1)}]^\top (\mathbf{D}_t^{(\ell+1)} \circ \sigma^\prime(\widetilde{\mathbf{Z}}_t^{(\ell)})) - \widetilde{\mathbf{G}}^{(\ell)}_t \|_{\mathrm{F}}^2 \\
    &= \Big\| \Big[ [ \mathbf{L} \widetilde{\mathbf{H}}_t^{(\ell-1)}] ^\top (\mathbf{D}_t^{(\ell+1)} \circ \sigma^\prime(\widetilde{\mathbf{Z}}_t^{(\ell)})) - [\mathbf{L} \widetilde{\mathbf{H}}_{t-1}^{(\ell-1)}]^\top (\mathbf{D}_{t-1}^{(\ell+1)} \circ \sigma^\prime(\widetilde{\mathbf{Z}}_{t-1}^{(\ell)})) \Big] \\
    & \qquad + \Big[ [\mathbf{L} \widetilde{\mathbf{H}}_{t-1}^{(\ell-1)}]^\top (\mathbf{D}_{t-1}^{(\ell+1)} \circ \sigma^\prime(\widetilde{\mathbf{Z}}_{t-1}^{(\ell)})) - \widetilde{\mathbf{G}}_{t-1}^{(\ell)} \Big]  - \Big[ \widetilde{\mathbf{G}}_{t}^{(\ell)} - \widetilde{\mathbf{G}}_{t-1}^{(\ell)} \Big] \Big\|_{\mathrm{F}}^2 \\
    &\leq \Big\| \underbrace{ [\mathbf{L} \widetilde{\mathbf{H}}_t^{(\ell-1)}]^\top (\mathbf{D}_t^{(\ell+1)} \circ \sigma^\prime(\widetilde{\mathbf{Z}}_t^{(\ell)})) - [\mathbf{L} \widetilde{\mathbf{H}}_{t-1}^{(\ell-1)}] ^\top (\mathbf{D}_{t-1}^{(\ell+1)} \circ \sigma^\prime(\widetilde{\mathbf{Z}}_{t-1}^{(\ell)})) }_{(A_1)} \Big\|_{\mathrm{F}}^2 \\
    &\qquad + \| \underbrace{ [\mathbf{L} \widetilde{\mathbf{H}}_{t-1}^{(\ell-1)} ]^\top (\mathbf{D}_{t-1}^{(\ell+1)} \circ \sigma^\prime(\widetilde{\mathbf{Z}}_{t-1}^{(\ell)})) - \widetilde{\mathbf{G}}_{t-1}^{(\ell)}}_{A_2} \|_{\mathrm{F}}^2 + \| \underbrace{\widetilde{\mathbf{G}}_{t}^{(\ell)} - \widetilde{\mathbf{G}}_{t-1}^{(\ell)}}_{A_3} \|_{\mathrm{F}}^2 \\
    &\qquad + 2\langle A_1, A_2\rangle - 2\langle A_1, A_3\rangle - 2\langle A_2, A_3\rangle.
    \end{aligned}
\end{equation}

Taking expectation condition on $\mathcal{F}_t$ on both side, and using the fact that 
\begin{equation}
    \begin{aligned}
    \mathbb{E}[\widetilde{\mathbf{G}}_t^{(\ell)} - \widetilde{\mathbf{G}}_{t-1}^{(\ell)} | \mathcal{F}_t] 
    &= \mathbb{E}[ [\widetilde{\mathbf{L}}^{(\ell)}_t \widetilde{\mathbf{H}}_t^{(\ell-1)}]^\top (\mathbf{D}_t^{(\ell+1)} \circ \sigma^\prime(\widetilde{\mathbf{Z}}_t^{(\ell)})) - [\widetilde{\mathbf{L}}^{(\ell)}_t \widetilde{\mathbf{H}}_{t-1}^{(\ell-1)}]^\top (\mathbf{D}_{t-1}^{(\ell+1)} \circ \sigma^\prime(\widetilde{\mathbf{Z}}_{t-1}^{(\ell)})) | \mathcal{F}_t ] \\
    &= [\mathbf{L} \widetilde{\mathbf{H}}_t^{(\ell-1)}]^\top (\mathbf{D}_t^{(\ell+1)} \circ \sigma^\prime(\widetilde{\mathbf{Z}}_t^{(\ell)}))  - [\mathbf{L} \widetilde{\mathbf{H}}_{t-1}^{(\ell-1)}]^\top (\mathbf{D}_{t-1}^{(\ell+1)} \circ \sigma^\prime(\widetilde{\mathbf{Z}}_{t-1}^{(\ell)})) .
    \end{aligned}
\end{equation}
we have the following inequality holds
\begin{equation}
    \begin{aligned}
    &\mathbb{E}[\| [\mathbf{L} \widetilde{\mathbf{H}}_t^{(\ell)}]^\top (\mathbf{D}_t^{(\ell+1)} \circ \sigma^\prime(\widetilde{\mathbf{Z}}_t^{(\ell-1)})) - \widetilde{\mathbf{G}}^{(\ell)}_t \|_{\mathrm{F}}^2 | \mathcal{F}_t ] \\
    &\leq \| [\mathbf{L} \widetilde{\mathbf{H}}_{t-1}^{(\ell-1)}]^\top (\mathbf{D}_{t-1}^{(\ell+1)} \circ \sigma^\prime(\widetilde{\mathbf{Z}}_{t-1}^{(\ell)}))  - \widetilde{\mathbf{G}}_{t-1}^{(\ell)} \|_{\mathrm{F}}^2 + \mathbb{E}[\| \widetilde{\mathbf{G}}_{t}^{(\ell)} - \widetilde{\mathbf{G}}_{t-1}^{(\ell)} \|_{\mathrm{F}}^2 | \mathcal{F}_t] \\
    &\qquad - \| [\mathbf{L} \widetilde{\mathbf{H}}_t^{(\ell-1)}]^\top (\mathbf{D}_t^{(\ell+1)} \circ \sigma^\prime(\widetilde{\mathbf{Z}}_t^{(\ell)})) - [\mathbf{L} \widetilde{\mathbf{H}}_{t-1}^{(\ell-1)}]^\top (\mathbf{D}_{t-1}^{(\ell+1)} \circ \sigma^\prime(\widetilde{\mathbf{Z}}_{t-1}^{(\ell)})) \mathbf{W}_{t-1}^{(\ell)} \|_{\mathrm{F}}^2.
    \end{aligned}
\end{equation}

Then, taking expectation over $\mathcal{F}_t$, we can write Eq.~\ref{eq:gradient_W_SGCN++eq1} by
\begin{equation}
    \begin{aligned}
    &\mathbb{E}[\|[\mathbf{L} \widetilde{\mathbf{H}}_t^{(\ell-1)}]^\top (\mathbf{D}_t^{(\ell+1)} \circ \sigma^\prime(\widetilde{\mathbf{Z}}_t^{(\ell)}))  - \widetilde{\mathbf{G}}^{(\ell)}_t \|_{\mathrm{F}}^2 ] \\
    &\leq \mathbb{E}[ \| [\mathbf{L} \widetilde{\mathbf{H}}_{t-1}^{(\ell-1)}] ^\top (\mathbf{D}_{t-1}^{(\ell+1)} \circ \sigma^\prime(\widetilde{\mathbf{Z}}_{t-1}^{(\ell)})) \mathbf{W}_{t-1}^{(\ell)} - \widetilde{\mathbf{G}}_{t-1}^{(\ell)} \|_{\mathrm{F}}^2 ] + \mathbb{E}[\| \widetilde{\mathbf{G}}_{t}^{(\ell)} - \widetilde{\mathbf{G}}_{t-1}^{(\ell)} \|_{\mathrm{F}}^2 ] \\
    &\qquad - [ \| [\mathbf{L} \widetilde{\mathbf{H}}_t^{(\ell-1)}]^\top (\mathbf{D}_t^{(\ell+1)} \circ \sigma^\prime(\widetilde{\mathbf{Z}}_t^{(\ell)})) - [\mathbf{L} \widetilde{\mathbf{H}}_{t-1}^{(\ell-1)}]^\top (\mathbf{D}_{t-1}^{(\ell+1)} \circ \sigma^\prime(\widetilde{\mathbf{Z}}_{t-1}^{(\ell)}))\|_{\mathrm{F}}^2 ].
    \end{aligned}
\end{equation}

Let suppose $t\in\{e_{s-1}+1,\ldots,e_{s}\}$. Then we can denote $t=e_{s-1}+k$ for some $k\leq K$ such that
\begin{equation}
    \begin{aligned}
    &\mathbb{E}[\| [\mathbf{L} \widetilde{\mathbf{H}}_t^{(\ell-1)}]^\top (\mathbf{D}_t^{(\ell+1)} \circ \sigma^\prime(\widetilde{\mathbf{Z}}_t^{(\ell)}))  - \widetilde{\mathbf{G}}^{(\ell)}_t \|_{\mathrm{F}}^2 ] \\
    &= \mathbb{E}[\| [\mathbf{L} \widetilde{\mathbf{H}}_{e_{s-1}+k}^{(\ell-1)}]^\top (\mathbf{D}_{e_{s-1}+k}^{(\ell+1)} \circ \sigma^\prime(\widetilde{\mathbf{Z}}_{e_{s-1}+k}^{(\ell)})) - \widetilde{\mathbf{G}}^{(\ell)}_{e_{s-1}+k} \|_{\mathrm{F}}^2 ] \\
    &\leq \mathbb{E} [ \| [\mathbf{L} \widetilde{\mathbf{H}}_{e_{s-1}}^{(\ell-1)}]^\top (\mathbf{D}_{e_{s-1}}^{(\ell+1)} \circ \sigma^\prime(\widetilde{\mathbf{Z}}_{e_{s-1}}^{(\ell)})) - \widetilde{\mathbf{G}}_{e_{s-1}}^{(\ell)} \|_{\mathrm{F}}^2 ] + \sum_{t=e_{s-1}+1}^{e_{s}} \Big( \mathbb{E}[\| \widetilde{\mathbf{G}}_{t}^{(\ell)} - \widetilde{\mathbf{G}}_{t-1}^{(\ell)} \|_{\mathrm{F}}^2 ] \\
    &\qquad - [ \| [\mathbf{L} \widetilde{\mathbf{H}}_t^{(\ell-1)}]^\top (\mathbf{D}_t^{(\ell+1)} \circ \sigma^\prime(\widetilde{\mathbf{Z}}_t^{(\ell)})) - [\mathbf{L} \widetilde{\mathbf{H}}_{t-1}^{(\ell-1)}]^\top (\mathbf{D}_{t-1}^{(\ell+1)} \circ \sigma^\prime(\widetilde{\mathbf{Z}}_{t-1}^{(\ell)})) \|_{\mathrm{F}}^2 ] \Big) \\
    &\leq \mathbb{E}[\| [ \mathbf{L} \widetilde{\mathbf{H}}_{e_{s-1}}^{(\ell-1)} ]^\top (\mathbf{D}_{e_{s-1}}^{(\ell+1)} \circ \sigma^\prime(\widetilde{\mathbf{Z}}_{e_{s-1}}^{(\ell)}))  - \widetilde{\mathbf{G}}_{e_{s-1}}^{(\ell)} \|_{\mathrm{F}}^2 ] \\
    &\qquad + \sum_{t=e_{s-1}+1}^{e_{s}} \Big( \mathbb{E}[\| [\widetilde{\mathbf{L}}^{(\ell)}_t \widetilde{\mathbf{H}}_t^{(\ell-1)}]^\top (\mathbf{D}_t^{(\ell+1)} \circ \sigma^\prime(\widetilde{\mathbf{Z}}_t^{(\ell)}))  - [\widetilde{\mathbf{L}}^{(\ell)}_t \widetilde{\mathbf{H}}_{t-1}^{(\ell-1)} ]^\top (\mathbf{D}_{t-1}^{(\ell+1)} \circ \sigma^\prime(\widetilde{\mathbf{Z}}_{t-1}^{(\ell)}))   \|_{\mathrm{F}}^2 ] \\
    &\qquad - \mathbb{E}[ \| [\mathbf{L} \widetilde{\mathbf{H}}_t^{(\ell-1)} ]^\top (\mathbf{D}_t^{(\ell+1)} \circ \sigma^\prime(\widetilde{\mathbf{Z}}_t^{(\ell)})) - [\mathbf{L} \widetilde{\mathbf{H}}_{t-1}^{(\ell-1)}]^\top (\mathbf{D}_{t-1}^{(\ell+1)} \circ \sigma^\prime(\widetilde{\mathbf{Z}}_{t-1}^{(\ell)})) \|_{\mathrm{F}}^2 ] \Big).
    \end{aligned}
\end{equation}

Knowing that full-batch GD is used when $t = e_{s-1}$ such that $\mathbb{E}[\| [ \mathbf{L} \widetilde{\mathbf{H}}_{e_{s-1}}^{(\ell-1)} ]^\top (\mathbf{D}_{e_{s-1}}^{(\ell+1)} \circ \sigma^\prime(\widetilde{\mathbf{Z}}_{e_{s-1}}^{(\ell)}))  - \widetilde{\mathbf{G}}_{e_{s-1}}^{(\ell)} \|_{\mathrm{F}}^2 ] = 0$. Therefore, we have 
\begin{equation}
    \begin{aligned}
    &\mathbb{E}[\| [\mathbf{L} \widetilde{\mathbf{H}}_t^{(\ell-1)}]^\top (\mathbf{D}_t^{(\ell+1)} \circ \sigma^\prime(\widetilde{\mathbf{Z}}_t^{(\ell)}))  - \widetilde{\mathbf{G}}^{(\ell)}_t \|_{\mathrm{F}}^2 ] \\
    &\leq \sum_{t=e_{s-1}+1}^{e_{s}} 
    \Big( | \mathbb{E}[\|\widetilde{\mathbf{L}}^{(\ell)}_t\|_{\mathrm{F}}^2] - \| \mathbf{L} \|_{\mathrm{F}}^2 | \Big) \times \underbrace{\mathbb{E}[ \| [\widetilde{\mathbf{H}}_t^{(\ell-1)}]^\top (\mathbf{D}_t^{(\ell+1)} \circ \sigma^\prime(\widetilde{\mathbf{Z}}_t^{(\ell)})) - [\widetilde{\mathbf{H}}_{t-1}^{(\ell-1)}]^\top (\mathbf{D}_{t-1}^{(\ell+1)} \circ \sigma^\prime(\widetilde{\mathbf{Z}}_{t-1}^{(\ell)})) \|_{\mathrm{F}}^2 ]}_{(B)}.
    \end{aligned}
\end{equation}

Let take closer look at term $(B)$.
\begin{equation}
    \begin{aligned}
    & \mathbb{E}[ \| [\widetilde{\mathbf{H}}_t^{(\ell-1)}]^\top  (\mathbf{D}_t^{(\ell+1)} \circ \sigma^\prime(\widetilde{\mathbf{Z}}_t^{(\ell)})) - [\widetilde{\mathbf{H}}_{t-1}^{(\ell-1)}]^\top  (\mathbf{D}_{t-1}^{(\ell+1)} \circ \sigma^\prime(\widetilde{\mathbf{Z}}_{t-1}^{(\ell)})) \|_{\mathrm{F}}^2] \\
    &\leq  3\mathbb{E}[ \| [\widetilde{\mathbf{H}}_t^{(\ell-1)}]^\top  (\mathbf{D}_t^{(\ell+1)} \circ \sigma^\prime(\widetilde{\mathbf{Z}}_t^{(\ell)})) - [\widetilde{\mathbf{H}}_{t-1}^{(\ell-1)}]^\top  (\mathbf{D}_t^{(\ell+1)} \circ \sigma^\prime(\widetilde{\mathbf{Z}}_t^{(\ell)})) \|_{\mathrm{F}}^2] \\
    &\quad + 3\mathbb{E}[ \| [\widetilde{\mathbf{H}}_{t-1}^{(\ell-1)}]^\top  (\mathbf{D}_t^{(\ell+1)} \circ \sigma^\prime(\widetilde{\mathbf{Z}}_t^{(\ell)})) - [\widetilde{\mathbf{H}}_{t-1}^{(\ell-1)}]^\top  (\mathbf{D}_{t-1}^{(\ell+1)} \circ \sigma^\prime(\widetilde{\mathbf{Z}}_t^{(\ell)})) \|_{\mathrm{F}}^2] \\
    &\quad + 3\mathbb{E}[ \| [\widetilde{\mathbf{H}}_{t-1}^{(\ell-1)}]^\top  (\mathbf{D}_{t-1}^{(\ell+1)} \circ \sigma^\prime(\widetilde{\mathbf{Z}}_t^{(\ell)})) - [\widetilde{\mathbf{H}}_{t-1}^{(\ell-1)}]^\top  (\mathbf{D}_{t-1}^{(\ell+1)} \circ \sigma^\prime(\widetilde{\mathbf{Z}}_{t-1}^{(\ell)})) \|_{\mathrm{F}}^2] \\
    &\leq  3\mathbb{E}[ \| [( \widetilde{\mathbf{H}}_t^{(\ell-1)} - \widetilde{\mathbf{H}}_{t-1}^{(\ell-1)} )]^\top  (\mathbf{D}_t^{(\ell+1)} \circ \sigma^\prime(\widetilde{\mathbf{Z}}_t^{(\ell)})) \|_{\mathrm{F}}^2] \\
    &\quad + 3\mathbb{E}[ \| [\widetilde{\mathbf{H}}_{t-1}^{(\ell-1)}]^\top  (( \mathbf{D}_t^{(\ell+1)} - \mathbf{D}_{t-1}^{(\ell+1)} ) \circ \sigma^\prime(\widetilde{\mathbf{Z}}_t^{(\ell)})) \|_{\mathrm{F}}^2] \\
    &\quad + 3\mathbb{E}[ \| [\widetilde{\mathbf{H}}_{t-1}^{(\ell-1)}]^\top  (\mathbf{D}_{t-1}^{(\ell+1)} \circ ( \sigma^\prime(\widetilde{\mathbf{Z}}_t^{(\ell)}) - \sigma^\prime(\widetilde{\mathbf{Z}}_{t-1}^{(\ell)})) ) \|_{\mathrm{F}}^2] \\
    &\leq 3 \beta^2 B_D^2 C_\sigma^2 \mathbb{E}[\| \widetilde{\mathbf{H}}_t^{(\ell-1)} - \widetilde{\mathbf{H}}_{t-1}^{(\ell-1)} \|_\mathrm{F}^2 ] + 3\alpha^2 B_H^2 C_\sigma^2 \mathbb{E}[\| \mathbf{D}_t^{(\ell+1)} - \mathbf{D}_{t-1}^{(\ell+1)} \|_\mathrm{F}^2 ] \\
    &\quad + 3 \alpha^2 \beta^2 B_H^2 B_D^2 C_\sigma^2 \mathbb{E}[\| \widetilde{\mathbf{Z}}_t^{(\ell-1)} - \widetilde{\mathbf{Z}}_{t-1}^{(\ell-1)} \|_{\mathrm{F}}^2] \\
    &\leq 3(\beta^2 B_D^2 C_\sigma^4 + \alpha^2 \beta^2 B_H^2 B_D^2 L_\sigma^2) \underbrace{\mathbb{E}[\| \widetilde{\mathbf{Z}}_t^{(\ell-1)} - \widetilde{\mathbf{Z}}_{t-1}^{(\ell-1)} \|_{\mathrm{F}}^2]}_{(C_1)} + 3 \alpha^2 B_H^2 C_\sigma^2 \underbrace{\mathbb{E}[\| \mathbf{D}_t^{(\ell+1)} - \mathbf{D}_{t-1}^{(\ell+1)} \|_{\mathrm{F}}^2]}_{(C_2)} .
    \end{aligned}
\end{equation}

For term $(C_2)$ by definition we know 
\begin{equation}
    \begin{aligned}
    \|\mathbf{D}_t^{(\ell+1)} - \mathbf{D}_{t-1}^{(\ell+1)} \|_{\mathrm{F}}^2 &= \Big\|\Big( \mathbf{L}^\top (\mathbf{D}_t^{(\ell+2)} \circ \sigma^\prime(\mathbf{Z}_t^{(\ell+1)}))\mathbf{W}_t^{(\ell+1)} \Big) -  \Big( \mathbf{L}^\top (\mathbf{D}_{t-1}^{(\ell+2)} \circ \sigma^\prime(\mathbf{Z}_{t-1}^{(\ell+1)}))\mathbf{W}_{t-1}^{(\ell+1)} \Big) \Big\|_{\mathrm{F}}^2 \\
    &\leq 3 \Big\| \Big( \mathbf{L}^\top (\mathbf{D}_t^{(\ell+2)} \circ \sigma^\prime(\mathbf{Z}_t^{(\ell+1)}))\mathbf{W}_t^{(\ell+1)} \Big) -  \Big( \mathbf{L}^\top (\mathbf{D}_{t-1}^{(\ell+2)} \circ \sigma^\prime(\mathbf{Z}_t^{(\ell+1)}))\mathbf{W}_t^{(\ell+1)} \Big) \Big\|_{\mathrm{F}}^2 \\
    &\quad + 3 \Big\| \Big( \mathbf{L}^\top (\mathbf{D}_{t-1}^{(\ell+2)} \circ \sigma^\prime(\mathbf{Z}_t^{(\ell+1)}))\mathbf{W}_t^{(\ell+1)} \Big) - \Big( \mathbf{L}^\top (\mathbf{D}_{t-1}^{(\ell+2)} \circ \sigma^\prime(\mathbf{Z}_{t-1}^{(\ell+1)}))\mathbf{W}_t^{(\ell+1)} \Big) \Big\|_{\mathrm{F}}^2 \\
    &\quad + 3 \Big\| \Big( \mathbf{L}^\top (\mathbf{D}_{t-1}^{(\ell+2)} \circ \sigma^\prime(\mathbf{Z}_{t-1}^{(\ell+1)}))\mathbf{W}_t^{(\ell+1)} \Big) - \Big( \mathbf{L}^\top (\mathbf{D}_{t-1}^{(\ell+2)} \circ \sigma^\prime(\mathbf{Z}_{t-1}^{(\ell+1)}))\mathbf{W}_{t-1}^{(\ell+1)} \Big) \Big\|_{\mathrm{F}}^2 \\
    &\leq \mathcal{O}(\|\mathbf{D}_t^{(\ell+2)} - \mathbf{D}_{t-1}^{(\ell+2)} \|_{\mathrm{F}}^2) + \mathcal{O}(\beta^2 \|\mathbf{Z}_t^{(\ell+1)} - \mathbf{Z}_{t-1}^{(\ell+1)} \|_{\mathrm{F}}^2) \\
    &\quad + \mathcal{O}( \beta^2 \|\mathbf{W}_t^{(\ell+1)} - \mathbf{W}_{t-1}^{(\ell+1)} \|_{\mathrm{F}}^2).
    \end{aligned}
\end{equation}
By induction, we have
\begin{equation}
    \begin{aligned}
    \|\mathbf{D}_t^{(\ell+1)} - \mathbf{D}_{t-1}^{(\ell+1)} \|_{\mathrm{F}}^2 
    &\leq \underbrace{\mathcal{O}(\|\mathbf{D}^{(L+1)}_t - \mathbf{D}^{(L+1)}_{t-1}\|_{\mathrm{F}}^2)}_{(D_1)} \\
    &\quad + \underbrace{\mathcal{O}(\beta^2 \|\mathbf{Z}_t^{(\ell+1)} - \mathbf{Z}_{t-1}^{(\ell+1)} \|_{\mathrm{F}}^2)}_{(D_2)} + \ldots + \mathcal{O}(\beta^2\|\mathbf{Z}_t^{(L)} - \mathbf{Z}_{t-1}^{(L)} \|_{\mathrm{F}}^2) \\
    &\quad + \mathcal{O}(\beta^2\|\mathbf{W}_t^{(\ell+1)} - \mathbf{W}_{t-1}^{(\ell+1)} \|_{\mathrm{F}}^2) + \ldots + \mathcal{O}(\beta^2\|\mathbf{W}_t^{(L)} - \mathbf{W}_{t-1}^{(L)} \|_{\mathrm{F}}^2).
    \end{aligned}
\end{equation}

For term $(D_1)$ we have
\begin{equation}
    \begin{aligned}
    \| \mathbf{D}^{(L+1)}_t - \mathbf{D}^{(L+1)}_{t-1}\|_{\mathrm{F}}^2 
    &= \| \frac{\partial \mathcal{L}(\bm{\theta}_t) }{\partial \mathbf{W}_t^{(L)}} - \frac{\partial \mathcal{L}(\bm{\theta}_{t-1}) }{\partial \mathbf{W}_{t-1}^{(L)}} \|_{\mathrm{F}}^2 \\
    &\leq L_\text{loss}^2 C_\sigma^2 \| \mathbf{Z}^{(L)}_t - \mathbf{Z}^{(L)}_{t-1} \|_{\mathrm{F}}^2.
    \end{aligned}
\end{equation}
For term $(D_2)$ we have
\begin{equation}
    \begin{aligned}
    & \|\mathbf{Z}_t^{(\ell+1)} - \mathbf{Z}_{t-1}^{(\ell+1)} \|_{\mathrm{F}}^2 \\
    &\leq \| \mathbf{L} \mathbf{H}^{(\ell)}_t \mathbf{W}^{(\ell+1)}_t - \mathbf{L} \mathbf{H}^{(\ell)}_{t-1} \mathbf{W}^{(\ell+1)}_{t-1} \|_{\mathrm{F}}^2 \\
    &\leq B_{LA}^2 \| \mathbf{H}^{(\ell)}_t \mathbf{W}^{(\ell+1)}_t - \mathbf{H}^{(\ell)}_{t-1} \mathbf{W}^{(\ell+1)}_t + \mathbf{H}^{(\ell)}_{t-1} \mathbf{W}^{(\ell+1)}_t - \mathbf{H}^{(\ell)}_{t-1} \mathbf{W}^{(\ell+1)}_{t-1} \|_{\mathrm{F}}^2 \\
    &\leq 2 B_{LA}^2 \| \mathbf{H}^{(\ell)}_t \mathbf{W}^{(\ell+1)}_t - \mathbf{H}^{(\ell)}_{t-1} \mathbf{W}^{(\ell+1)}_t \|_{\mathrm{F}}^2 + 2 B_{LA}^2 \|\mathbf{H}^{(\ell)}_{t-1} \mathbf{W}^{(\ell+1)}_t - \mathbf{H}^{(\ell)}_{t-1} \mathbf{W}^{(\ell+1)}_{t-1} \|_{\mathrm{F}}^2 \\
    &\leq 2 B_{LA}^2 B_W^2 \| \mathbf{Z}^{(\ell)}_t  - \mathbf{Z}^{(\ell)}_{t-1} \|_{\mathrm{F}}^2 + 2 \alpha^2 C_\sigma^2 B_{LA}^2 B_H^2 \| \mathbf{W}^{(\ell+1)}_t  - \mathbf{W}^{(\ell+1)}_{t-1} \|_{\mathrm{F}}^2.
    \end{aligned}
\end{equation}
By induction we have
\begin{equation} \label{eq:gradient_W_SGCN++eq2}
    \|\mathbf{Z}_t^{(\ell+1)} - \mathbf{Z}_{t-1}^{(\ell+1)} \|_{\mathrm{F}}^2 \leq \mathcal{O}( \alpha^2 \| \mathbf{W}^{(\ell+1)}_t  - \mathbf{W}^{(\ell+1)}_{t-1} \|_{\mathrm{F}}^2) + \ldots + \mathcal{O}(\alpha^2 \| \mathbf{W}^{(1)}_t  - \mathbf{W}^{(1)}_{t-1} \|_{\mathrm{F}}^2).
\end{equation}

The upper bound of term $(C_1)$ is similar to the derivation in Eq~\ref{eq:gradient_W_SGCN++eq2}.


Plugging $(D_1), (D_2)$ to $(C_2)$ and $(C_1)$, $(C_2)$ to $(B)$, we have
\begin{equation}
    \begin{aligned}
        &\mathbb{E}[\|[\mathbf{L} \widetilde{\mathbf{H}}_t^{(\ell-1)}]^\top (\mathbf{D}_t^{(\ell+1)} \circ \sigma^\prime(\widetilde{\mathbf{Z}}_t^{(\ell)})) - \widetilde{\mathbf{G}}^{(\ell)}_t \|_{\mathrm{F}}^2 ] \\
        &\leq \sum_{t=e_{s-1}+1}^{e_{s}} \eta^2 \mathcal{O}\left( \Big( \sum_{\ell=1}^L | \mathbb{E}[\|\widetilde{\mathbf{L}}^{(\ell)}_t\|_{\mathrm{F}}^2] - \| \mathbf{L} \|_{\mathrm{F}}^2 | \Big) \times (\alpha^2 \beta^2 + \alpha^4 \beta^2) \mathbb{E}[\|\nabla \widetilde{\mathcal{L}}(\bm{\theta}_{t-1})\|_{\mathrm{F}}^2] \right)
    \end{aligned}
\end{equation}

\end{proof}

Using the previous lemma, we provide the upper-bound of Eq.~\ref{eq:sgcn_pplus_important_steps_3}, which is one of the three key factors that affect the mean-square error of stochastic gradient at the $\ell$th layer.

\begin{lemma}\label{lemma:sgcn_pplus_support_lemma2}
Let suppose $t\in\{e_{s-1}+1,\ldots,e_{s}\}$. 
The upper-bound on the difference of the gradient with respect to the weight of the $\ell$th graph convolutional layer given the same input $\mathbf{D}_t^{(\ell+1)}$ but different input $\widetilde{\mathbf{H}}_t^{(\ell-1)}, \mathbf{H}_t^{(\ell-1)}$ is defined as
\begin{equation}
    \begin{aligned}
    &\mathbb{E}[ \| \nabla_W \widetilde{f}^{(\ell)}(\mathbf{D}_t^{(\ell+1)}, \widetilde{\mathbf{H}}_t^{(\ell-1)}, \mathbf{W}_t^{(\ell)}) - \nabla_W f^{(\ell)}(\mathbf{D}_t^{(\ell+1)}, \mathbf{H}_t^{(\ell-1)}, \mathbf{W}_t^{(\ell)}) \|_{\mathrm{F}}^2 ] \\
    &\leq  \sum_{t=e_{s-1}+1}^{e_{s}} \eta^2 \mathcal{O}\left( \Big( \sum_{\ell=1}^L | \mathbb{E}[\|\widetilde{\mathbf{L}}^{(\ell)}_t\|_{\mathrm{F}}^2] - \| \mathbf{L} \|_{\mathrm{F}}^2 | \Big) \times (\alpha^2 \beta^2 + \alpha^4 \beta^2 + \alpha^4) \mathbb{E}[\|\nabla \widetilde{\mathcal{L}}(\bm{\theta}_{t-1})\|_{\mathrm{F}}^2] \right).
    \end{aligned}
\end{equation}

\end{lemma}
\begin{proof}
For the gradient w.r.t. the weight matrices, we have
\begin{equation}
    \begin{aligned}
    &\mathbb{E}[ \| \nabla_W \widetilde{f}^{(\ell)}(\mathbf{D}_t^{(\ell+1)}, \widetilde{\mathbf{H}}_t^{(\ell-1)}, \mathbf{W}_t^{(\ell)}) - \nabla_W f^{(\ell)}(\mathbf{D}_t^{(\ell+1)}, \mathbf{H}_t^{(\ell-1)}, \mathbf{W}_t^{(\ell)}) \|_{\mathrm{F}}^2 ] \\
    &= \mathbb{E}[ \| \Big( \widetilde{\mathbf{G}}_{t-1}^{(\ell)} + [\widetilde{\mathbf{L}}^{(\ell)}_t \widetilde{\mathbf{H}}_t^{(\ell-1)}]^\top (\mathbf{D}_t^{(\ell+1)} \circ \sigma^\prime(\widetilde{\mathbf{Z}}_t^{(\ell)})) - [\widetilde{\mathbf{L}}^{(\ell)} \widetilde{\mathbf{H}}_{t-1}^{(\ell-1)}]^\top (\mathbf{D}_{t-1}^{(\ell+1)} \circ \sigma^\prime(\widetilde{\mathbf{Z}}_{t-1}^{(\ell)}))  \Big) \\
    &\qquad - \Big( [\mathbf{L} \mathbf{H}_t^{(\ell-1)}]^\top (\mathbf{D}_t^{(\ell+1)} \circ \sigma^\prime(\mathbf{Z}_t^{(\ell)})) \Big) \|_{\mathrm{F}}^2 ] \\
    &\leq 2\underbrace{\mathbb{E}[ \| \Big( \widetilde{\mathbf{G}}_{t-1}^{(\ell)} + [\widetilde{\mathbf{L}}^{(\ell)} \widetilde{\mathbf{H}}_t^{(\ell-1)} ]^\top (\mathbf{D}_t^{(\ell+1)} \circ \sigma^\prime(\widetilde{\mathbf{Z}}_t^{(\ell)})) - [\widetilde{\mathbf{L}}^{(\ell)} \widetilde{\mathbf{H}}_t^{(\ell-1)} ]^\top (\mathbf{D}_{t-1}^{(\ell+1)} \circ \sigma^\prime(\widetilde{\mathbf{Z}}_{t-1}^{(\ell)})) \Big)}_{(A)} \\
    &\qquad \underbrace{- \Big( [\mathbf{L} \widetilde{\mathbf{H}}_t^{(\ell-1)}]^\top (\mathbf{D}_t^{(\ell+1)} \circ \sigma^\prime(\widetilde{\mathbf{Z}}_t^{(\ell)})) \Big) \|_{\mathrm{F}}^2 ]}_{(A)} \\
    &\qquad + 2\underbrace{\mathbb{E}[\| \Big( [\mathbf{L} \widetilde{\mathbf{H}}_t^{(\ell-1)}]^\top (\mathbf{D}_t^{(\ell+1)} \circ \sigma^\prime(\widetilde{\mathbf{Z}}_t^{(\ell)})) \Big) - \Big( [\mathbf{L} \mathbf{H}_t^{(\ell-1)}]^\top (\mathbf{D}_t^{(\ell+1)} \circ \sigma^\prime(\mathbf{Z}_t^{(\ell)}))\Big) \|_{\mathrm{F}}^2]}_{(B)} .
    \end{aligned}
\end{equation}
Let first take a closer look at term $(A)$. Let suppose $t\in\{e_{s-1}+1,\ldots,e_{s}\}$, then we can denote $t= e_{s-1} + k$ for some $k \leq K$. By Lemma~\ref{lemma:gradient_W_SGCN++}, term $(A)$ can be bounded by
\begin{equation}
    (A) \leq \sum_{t=e_{s-1}+1}^{e_{s}} \eta^2 \mathcal{O}\left( \Big( \sum_{\ell=1}^L | \mathbb{E}[\|\widetilde{\mathbf{L}}^{(\ell)}_t\|_{\mathrm{F}}^2] - \| \mathbf{L} \|_{\mathrm{F}}^2 | \Big) \times (\alpha^2 +\beta^2) \mathbb{E}[\|\nabla \widetilde{\mathcal{L}}(\bm{\theta}_{t-1})\|_{\mathrm{F}}^2] \right).
\end{equation}

Then we take a closer look at term $(B)$. 
\begin{equation}
    \begin{aligned}
    &\mathbb{E}[\| \Big( [\mathbf{L} \widetilde{\mathbf{H}}_t^{(\ell-1)}]^\top (\mathbf{D}_t^{(\ell+1)} \circ \sigma^\prime(\widetilde{\mathbf{Z}}_t^{(\ell)})) \Big) - \Big( [\mathbf{L} \mathbf{H}_t^{(\ell-1)}]^\top (\mathbf{D}_t^{(\ell+1)} \circ \sigma^\prime(\mathbf{Z}_t^{(\ell)}))\Big) \|_{\mathrm{F}}^2] \\
    &\leq 2 \mathbb{E}[\| \Big( [\mathbf{L} \widetilde{\mathbf{H}}_t^{(\ell-1)}]^\top (\mathbf{D}_t^{(\ell+1)} \circ \sigma^\prime(\widetilde{\mathbf{Z}}_t^{(\ell)})) \Big) - \Big( [\mathbf{L} \mathbf{H}_t^{(\ell-1)}]^\top (\mathbf{D}_t^{(\ell+1)} \circ \sigma^\prime(\widetilde{\mathbf{Z}}_t^{(\ell)}))\Big) \|_{\mathrm{F}}^2] \\
    &\quad + 2 \mathbb{E}[\| \Big( [\mathbf{L} \mathbf{H}_t^{(\ell-1)}]^\top (\mathbf{D}_t^{(\ell+1)} \circ \sigma^\prime(\widetilde{\mathbf{Z}}_t^{(\ell)})) \Big) - \Big( [\mathbf{L} \mathbf{H}_t^{(\ell-1)}]^\top (\mathbf{D}_t^{(\ell+1)} \circ \sigma^\prime(\mathbf{Z}_t^{(\ell)}))\Big) \|_{\mathrm{F}}^2] \\
    &\leq 2(\beta^2 B_{LA}^2 B_D^2 C_\sigma^4 + \alpha^2 B_{LA}^2 B_H^2 B_D^2 ) \underbrace{\mathbb{E}[\| \widetilde{\mathbf{Z}}_t^{(\ell)} - \mathbf{Z}_t^{(\ell)}\|_{\mathrm{F}}^2]}_{(C)}.
    \end{aligned}
\end{equation}

Let suppose $t\in\{e_{s-1}+1,\ldots,e_{s}\}$. Then we can denote $t=e_{s-1}+k$ for some $k\leq K$.
From Eq.~\ref{eq:sgcn_pplus_support_lemma1_eq5} and Eq.~\ref{eq:sgcn_pplus_support_lemma1_eq6}, we know that
\begin{equation}
    \begin{aligned}
    (C)
    &\leq \sum_{t=e_{s-1}+1}^{e_{s}} \mathcal{O} \left( \Big(\sum_{\ell=1}^L |\mathbb{E}[\|\widetilde{\mathbf{L}}^{(\ell)}_t \|_{\mathrm{F}}^2] - \|\mathbf{L}\|_{\mathrm{F}}^2| \Big) \times \alpha^{2} \mathbb{E}[\| \mathbf{W}_t^{(\ell)} - \mathbf{W}_{t-1}^{(\ell)} \|_{\mathrm{F}}^2 ]\right).
    \end{aligned}
\end{equation}
Plugging $(C)$ back to $(B)$, combing with $(A)$ we have
\begin{equation}
    \begin{aligned}
    &\mathbb{E}[ \| \nabla_W \widetilde{f}^{(\ell)}(\mathbf{D}_t^{(\ell+1)}, \widetilde{\mathbf{H}}_t^{(\ell-1)}, \mathbf{W}_t^{(\ell)}) - \nabla_W f^{(\ell)}(\mathbf{D}_t^{(\ell+1)}, \mathbf{H}_t^{(\ell-1)}, \mathbf{W}_t^{(\ell)}) \|_{\mathrm{F}}^2 ] \\
    &\leq \sum_{t=e_{s-1}+1}^{e_{s}} \eta^2 \mathcal{O}\left( \Big( \sum_{\ell=1}^L | \mathbb{E}[\|\widetilde{\mathbf{L}}^{(\ell)}_t\|_{\mathrm{F}}^2] - \| \mathbf{L} \|_{\mathrm{F}}^2 | \Big) \times (\alpha^2 \beta^2 + \alpha^4 \beta^2 + \alpha^4) \mathbb{E}[\|\nabla \widetilde{\mathcal{L}}(\bm{\theta}_{t-1})\|_{\mathrm{F}}^2] \right).
    \end{aligned}
\end{equation}

\end{proof}

In the following lemma,  we provide the upper-bound of Eq.~\ref{eq:sgcn_pplus_important_steps_1}, which is one of the three key factors that affect the mean-square error of stochastic gradient at the $\ell$th layer.
\begin{lemma}\label{lemma:sgcn_pplus_support_lemma3}
Let suppose $t\in\{e_{s-1}+1,\ldots,e_{s}\}$. Then the upper bound of 
\begin{equation}
    \mathbb{E}[\| \widetilde{\mathbf{D}}_t^{(L+1)} -\mathbf{D}_t^{(L+1)} \|_{\mathrm{F}}^2] \leq \sum_{t=e_{s-1}+1}^{e_{s}} \eta^2 \times \mathcal{O} \left( \Big(\sum_{\ell=1}^L |\mathbb{E}[\|\widetilde{\mathbf{L}}^{(\ell)}\|_{\mathrm{F}}^2] - \|\mathbf{L}\|_{\mathrm{F}}^2| \Big) \times \alpha^2 \mathbb{E}[\| \nabla \widetilde{\mathcal{L}}(\bm{\theta}_{t-1}) \|_{\mathrm{F}}^2 ]\right).
\end{equation}
\end{lemma}
\begin{proof}
By definition we have
\begin{equation}
    \begin{aligned}
    \| \widetilde{\mathbf{D}}_t^{(L+1)} - \mathbf{D}_t^{(L+1)} \|_{\mathrm{F}}^2 &\leq \Big\| \frac{\partial \widetilde{\mathcal{L}}(\bm{\theta}_t)}{\partial \widetilde{\mathbf{H}}^{(L)}_t} - \frac{\partial \mathcal{L}(\bm{\theta}_t)}{\partial \mathbf{H}^{(L)}_{t-1}} \Big\|_{\mathrm{F}}^2 \\
    &\leq L_\text{loss}^2\| \widetilde{\mathbf{H}}^{(L)}_t - \mathbf{H}^{(L)}_t\|_{\mathrm{F}}^2 \\
    &\leq L_\text{loss}^2 C_\sigma^2 \underbrace{\| \widetilde{\mathbf{Z}}^{(L)}_t - \mathbf{Z}^{(L)}_t\|_{\mathrm{F}}^2}_{(A)}.
    \end{aligned}
\end{equation}

Let take closer look at term $(A)$:
\begin{equation}
    \begin{aligned}
    \| \widetilde{\mathbf{Z}}^{(L)}_t - \mathbf{Z}^{(L)}_t\|_{\mathrm{F}}^2 
    &= \| \widetilde{f}^{(L)}(\widetilde{\mathbf{H}}_t^{(L-1)}, \mathbf{W}^{(L)}) - f^{(L)}(\mathbf{H}_t^{(L-1)}, \mathbf{W}^{(L)}) \|_{\mathrm{F}}^2 \\
    &\leq 2\| \widetilde{f}^{(L)}(\widetilde{\mathbf{H}}_t^{(L-1)}, \mathbf{W}^{(L)}) - f^{(L)}(\widetilde{\mathbf{H}}_t^{(L-1)}, \mathbf{W}^{(L)}) \|_{\mathrm{F}}^2 \\
    &\qquad + 2\| f^{(L)}(\widetilde{\mathbf{H}}_t^{(L-1)}, \mathbf{W}^{(L)}) - f^{(L)}(\mathbf{H}_t^{(L-1)}, \mathbf{W}^{(L)}) \|_{\mathrm{F}}^2 \\
    &= 2\| \widetilde{f}^{(L)}(\widetilde{\mathbf{H}}_t^{(L-1)}, \mathbf{W}^{(L)}) - f^{(L)}(\widetilde{\mathbf{H}}_t^{(L-1)}, \mathbf{W}^{(L)}) \|_{\mathrm{F}}^2 \\
    &\qquad + 2\| \sigma(\mathbf{L} \widetilde{\mathbf{H}}_t^{(L-1)} \mathbf{W}^{(L)}) - \sigma(\mathbf{L} \mathbf{H}_t^{(L-1)} \mathbf{W}^{(L)}) \|_{\mathrm{F}}^2 \\
    &\leq 2\| \widetilde{f}^{(L)}(\widetilde{\mathbf{H}}_t^{(L-1)}, \mathbf{W}^{(L)}) - f^{(L)}(\widetilde{\mathbf{H}}_t^{(L-1)}, \mathbf{W}^{(L)}) \|_{\mathrm{F}}^2 \\
    &\qquad + 2C_\sigma^4 B_{LA}^2 B_W^2 \| \widetilde{\mathbf{Z}}^{(L-1)}_t - \mathbf{Z}^{(L-1)}_t\|_{\mathrm{F}}^2 .
    \end{aligned}
\end{equation}

By induction, we have
\begin{equation}
    \| \widetilde{\mathbf{Z}}^{(L)}_t - \mathbf{Z}^{(L)}_t\|_{\mathrm{F}}^2 
    \leq  \mathcal{O} \left( \sum_{\ell=1}^L \| \widetilde{f}^{(\ell)}(\widetilde{\mathbf{H}}_t^{(\ell-1)}, \mathbf{W}^{(\ell)}) - f^{(\ell)}(\widetilde{\mathbf{H}}_t^{(\ell-1)}, \mathbf{W}^{(\ell)}) \|_{\mathrm{F}}^2 \right) .
\end{equation}

Let suppose $t\in\{e_{s-1}+1,\ldots,e_{s}\}$. Then we can denote $t=e_{s-1}+k$ for some $k\leq K$.
From Eq.~\ref{eq:sgcn_pplus_support_lemma1_eq6}, we know that
\begin{equation}
    \begin{aligned}
    &\| \widetilde{f}^{(\ell)}(\widetilde{\mathbf{H}}_t^{(\ell-1)}, \mathbf{W}^{(\ell)}) - f^{(\ell)}(\widetilde{\mathbf{H}}_t^{(\ell-1)}, \mathbf{W}^{(\ell)}) \|_{\mathrm{F}}^2\\
    &\leq \sum_{t=e_{s-1}+1}^{e_{s}} \mathcal{O}\Big(|\mathbb{E}[\|\widetilde{\mathbf{L}}^{(\ell)}\|_{\mathrm{F}}^2] - \|\mathbf{L}\|_{\mathrm{F}}^2| \times \alpha^2 \mathbb{E}[\| \mathbf{W}_t^{(\ell)} - \mathbf{W}_{t-1}^{(\ell)} \|_{\mathrm{F}}^2 ]\Big).
    \end{aligned}
\end{equation}

Therefore, we know
\begin{equation}
\begin{aligned}
    & \| \widetilde{\mathbf{D}}_t^{(L+1)} - \mathbf{D}_t^{(L+1)} \|_{\mathrm{F}}^2 \\
    &\leq L_\text{loss}^2 C_\sigma^2 \| \widetilde{\mathbf{Z}}^{(L)}_t - \mathbf{Z}^{(L)}_t\|_{\mathrm{F}}^2 \\
    &\leq \sum_{t=e_{s-1}+1}^{e_{s}} \mathcal{O} \left( \Big( \sum_{\ell=1}^L  |\mathbb{E}[\|\widetilde{\mathbf{L}}^{(\ell)}\|_{\mathrm{F}}^2] - \|\mathbf{L}\|_{\mathrm{F}}^2| \Big) \times \alpha^2 \mathbb{E}[\| \mathbf{W}_t^{(\ell)} - \mathbf{W}_{t-1}^{(\ell)} \|_{\mathrm{F}}^2 ]\right) \\
    &= \sum_{t=e_{s-1}+1}^{e_{s}} \eta^2 \times \mathcal{O} \left( \Big(\sum_{\ell=1}^L |\mathbb{E}[\|\widetilde{\mathbf{L}}^{(\ell)}\|_{\mathrm{F}}^2] - \|\mathbf{L}\|_{\mathrm{F}}^2| \Big) \times \alpha^2 \mathbb{E}[\| \nabla \widetilde{\mathcal{L}}(\bm{\theta}_{t-1}) \|_{\mathrm{F}}^2 ]\right).
\end{aligned}
\end{equation}
which conclude the proof.
\end{proof}

 Combing the upper-bound of Eq.~\ref{eq:sgcn_pplus_important_steps_1}, \ref{eq:sgcn_pplus_important_steps_2}, \ref{eq:sgcn_pplus_important_steps_3}, we provide the upper-bound of mean-suqare error of stochastic gradient in \texttt{SGCN++}.
\begin{lemma}\label{lemma:upper-bound-sgcn-plus-plus-mse}
Let suppose $t\in\{e_{s-1}+1,\ldots,e_{s}\}$. Then we can denote $t=e_{s-1}+k$ for some $k\leq K$ such that
\begin{equation}
    \begin{aligned}
        &\mathbb{E}[\| \nabla \widetilde{\mathcal{L}}(\bm{\theta}_t) - \nabla \mathcal{L}(\bm{\theta}_t) \|_{\mathrm{F}}^2 ] \\
    &\leq \sum_{t=e_{s-1}+1}^{e_{s}} \eta^2 \mathcal{O} \left( \Big(\sum_{\ell=1}^L |\mathbb{E}[\|\widetilde{\mathbf{L}}^{(\ell)}\|_{\mathrm{F}}^2] - \|\mathbf{L}\|_{\mathrm{F}}^2| \Big) \times (\alpha^2 + 1) (\alpha^2 + \beta^2 + \alpha^2 \beta^2) \mathbb{E}[\| \nabla \widetilde{\mathcal{L}}(\bm{\theta}_{t-1}) \|_{\mathrm{F}}^2 ]\right).
    \end{aligned}
\end{equation}
\end{lemma}
\begin{proof}
From Lemma~\ref{lemma:sgcn_pplus_lemma1}, we have
\begin{equation}
        \begin{aligned}
        &\mathbb{E}[\|\widetilde{\mathbf{G}}^{(\ell)}_t - \mathbf{G}^{(\ell)}_t \|F^2] \\
        &\leq \mathcal{O}(\mathbb{E}[ \|\widetilde{\mathbf{D}}^{(L+1)}_t - \mathbf{D}^{(L+1)}_t \|_{\mathrm{F}}^2 ]) \\
        &\qquad + \mathcal{O}(\mathbb{E}[ \| \nabla_H \widetilde{f}^{(L)}( \mathbf{D}^{(L+1)}_t, \widetilde{\mathbf{H}}^{(L-1)}_t \mathbf{W}^{(L)}_t ) - \nabla_H f^{(L)}( \mathbf{D}^{(L+1)}_t, \mathbf{H}^{(L-1)}_t \mathbf{W}^{(L)}_t ) \|_{\mathrm{F}}^2 ]) + \ldots\\
        &\qquad + \mathcal{O}(\mathbb{E}[ \| \nabla_H \widetilde{f}^{(\ell+1)}( \mathbf{D}^{(\ell+2)}, \widetilde{\mathbf{H}}^{(\ell)}, \mathbf{W}^{(\ell+1)}) - \nabla_H f^{(\ell+1)}( \mathbf{D}^{(\ell+2)}, \mathbf{H}^{(\ell)}, \mathbf{W}^{(\ell+1)})\|_{\mathrm{F}}^2 ]) \\
        &\qquad + \mathcal{O}(\mathbb{E}[ \| \nabla_W \widetilde{f}^{(\ell)}( \mathbf{D}^{(\ell+1)}, \widetilde{\mathbf{H}}^{(\ell-1)}, \mathbf{W}^{(\ell)}) - \nabla_W f^{(\ell)}( \mathbf{D}^{(\ell+1)}, \mathbf{H}^{(\ell-1)}, \mathbf{W}^{(\ell)}) \|_{\mathrm{F}}^2 ]).
        \end{aligned}
    \end{equation}
    
Plugging the result from support lemmas, i.e., Lemma~\ref{lemma:sgcn_pplus_support_lemma1}, \ref{lemma:sgcn_pplus_support_lemma2}, \ref{lemma:sgcn_pplus_support_lemma3} and using the definition of stochastic gradient for all model parameters $\nabla \widetilde{\mathcal{L}}(\bm{\theta}_t) = \{ \widetilde{\mathbf{G}}_t^{(\ell)}\}_{\ell=1}^L$ we have
\begin{equation}
    \begin{aligned}
        &\mathbb{E}[\| \nabla \widetilde{\mathcal{L}}(\bm{\theta}_t) - \nabla \mathcal{L}(\bm{\theta}_t) \|_{\mathrm{F}}^2 ] \\
    &\leq \sum_{t=e_{s-1}+1}^{e_{s}} \eta^2 \mathcal{O} \left( \Big(\sum_{\ell=1}^L |\mathbb{E}[\|\widetilde{\mathbf{L}}^{(\ell)}\|_{\mathrm{F}}^2] - \|\mathbf{L}\|_{\mathrm{F}}^2| \Big) \times (\alpha^2 + 1) (\alpha^2 + \beta^2 + \alpha^2 \beta^2 ) \mathbb{E}[\| \nabla \widetilde{\mathcal{L}}(\bm{\theta}_{t-1}) \|_{\mathrm{F}}^2 ]\right).
    \end{aligned}
\end{equation}
\end{proof}

\subsection{Remaining steps toward Theorem~\ref{theorem:convergence_of_sgcn_plus_plus}}
By the smoothness of $\mathcal{L}(\bm{\theta}_t)$, we have 
\begin{equation}
    \begin{aligned}
    \mathcal{L}(\bm{\theta}_{T+1}) &\leq \mathcal{L}(\bm{\theta}_t) + \langle \nabla \mathcal{L}(\bm{\theta}_t), \bm{\theta}_{t+1} - \bm{\theta}_t \rangle + \frac{L_f}{2}\|\bm{\theta}_{t+1} - \bm{\theta}_t\|_{\mathrm{F}}^2 \\
    &= \mathcal{L}(\bm{\theta}_t) - \eta \langle \nabla \mathcal{L}(\bm{\theta}_t), \nabla \widetilde{\mathcal{L}}(\bm{\theta}_t) \rangle + \frac{\eta^2 L_f }{2}\|\nabla \widetilde{\mathcal{L}}(\bm{\theta}_t)\|_{\mathrm{F}}^2 \\
    &\underset{(a)}{=} \mathcal{L}(\bm{\theta}_t) - \frac{\eta}{2} \|\nabla \mathcal{L}(\bm{\theta}_t)\|_{\mathrm{F}}^2 + \frac{\eta}{2} \|\nabla \mathcal{L}(\bm{\theta}_t) - \nabla \widetilde{\mathcal{L}}(\bm{\theta}_t)\|_{\mathrm{F}}^2 - \Big(\frac{\eta}{2}-\frac{L_f\eta^2}{2} \Big) \|\nabla \widetilde{\mathcal{L}}(\bm{\theta}_t)\|_{\mathrm{F}}^2.
    \end{aligned}
\end{equation}
where equality $(a)$ is due to the fact $2\langle \bm{x}, \bm{y} \rangle = \|\bm{x}\|_{\mathrm{F}}^2 + \|\bm{y}\|_{\mathrm{F}}^2 - \|\bm{x}-\bm{y}\|_{\mathrm{F}}^2$ for any $\bm{x},\bm{y}$.

Take expectation on both sides, we have
\begin{equation}
     \mathbb{E}[\mathcal{L}(\bm{\theta}_{T+1})] \leq \mathbb{E}[ \mathcal{L}(\bm{\theta}_t) ] - \frac{\eta}{2} \mathbb{E}[ \|\nabla \mathcal{L}(\bm{\theta}_t)\|_{\mathrm{F}}^2 ] + \frac{\eta}{2} \mathbb{E}[\|\nabla \mathcal{L}(\bm{\theta}_t) - \nabla \widetilde{\mathcal{L}}(\bm{\theta}_t)\|_{\mathrm{F}}^2] - \Big(\frac{\eta}{2}-\frac{L_f\eta^2}{2} \Big) \mathbb{E}[\|\nabla \widetilde{\mathcal{L}}(\bm{\theta}_t)\|_{\mathrm{F}}^2].
\end{equation}

By summing over $t=1,\ldots,T$ where $T$ is the inner-loop size, we have
\begin{equation}
     \begin{aligned}
     &\sum_{t=1}^T \mathbb{E}[ \|\nabla \mathcal{L}(\bm{\theta}_t)\|_{\mathrm{F}}^2 ] \\
     &\leq \frac{2}{\eta} \Big( \mathbb{E}[  \mathcal{L}(\bm{\theta}_1) ] - \mathbb{E}[\mathcal{L}(\bm{\theta}_{T+1})] \Big) + \sum_{t=1}^T [ \mathbb{E}[\|\nabla \mathcal{L}(\bm{\theta}_t) - \nabla \widetilde{\mathcal{L}}(\bm{\theta}_t)\|_{\mathrm{F}}^2] - \Big( 1-L_f\eta \Big) \mathbb{E}[\|\nabla \widetilde{\mathcal{L}}(\bm{\theta}_t)\|_{\mathrm{F}}^2] ] \\
     &\leq \frac{2}{\eta} \Big( \mathbb{E}[ \mathcal{L}(\bm{\theta}_1) ] - \mathbb{E}[\mathcal{L}(\bm{\theta}^\star)] \Big) + \sum_{t=1}^T [ \mathbb{E}[\|\nabla \mathcal{L}(\bm{\theta}_t) - \nabla \widetilde{\mathcal{L}}(\bm{\theta}_t)\|_{\mathrm{F}}^2] - \Big( 1-L_f\eta \Big) \mathbb{E}[\|\nabla \widetilde{\mathcal{L}}(\bm{\theta}_t)\|_{\mathrm{F}}^2] ] .
     \end{aligned}
\end{equation}
where $\mathcal{L}(\bm{\theta}^\star)$ is the global optimal solution.

Let us consider each inner-loop with $s \in \{ 1, \ldots, S\}$ and $t\in\{e_{s-1}+1,\ldots,e_{s}\}$, where $E_s = e_s - e_{s-1}$ and $\sum_{s=1}^S E_s = T$. 
Let define $C_{\alpha\beta} = (\alpha^2 + 1)(\alpha^2 + \beta^2 + \alpha^2 \beta^2 )$.
Using Lemma~\ref{lemma:upper-bound-sgcn-plus-plus-mse} we have
\begin{equation}\label{eq:residual_error_in_variance_reduction}
\begin{aligned}
    &\sum_{s=1}^{S} \sum_{t=e_{s-1}+1}^{e_{s}} \mathbb{E}[\|\mathcal{L}(\bm{\theta}_t) - \nabla \widetilde{\mathcal{L}}(\bm{\theta}_t)\|_{\mathrm{F}}^2] - \Big( 1-L_f\eta \Big) \sum_{s=1}^{S} \sum_{t=e_{s-1}+1}^{e_{s}} \mathbb{E}[\|\nabla \widetilde{\mathcal{L}}(\bm{\theta}_t)\|_{\mathrm{F}}^2] \\
    &\underset{(a)}{\leq} \sum_{s=1}^{S} \sum_{t=e_{s-1}+1}^{e_{s}}  \eta^2 K \mathcal{O}\left( \Big(\sum_{\ell=1}^L |\mathbb{E}[\|\widetilde{\mathbf{L}}^{(\ell)}_t \|_{\mathrm{F}}^2] - \|\mathbf{L}\|_{\mathrm{F}}^2| \Big) \times C_{\alpha\beta} \mathbb{E}[\|\nabla \widetilde{\mathcal{L}}(\bm{\theta})_{t-1}\|_{\mathrm{F}}^2] \right) \\
    &\qquad - \Big( 1-L_f\eta \Big) \sum_{s=1}^{S} \sum_{t=e_{s-1}+1}^{e_{s}} \mathbb{E}[\|\nabla \widetilde{\mathcal{L}}(\bm{\theta}_t)\|_{\mathrm{F}}^2] \\
    &\leq \left[ \eta^2 K \times \mathcal{O}\Big( C_{\alpha \beta} \sum_{\ell=1}^L |\mathbb{E}[\|\widetilde{\mathbf{L}}^{(\ell)}\|_{\mathrm{F}}^2 ] - \|\mathbf{L}\|_{\mathrm{F}}^2| \Big) - \Big( 1-L_f\eta \Big) \right] \sum_{s=1}^{S} \sum_{t=e_{s-1}+1}^{e_{s}} \mathbb{E}[\|\nabla \widetilde{\mathcal{L}}(\bm{\theta}_t)\|_{\mathrm{F}}^2] \\
    &= [ \eta^2 \Delta_{\mathbf{b}+\mathbf{n}}^{++\prime}  - ( 1-L_f\eta) ] \sum_{s=1}^{S} \sum_{t=e_{s-1}+1}^{e_{s}} \mathbb{E}[\|\nabla \widetilde{\mathcal{L}}(\bm{\theta}_t)\|_{\mathrm{F}}^2],
\end{aligned}
\end{equation}
where $(a)$ is due to $E_s \leq K$ and Lemma~\ref{lemma:upper-bound-sgcn-plus-plus-mse}.

Notice that $\eta = \frac{2}{L_f+\sqrt{L_f^2+4\Delta_{\mathbf{b}+\mathbf{n}}^{++\prime}}}$ is a root of equation $\eta^2 \Delta_{\mathbf{b}+\mathbf{n}}^{++\prime} - ( 1-L_f\eta) = 0$.
Therefore we have 
\begin{equation}
    \sum_{t=1}^T \mathbb{E}[ \|\nabla \mathcal{L}(\bm{\theta}_t)\|_{\mathrm{F}}^2 ] \leq \frac{2}{\eta} \Big( \mathbb{E}[ \mathcal{L}(\bm{\theta}_1) ] - \mathbb{E}[\mathcal{L}(\bm{\theta}^\star)] \Big)
\end{equation}
which implies
\begin{equation}
    \frac{1}{T}\sum_{t=1}^T \mathbb{E}[ \|\nabla \mathcal{L}(\bm{\theta}_t)\|_{\mathrm{F}}^2 ] \leq \frac{1}{T} \Big(L_f+\sqrt{L_f^2+4\Delta_{\mathbf{b}+\mathbf{n}}^{++\prime}}\Big)\Big( \mathbb{E}[ \mathcal{L}(\bm{\theta}_1) ] - \mathbb{E}[\mathcal{L}(\bm{\theta}^\star)] \Big)
\end{equation}

%% file: appendix/connection_to_existing_results.tex
\section{Connection to composite optimization}\label{supp:connect_to_comp_opt}
\input{section/connection_to_composite_optimization}

%% file: section/connection_to_composite_optimization.tex
In this section, we formally compare the optimization problem in training GCNs to the standard composite optimization and highlight the key differences that necessitates developing a completely different variance reduction schema and convergence analysis compared to the composite optimization counterparts (e.g., see~\cite{fang2018spider}). 

\paragraph{Different objective function.}
In composite optimization, the output of the lower-level function is treated as the \emph{parameter} of the outer-level function. 
However in GCN, the output of the lower-level function is used as the \emph{input} of the outer-level function, and the parameter of the outer-level function is independent of the output of the inner-layer result. 

More specifically, a two-level composite optimization problem can be formulated as
\begin{equation}\label{eq:2_level_obj}
    F(\bm{\theta}) = \frac{1}{N}\sum_{i=1}^N f_i\Big(\frac{1}{M}\sum_{j=1}^M g_j(\mathbf{w}) \Big),~\bm{\theta} = \{\mathbf{w}\},
\end{equation}
where $f_i(\cdot)$ is the outer-level function computed on the $i$th data point, $g_j(\cdot)$ is the inner-level function computed on the $j$th data point, and $\mathbf{w}$ is the parameter. 
We denote $\nabla f_i(\cdot)$ and $\nabla g_j(\cdot)$ as the gradient. 
Then, the gradient for Eq.~\ref{eq:2_level_obj} is computed as
\begin{equation}\label{eq:2_level_obj_grad}
    \nabla F(\bm{\theta}) = \Big[ \frac{1}{N}\sum_{i=1}^N \nabla f_i \Big( \frac{1}{M}\sum_{j=1}^M g_j(\mathbf{w}) \Big) \Big] \Big( \frac{1}{M}\sum_{j=1}^M \nabla g_j(\mathbf{w}) \Big),~\bm{\theta} = \{\mathbf{w}\},
\end{equation}
where the dependency between inner- and outer-level sampling are not considered.
One can independently sample inner layer data to estimate $\tilde{g}\approx \frac{1}{M}\sum_{j=1}^M g_j(\mathbf{w})$ and $\nabla \tilde{g} \approx \frac{1}{M}\sum_{j=1}^M \nabla g_j (\mathbf{w})$, sample outer layer data to estimate $\nabla \tilde{f} \approx \frac{1}{N}\sum_{i=1}^N \nabla f_i(\tilde{g})$, then estimate $\nabla F(\bm{\theta})$ by using $[\nabla \tilde{f}]^\top \nabla \tilde{g}$.

By casting the optimizaion problem in GCN as composite optimization problem in Eq.~\ref{eq:2_level_obj}, we have 
\begin{equation}\label{eq:gcn_obj_from_comp_obj}
    \begin{aligned}
    \mathcal{L}(\bm{\theta}) &= \frac{1}{B}\sum_{i\in\mathcal{V}_\mathcal{B}} \text{Loss}(\mathbf{h}_i^{(L)}, y_i),~\text{where}~\bm{\theta} = \{ \mathbf{W}^{(1)}\}, \\
    \mathbf{H}^{(L)} &= \sigma(\widetilde{\mathbf{L}}^{(L)} \mathbf{X} \widetilde{\mathbf{W}}^{(L)}),~\text{and}~ \widetilde{\mathbf{W}}^{(\ell)} = \sigma ( \widetilde{\mathbf{L}}^{(\ell-1)} \mathbf{X} \widetilde{\mathbf{W}}^{(\ell-1)} ),
    \end{aligned}
\end{equation}

which is different from the vanilla GCN model. To see this, we note that in vanilla GCNs, since the sampled nodes at the $\ell$th layer are dependent from the nodes sampled at the $(\ell+1)$th layer, we have $\mathbb{E}[\widetilde{\mathbf{L}}^{(\ell)}] = \mathbf{\mathbf{P}}^{(\ell)} \neq \mathbf{L}$.  However in Eq.~\ref{eq:gcn_obj_from_comp_obj}, since the sampled nodes have no dependency on the weight matrices or nodes sampled at other layers, we can easily obtain $\mathbb{E}[\widetilde{\mathbf{L}}^{(\ell)}] = \mathbf{L}$.  These key differences makes the analysis more involved and are reflected in all three theorems, that give us  different results.

\paragraph{Different gradient computation and algorithm.} The stochastic gradients to update the parameters in Eq.~\ref{eq:gcn_obj_from_comp_obj} are computed as 
\begin{equation}
    \frac{\partial \mathcal{L}(\bm{\theta})}{\partial \widetilde{\mathbf{W}}^{(\ell)}} = \frac{\partial \mathcal{L}(\bm{\theta})}{\partial \widetilde{\mathbf{W}}^{(L)}} \Big( \prod_{j=\ell+1}^L \frac{\partial \widetilde{\mathbf{W}}^{(j)} }{\partial \widetilde{\mathbf{W}}^{(j-1)}}\Big).
\end{equation}
However in GCN, there are two types of gradient at each layer (i.e., $\tilde{\mathbf{D}}^{(\ell)}$ and $\tilde{\mathbf{G}}^{(\ell)}$) that are fused with each other (i.e., $\tilde{\mathbf{D}}^{(\ell)}$ is a part of $\tilde{\mathbf{G}}^{(\ell-1)}$ and $\tilde{\mathbf{D}}^{(\ell)}$ is a part of $\tilde{\mathbf{D}}^{(\ell-1)}$) but with different functionality.
$\tilde{\mathbf{D}}^{(\ell)}$ is passing gradient between different layers, $\tilde{\mathbf{G}}^{(\ell)}$ is passing gradient to weight matrices. 

These two types of gradient and their coupled relation make both algorithm and analysis different from~\cite{zhang2019multi}. 
For example in~\cite{zhang2019multi}, the zeroth-order variance reduction is applied to $\widetilde{\mathbf{W}}_t^{(\ell)}$ in Eq.~\ref{eq:2_level_obj} (please refer to Algorithm 3 in~\cite{zhang2019multi}), where $\widetilde{\mathbf{W}}_{t-1}^{(\ell)}$ is used as a control variant to reduce the variance of $\widetilde{\mathbf{W}}_{t}^{(\ell)}$, i.e.,
\begin{equation}
    \widetilde{\mathbf{W}}_t^{(\ell+1)} = \widetilde{\mathbf{W}}_{t-1}^{(\ell+1)} + \sigma(\widetilde{\mathbf{L}}_t^{(\ell)} \mathbf{X} \widetilde{\mathbf{W}}_t^{(\ell)}) - \sigma(\widetilde{\mathbf{L}}_t^{(\ell)} \mathbf{X} \widetilde{\mathbf{W}}_{t-1}^{(\ell)}).
\end{equation}
However in \texttt{SGCN++}, the zeroth-order variance reduction is applied to $\widetilde{\mathbf{H}}_t^{(\ell)}$. Because the node sampled at the $t$th and $(t-1)$th iteration are unlikely the same, we cannot directly use $\mathbf{H}^{(\ell)}_{t-1}$ to reduce the variance of $\mathbf{H}^{(\ell)}_{t}$. Instead, the control variant in \texttt{SGCN++} is computed by applying historical weight $\mathbf{W}^{(\ell)}_{t-1}$ on the historical node embedding from previous layer $\mathbf{H}^{(\ell-1)}_{t-1}$, i.e.,
\begin{equation}
    \widetilde{\mathbf{H}}_t^{(\ell)} = \widetilde{\mathbf{H}}_{t-1}^{(\ell)} + \sigma(\widetilde{\mathbf{L}}_t^{(\ell)} \mathbf{H}^{(\ell-1)}_{t} \mathbf{W}_{t}^{(\ell)}) - \sigma(\widetilde{\mathbf{L}}_t^{(\ell)} \mathbf{H}^{(\ell-1)}_{t-1} \mathbf{W}_{t-1}^{(\ell)}).
\end{equation}
These changes are not simply heuristic modifications, but all reflected in the analysis and the result.

\paragraph{Different theoretical results and intuition.} The aforementioned differences further result in a novel analysis of Theorem~\ref{theorem:convergence_of_sgcn}, where we show that the vanilla sampling-based GCNs suffer a residual error $\Delta_\mathbf{b}$ that is not decreasing as the number of iterations $T$ increases, and this residual error is strongly connected to the difference between sampled and full Laplacian matrices. This is one of our novel observations for GCNs, when compared to (1) multi-level composite optimization with layerwise changing learning rate~\cite{yang2019multilevel,chen2020solving}, (2) variance reduction based methods~\cite{zhang2019multi}, and (3) the previous analysis on the convergence of GCNs~\cite{fastgcn,biased_but_consistent}.
Our observation can be used as a theoretical motivation on using first-order and doubly variance reduction, and can mathematically explain why VRGCN outperform GraphSAGE, even with fewer nodes during training.
Furthermore, as the algorithm and gradient computation are different, the theoretical results in Theorems~\ref{theorem:convergence_of_sgcn_plus} and~\ref{theorem:convergence_of_sgcn_plus_plus} are also different. 


\paragraph{Different mean-square error formulation.}
Both \cite{cong2020minimal} and ours are inspired by bias-variance decomposition, with a key difference in formulation. 
Let denote $\mathbf{g}$ as the stochastic gradient and $\nabla F(\bm{\theta})$ as the full-batch gradient. Our formulation follows the standard definition of mean-square error of stochastic gradient, which is $\mathbb{E}[\| \mathbf{g} - \nabla F(\bm{\theta})\|_\mathrm{F}^2] = \mathbb{E}[\| \mathbf{g} - \mathbb{E}[\mathbf{g}]\|_\mathrm{F}^2] + \mathbb{E}[\| \mathbb{E}[\mathbf{g}] - \nabla F(\bm{\theta})\|_\mathrm{F}^2]$, where $\mathbb{E}[\| \mathbf{g} - \mathbb{E}[\mathbf{g}]\|_\mathrm{F}^2]$ is known as the variance and $\mathbb{E}[\| \mathbb{E}[\mathbf{g}] - \nabla F(\bm{\theta})\|_\mathrm{F}^2]$ is known as the bias. [5] treats $\mathbb{E}[\mathbf{g}]$ as the gradient computed using all neighbors with mini-batch sampling (Eq.~2 in [5]), which is different from ours. This difference results in a different analysis and theoretical results. Furthermore, [5] does not provide any further analysis (e.g., convergence) based on their observation of bias-variance decomposition.